\documentclass[12pt]{article}
\usepackage[margin=1in, paperwidth=8.5in, paperheight=11in]{geometry}
\usepackage[round]{natbib}
\usepackage{url}
\usepackage{titlesec}

\usepackage[]{graphicx}
\usepackage{enumerate}
\usepackage{subfigure}
\usepackage{hyphenat}
\usepackage{epstopdf}
\usepackage{amssymb}
\usepackage{amsthm}

\newcommand{\mpmargin}[2]{{\color{cyan}#1}\marginpar{\color{cyan}\raggedright\footnotesize [MP]:
#2}}

\newcommand{\edit}[1]{{\color{black}#1}}
\graphicspath{{./fig/}}
\usepackage{color}
\usepackage{amsmath}
\usepackage{amssymb}
\usepackage{graphicx}
\usepackage{comment,xspace}
\usepackage{fancybox}

\newtheorem{theorem}{Theorem}[section]

\newtheorem{lemma}[theorem]{Lemma}

\newtheorem{remark}[theorem]{Remark}

\newcommand{\real}{{\mathbb{R}}}
\newcommand{\reals}{\real}

\renewcommand{\natural}{{\mathbb{N}}}
\newcommand{\naturals}{\natural}

\newcommand{\xfree}{\mathcal X_{\text{free}}}
\newcommand{\xobs}{\mathcal X_{\text{obs}}}
\newcommand{\xgoal}{\mathcal X_{\text{goal}}}
\newcommand{\xinit}{x_{\mathrm{init}}}
\newcommand{\FMT}{$\text{FMT}^*\, $}

\newcommand{\card}{\operatorname{card}}

\newcommand{\p}[1]{\mbox{$\mathbb{P}\left(#1\right)$}} 
 
\newcommand{\probcond}[2]{\mbox{$\mathbb{P}\left(#1 \,| \, #2\right)$}}

\usepackage{color}
\usepackage{url}
\usepackage{algorithmic}
\usepackage{algorithm}

\newcommand{\figWidth}{0.38}
\newcommand{\dist}{\texttt{dist}}
\newcommand{\RRTstar}{RRT$^{\ast}\,$}
\newcommand{\RRTsharp}{RRT$^{\#}\,$}
\newcommand{\PRMstar}{PRM$^\ast\,$}
\newcommand{\kFMT}{$k$-nearest \FMT}
\newcommand{\knFMT}{$k_n$-nearest \FMT}
\newcommand{\Hset}{V_{\mathrm{open}}}
\newcommand{\Wset}{V_{\mathrm{unvisited}}}
\newcommand{\Vc}{V_{\mathrm{closed}}}
\newcommand{\xfinal}{x_{\mathrm{terminal}}}
\newcommand{\Hsetnew}{V_{\mathrm{open, \, new}}}

\begin{document}

\title{ Fast Marching Tree: a Fast Marching\\ Sampling-Based Method for \\Optimal Motion Planning in Many Dimensions\thanks{This work was originally presented at the 16th International Symposium on Robotics Research, ISRR 2013. This revised version includes an extended description of the \FMT algorithm, proofs of all results, extended discussions about convergence rate and computational complexity, extensions to non-uniform sampling distributions and general costs, a $k$-nearest version of \FMT\!, and a larger set of numerical experiments.}}

\author{\large Lucas Janson\\ 
\normalsize Department of Statistics, Stanford University\\
 \normalsize \url{ljanson@stanford.edu}
\and
Edward Schmerling\\ 
\normalsize Institute for Computational \& Mathematical Engineering, Stanford University\\
 \normalsize \url{schmrlng@stanford.edu}
 \and
Ashley Clark \\ 
 \normalsize Department of Aeronautics and Astronautics, Stanford University\\
 \normalsize  \url{aaclark@stanford.edu}
\and
Marco Pavone\\
  \normalsize Department of Aeronautics and Astronautics, Stanford University\\
 \normalsize  \url{pavone@stanford.edu}
}
%\author{Lucas Janson, Ashley Clark,  and Marco Pavone \thanks{aS}}

%\institute{Lucas Janson \at Department of Statistics, Stanford University, Stanford, CA 94305, \email{ljanson@stanford.edu}
%\and Marco Pavone \at Department of Aeronautics and Astronautics, Stanford University, Stanford, CA 94305, \email{pavone@stanford.edu}
%\and A preliminary version of this work has been \emph{orally} presented at the workshop on ``Robotic Exploration, Monitoring, and Information Collection: Nonparametric Modeling, Information-based Control, and Planning under Uncertainty" at the Robotics: Science and Systems 2013 conference. This work has neither appeared elsewhere for publication, nor is under review for another refereed publicatbion.}
\maketitle

\begin{abstract}
In this paper we present a novel probabilistic sampling-based motion
planning algorithm called the Fast Marching Tree algorithm
(\FMT$\!$). The algorithm is specifically aimed at solving complex
motion planning problems in high-dimensional configuration
spaces. This algorithm is proven to be asymptotically optimal and is
shown to converge to an optimal solution faster than its
state-of-the-art counterparts, chiefly \PRMstar and \RRTstar\!. The \FMT algorithm performs a
``lazy" dynamic programming recursion on a predetermined number of 
probabilistically-drawn samples to grow a tree of paths, which moves
steadily outward in cost-to-{\color{black}arrive} space. As such, this algorithm
combines features of both single-query algorithms (chiefly RRT) and
multiple-query algorithms (chiefly PRM), and is reminiscent of
the Fast Marching Method for the solution of Eikonal equations. As
a departure from previous analysis approaches that are based on the
notion of almost sure convergence, the \FMT algorithm is analyzed
under the notion of convergence in probability: the extra mathematical
flexibility of this approach allows for convergence rate bounds---the
first in the field of optimal sampling-based motion planning.
Specifically, for a certain selection of tuning
  parameters and configuration {\color{black}spaces}, we obtain a convergence rate
  bound of order $O(n^{-1/d + \rho})$, where $n$ is the number of
  sampled points, $d$ is the dimension of the configuration space, and
  $\rho$ is an arbitrarily small constant. We go on to demonstrate
asymptotic optimality for a number of variations on \FMT\!, namely
when the configuration space is sampled non-uniformly, when the cost
is not arc length, and when connections are made based on {\color{black}the number of} nearest
neighbors instead of a fixed connection radius. Numerical experiments over
a range of dimensions and obstacle configurations confirm our
theoretical and heuristic arguments by showing that \FMT\!, for a
given execution time, returns substantially better solutions than either \PRMstar or \RRTstar\!, especially in high-dimensional configuration spaces and in scenarios where collision-checking is expensive.
\end{abstract}

\section{Introduction}\label{sec:intro}

Probabilistic sampling-based algorithms represent a particularly successful approach to robotic motion planning problems in high-dimensional configuration spaces, which naturally arise, e.g., when controlling the motion of high degree-of-freedom robots or planning under uncertainty \citep{Thrun.et.al:05, Lavalle:06}. Accordingly, the design of rapidly converging sampling-based algorithms with sound performance guarantees has emerged as a central topic in robotic motion planning and represents the main thrust of this paper.

Specifically, the key idea behind probabilistic sampling-based
algorithms is to avoid the explicit construction of the configuration
space (which can be prohibitive in complex planning problems) and
instead conduct a search that probabilistically probes the
configuration space with a sampling scheme. This probing is enabled by
a collision detection module, which the motion planning algorithm
considers as a ``black box" \citep{Lavalle:06}.  Probabilistic
sampling-based algorithms \edit{may be classified into two categories:} multiple-query and
single-query. Multiple-query algorithms construct a topological graph
called a roadmap, which allows a user to efficiently solve multiple
initial-state/goal-state queries. This family of algorithms includes
the probabilistic roadmap algorithm (PRM)  \citep{Kavraki.ea:TRA96}
and its variants, e.g., Lazy-PRM  \citep{Bohlin.Kavraki:ICRA00},
dynamic PRM  \citep{Jaillet.Simeon:04},  and \PRMstar
\citep{Karaman.Frazzoli:IJRR2011}. In single-query algorithms, on the
other hand, a single initial-state/goal-state pair is given, and the
algorithm must search until it finds a solution{\color{black},} or it may report
early failure.  This family of algorithms includes the rapidly
exploring random trees algorithm (RRT)  \citep{LaValle.ea:IJRR01}, the
rapidly exploring dense trees algorithm (RDT) \citep{Lavalle:06}, and
their variants, e.g., \RRTstar
\citep{Karaman.Frazzoli:IJRR2011}. Other notable sampling-based
planners include expansive space trees (EST) \citep{Hsu:IGCGA99,
  Phillips.ea:04}, sampling-based roadmap of trees (SRT)
\citep{Plaku.et.al:TR05}, rapidly-exploring roadmap (RRM) \citep{Alterovitz.et.al:ICRA2011}, and the ``cross-entropy" planner in \citep{Kobilarov:IJRR12}. Analysis in terms of convergence to feasible or even optimal solutions for multiple-query and single-query algorithms is provided in \citep{LK.MK.ea:96, Hsu:IGCGA99, Barraquand.et.al:IJRR00, Ladd.et.al:TRA2004, Hsu.et.al:IJRR06, Karaman.Frazzoli:IJRR2011}. A central result is that these algorithms provide \emph{probabilistic completeness} guarantees in the sense that the probability that the planner fails to return a solution, if one exists, decays to zero as the number of samples approaches infinity \citep{Barraquand.et.al:IJRR00}. Recently, it has been proven that both \RRTstar and \PRMstar are asymptotically optimal, i.e., the cost of the returned solution converges almost surely to the optimum \citep{Karaman.Frazzoli:IJRR2011}. Building upon the results in \citep{Karaman.Frazzoli:IJRR2011}, the work in \citep{Marble.Bekris:ICRA12} presents an algorithm with provable ``sub-optimality" guarantees, which ``trades" optimality with faster computation, while the work in \citep{OA.PT:13} presents a variant of \RRTstar\!, named \RRTsharp\!, that is also asymptotically optimal and aims \edit{to mitigate} the ``greediness" of \RRTstar\!.

\emph{Statement of Contributions}: The objective of this paper is to
propose and analyze a novel probabilistic motion planning algorithm
that is asymptotically optimal and improves upon state-of-the-art
asymptotically-optimal algorithms{\color{black},} namely \RRTstar and \PRMstar{\color{black}. 
Improvement is measured }in
terms of the convergence rate to the optimal solution, where
convergence rate is interpreted with respect to execution time. The
algorithm, named the Fast Marching Tree algorithm (\FMT \!), is
designed to {\color{black}reduce} the number of obstacle
  collision-checks and is particularly efficient in high-dimensional
environments cluttered with obstacles. \FMT essentially performs a forward dynamic programming recursion on a
  predetermined number  of probabilistically-drawn samples in the
configuration space, see Figure
  \ref{fig:tree_growth}. The recursion is characterized
  by three key features, namely (1) it is \emph{tailored} to
  disk-connected graphs, (2) it \emph{concurrently} performs  graph
  construction and graph search, and (3) it \emph{lazily} skips
  collision-checks when evaluating local connections. This lazy
  collision-checking strategy may introduce suboptimal
  connections---the crucial property of \FMT is that such
    suboptimal connections become vanishingly rare as the
    number of samples goes to infinity.

%combines some of the features of multiple-query algorithms with those of single-query algorithms, by performing
\FMT combines features of PRM and SRT (which is similar to RRM) and grows a
tree of trajectories like RRT. Additionally, \FMT is
  reminiscent of the Fast Marching Method, one of the main methods
for solving stationary Eikonal equations
\citep{Sethian:NAS96}. We refer the reader to
  \citep{Gomez:RAM13} and references therein for a recent overview of
path planning algorithms inspired by the Fast Marching Method. As in
the Fast Marching Method, the main idea is to exploit a heapsort
technique to systematically locate the proper sample point to update
and to incrementally build the solution in an ``outward" direction, so
that {\color{black}the algorithm} needs never backtrack over previously evaluated sample
points. Such a \emph{one-pass} property is what makes both the Fast
Marching Method and \FMT (in addition to its lazy
  strategy) particularly efficient\footnote{We note,
    however, that the Fast Marching Method and \FMT differ in a number
    of important aspects. Chiefly, the Fast Marching Method hinges
    upon upwind approximation schemes for the solution to the Eikonal
    equation over orthogonal grids or triangulated domains, while \FMT
    hinges upon the application of the Bellman principle of optimality over a randomized grid {\color{black}within a sampling-based framework.}}.

The end product of the \FMT algorithm is a tree, which, together with the connection to the Fast Marching Method, gives the algorithm its name. 
%An advantage of \FMT with respect to \PRMstar (in addition to a faster convergence rate) is the fact that \FMT builds and maintains paths in a tree-like structure, which is especially useful when planning under differential or integral constraints. 
Our simulations across a variety of problem instances, ranging in
obstacle clutter and in dimension {\color{black}from 2D to 7D}, show that \FMT
outperforms state-of-the-art algorithms such as \PRMstar and \RRTstar\!, often by a significant margin. 
The speedups are particularly prominent in higher
  dimensions and in scenarios where collision-checking is expensive,
  which is exactly the regime in which sampling-based algorithms \edit{excel.}
  \FMT also presents a number of ``structural"
  advantages{\color{black}, such as} maintaining a tree structure at all times and
  expanding in cost-to-{\color{black}arrive} space, which have been recently leveraged
  to include differential constraints \citep{ES-LJ-MP:14, ES-LJ-MP:14b}, to provide a bidirectional implementation \citep{JS-ES-LJ-MP:14}, and to speed up the convergence rate even further via the inclusion of lower bounds on cost \citep{OS-DH:14} and heuristics \citep{JG-ES-SS-TB:14}.

It is important to note that in this paper we use a notion of asymptotic optimality (AO) different from the one used  in \citep{Karaman.Frazzoli:IJRR2011}.  In \citep{Karaman.Frazzoli:IJRR2011}, AO is defined through the notion of convergence almost everywhere (a.e.).  Explicitly, in \citep{Karaman.Frazzoli:IJRR2011}, an algorithm is considered AO if the cost of the solution it returns converges a.e. to the optimal cost as the number of samples $n$ approaches infinity. % \rightarrow \infty$.  
This definition is \edit{apt} when the algorithm is
sequential in $n$, such as \RRTstar \citep{Karaman.Frazzoli:IJRR2011},
in the sense that it requires that with probability 1 the sequence of
solutions converges to an optimal one, with the solution at $n+1$
heavily related to that at $n$.  However, for non-sequential
algorithms such as \PRMstar and \FMT\!, there is no connection between
the solutions at $n$ and $n+1$.  Since these algorithms process all
the samples at once, the solution at $n+1$ is based on $n+1$ new samples,
sampled independently of those used in the solution at $n$.  This
motivates the definition of AO used in this paper, which is that the
cost of the solution returned by an algorithm must converge \emph{in
  probability} to the optimal cost.  Although
convergence in probability is a {\color{black}mathematically} weaker notion than convergence
a.e. (the latter implies the former), in practice there is no
distinction when an algorithm is only run on a predetermined, fixed
number of samples.  In this case, all that matters is that the
probability that the cost of the solution returned by the algorithm is
less than an $\varepsilon$ fraction greater than the optimal cost goes
to 1 as $n \rightarrow \infty$, for any $\varepsilon > 0$, which is
exactly the statement of convergence in probability. Since this
convergence is a
mathematically weaker, but practically identical condition, we sought
to capitalize on the extra mathematical flexibility, and indeed find
that our proof of AO for \FMT allows for a tighter theoretical lower
bound on the search radius of  \PRMstar {\color{black}than was found} in \citep{Karaman.Frazzoli:IJRR2011}. In this regard, an additional important contribution of this paper is the analysis of AO under the notion of convergence in probability, which is of independent interest and could enable the design and analysis of other AO sampling-based algorithms.

Most importantly, our proof of AO gives a \emph{convergence rate bound} with respect to the number of sampled points both for \FMT and \PRMstar---the first in the field of optimal sampling-based motion planning.  Specifically, for a certain selection of tuning parameters and configuration space, we derive a convergence rate bound of $O(n^{-1/d+\rho})$, where $n$ is the number of sampled points, $d$ is the dimension of the configuration space, and $\rho$ is an arbitrarily small constant. {\color{black}While the algorithms exhibit the slow convergence rate typical of sampling-based algorithms, the rate is} at least a power of $n$.

\begin{figure}
  \centering
  \subfigure{
  \includegraphics[width=50mm]{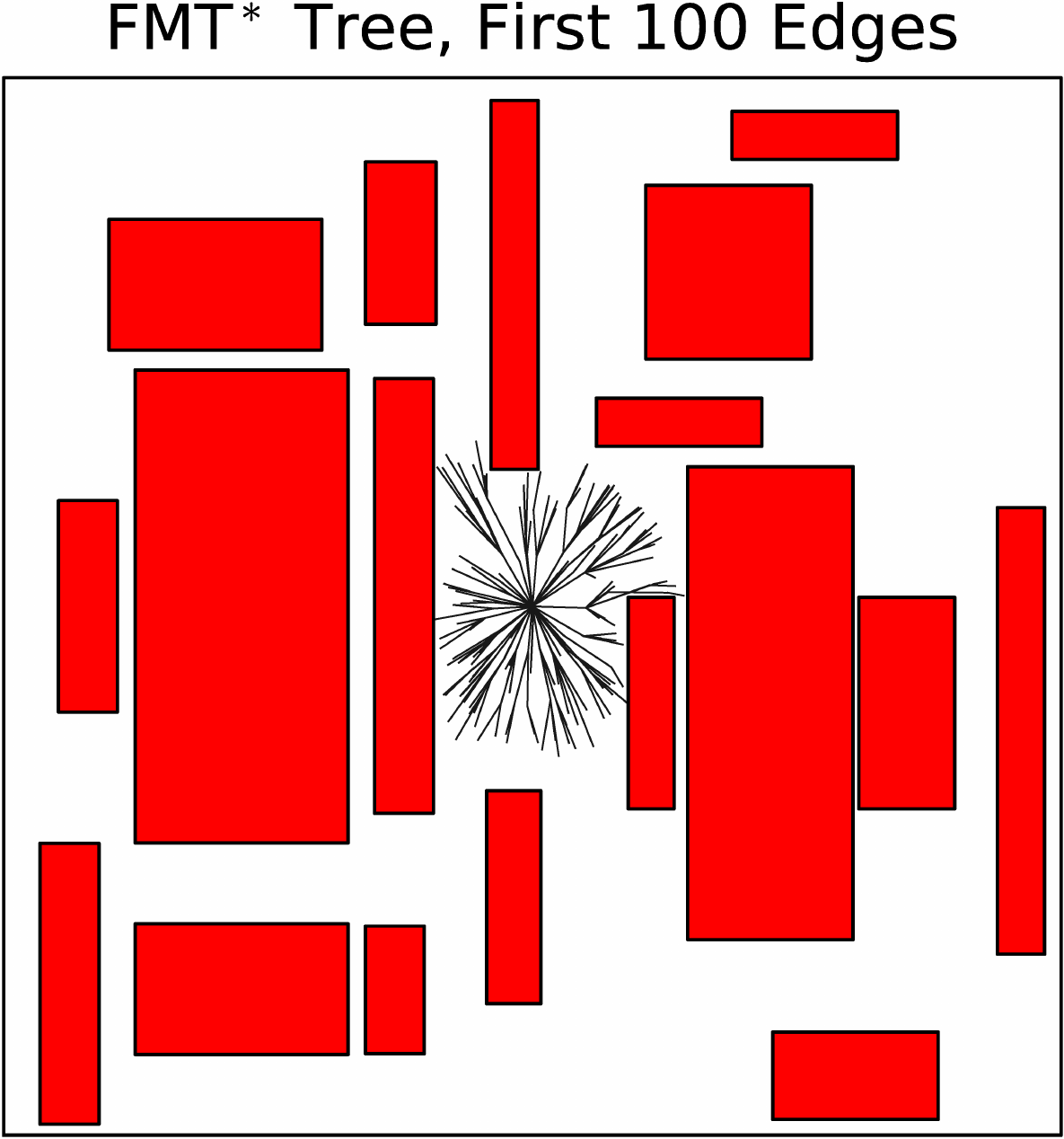}
  }
  \subfigure{
  \includegraphics[width=50mm]{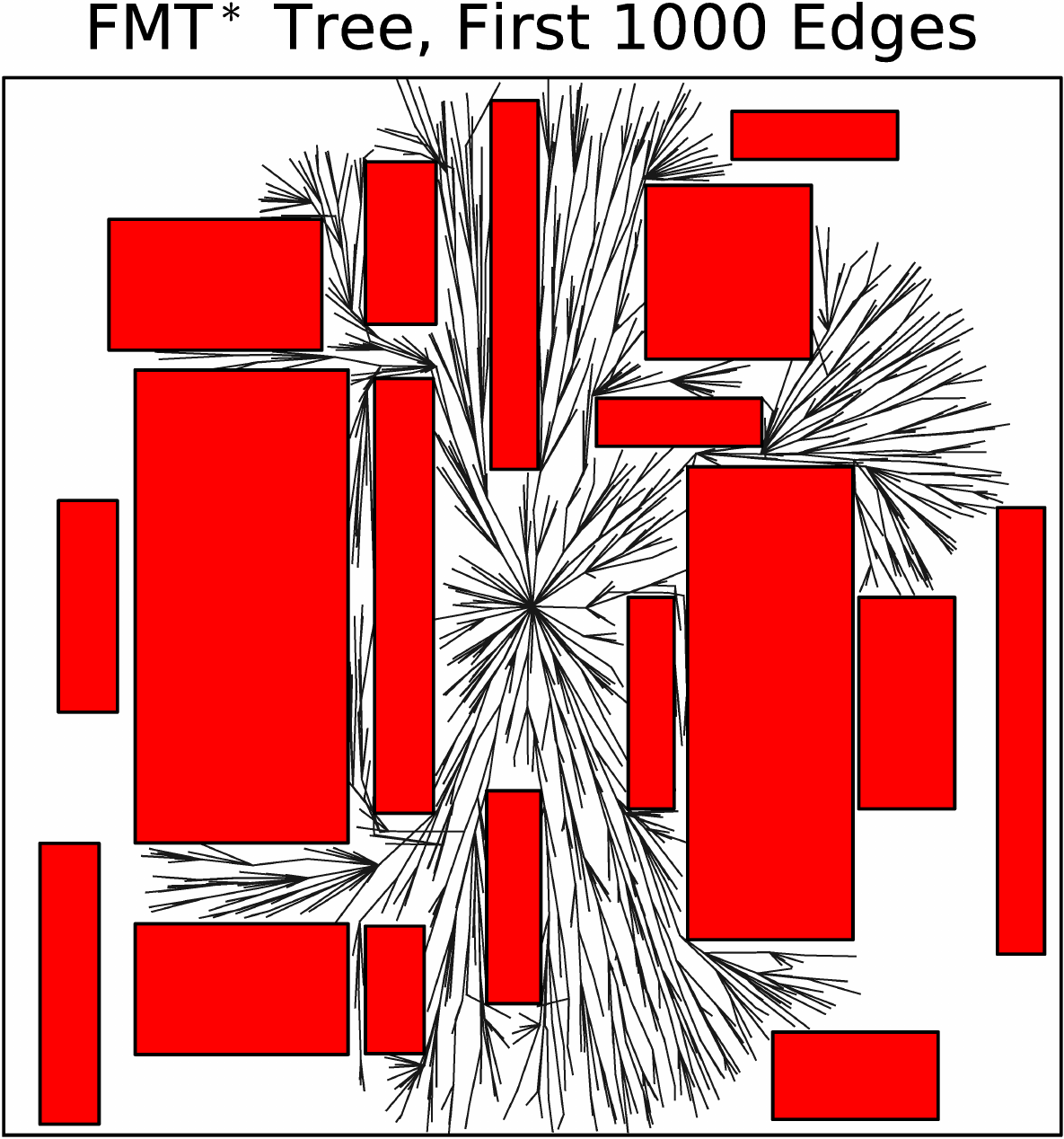}
  }
  \subfigure{
  \includegraphics[width=50mm]{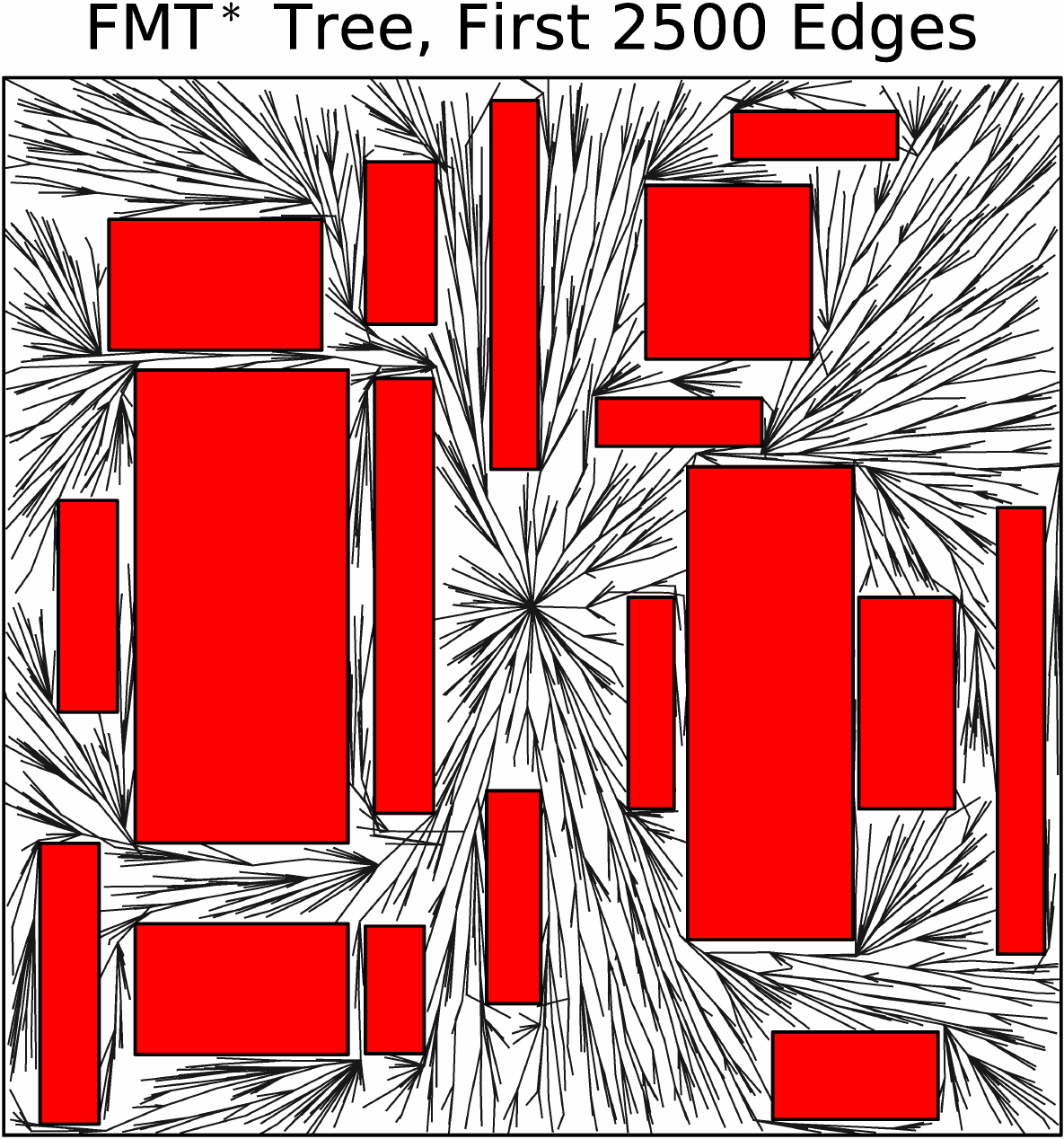}
  }

  \caption{The \FMT algorithm generates a tree by moving steadily
    outward in cost-to-{\color{black}arrive} space. This figure portrays the growth of
    the tree in a 2D environment with 2,500 samples (only edges are shown).}
  \label{fig:tree_growth}
\end{figure}

\emph{Organization}: This paper is structured as follows. In Section
\ref{sec:setup} we formally define the optimal path planning
problem. In Section \ref{prtintro} we present a
  high-level description of \FMT$\!$, describe the main intuition
  behind its correctness, conceptually compare it  to existing AO
  algorithms, and discuss its implementation details.  In Section
\ref{sec:AO} we prove the asymptotic optimality of \FMT\!, derive
convergence rate bounds, and characterize its computational complexity. In Section \ref{sec:extension} we extend \FMT along three main directions, namely non-uniform sampling strategies, general cost functions, and a variant of the algorithm that relies on $k$-nearest-neighbor computations.
In Section \ref{sec:sims}  we present results from numerical experiments supporting our statements. Finally, in Section \ref{sec:conc}, we draw some conclusions and discuss directions for future work.

\emph{Notation}: Consider the Euclidean space in $d$ dimensions, i.e.,
$\reals^d$. A ball of radius $r>0$ centered at $\bar x \in \reals^d$
is defined as $B(\bar  x;\ r):=\{x\in \reals^d \,  | \, \|x - \bar x
\|< r\}$. Given a subset $\mathcal X$ of $\reals^d$, its boundary is
denoted by $\partial \mathcal X$ and its closure is denoted by $\mathrm{cl}(\mathcal X)$. Given two points $x$ and $y$ in
$\reals^d$, the line connecting them is denoted by
$\overline{xy}$. Let $\zeta_d$ denote the volume of the unit ball in
$d$-dimensional Euclidean space. The cardinality of a set $S$ is
written as $\card{S}$. Given a set $\mathcal X \subseteq \reals^d$,
$\mu(\mathcal X)$ denotes its $d$-dimensional Lebesgue measure. Finally, the complement of a probabilistic event $A$ is denoted by $A^c$.

\section{Problem Setup}\label{sec:setup}
The problem formulation follows closely the problem formulation in
\citep{Karaman.Frazzoli:IJRR2011}, with two subtle, yet
  important differences, namely a notion of regularity for goal
  regions and a refined definition of path clearance. Specifically,
let $\mathcal X =[0,\, 1]^d$ be the configuration space, where
the dimension, $d$, is an integer larger than or equal to
  two. Let $\xobs$ be the obstacle region, such that $\mathcal{X}
\setminus \xobs$ is an open set (we consider $\partial \mathcal{X}
\subset \xobs$). The obstacle-free space is defined as
$\xfree = \text{cl}(\mathcal X \setminus \mathcal \xobs)$. The initial
condition $\xinit$ is an element of $\xfree$, and the goal region
$\xgoal$ is an open subset of $\xfree$. A path planning problem is
denoted by a triplet $(\xfree, \xinit, \xgoal)$. A function $\sigma : [0, 1] \to  \reals^d$ is called a
\emph{path} if it {\color{black} is continuous and} has 
\emph{bounded variation}{\color{black}, see
\cite[Section 2.1]{Karaman.Frazzoli:IJRR2011} for a formal
definition. In the setup of this paper, namely, for continuous functions
on a bounded, one-dimensional domain, bounded variation is exactly
equivalent to finite length.} A path is said to be \emph{collision-free} if $\sigma(\tau)\in \xfree$ for all $\tau\in [0,\, 1]$. A path is said to be a \emph{feasible path} for the planning problem $(\xfree, \xinit, \xgoal)$ if it is collision-free, $\sigma(0) = \xinit$, and $\sigma(1)\in \mathrm{cl}(\xgoal)$.

A goal region $\mathcal{X}_{\text{goal}}$ is said to be \emph{regular}
if there exists $ \xi > 0$ such that $\forall x \in \partial
\mathcal{X}_{\text{goal}}$, there exists a ball in the goal region, say $B(\bar x; \xi) \subseteq
\mathcal{X}_{\text{goal}}$, such that $x$ is on the boundary of the ball,
i.e., $x \in \partial B(\bar x; \xi)$. In other words, a regular goal region
is a ``well-behaved" set where the boundary has bounded curvature. We
will say $\mathcal{X}_{\text{goal}}$ is $\xi$-regular if
$\mathcal{X}_{\text{goal}}$ is regular for the parameter
$\xi$. Such a notion of regularity, not present in
  \citep{Karaman.Frazzoli:IJRR2011}, is needed because to return a
  feasible solution, there must be samples in
  $\mathcal{X}_{\text{goal}}$, and for that solution to be
  near-optimal, some samples must be near the edge of
  $\mathcal{X}_{\text{goal}}$ where the optimal path meets
  it. The notion of $\xi$-regularity essentially formalizes the notion of
  $\mathcal{X}_{\text{goal}}$ having enough measure near this edge to
  ensure that points are sampled near it.

Let $\Sigma$ be the set of all paths. A cost function for the planning
problem $(\xfree, \xinit, \xgoal)$ is a function $c:\Sigma \to
\reals_{\geq 0}$ from the set of paths to the set of nonnegative real
numbers; in this paper we will mainly consider cost functions $c(\sigma)$
that are the \emph{arc length} of $\sigma$ with respect to the
Euclidean metric in $\mathcal X$ (recall that $\sigma$ is, by
definition, rectifiable). Extension to more general cost functions, potentially not satisfying the triangle inequality are discussed in Section \ref{subsec:cost}. The optimal path planning problem is then defined as follows:

\begin{quote}{\bf Optimal path planning problem}: 
Given a path planning problem $(\xfree, \xinit, \xgoal)$ with a regular goal region and an arc length function $c:~\Sigma \to \reals_{\geq 0}$, find a feasible path $\sigma^{*}$ such that $c(\sigma^{*} )= \min\{c(\sigma):\sigma \text{ is feasible}\}$. If no such path exists, report failure.
\end{quote}

Finally, we introduce some definitions concerning the \emph{clearance} of a path, i.e., its ``distance" from $\xobs$  \citep{Karaman.Frazzoli:IJRR2011}. For a given $\delta>0$, the $\delta$-interior of $\xfree$ is defined as 
 the set of all points that are at least a distance $\delta$ away from
any point in $\xobs$. A collision-free path $\sigma$  is said to have
strong $\delta$-clearance if it lies entirely inside the
$\delta$-interior of $\xfree$.
A path planning problem with optimal path cost $c^*$ is called
$\delta$-robustly feasible 
  if there exists a strictly positive
  sequence $\delta_n \rightarrow 0$, with $\delta_n \le \delta \;\, \forall n \in \mathbb{N}$, and a sequence $\{\sigma_n\}_{n=1}^{\infty}$ of
  feasible paths such that $\lim_{n \rightarrow \infty} c(\sigma_n) =
  c^*$ and for all $n \in \mathbb{N}$, $\sigma_n$ has strong
  $\delta_n$-clearance, 
  $\sigma_n(1) \in \partial \mathcal{X}_{\text{goal}}$,
  $\sigma_n(\tau) \notin \mathcal{X}_{\text{goal}}$ for all $\tau \in
  (0,1)$, and $\sigma_n(0) = x_{\text{init}}$. Note this definition is slightly
different mathematically than admitting a \emph{robustly optimal solution} as in
\citep{Karaman.Frazzoli:IJRR2011}, but the two are nearly identical in
practice. Briefly, the difference is necessitated by the
  definition of a homotopy class only involving pointwise limits, as
  opposed to limits in bounded variation norm, making the conditions of a
  robustly optimal solution potentially vacuously satisfied.

\section{The Fast Marching Tree Algorithm (\FMT\!)}
\label{prtintro}

In this section we present the Fast Marching Tree algorithm (\FMT\!). In Section \ref{subsec:high_lev} we provide a high-level description. In Section \ref{subsec:basic} we present some basic properties and discuss the main intuition behind \FMT\!'s design. In Section \ref{subsec:comp} we conceptually compare \FMT to existing AO algorithms and discuss its structural advantages. Finally, in Section \ref{subsec:detail_des} we provide a detailed description of \FMT together with implementation details, which will be instrumental to the computational complexity analysis given in Section \ref{subsec:complexity}.

%he proof of its (asymptotic) optimality and
%a convergence rate bound will be presented in Section \ref{sec:AO}.

%Let $r_n$ be as in Karaman \& Frazzoli (2011), namely: 
%\begin{equation}
%\label{radius}
%r_n = (1+\varepsilon) \cdot 2(1+\frac{1}{d})^{\frac{1}{d}} (\frac{\mu(\mathcal{X}_{\text{free}})}{\zeta_d})^{\frac{1}{d}} (\frac{\log(n)}{n})^{\frac{1}{d}},
%\end{equation}
%where $\varepsilon > 0$, $\mu$ is the Lebesgue measure in $\mathbb{R}^d$ and $\zeta_d$ is the measure (volume) of the unit ball in $d$ dimensions.
%\subsection{The \FMT Algorithm}
\subsection{High-Level Description}\label{subsec:high_lev}

The \FMT algorithm performs a forward dynamic programming
recursion over a predetermined number of sampled points and correspondingly generates a \emph{tree of paths} by moving steadily outward in cost-to-{\color{black}arrive} space (see Figure \ref{fig:tree_growth}). The dynamic programming recursion performed by  \FMT is characterized by three key features:
\begin{itemize}\itemsep-1pt
\item It is \emph{tailored} to disk-connected graphs, where two
samples are considered \emph{neighbors}{\color{black}, and }hence connectable{\color{black},} if their
distance is below a given bound, referred to as the \emph{connection
radius}. 
\item It performs graph construction and graph search
  \emph{concurrently}.
\item For the evaluation of the immediate cost in the dynamic
programming recursion, \edit{the algorithm} ``lazily" ignores the presence of
obstacles, and whenever a locally-optimal (assuming no obstacles)
connection to a new sample intersects an obstacle, that sample is
simply skipped and left for later as opposed to looking for other
connections in the neighborhood. 
\end{itemize}

{\color{black} The first feature concerns the fact that \FMT exploits the structure of disk-connected graphs to run dynamic programming for shortest path computation, in contrast with successive approximation schemes (as employed, e.g., by label-correcting methods). This aspect of the algorithm  is illustrated in Section \ref{subsec:basic}, in particular, in Theorem \ref{thrm:sp} and Remark \ref{remark:spa}. An extension of \FMT to $k$-nearest-neighbor graphs, which are structurally very similar to disk-connected graphs, is studied in Section \ref{subsec:FMTkNN} and numerically evaluated in Section \ref{sec:sims}.} The last feature, which makes the algorithm ``lazy" {\color{black} and represents the key innovation}, {\color{black} dramatically reduces the number of costly collision-check computations. However,} it may cause
\emph{suboptimal} connections. A central property of \FMT is that the
cases where a suboptimal connection is made become vanishingly rare as
the number of samples goes to infinity, which \edit{is} key in proving that the
algorithm is AO {\color{black} (Sections  \ref{subsec:basic} and \ref{sec:AO})}. 

\begin{algorithm}
\caption{{\color{black} Fast Marching Tree Algorithm (\FMT\!): Basics}}
\label{pseudofmt}
\begin{algorithmic}[1]
{\small 
\REQUIRE sample set $V$ comprising of $\xinit$ and $n$ samples in $\xfree$, at least one of which is also in $\xgoal$ \label{line:sample}
%\STATE Sample $n$ nodes in $\xfree$ and add them to node set $V$ together with $\xinit$\label{line:sample}
\STATE Place $\xinit$ in $\Hset$ and all other samples in $\Wset$; initialize tree with root node $\xinit$\label{line:setSetup}
\STATE Find lowest-cost node $z$ in  $\Hset$ \label{line:extract}
\STATE $\qquad$For each of {\color{black}$z$'s} neighbors $x$ in $\Wset$: \label{line:neighSearch}
\STATE $\qquad\qquad$Find neighbor nodes $y$ in $\Hset$ \label{line:findYnear}
\STATE $\qquad\qquad$Find locally-optimal one-step connection to $x$
from among nodes $y$ \label{line:findPath}
\STATE $\qquad\qquad$If that connection is collision-free, add edge to tree of paths \label{line:ins1}
%\STATE $\qquad\qquad \qquad$Remove $x$ from $\Wset$ \label{line:ins2}
\STATE $\qquad$Remove successfully connected \edit{nodes $x$} from $\Wset$ and add them to $\Hset$ \label{line:ins3}
\STATE $\qquad$Remove $z$ from $\Hset$ and add it to $\Vc$ \label{line:removeZ}
\STATE $\qquad$Repeat until either: \\ \label{line:terminate}
$\qquad\qquad$(1) $\Hset$ is empty $\Rightarrow$ report failure\\
$\qquad\qquad$(2) Lowest-cost node $z$ in $\Hset$ is in $\xgoal$ $\Rightarrow$ return unique path to $z$ and\\
$\qquad\qquad\qquad\qquad$ report success
}
\end{algorithmic}
\end{algorithm}

%\begin{figure}
%  \centering
%  \includegraphics[width=140mm]{diagram.pdf}
%  \caption{Diagram.}
%  \label{fig:diagram}
%\end{figure}
A basic pseudocode description of \FMT is given in Algorithm \ref{pseudofmt}.
% and a diagrammatic illustration is given in Figure \ref{fig:diagram}. 
The input to the algorithm, besides the path planning problem
definition, i.e., ($\xfree, \xinit, \xgoal$){\color{black},} is a sample set $V$ comprising $\xinit$ and $n$ samples in $\xfree$ (line \ref{line:sample}). {\color{black}We refer to samples added to the tree of paths as nodes.} 
Two samples $u,v \in V$ are considered \emph{neighbors} if their Euclidean distance is smaller than
\[
r_n = \gamma \, \biggl(\frac{\log(n)}{n}\biggr)^{1/d},
\]
where $ \gamma > 2 \, \Bigl (1/d \Bigr)^{1/d} \, \Bigl (\mu(\mathcal{X}_{\text{free}})/\zeta_d \Bigr)^{1/d}$ is a tuning parameter.  The algorithm makes use of a partition of $V$ into three subsets, namely $\Wset$, $\Hset$, and $\Vc$.
%, which are collectively exhaustive (i.e., $\Wset \cup \Hset \cup \Vc = V$) and mutually exclusive (i.e., $\Wset \cap \Hset = \emptyset$, $\Hset \cap \Vc = \emptyset$, and $\Wset \cap \Vc =  \emptyset$). 
The set $\Wset$ consists of all of the samples that have not yet been
considered for addition to the \edit{incrementally
grown tree of paths}.  The set  $\Hset$ contains samples that are currently active, in
the sense that they have already been added to the tree (i.e., a
collision-free path from $\xinit$ with a given cost-to-{\color{black}arrive} has been
found) and are candidates for further connections to samples in
$\Wset$. The set $\Vc$ contains samples that have been added to the
tree and are no longer considered for any new connections. Intuitively, these samples are not near enough to the edge of the expanding tree to actually have any new connections made with $\Wset$. 
{\color{black}Removing them from $\Hset$ reduces the number of nodes that need to be considered as neighbors for sample $x$.}
%; it  As it progresses, the algorithm maintains dual sets $H$ (in purple) and $W$ (in red). 
%A sample execution of the \FMT algorithm is illustrated in Figure \ref{fig:FMT_execution}. 
The \FMT algorithm initially places $\xinit$ into $\Hset$ and all
other samples in $\Wset$, while $\Vc$ is initially empty (line
\ref{line:setSetup}). The algorithm then progresses by extracting the
node with the lowest cost-to-{\color{black}arrive} in $\Hset$ (line \ref{line:extract},
Figure \ref{fig:FMT_1}), call it $z$, and finds all its neighbors
within $\Wset$, call them $x$ samples (line \ref{line:neighSearch},
Figure \ref{fig:FMT_1}). For each sample $x$, \FMT finds all its
neighbors within $\Hset$, call them $y$ nodes (line
\ref{line:findYnear}, Figure \ref{fig:FMT_2}). The algorithm then
evaluates the cost of all paths to $x$ obtained by
concatenating previously computed paths to nodes $y$  with straight
lines connecting them to $x$, referred to as ``local one-step"
connections. Note that this step \emph{lazily} ignores the presence
of obstacles. \FMT then picks the path with lowest cost-to-{\color{black}arrive} to $x$
(line \ref{line:findPath}, Figure \ref{fig:FMT_2}). If the last edge
of this path, i.e., the one connecting $x$ with one of its neighbors
in  $\Hset$, is collision-free, then it is added to the tree (line \ref{line:ins1}, Figure  \ref{fig:FMT_3}).
When all samples $x$ have been considered, the ones that have
been successfully connected to the tree are added to $\Hset$ and removed from $\Wset$ (line \ref{line:ins3}, Figure \ref{fig:FMT_4}), while the \edit{others} remain in $\Wset$ until a further iteration of the
algorithm\footnote{In this paper we consider {\color{black}a batch} implementation, whereby all
successfully connected $x$ are added to $\Hset$ {\color{black}in batch} \emph{after} all the
samples $x$ have been considered. {\color{black} It is easy to show that if, instead, each sample $x$ were added
to $\Hset$ as soon as its obstacle-free connection was found, then with probability 1, the algorithm would make all the
same connections as in the batch setting, regardless of what 
order the $x$ were considered in. Thus, since adding the
samples $x$ serially or in batch makes no difference to the
algorithm's output, we prefer the batch implementation for its
simplicity and parallelizability.}}. 
%In this paper we consider {\color{black}a batch} implementation, whereby the order in
%which the samples $x$ are considered does not matter, as all
%successfully connected $x$ are added to $\Hset$ {\color{black}in a batch} \emph{after} all the
%samples $x$ have been considered. If, instead, $x$ were added
%to $\Hset$ as soon as its obstacle-free connection was found, {\color{black}it is
%easy to show that with probability 1, the algorithm would make all the
%same connections as in the batch setting, regardless of what 
%order the $x$ were considered in. Thus, since adding the
%samples $x$ serially or in batch makes no difference to the
%algorithm's output, we prefer the batch implementation for its
%simplicity and parallelizability.}
Additionally, node $z$
is inserted into $\Vc$ (line \ref{line:removeZ}, Figure
\ref{fig:FMT_4}), and \FMT moves to the next iteration (an iteration
comprises lines \ref{line:extract}--\ref{line:removeZ}). \edit{The algorithm terminates when the lowest-cost node
in $\Hset$ is also in the goal region or when $\Hset$ becomes empty.} Note that at the beginning of each iteration every sample in $V$ is either in $\Hset$ \emph{or} in $\Wset$ \emph{or} in $\Vc$.

%As will be proven later, without obstacles, this would ensure the shortest path to $x$. 
%
%% If a collision-free connection can be formed between any of the
%% neighbors and the start node, then an edge is created (Figure
%% \ref{fig:FMT_execution}, Slides 2 -- 4). After all neighbors have
%% been evaluated, the nodes with newly formed edges are removed from
%% $\Wset$ and inserted into $\Hset$. The start node is removed from
%% $\Hset$, since all of its potential connections have been
%% evaluated, and it will no longer be useful for future connections
%% (Figure \ref{fig:FMT_execution}, Slides 6 -- 7). The algorithm then selects the lowest cost-to-come node in $\Hset$, call it $z$, and repeats this process. All nodes in $\Wset$ that are neighbors of $z$ are evaluated. 
%For each neighbor node $x$ in $\Wset$, \FMT find the path to $x$ with lowest cost-to-come among those paths obtained by concati
%
%
A few comments are in order. First, the choice of the connection radius relies
on a trade-off between computational complexity (roughly speaking,
more neighbors lead to more computation) and quality of the computed
path (roughly speaking, more neighbors lead to  more paths
{\color{black} to optimize} over), and is \edit{an important parameter in the analysis and implementation} of \FMT
$\!$. This \edit{choice} will be studied theoretically in Section \ref{sec:AO} and numerically in Section \ref{subsubsec:rmtuning}. Second, as shown in Figure \ref{fig:algoBasic}, \FMT concurrently performs graph construction and graph search, which is carried out via a dynamic programming recursion tailored to disk graphs (see Section \ref{subsec:basic}). 
This recursion lazily skips collision-checks and may indeed introduce
\emph{suboptimal} connections. In Section \ref{subsec:basic} we will
intuitively discuss why such suboptimal connections are very
rare and still allow the algorithm to asymptotically approach an
optimal solution (Theorem \ref{thrm:AO}). Third, the lazy
collision-checking strategy employed by \FMT is fundamentally
different from the one proposed in the past within the probabilistic
roadmap framework  \citep{Bohlin.Kavraki:ICRA00},
\citep{GS.JL:03}. Specifically, the lazy $\mathrm{PRM}$ algorithm
presented in  \citep{Bohlin.Kavraki:ICRA00} first constructs a graph
assuming that all connections are collision-free (refer to this graph
as the \emph{optimistic graph}).  Then, it searches for a shortest
\emph{collision-free} path by repeatedly searching for a shortest path
over the optimistic graph and then checking whether it is
collision-free or not. Each time a collision is found, the corresponding edge is removed from the optimistic graph and a new shortest path is computed.  The ``Single-query,
Bi-directional, Lazy in collision-checking" algorithm, $\mathrm{SBL}$ \citep{GS.JL:03}, implements a similar idea within the context of bidirectional search. In contrast to lazy $\mathrm{PRM}$ and $\mathrm{SBL}$,  \FMT 
\emph{concurrently} performs graph construction and graph search, and
as soon as a shortest path to the goal region is found, that path is
guaranteed to be collision-free. This approach \edit{provides}
computational savings in especially cluttered environments, wherein
lazy PRM-like algorithms will require a large number of attempts to
find a collision-free shortest path.
\begin{comment}
Philosophically, lazy $\mathrm{PRM}$
and $\mathrm{SBL}$, as well as the recently presented
$\mathrm{LazyWeightedA}^{*}$ algorithm \citep{BK.MP.ea:14},
\emph{delay} collision-checking calls, while \FMT \emph{skips}
collision-checking calls at the cost of possibly introducing a few
suboptimal connections, see Section \ref{subsec:basic}. 
\end{comment}

\begin{figure}
  \centering
%  \subfigure{
%  \includegraphics[width=50mm]{pan_1.pdf}
%  }
  \subfigure[Lines \ref{line:extract}--\ref{line:neighSearch}: \FMT
  selects {\color{black}the} lowest-cost node $z$ from set $\Hset$ and finds its neighbors within $\Wset$.]{
  \includegraphics[width=75mm]{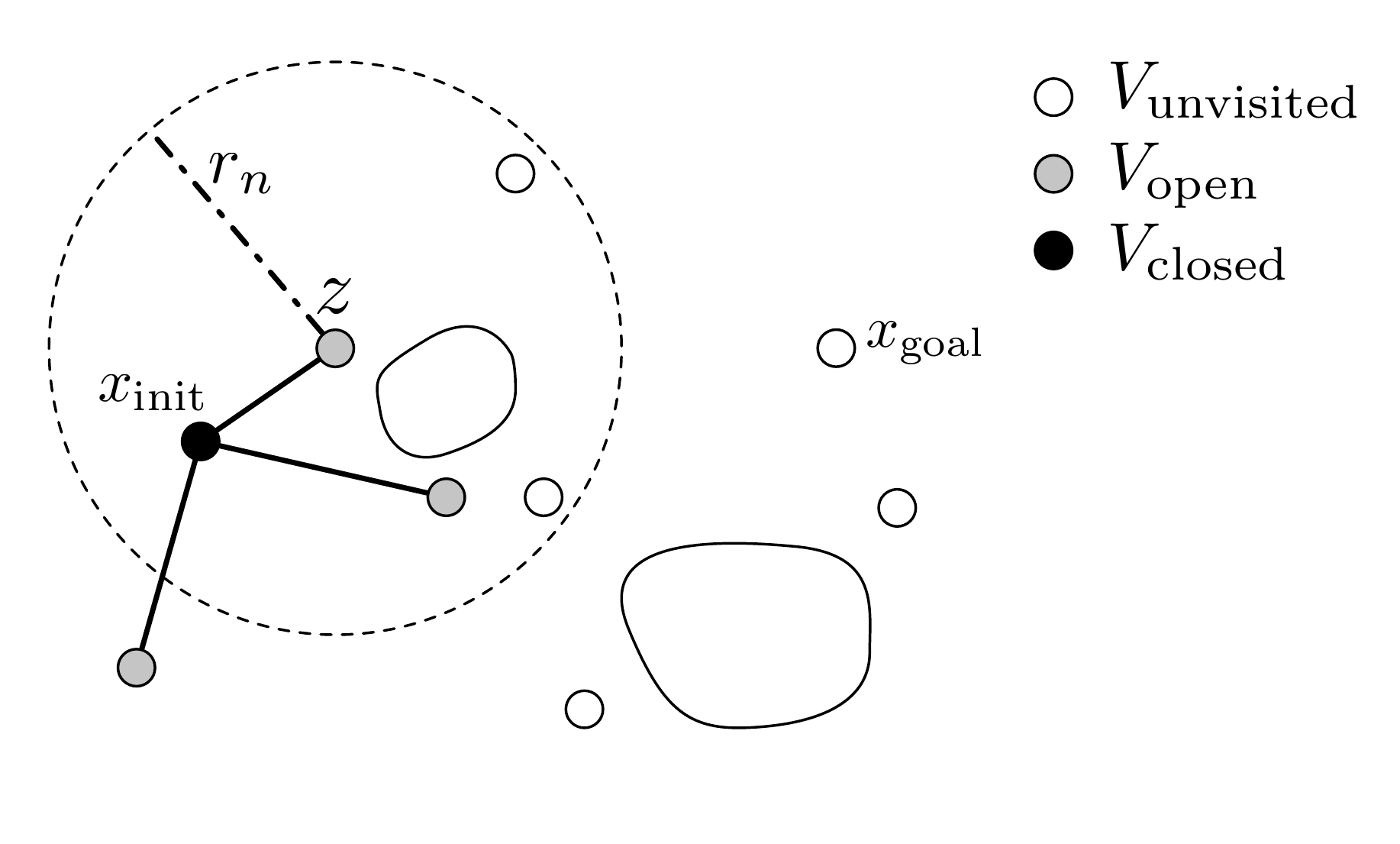}\label{fig:FMT_1}
  }\quad
  \subfigure[Lines \ref{line:findYnear}--\ref{line:findPath}: given
  a neighboring node $x$, \FMT finds the neighbors of $x$ within
  $\Hset$ and searches for a locally-optimal one-step connection. Note that paths intersecting obstacles are also lazily considered.]{
  \includegraphics[width=75mm]{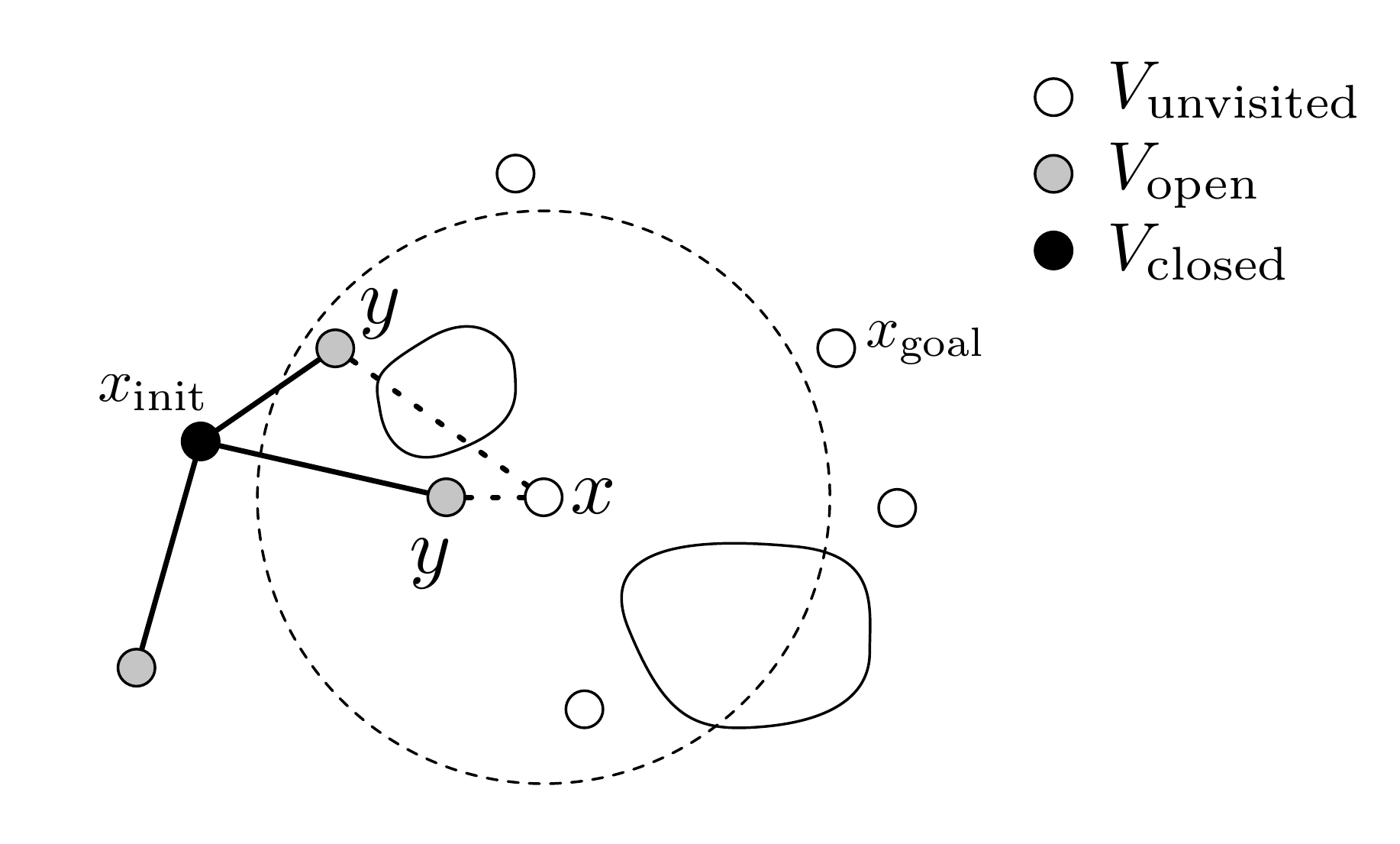}\label{fig:FMT_2}
  }
  \subfigure[Line \ref{line:ins1}: \FMT selects the locally-optimal one-step connection to $x$ ignoring obstacles, and adds that connection to the tree if it is collision-free.]{
  \includegraphics[width=75mm]{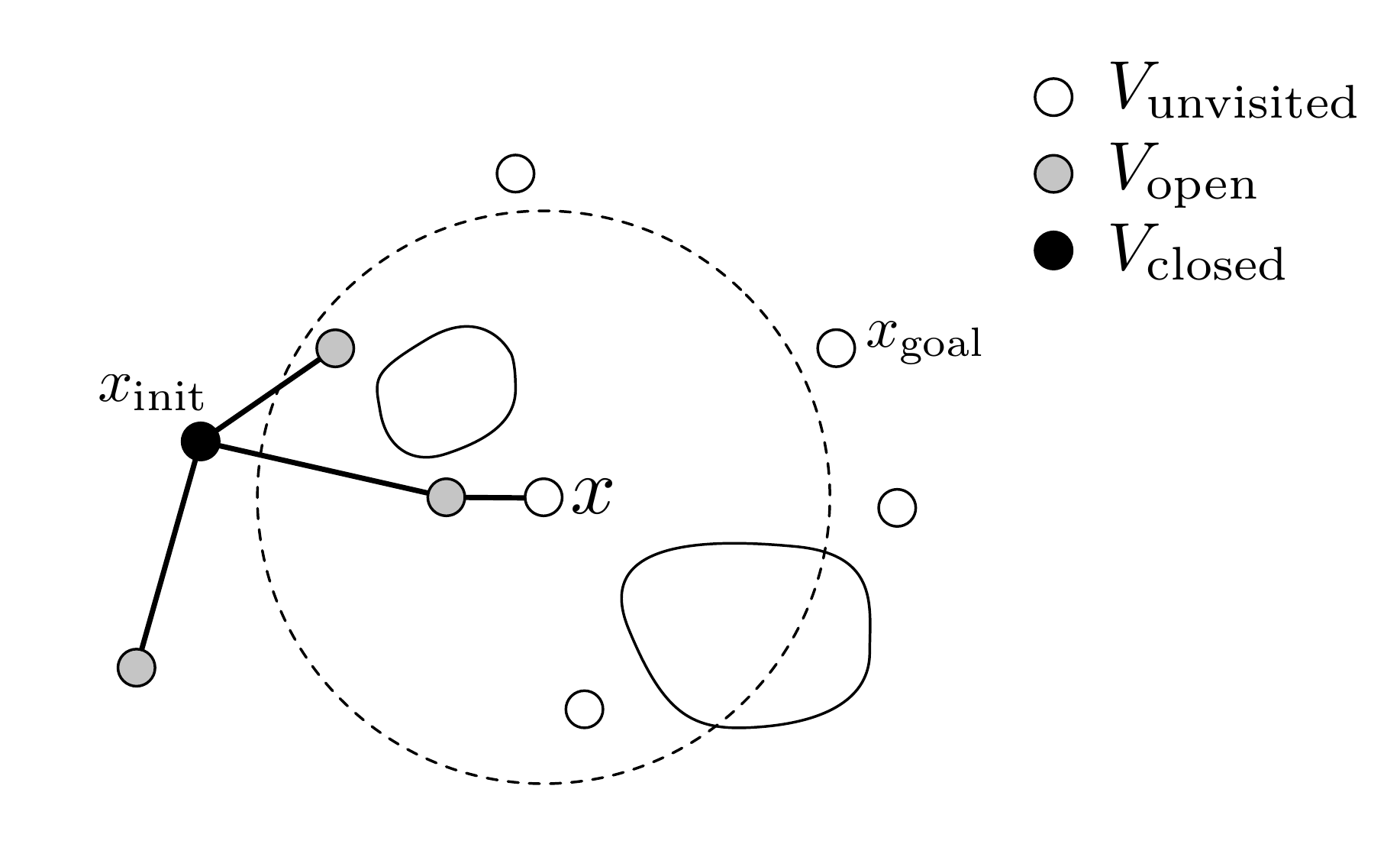}\label{fig:FMT_3}
  }\quad
    \subfigure[Lines \ref{line:ins3}--\ref{line:removeZ}: After all neighbors of $z$ in $\Wset$ have been explored,  \FMT adds successfully connected nodes to $\Hset$, places $z$ in $\Vc$, and moves to the next iteration.]{
  \includegraphics[width=75mm]{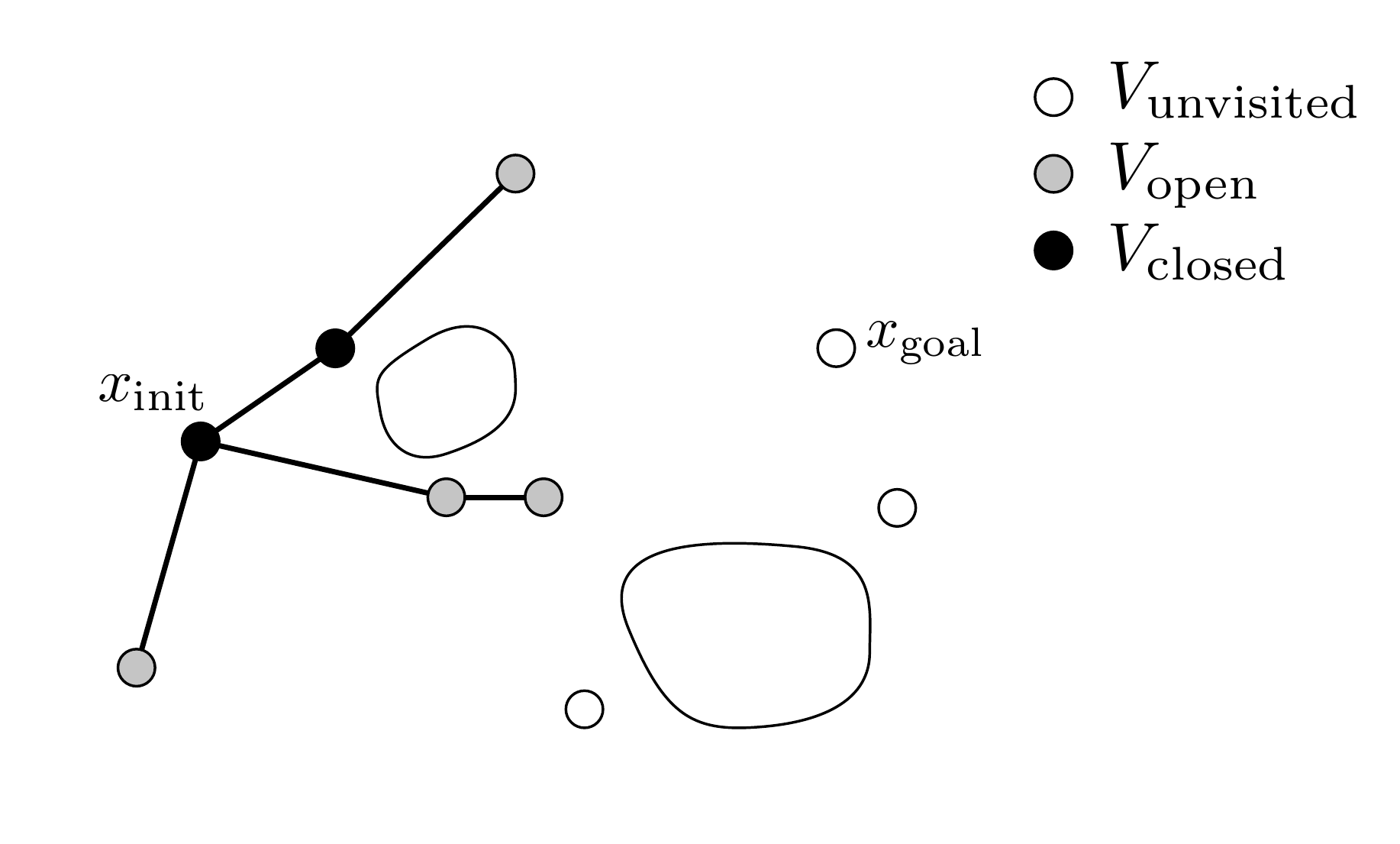}\label{fig:FMT_4}
  }
  \caption{An iteration of the \FMT algorithm.  \FMT \emph{lazily} and \emph{concurrently} performs  graph construction and graph search. Line references are with respect to Algorithm \ref{pseudofmt}. In panel (b), node $z$ is re-labeled as node $y$ since it is one of the neighbors of node $x$.}
  \label{fig:algoBasic}
\end{figure}

\subsection{Basic Properties and Intuition}\label{subsec:basic}
This section discusses basic properties of the \FMT algorithm and provides intuitive reasoning about its correctness and effectiveness. We start by showing that the algorithm terminates in at most $n$ steps, where $n$ is the number of samples.
\begin{theorem}[Termination]
\label{termination}
Consider a path planning problem $(\mathcal{X}_{\text{free}}, x_{\text{init}}, \mathcal{X}_{\text{goal}})$ and any $n \in \naturals$. The \FMT algorithm always terminates in at most $n$ iterations (i.e., in $n$ loops through Algorithm \ref{pseudofmt} lines \ref{line:extract}--\ref{line:removeZ}).
\end{theorem} 
\begin{proof}
Note two key facts: (i) \FMT terminates and reports failure if $\Hset$ is ever empty, and (ii) the lowest-cost node in $\Hset$ is removed from $\Hset$ at each iteration. 
% note first that all for loops are evaluated over a set which remains static for the duration of the for loop, and thus \FMT can never get stuck in a for loop, but only in the while loop.  
Therefore, to prove the theorem it suffices to prove the invariant
that any sample that has ever been added to $\Hset$ can never be added
again.  To establish the invariant, observe that at a given
iteration, only samples in $\Wset$ can be added to $\Hset$, and each time
a sample is added, it is removed from $\Wset$. Finally, since $\Wset$
never has samples added to it, a sample can only be added to $\Hset$
once. Thus the invariant is proved, and, in turn, the theorem.
\end{proof}
To {\color{black}understand the correctness} of the algorithm, consider
first the case without obstacles and where there is only one sample in
$\xgoal$, denoted by $\xfinal$. In this case \FMT uses dynamic
programming to find the shortest path from $\xinit$ to $\xfinal${\color{black},} if
one exists{\color{black},} over the $r_n$-disk graph induced by $V$, i.e., over the
graph where there exists an edge between two samples $u, v\in V$ if
and only if $\|u  -v \|\ < r_n$. This fact is proven in the following
theorem{\color{black}, the proof of which} highlights how \FMT applies dynamic programming over an $r_n$-disk graph.

\begin{theorem}[\FMT in obstacle-free environments]
\label{thrm:sp}
Consider a path planning problem $(\xfree, \xinit, \xgoal)$, where $\xfree = \mathcal X$ (i.e., there are no obstacles) and $\xgoal = \{\xfinal\}$ (i.e., there is a single node in $\xgoal$). Then, \FMT computes a shortest path from $\xinit$ to $\xfinal$ (if one exists)  over the $r_n$-disk graph induced by $V$.
\end{theorem} 
\begin{proof}
For a sample $v\in V$, let $c(v)$ be the length of a shortest path to
$v$ from $\xinit$ over the $r_n$-disk graph induced by $V${\color{black}, where
}$c(v)=\infty$ if no path to $v$ exists. Furthermore, let
$\texttt{Cost}(u,v)$ be the length of the edge connecting samples $u$
and $v$ (i.e., its Euclidean distance). It is well known that shortest path distances satisfy the Bellman principle of optimality \citep[Chapter 24]{Cromen.ea:01}, namely
\begin{equation}\label{eq:Bellman}
c(v) = \min_{u: \|u-v\| <  r_n} \, {\color{black}\{}c(u) + \texttt{Cost}(u,v){\color{black}\}}.
\end{equation}
\FMT repeatedly applies this relation in a way that exploits the geometry of $r_n$-disk graphs. Specifically, \FMT maintains two loop invariants:
\begin{quote}
{\bf Invariant 1}: At the beginning of each iteration, the shortest
path in the $r_n$-disk graph to a sample $v\in \Wset$ must pass through a node $u \in \Hset$.
\end{quote}
To prove Invariant 1, assume \edit{for} contradiction that the invariant is
not true, that is there exists a sample $v\in \Wset$ {\color{black} with
  a shortest path that}
does not contain any node in $\Hset$. At the first iteration this
condition is clearly false, as $\xinit$ is in $\Hset$. For
subsequent iterations, the contradiction assumption implies that along
the shortest path there is at least one edge $(u, w)$ where $u\in \Vc$
and $w \in \Wset$. This situation is, however, impossible as before
$u$ is placed in $\Vc$, all its neighbors, including $v$, must have
been extracted from $\Wset$ and inserted into $\Hset${\color{black}, since }insertion into
$\Hset$ is ensured {\color{black}when} there are no obstacles{\color{black}. Thus, we have }a contradiction. 

The second invariant is:
\begin{quote}
{\bf Invariant 2}: At the end of each iteration, all neighbors of $z$ in $\Wset$ are placed in $\Hset$ with their shortest paths computed.
\end{quote}
To see this, let us \edit{induct} on the number of iterations. At the
first iteration, Invariant 2 is trivially true. Consider, then,
iteration $i+1$ and let $x \in \Wset$ be a neighbor of $z$. In line
\ref{line:findPath} of Algorithm \ref{pseudofmt}, \FMT computes a path
to $x$ {\color{black} with cost $\tilde c(x)$} given by
\[
\tilde c(x) = \min_{u \in \Hset : \,\|u - x\| <  r_n} \, {\color{black}\{}c(u) +  \texttt{Cost}(u,x){\color{black}\}},
\]
where by the inductive hypothesis the shortest paths to nodes in
$\Hset$ are all known{\color{black}, since} all nodes placed in $\Hset$ before or at
iteration $i$ have had their shortest paths computed.
To prove that $\tilde c(x)$ is indeed equal to the cost of a shortest path to $x$, i.e., $c(x)$, {\color{black}we} need to prove that the Bellman principle of optimality is satisfied, that is
\begin{equation}\label{eq:dp}
 \min_{u \in \Hset : \,\|u - x \| <  r_n} \, {\color{black}\{}c(u) +  \texttt{Cost}(u,x){\color{black}\}} = \min_{u: \, \|u - x \| < r_n} \, {\color{black}\{}c(u) +  \texttt{Cost}(u,x){\color{black}\}}.
\end{equation}
To prove the above equality, note first that there are no nodes $u\in
\Vc$ such that $\|u - x \|< r_n$, otherwise $x$ could not be in
$\Wset$ (by using the same argument \edit{from} the proof of
Invariant 1). Consider, then, samples  $u \in \Wset$ such that $\|u -
v \|< r_n$.  From Invariant 1 we know that a shortest path to $u$ must
pass through a node $w \in \Hset$. If $w$ is within a distance $r_n$
from $x$, then, by the triangle inequality, {\color{black}we} obtain a shorter path
by concatenating a shortest path to $w$ with the edge connecting $w$
and $x$---hence, $u$ can be discarded when looking for a shortest path
to $x$. If, instead, $w$ is farther than a distance $r_n$ from $x$,
{\color{black}we} can write by repeatedly applying the triangle inequality:
\[
c(u) + \texttt{Cost}(u,x) \geq c(w) +  \texttt{Cost}(w,x)  \geq c(w) + r_n.
\]
Since $c(w) \geq c(z)$ due to the fact that nodes are extracted from $\Hset$ in order of their cost-to-{\color{black}arrive}, and since $\texttt{Cost}(z,x) <  r_n$, {\color{black}we} obtain
\[
c(u) + \texttt{Cost}(u,x) >  c(z) + \texttt{Cost}(z,x),
\]
which implies that, again, $u$ can be discarded when looking for
a shortest path to $x$. Thus, equality \eqref{eq:dp} is proved and, in turn, Invariant 2.

Given Invariant 2, the theorem is proven by showing that, if there
exists a path from $\xinit$ to $\xfinal$, at some iteration the
lowest-cost node in $\Hset$ is $\xfinal$ and \FMT terminates{\color{black},} reporting
``success," see line \ref{line:terminate} in Algorithm \ref{pseudofmt}. We already know, by Theorem
\ref{termination}, that \FMT terminates in at most $n$
iterations. Assume \edit{by contradiction} that upon termination $\Hset$ is
empty, which implies that $\xfinal$ never entered $\Hset$ and hence is
in $\Wset$. This situation is impossible, since the shortest
path to $\xfinal$ would contain at least one edge $(u, w)$ with $u\in
\Vc$ and $w \in \Wset$, which as argued in the proof of Invariant 1
cannot happen. Thus the theorem is proved.
\end{proof}

\begin{remark}[{\color{black} \FMT\!\!, dynamic programming, and disk-graphs}]\label{remark:spa}
{\color{black} The functional equation \eqref{eq:Bellman} does not constitute an algorithm, it only stipulates an optimality condition. \FMT implements equation \eqref{eq:Bellman} by exploiting the structure of disk-connected graphs. Specifically, in the obstacle-free case, the disk\hyp{}connectivity structure ensures that \FMT visits nodes in a ordering compatible with directly computing \eqref{eq:Bellman}, that is, while computing the left hand side of equation \eqref{eq:Bellman} (i.e., the shortest path value $c(v)$), all the  relevant shortest path values on the right hand
side (i.e., the values $c(u)$) have already been computed (see proof
of Invariant 2). In this sense,} \FMT computes {\color{black} shortest}
paths by running direct dynamic programming, as opposed to
performing successive approximations as done {\color{black} by
  label-setting or label-correcting algorithms, e.g., Dijkstra's
  algorithm or the Bellman--Ford algorithm} \citep[Chapter 2]{DB:05}.
{\color{black} We refer the reader to \cite{MS:06} for an in-depth
  discussion of the differences between  direct dynamic programming
  methods (such as \FMT\!) and successive approximation methods (such
  as Dijkstra's algorithm) for shortest path computation.}
{\color{black} Such a direct approach is desirable since the
  cost-to-arrive value for each node is updated only once, and thus only
  one collision check is required per node in the obstacle-free case. When obstacles are
  introduced, \FMT sacrifices the ability to return an exact solution
  on the obstacle-free disk graph in order to retain the computational
  efficiency of the direct approach. The suboptimality introduced in
  this way is slight, as we prove in Section~\ref{sec:AO}, and
  only one collision check is required for the majority of nodes.}
%Specifically, \FMT exploits the
%structure of disk graphs to \emph{directly} apply the Bellman
%principle of optimality (equation \eqref{eq:Bellman}) over such 
%graphs. 
{\color{black} \FMT\!\!'s strategy} is reminiscent of the approach used for the
computation of shortest paths over acyclic graphs
\citep{MS:06}. Indeed, the idea of leveraging graph structure to
compute shortest paths over disk graphs is not new and was recently
investigated in \citep{LR.MS:11}---under the name of bounded leg
shortest path problem---and in \citep{SC.MJ:14}. Both works, however, do not use ``direct" dynamic programming arguments, but rather combine Dijkstra's algorithm with the concept of bichromatic closest pairs \citep{TC.AE:01}. 
\end{remark}

Theorem \ref{thrm:sp} shows that in the obstacle-free case \FMT
returns a shortest path{\color{black},} if one exists{\color{black},} over the $r_n$-disk graph
induced by the sample set $V$. This \edit{statement no longer holds, however,} when there are obstacles, as in this case \FMT might make connections that are suboptimal, i.e., that do not satisfy the Bellman principle of optimality. Specifically, 
\FMT will make a suboptimal connection  when exactly four conditions
are satisfied. Let $u_1$ be the optimal parent of $x$ with respect to
the $r_n$-disk graph where edges intersecting obstacles are removed{\color{black}. T}his graph is the ``correct'' graph \FMT should plan over if it were
not lazy. The sample $x$ will \edit{not} be connected to
$u_1$ by \FMT only if when $u_1$ is the
lowest-cost node in $\Hset$, there is another node
$u_2 \in \Hset$ such that (a) $u_2$ is
within a radius $r_n$ of $x$, (b) $u_2$ has
greater cost-to-{\color{black}arrive} than $u_1$, (c) obstacle-free
connection of $x$ to $u_2$ would have lower
cost-to-{\color{black}arrive} than connection to $u_1$, and (d) $u_2$ is blocked from
connecting to $x$ by an obstacle. These four conditions are
illustrated in Figure \ref{fig:fail_condition}. Condition (a) is
required because in order for $u_2$ to be connected to $x$, it must be within the connection radius of $x$.  Conditions (b), (c), and (d) combine as follows: condition (b) dictates that $u_1$ will be pulled from $\Hset$ before $u_2$ is. Due to (c), $u_2$ will be chosen as the potential parent of $x$. Condition (d) will cause the algorithm to discard the edge between them, and $u_1$ will be removed from $\Hset$, never to be evaluated again. Thus, in the future, the algorithm will never realize that $u_1$ was a better parent for $x$. If condition (b) were to fail, then $u_2$ would be pulled from $\Hset$ first, would unsuccessfully attempt to connect to $x$, and then would be removed from $\Hset$, leaving $x$ free to connect to $u_1$ in a future iteration. If condition (c) were to fail, the algorithm would attempt to connect $x$ to $u_1$ instead of $u_2$ and would therefore find the optimal connection. If condition (d) were to fail, then $u_2$ would indeed be the optimal parent of $x$, and so the optimal connection would be formed. Thus, if any of one these four conditions fail, then at some iteration (possibly not the first), $x$ will be connected optimally with respect to the ``correct" graph.  Note that the combination of conditions
 (a), (b), (c), and (d) make such suboptimal connections quite rare{\color{black}.
Additionally, samples must be within distance $r_n$ of an obstacle to achieve 
joint satisfaction of conditions (a), (b), (c), and (d), and Lemma \ref{considerations}
shows} that the fraction of samples which lie within $r_n$ of an
  obstacle goes to zero as $n \rightarrow \infty$. Furthermore,
  Theorem \ref{thrm:AO} shows that such suboptimal connections do not affect the
  AO of \FMT.
\begin{figure}
  \centering
  \includegraphics[width=140mm]{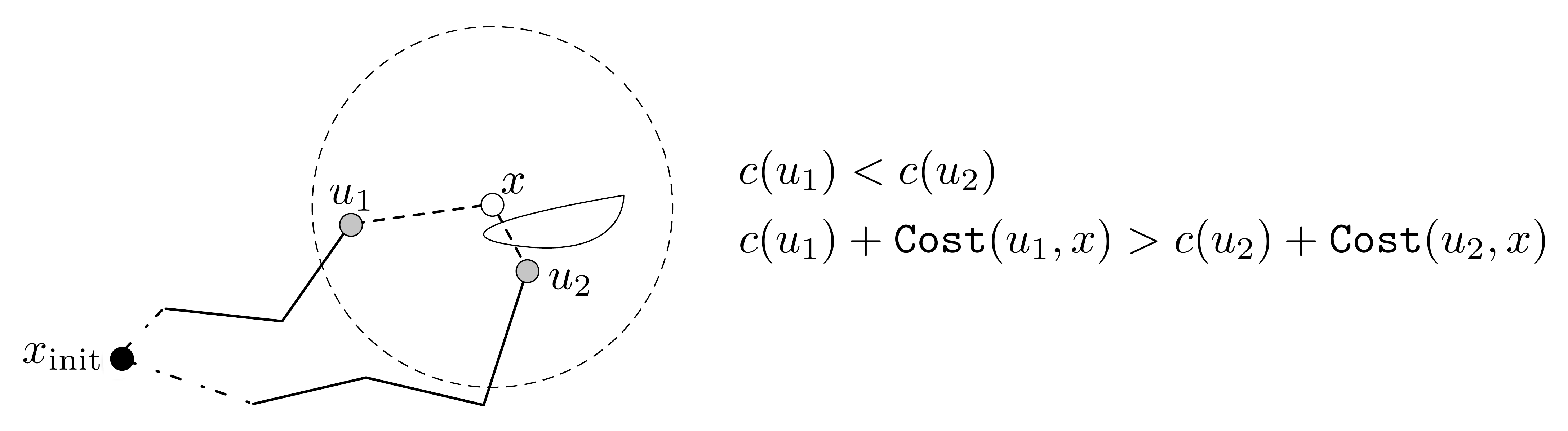}
  \caption{Illustration of a case where \FMT would make a suboptimal connection. 
  \FMT is designed so that suboptimal connections are ``rare" {\color{black}in general,
  and} vanishingly rare as $n \rightarrow \infty$.}
  \label{fig:fail_condition}
\end{figure}

\subsection{Conceptual Comparison with Existing AO Algorithms and Advantages of \FMT}\label{subsec:comp}

When there are no obstacles, \FMT reports the exact same solution or failure
as \PRMstar. This property follows from the fact that, without obstacles, \FMT is indeed
using dynamic programming to build the minimum-cost spanning tree,
as shown in Theorem \ref{thrm:sp}. With
  obstacles, for a given sample set, \FMT finds a path
  {\color{black}with a cost that is lower-bounded by, and does not
  substantially exceed,} the cost of the path found by \PRMstar\!, due
  to the suboptimal connections made by lazily ignoring  obstacles in
  the dynamic programming recursion. However, as will be shown in
  Theorem \ref{thrm:AO}, the cases where \FMT makes a suboptimal
  connection are rare enough that as $n \to \infty$, \FMT\!, like \PRMstar\!, converges to an optimal solution. While lazy collision-checking might introduce suboptimal connections,  it leads to a key computational advantage. By only checking for
collision on the locally-optimal (assuming no obstacles) one-step connection,
as opposed to every possible connection {\color{black}as} is done in
\PRMstar\!, \FMT saves a large number of costly collision-check
computations{\color{black}. I}ndeed, the ratio of the number of collision-check
computations in \FMT to those in \PRMstar goes to zero as the number
of samples goes to infinity.
Hence, we expect \FMT to outperform \PRMstar
  in terms of solution cost as a function of time.

A conceptual comparison {\color{black}to} \RRTstar is more difficult, given how
differently \RRTstar generates paths as compared {\color{black}with} \FMT\!.
The graph expansion
procedure of \RRTstar is fundamentally different from that of \FMT\!.
While \FMT samples points throughout the free space and makes connections
independently of the order in which the samples are drawn, at each iteration \RRTstar steers
towards a new sample only from the regions it has reached up until that time.
In problems where the solution path is necessarily long and winding it may take a long time for an ordered
set of points traversing the path to present steering targets for \RRTstar\!. In this case,
a lot of time can be wasted by steering in inaccessible directions before a feasible solution is found.
Additionally, even once the search trees for both algorithms have explored the whole space, one may expect \FMT
to show some improvement in solution quality per number of samples
placed.  This improvement comes from the fact that, for a given set of samples, \FMT creates
connections nearly optimally (exactly optimally when there are no
obstacles) within the radius constraint, while \RRTstar\!, even with its
rewiring step, is {\color{black}fundamentally} a greedy algorithm.
It is, however, hard
to conceptually assess  how long the algorithms might take to run on a
given set of samples, although in terms of
  collision-check computations, we will show in Lemma
  \ref{considerations} that \FMT performs $O(1)$ collision-checks per
  sample, while \RRTstar performs $O(\log(n))$ per sample. 
In Section \ref{advantages} we will present
results from numerical experiments to make these conceptual comparisons concrete
and assess the benefits of
\FMT over \RRTstar.

An effective approach to address the greedy behavior of \RRTstar is to
leverage \emph{relaxation methods} for the exploitation of new
connections  \citep{OA.PT:13}. This approach is the main
idea behind the \RRTsharp algorithm \citep{OA.PT:13}, which constructs
a spanning tree rooted at the initial condition and  containing
lowest-cost path information for nodes which have the potential to be
part of a shortest path to the goal region. This approach is also very
similar to what is done by \FMT\!. However, \RRTsharp grows {\color{black}the} tree
in a fundamentally different way, by interleaving the addition of new
nodes and corresponding edges to the graph with a Gauss--Seidel
relaxation of the Bellman equation \eqref{eq:Bellman}{\color{black}; it is} essentially the
same relaxation used in the $\mathrm{LPA}^{*}$ algorithm
\citep{SK.ML.ea:04}. This last step propagates the new information
gained with a node addition across  the \emph{whole} graph in order to
improve the cost-to-{\color{black}arrive} values  of ``promising" nodes
\citep{OA.PT:13}. In contrast, \FMT directly implements the Bellman
equation \eqref{eq:Bellman} and, whenever a new node is added to the
tree, considers only \emph{local}, {\color{black}i.e.} within a neighborhood,
connections. Furthermore, and perhaps most importantly, \FMT
implements a lazy collision-checking strategy, which on the practical
side may significantly reduce the number of costly collision-checks,
while on the theoretical side requires a careful analysis of possible
suboptimal local connections (see Section \ref{subsec:basic} and
Theorem \ref{thrm:AO}). It is also worth mentioning that over $n$
samples \FMT has a computational complexity that is $O(n \, \log n)$
(Theorem \ref{thrm:CC}), while \RRTsharp has a computational
complexity of $O(n^2 \, \log n)$ \citep{OA.PT:13}.

Besides providing fast convergence
to high quality solutions, \FMT has some ``structural" advantages with
respect to its state-of-the-art counterparts. First, \FMT\!{\color{black}, like}
\PRMstar\!{\color{black},} relies on the choice of two parameters, namely the number
of samples and the constant appearing in the connection radius in 
equation \eqref{radiusprt}. In contrast, \RRTstar requires the 
choice of four parameters, namely, {\color{black}the} number of samples or termination
time, {\color{black}the} steering radius, {\color{black}the} goal biasing, and {\color{black} the constant
  appearing in the} connection radius. An advantage of \FMT over \PRMstar\!, besides the
reduction in the number of collision-checks (see Section
\ref{subsec:high_lev}), is that \FMT builds and maintains paths in a tree
structure at \emph{all times}, which is advantageous 
when differential constraints are added to the paths. In particular, far fewer two-point boundary value problems need to be solved (see the recent work in \citep{ES-LJ-MP:14}). Also, the fact that
the tree grows in cost-to-{\color{black}arrive} space simplifies a bidirectional
implementation, as discussed in \citep{JS-ES-LJ-MP:14}. Finally, while
\FMT\!,  by running on a \emph{predetermined} number of samples, is \emph{not} an anytime algorithm (roughly speaking, an
algorithm is called anytime if, given extra time, it continues to run
and further improve its solution until time runs out---a \edit{notable}
example is \RRTstar\!), it can be cast into this framework by
repeatedly adding batches of samples and carefully reusing previous
computation until time runs out, as recently presented in \citep{OS-DH:14}.

\begin{algorithm}
\caption{Fast Marching Tree Algorithm (\FMT): Details}
\label{prtalg}
\algsetup{linenodelimiter=}
\begin{algorithmic}[1]
\STATE $V \leftarrow \{x_{\text{init}}\} \cup \texttt{SampleFree}(n)$; $E \leftarrow \emptyset$
\STATE $\Wset \leftarrow V \backslash \{x_{\text{init}}\}$; $\Hset \leftarrow \{x_{\text{init}}\}$, $\Vc \leftarrow \emptyset$
\STATE $z \leftarrow x_{\text{init}}$
\STATE $N_z \leftarrow \texttt{Near}(V \backslash \{z\}, z, r_n)$ \label{line:Nz}
\STATE $\texttt{Save}(N_z, z)$ \label{line"save_1}
\WHILE{$z \notin \mathcal{X}_{\text{goal}}$} \label{stoppingcond}
\STATE $\Hsetnew \leftarrow \emptyset$
\STATE $X_{\text{near}} = N_z{\color{black}\cap} \Wset$ \label{line:intersectW}
\FOR{$x \in X_{\text{near}}$} \label{line:forXnear}
\STATE $N_x \leftarrow \texttt{Near}(V \backslash \{x\}, x, r_n)$\label{save_2_0}
\STATE $ \texttt{Save}(N_x, x)$ \label{line"save_2}
\STATE $Y_{\text{near}} \leftarrow N_x {\color{black}\cap} \Hset)$ \label{line:intersect}
\STATE $y_{\text{min}} \leftarrow \arg\min_{y \in
  Y_{\text{near}}} \, {\color{black}\{}c(y) +
\texttt{Cost}(y,x){\color{black}\}}$ \COMMENT{\, // {\small dynamic programming equation}} \label{line:ymin}%\{\texttt{Cost}(y, T = (V, E)) \!+\!\texttt{Cost}(\overline{yx})\}$ 
\IF{$\texttt{CollisionFree}(y_{\text{min}}, x)$}  \label{line:collisionfreecheck}
\STATE $E \leftarrow E \cup \{(y_{\text{min}}, x)\}$  \COMMENT{\, // {\small straight line joining $y_{\text{min}}$ and $x$ is collision-free}}
\STATE $\Hsetnew \leftarrow \Hsetnew \cup \{x\}$ \label{alg:H_1}
\STATE $\Wset \leftarrow \Wset \backslash \{x\}$
\STATE $c(x) = c(y_{\text{min}}) + \texttt{Cost}(y_{\text{min}},x)$
\COMMENT{// {\small cost-to-{\color{black}arrive} from $\xinit$ in tree $T = (\Hset \cup \Vc, E)$}}
\ENDIF
\ENDFOR
\STATE $\Hset \leftarrow (\Hset \cup \Hsetnew) \backslash \{z\}$ \label{alg:H_2}
\STATE $\Vc \leftarrow \Vc \cup \{z\}$
\IF{$\Hset = \emptyset$}
\RETURN Failure
\ENDIF
\STATE $z \leftarrow \arg\min_{y \in \Hset} \, {\color{black}\{}c(y){\color{black}\}}$%\{\texttt{Cost}(y, T = (V, E))\}$
\ENDWHILE
\RETURN $\texttt{Path}(z, T = (\Hset\cup \Vc, E))$
\end{algorithmic}
\end{algorithm}

\subsection{Detailed Description and Implementation Details}\label{subsec:detail_des}

This section provides a detailed pseudocode description of Algorithm \ref{pseudofmt}, which highlights a number of implementation details that will be instrumental to the computational complexity analysis given in Section  \ref{subsec:complexity}.

Let $\texttt{SampleFree}(n)$ be a function that returns a set of $n
\in \mathbb{N}$ points (samples) sampled independently and identically from the
uniform distribution on $\mathcal{X}_{\text{free}}$. We
  discuss the extension to non-uniform sampling distributions in
  Section \ref{subsec:nonunif}. Let $V$ be a set of
  samples containing the initial state $\xinit$ and a set of $n$
  points sampled according to $\texttt{SampleFree}(n)$.
  Given a subset $V^{\prime}\subseteq V$, and a sample
  $v\in V$, let $\texttt{Save}(V^{\prime}, v)$ be a
function that stores in memory a set of samples
  $V^{\prime}$ associated with sample $v$.  Given a set
of samples $V$, a sample $v\in V$, and a
positive number $r$, let $\texttt{Near}(V, v, r)$ be a
function that returns the set of samples 
  $\{u \in V: \|u-v\| < r\}$.  {\color{black}\texttt{Near} checks first
    to see if the required set of samples has already been computed
    and saved using $\texttt{Save}$, in which case it loads the set
    from memory, otherwise it computes the required set from scratch.} Paralleling the notation
  in the proof of Theorem \ref{thrm:sp}, given a tree $T =
(V^{\prime}, E)$, where the node set $V^{\prime}\subseteq V$ contains
$\xinit$ and $E$ is the edge set, and a node $v \in V'$,
  let $c(v)$ be the cost of the unique path in the graph $T$ from
  $x_{\text{init}}$ to $v$. Given two samples $u,v\in V$, let
$\texttt{Cost}(u,v)$ be the cost of the \emph{straight line} joining
$u$ and $v$ (in the current setup $\texttt{Cost}(u,v) = \| v-u\|$,
more general costs will be discussed in Section \ref{subsec:cost}). Note that $\texttt{Cost}(u,v)$ is well{\color{black}-}defined regardless of the line joining $u$ and $v$ being collision-free. Given two samples  $u, v\in V$, let $\texttt{CollisionFree}(u, v)$ denote the boolean function which is true if and only if the line joining $u$ and $v$ does not intersect an obstacle.
Given a tree $T = (V^{\prime}, E)$, where the node set
  $V^{\prime}\subseteq V$ contains $\xinit$ and $E$ is the edge set,
and a node $v\in V^{\prime}$, let
$\texttt{Path}(v, T)$ be the function returning the unique path in the
tree $T$ from $x_{\text{init}}$ to $v$. The detailed \edit{\FMT} algorithm is given in Algorithm \ref{prtalg}.

%Cost can be defined as any non-negative function on piecewise linear paths which is additive with respect to concatenation of paths.

%\subsection{Implementation Details}\label{subsec:imp_det}
The set $\Hset$ should be implemented as a binary min heap, ordered by
cost-to-{\color{black}arrive}, with a parallel set of nodes that exactly tracks the
nodes in $\Hset$ in no particular order{\color{black},} and that is used
  to efficiently carry out the intersection operation in line
\ref{line:intersect} of the algorithm. Set $\Hsetnew$  contains successfully connected $x$ samples that will be added to $\Hset$ once all $x$ samples have been considered (compare with line \ref{line:ins3} in Algorithm \ref{pseudofmt}). At initialization
  (line \ref{line"save_1}) and during the main while loop (line
  \ref{line"save_2}), \FMT saves the information regarding the nearest
  neighbor set of a node $v$, that is $N_v$. This operation is needed
  to avoid unnecessary repeated {\color{black}computations of near
    neighbors by allowing the \texttt{Near} function to load from memory,}
  and will be important  for the characterization of the computational
  complexity of \FMT in Theorem \ref{thrm:CC}{\color{black}. S}ubstituting lines
  \ref{save_2_0}--\ref{line:intersect} with the line $Y_{\mathrm{near}}
  \leftarrow\texttt{Near}(\Hset, x, r_n)$, while algorithmically
  correct, would cause a larger number of {\color{black}unnecessary near neighbor computations}. Additionally,  for each node $u \in N_v$, one should also save the real value $\texttt{Cost}(u,v)$ and the boolean value $\texttt{CollisionFree}(u, v)$.  Saving both of these values whenever they are first computed guarantees that \FMT will never compute them more than once for a given pair of nodes.

%\subsection{Termination}\label{subsec:term}
%One might wonder if, in the first place, \FMT is guaranteed to terminate. The following theorem shows that \FMT always terminates, i.e., it does not cycle indefinitely through the sets $\Hset$ and $W$.
%\begin{theorem}[Termination]
%\label{termination}
%Consider a path planning problem $(\mathcal{X}_{\text{free}}, x_{\text{init}}, \mathcal{X}_{\text{goal}})$ and any $n \in \naturals$. The \FMT algorithm always terminates  in at most $n$ iterations of the while loop.
%\end{theorem} 
%\begin{proof}
%To prove that \FMT always terminates, note two key facts: (i) \FMT terminates and reports failure if $\Hset$ is ever empty, and (ii) the minimum-cost node in $\Hset$ is removed from $\Hset$ at each while loop iteration. 
%% note first that all for loops are evaluated over a set which remains static for the duration of the for loop, and thus \FMT can never get stuck in a for loop, but only in the while loop.  
%Therefore, to prove the theorem it suffices to prove the invariant that any node that has ever been added to $\Hset$ can never be added again (this, in fact, would imply that the while loop goes through at most $n$ iterations).  To establish the invariant, observe that at a given iteration only nodes in $X_{\text{near}}$ are added to $\Hset$.  However, $X_{\text{near}} \subseteq W$, so only nodes in $W$ can be added to $\Hset$, and each time a node is added, it is removed from $W$. Finally, since $W$ never has nodes added to it, a node can only be added to $\Hset$ once. This proves the invariant, and, in turn, the claim.
%\end{proof}

\section{Analysis of \FMT}\label{sec:AO}
In this section we characterize the asymptotic optimality of \FMT
(Section \ref{subsec:AO}), provide a convergence rate to the optimal
solution (Section \ref{subsec:rate_bound}), and finally
  characterize its computational complexity (Section \ref{subsec:complexity}).
\subsection{Asymptotic Optimality}\label{subsec:AO}
%The second difference is, not surprisingly, that \FMT does not in general return the same solution as \PRMstar.  These changes mainly manifest themselves in how and when nearby points are connected, and as a result, we prove here AO when the cost is Euclidean path-length, and defer the more general cost case to future work.

The following theorem presents the main result of this paper.
\begin{theorem}[Asymptotic optimality of \FMT]
\label{thrm:AO}
Let $(\mathcal{X}_{\text{free}}, x_{\text{init}},
\mathcal{X}_{\text{goal}})$ be a $\delta$-robustly feasible path
planning problem in $d$ dimensions, with $\delta>0$ and
$\mathcal{X}_{\text{goal}}$ being $\xi$-regular. Let $c^*$ denote the arc
length of an optimal path $\sigma^*$,
and let $c_n$ denote the {\color{black} arc length} of the path returned by \FMT (or
$\infty$ if \FMT returns failure) with $n$ samples using the
following radius, 
\begin{equation}
\label{radiusprt}
r_n = (1+\eta) \,2 \, \biggl(\frac{1}{d}\biggr)^{1/d} \biggl(\frac{\mu(\mathcal{X}_{\text{free}})}{\zeta_d} \biggr)^{1/d} \biggl(\frac{\log(n)}{n}\biggr)^{1/d},
\end{equation}
for some $\eta > 0$.  Then  $\lim_{n \rightarrow \infty} \, \p{c_n > (1+\varepsilon)c^*} = 0 \,\,$ for all $ \varepsilon > 0$.
\end{theorem}

\begin{proof}
Note that $c^* = 0$ implies $x_{\text{init}} \in
\mathrm{cl}(\mathcal{X}_{\text{goal}})$, and the result is trivial, therefore
assume $c^* > 0$. Fix $\theta \in (0,1/4)$ and define the sequence of
paths $\sigma_n$ such that $\lim_{n \rightarrow \infty} c(\sigma_n) =
c^*$, $\sigma_n(1) \in \partial \mathcal{X}_{\text{goal}}$,
$\sigma_n(\tau) \notin \mathcal{X}_{\text{goal}}$ for all $\tau \in
(0,1)$, $\sigma_n(0) = \xinit$, and $\sigma_n$ has strong
$\delta_n$-clearance, where $\delta_n = \min\bigl\{\delta,
\frac{3+\theta}{2+\theta}r_n\bigr\}$. Such a sequence of paths must
exist by the $\delta$-robust feasibility of the path planning
problem. {\color{black} The parameter} $\theta$ will be used to
construct balls that cover a path of interest, and {\color{black}in
  particular will be the ratio of the separation of the ball centers
  to their radii (see Figure~\ref{fig:param} for an illustration).}

{\color{black}The path $\sigma_n$ ends at
$\partial\mathcal{X}_{\text{goal}}$; we will define $\sigma_n'$
as $\sigma_n$ with a short extension into 
the interior of $\mathcal{X}_{\text{goal}}$.} Specifically, $\sigma_n'$
is $\sigma_n$ concatenated with the line of length $\min\bigl
\{\xi, \frac{r_n}{2(2+\theta)}\bigr \}$ that
extends from $\sigma_n(1)$ into $\mathcal{X}_{\text{goal}}$, exactly perpendicular
to the tangent hyperplane of $\partial \mathcal{X}_{\text{goal}}$ at $\sigma_n(1)$. Note that this tangent
hyperplane is well-defined, since the regularity assumption for
$\mathcal{X}_{\text{goal}}$ ensures that its boundary is
differentiable.  Note that, trivially, $\lim_{n\rightarrow
  \infty}c(\sigma_n') = \lim_{n\rightarrow \infty}c(\sigma_n) = c^*$. % + \min\{\xi, \frac{r_n}{2(2+\theta)}\} = c^*$. 
 This line extension is needed because a path that only
  reaches the boundary of the goal region can be arbitrarily
  well-approximated in bounded variation norm by paths that are not
  actually feasible because they {\color{black}do not} reach the goal region, and we
  need to ensure that \FMT finds \emph{feasible} solution
  paths that approximate an optimal path.

Fix $\varepsilon \in (0,1)$, suppose $\alpha, \beta \in (0,\theta\varepsilon/8)$, and pick $n_0 \in \mathbb{N}$ such that for all $n \ge n_0$ the following conditions hold: (1)  $\frac{r_n}{2(2+\theta)} < \xi$, (2)
 $\frac{3+\theta}{2+\theta}r_n < \delta$, (3)
 $c(\sigma_n') < (1+\frac{\varepsilon}{4})\,c^*$, and (4)
 $\frac{r_n}{2+\theta} < \frac{\varepsilon}{8}\,c^*$. Both
   $\alpha$ and $\beta$ are parameters for controlling the
   {\color{black} smoothness}
   of \FMT\!'s solution, and will be used in the proofs of Lemmas \ref{lemma:claim_1}
   and \ref{lemma:claim_2}.

For the remainder of this proof, assume $n \ge n_0$.  From conditions
(1) and (2), $\sigma_n'$ has strong
$\frac{3+\theta}{2+\theta}r_n$-clearance.  For notational simplicity,
let $\kappa(\alpha, \beta, \theta) := 1+(2\alpha + 2\beta)/ \theta$,
in which case conditions (3) and (4) imply,
\begin{equation*}
\begin{split}
\kappa (\alpha,\beta,\theta) \, c(\sigma_n') + \frac{r_n}{2+\theta}&  \le \kappa (\alpha,\beta,\theta) \biggl (1+\frac{\varepsilon}{4} \biggr) \, c^* + \frac{\varepsilon}{8} \, c^* \\
& \le \Biggl ( \biggl (1+\frac{\varepsilon}{2}\biggr)\biggl(1+\frac{\varepsilon}{4} \biggr ) + \frac{\varepsilon}{8}\Biggl )c^*  \le (1+\varepsilon)c^*.
\end{split}
\end{equation*}
Therefore, 
\begin{equation}
\label{opt1}
\begin{split}
\p{c_n > (1+\varepsilon)c^*} & = 1 - \p{c_n \le (1+\varepsilon)c^*}  \leq 1 - \p{c_n \le \kappa(\alpha,\beta,\theta) \, c(\sigma_n') + \frac{r_n}{2+\theta}}.
\end{split}
\end{equation}

Define the sequence of balls $B_{n,1}, \dots, B_{n,M_n} \subseteq \mathcal{X}_{\text{free}}$ parameterized by $\theta$ as follows. For $m=1$ we define
$
 B_{n,1} := B\biggl(\sigma_n(\tau_{n,1});\,  \frac{r_n}{2+\theta}\biggr), \quad \text{with }  \tau_{n,1} = 0
 $.
For $m = 2, 3, \ldots$, let
$$
\Gamma_m = \biggl\{\tau \in (\tau_{n,m-1}, 1): \|\sigma_n(\tau) - \sigma_n(\tau_{n,m-1})\| = \frac{\theta r_n}{2+\theta}\biggr\};
$$ if $\Gamma_m \neq \emptyset$ we define
$ B_{n,m} := B\biggl (\sigma_n(\tau_{n,m}); \frac{r_n}{2+\theta}\biggr ),  \quad \text{with }   \tau_{n,m} = \min_{\tau} \, \Gamma_m
$.
Let $M_n$ be the first $m$ such that $\Gamma_m = \emptyset$, then, $ B_{n,M_n} := B\biggl(\sigma_n'(1); \frac{r_n}{2(2+\theta)}\biggr)$, 
and we stop the process, i.e., $B_{n,M_n}$ is the last ball placed
along the path $\sigma_n$ (note that the center of the last ball is
$\sigma_n'(1)$). Considering the construction of $\sigma_n'$ and
condition (1) above, we conclude that $B_{n,M_n} \subseteq
\mathcal{X}_{\text{goal}}$. See Figure \ref{fig:param}
  for an illustration of this construction.

\begin{figure}
  \centering
  \includegraphics[width=120mm]{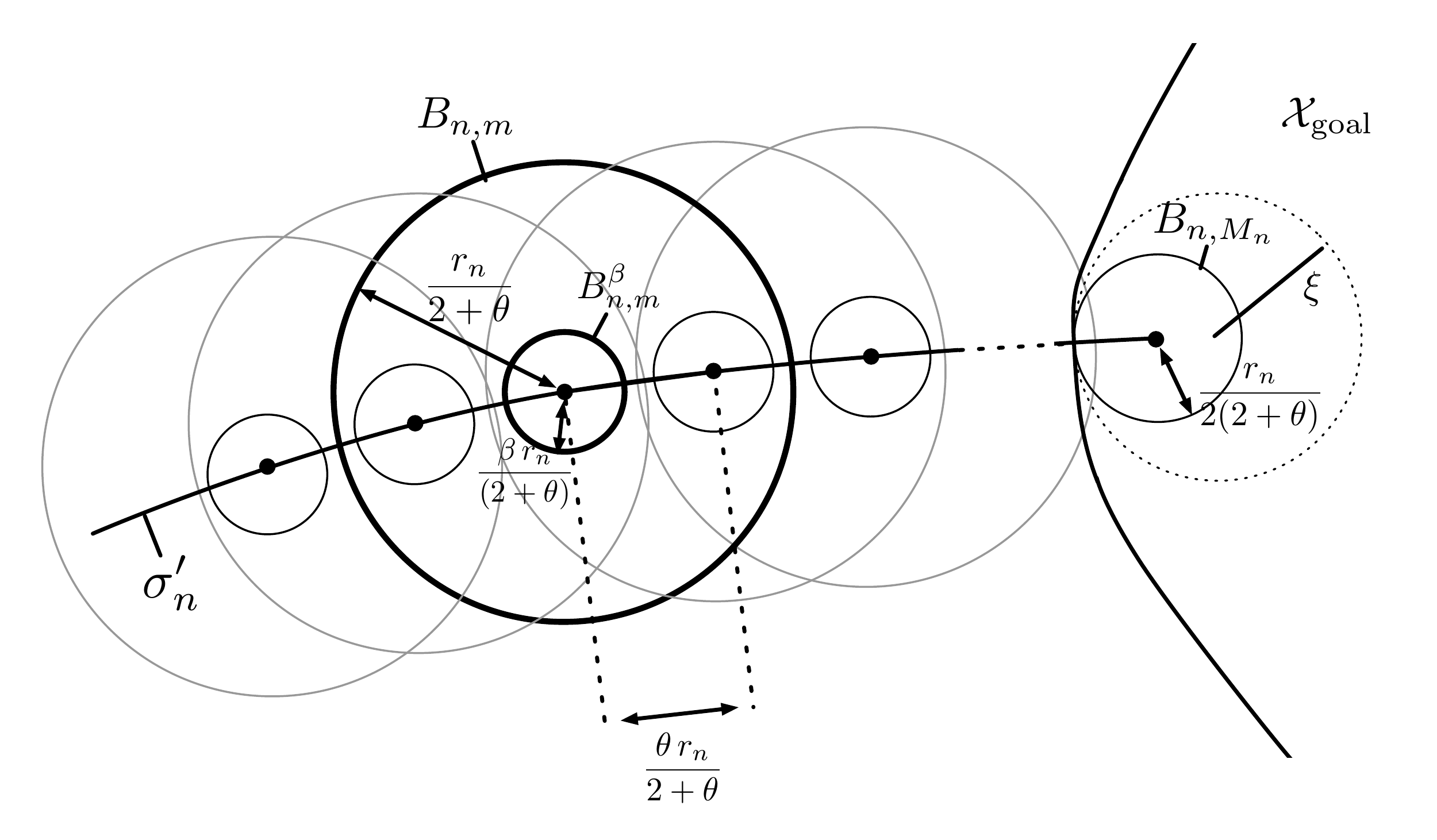}
  \caption{An illustration of the covering balls $B_{n,m}$ and
    associated smaller balls $B^{\beta}_{n,m}$. The figure also
    illustrates the role of $\xi$ in $\mathcal{X}_{\text{goal}}$ and
    the construction of $B_{n,M_n}$. Note that $\theta$ (the
    ratio of the separation of the centers of the $B_{n,m}$ to their
    radii) is depicted here as being around 2/3 for demonstration
    purposes only, as the proof requires $\theta < 1/4$.}
  \label{fig:param}
\end{figure}

Recall that $V$ is the set of samples available to algorithm \FMT (see line 1 in Algorithm \ref{prtalg}). We define the event
$
A_{n,\theta} := \bigcap_{m = 1}^{M_n} \{B_{n,m} \cap V \neq \emptyset\}$;
$A_{n,\theta}$ is the event that each ball contains at least one (not necessarily unique) sample in $V$. For clarity, we made the event's dependence on $\theta$, due to the dependence on $\theta$ of the balls, explicit.  Further, for all $m\in \{1, \dots, M_n-1\}$, let $B^{\beta}_{n,m}$ be the ball with the same center as $B_{n,m}$ and radius $\frac{\beta r_n}{2+\theta}$, where $0 \leq \beta \leq 1$, and let $K^{\beta}_n$ be the number of smaller balls $B^{\beta}_{n,m}$ not containing any of the samples in $V$, i.e.,
$K^{\beta}_n  := \card\{m \in \{1, \dots, M_n-1\}: B^{\beta}_{n,m}
\cap V = \emptyset\}$. We again point the reader to
  Figure \ref{fig:param} to see the $B^{\beta}_{n,m}$ depicted.

We now present three important lemmas{\color{black}; their} proofs can be found in
Appendix \ref{appA}.
\begin{lemma}[\FMT path quality]\label{lemma:claim_1} Under the assumptions of Theorem \ref{thrm:AO} and assuming $n \ge n_0$, the following inequality holds:
\[
\p{c_n \le \kappa(\alpha,\beta,\theta)\, c(\sigma_n') + \frac{r_n}{2+\theta}} \ge 1 \, - \, \mathbb{P}(K^{\beta}_n \ge \alpha (M_n-1)) - \mathbb{P}(A_{n,\theta}^c).
\]
\end{lemma}

\begin{lemma}[Tight approximation to most of the path]\label{lemma:claim_2} Under the assumptions of Theorem \ref{thrm:AO}, for all $\alpha\in (0,1)$ and $\beta \in (0,\theta/2)$, it holds that
\[
\lim_{n \rightarrow \infty} \mathbb{P}(K^{\beta}_n \ge \alpha (M_n-1)) = 0. 
\]
\end{lemma}

\begin{lemma}[Loose approximation to all of the path]\label{lemma:claim_3} Under the assumptions of Theorem \ref{thrm:AO}, 
assume that 
\[
r_n = \gamma \, \biggl (\frac{\log n}{n} \biggr)^{1/d},
\]
where 
\[\gamma = (1+\eta) \cdot 2 \,  \biggl(\frac{1}{d} \biggr)^{1/d}  \biggl(\frac{\mu (\mathcal{X}_{ \text{free}})}{\zeta_d}\biggr )^{1/d},
\]
and $\eta > 0$. Then for all $\theta < 2\eta$,
$ \lim_{n \rightarrow \infty} \mathbb{P}(A_{n,\theta}^c) = 0$. 
\end{lemma}

Essentially, Lemma \ref{lemma:claim_1} provides a lower bound for the
{\color{black} arc length} of the solution delivered by \FMT in terms of the probabilities that the ``big" balls and ``small" balls do not contain samples in $V$. Lemma \ref{lemma:claim_2} states that the probability that the fraction of small balls not containing samples in $V$ is larger than an $\alpha$ fraction of the total number of balls is asymptotically zero. Finally, Lemma \ref{lemma:claim_3} states that the probability that at least one ``big" ball does not contain any of the samples in $V$ is asymptotically zero.

The asymptotic optimality claim of the theorem then follows easily. Let $\varepsilon \in (0,1)$ and pick  $\theta \in (0,\min\{2\eta,1/4\})$ and $\alpha, \beta \in (0,\theta \varepsilon/8) \subset (0,\theta/2)$. From equation (\ref{opt1}) and Lemma \ref{lemma:claim_1}, {\color{black}we} can write
\[
\lim_{n \rightarrow \infty} \p{c_n > (1+\varepsilon)c^*} \leq \lim_{n \rightarrow \infty} \p{K^{\beta}_n \ge \alpha (M_n-1)} + \lim_{n \rightarrow \infty}\p{A_{n,\theta}^c}.
\]
The right-hand side of this equation equals zero by Lemmas \ref{lemma:claim_2} and \ref{lemma:claim_3}, and the claim is proven.  The case with general $\varepsilon$ follows by monotonicity in $\varepsilon$ of the above probability.
\end{proof}

\begin{remark}[Application of Theorem \ref{thrm:AO} to \PRMstar\!]
Since the solution returned by \FMT is never better than the one
returned by \PRMstar on a given set of nodes, the exact same result
holds for \PRMstar\!.  Note that this proof uses a $\gamma$ which is a
factor of $(d+1)^{1/d}$ smaller{\color{black},} and thus a $r_n$ which is
$(d+1)^{1/d}$ smaller{\color{black},} than that in
\cite{Karaman.Frazzoli:IJRR2011}. Since the number of cost
computations and collision-checks scales approximately as
$r_n^d$, this factor should reduce run time substantially for a given
number of nodes, especially in high dimensions.  This reduction is due to the difference in definitions of AO mentioned earlier which, again, makes no practical difference for \PRMstar or \FMT.
\end{remark}

\subsection{Convergence Rate}\label{subsec:rate_bound}
In this  section we provide a convergence rate bound for \FMT  (and
thus also for \PRMstar\!), \emph{assuming no obstacles}. As far as the
authors are aware, this bound is the first such
convergence rate result for an optimal sampling-based motion planning
algorithm and \edit{represents} an important step towards understanding the behavior
of this class of algorithms. The proof is deferred to Appendix \ref{appD}.

\begin{theorem}[Convergence rate of \FMT\!] 
\label{thrm:convrate}
Let the configuration space be $[0,1]^d$ with no obstacles and the
goal region be $[0,1]^d \cap B(\vec{1}; \xi)$, where $\vec{1} = (1,1,
\ldots, 1)$. Taking $x_{\text{init}}$ to be the center of the configuration space, the
shortest path has length $c^* = \sqrt{d}/2 - \xi$ and has clearance
$\delta = \xi \sqrt{(d-1)/d}$. Denote the {\color{black}arc length} of the path
returned by \FMT with $n$ samples {\color{black}as $c_n$}. For \FMT run using the radius given by
equation~\eqref{radiusprt}, namely,
\begin{equation*}
r_n = (1+\eta) \,2 \, \biggl(\frac{1}{d}\biggr)^{1/d} \biggl(\frac{\mu(\mathcal{X}_{\text{free}})}{\zeta_d} \biggr)^{1/d} \biggl(\frac{\log(n)}{n}\biggr)^{1/d},
\end{equation*}
for all $\varepsilon>0$, we have the following convergence rate bounds,
\begin{equation}\label{eq:rate}
\mathbb{P}(c_n > (1+\varepsilon)c^*) \in \left\{\begin{array}{lcl}
     O\left(\left(\log(n)\right)^{-\frac{1}{d}}n^{\frac{1}{d}\left(1-\left(1+\eta\right)^d\right)+\rho}\right)
    & \text{ if } & \eta \le \frac{2}{(2^d-1)^{1/d}} - 1,\\
    O\left(n^{-\frac{1}{d}\left(\frac{1+\eta}{2}\right)^d+\rho}\right)
    & \text{ if } & \eta > \frac{2}{(2^d-1)^{1/d}} - 1, \\
    \end{array} \right.
\end{equation}
as $n \rightarrow \infty$, where $\rho$ is an arbitrarily
small positive constant.

\end{theorem}

In agreement with the common opinion about sampling-based motion
planning algorithms, our convergence rate bound converges to zero
slowly, especially in high dimensions. Although the rate is slow, 
it scales as a power of $n$ rather than, say, logarithmically. We have not, however, studied
how tight the bound is---studying this rate is a potential
area for future work. As expected, the rate of convergence increases
as $\eta$ increases. However, increasing $\eta$ increases the amount
of computation per sample, hence, to optimize convergence rate
\emph{with respect to time} one needs to properly balance these two
competing effects. Note that  if {\color{black}we} select $\eta = 1$, from equation
\eqref{eq:rate} {\color{black}we} obtain a remarkably simple form for the rate,
namely $O(n^{-1/d+\rho})$, which holds for \PRMstar as well (we recall
that for a given number of samples the solution returned by \PRMstar
is not worse than the one returned by \FMT using the same connection radius). Note that the rate of
convergence to a \emph{feasible} (as opposed to optimal) solution for
PRM and RRT is known to be exponential
\citep{LK.MK.ea:96,LaValle.ea:IJRR01}; unsurprisingly, our bound for
converging to an \emph{optimal} solution decreases more
slowly{\color{black}, as it is not exponential}.

We emphasize that our bound does not have a constant
multiplying the rate that {\color{black}approaches infinity} as the arbitrarily small parameter (in
our case $\rho$) approaches zero. In fact, the asymptotic constant
multiplying the rate is 1, {\color{black}independent of} the value of
$\rho$, but the earliest $n$ at which that rate holds
approaches $\infty$ as $\rho \rightarrow 0$. Furthermore, although our
bound reflects the asymptotically dominant term (see
equation~\eqref{longrate} in the proof), there are two other terms which may
contribute substantially or even dominate for finite $n$.

It is also of interest to bound $\mathbb{P}(c_n > (1+\varepsilon)c^*)$ by an
  asymptotic expression in $\varepsilon$, but unfortunately this
  cannot be gleaned from our results, since the closed-form bound we
  use in the proof (see again equation~\eqref{longrate}) only holds
  for $n \ge n_0$, and $n_0 \stackrel{\varepsilon \rightarrow 0}{\longrightarrow}
  \infty$. Therefore fixing $n$ and sending $\varepsilon \rightarrow
  0$ just causes this bound to return 1 on a set $(0,
  \varepsilon_0(n))$, which tells us nothing about the rate
  at which the true probability approaches 1 as $\varepsilon \rightarrow 0$.

\subsection{Computational Complexity}\label{subsec:complexity}

The following theorem, {\color{black}proved} in Appendix \ref{appB},
characterizes the computational complexity of \FMT with respect to the
number of samples. It shows that  \FMT \edit{requires $O(n\log(n))$ operations} in
expectation, the same as \PRMstar and \RRTstar\!. It also highlights the computational savings of \FMT
over \PRMstar\!, since in expectation \FMT checks for edge collisions just
$O(n)$ times, while \PRMstar does so $O(n\log(n))$ times. Ultimately, the most relevant \edit{complexity measure} is how long it takes
for an algorithm to return a solution of a certain quality. \edit{This measure}, partially characterized in Section \ref{subsec:rate_bound}, will be studied numerically in Section \ref{sec:sims}.

\begin{theorem}[Computational complexity of \FMT]\label{thrm:CC}
Consider a path planning problem $(\mathcal{X}_{\text{free}},
x_{\text{init}}, \mathcal{X}_{\text{goal}})$ and a set of samples $V$ in
$\mathcal{X}_{\text{free}}$ of cardinality $n$, and fix 
$$r_n = \gamma \, 
\biggl(\frac{\log(n)}{n}\biggr)^{1/d},$$ 
for some positive constant
$\gamma$. In expectation, \FMT takes $O(n\log(n))$ time to compute a solution on
$n$ samples, and in doing so, makes $O(n)$ calls to
$\emph{\texttt{CollisionFree}}$ (again in expectation). \FMT also takes
$O(n\log(n))$ space in expectation.
\end{theorem}

\section{Extensions}\label{sec:extension}
This section presents three extensions to the setup
considered in the previous section, namely, (1) non-uniform sampling
strategies, (2) general cost functions instead of arc length, and
(3) a version of \FMT\!, named \kFMT\!, in which connections are sought to $k$ nearest-neighbor nodes, rather than to nodes within a given distance. 

For all three cases we discuss the changes needed to the baseline \FMT
algorithm presented in Algorithm \ref{prtalg} and then argue how \FMT
with these changes retains AO in Appendices
\ref{ao:nonunif}--\ref{ao:knn}{\color{black}. I}n the interest of brevity, we will
only discuss the required modifications to existing
  theorems, rather than proving
everything from scratch.

\subsection{Non-Uniform Sampling Strategies}
\label{subsec:nonunif}

\subsubsection{Overview}

Sampling nodes from a non-uniform
distribution can greatly help planning algorithms by incorporating outside
knowledge of the optimal path into the algorithm itself
\citep{Hsu.et.al:IJRR06}. (Of course if
no outside knowledge exists, the uniform distribution may be a natural
choice.) Specifically, we consider the setup whereby
$\texttt{SampleFree}(n)$ returns $n$ points sampled independently and
identically from a probability density function $\varphi$ supported
over $\xfree$. We assume that $\varphi$ is bounded below by a strictly positive
number $\ell$. This lower bound on $\varphi$ allows us to make a
connection between sampling from a non-uniform distribution and
sampling from a uniform distribution, for which the proof of AO
already exists (Theorem \ref{thrm:AO}). This argument is worked
through in Appendix \ref{ao:nonunif} to show that \FMT with
non-uniform sampling is AO.

\begin{comment}
\mpmargin{}{Lucas: add caption here}
\begin{figure}
  \centering
  \includegraphics[width=140mm]{density.pdf}
  \caption{Caption.}
  \label{fig:density}
\end{figure}
\end{comment}

\subsubsection{Changes to \FMT Implementation}

The only change that needs to be made to \edit{\FMT} is to multiply $r_n$ by $(1/\ell)^{1/d}$.
 
\subsection{General Costs}\label{subsec:cost}

Another extension of interest is when the cost function
  is not as simple as arc length. We may, for instance, want to
  consider some regions as more costly to move through than others, or
  a cost that weights/treats movement along different dimensions
  differently. In the following subsections, we explain how \FMT can be extended to
  other metric costs, as well as line integral costs, and why its AO still holds.

Broadly speaking, the main change that needs to happen to \FMT\!'s
implementation is that it needs to consider \emph{cost} instead of
Euclidean distance when searching for nearby points. For metric costs
besides Euclidean cost (Section \ref{subsubsec:metric}), a few
adjustments to the constants are all that is needed in order to ensure
AO. This is because the proof of AO in Theorem \ref{thrm:AO}
relies on the cost being additive and obeying the triangle
inequality. The same can be said for line integral costs \emph{if}
\FMT is changed to search along and connect points by cost-optimal
paths instead of straight lines (Section
\ref{subsubsec:internalSlow}). Since such an algorithm may be hard to
implement in practice, we lastly show in Section
\ref{subsubsec:internalFast} that by making some Lipschitz assumptions
on the cost function, {\color{black}we} get an approximate triangle inequality for
straight-line, cost-weighted connections. We
present an argument for why this approximation is sufficiently good to ensure that the
suboptimality introduced in how parent nodes are chosen and in the
edges themselves goes to zero asymptotically, and thus that AO is
retained. All arguments for AO in this subsection are deferred to
Appendix \ref{ao:cost}.

\subsubsection{Metric Costs}\label{subsubsec:metric}

\noindent{\bf Overview}:
Consider as cost function any metric on $\mathcal X$, denoted by
$\dist: \mathcal X \times \mathcal X \to \reals$. If the distance
between points in $\mathcal X$ is measured according to $\dist$, the
\FMT algorithm requires very minor modifications, namely just a
modified version of the $\texttt{Near}$ function. Generalized metric
costs allow one to account for, e.g., different weightings on different
dimensions, or an angular dimension which wraps around at $2\pi$. 

\vspace{1 mm}
\noindent {\bf Changes to \FMT\!'s implementation}: Given two samples
$u,v \in V$, $\texttt{Cost}(u, v) = \dist(u,v)$. Accordingly, given a set
of samples $V$, a sample $v \in V$, and a positive number $r$,
$\texttt{Near}(V, v, r)$ returns the set of samples $\{u \in V :
\texttt{Cost}(u,v) < r\}$. We refer to such sets as \emph{cost
  balls}. Formally, everything else in Algorithm \ref{prtalg} stays
the same, except $\zeta_d$ in the definition of $r_n$ needs to be defined as the
Lebesgue measure of the unit cost-ball. 

\subsubsection{Line Integral Costs with Optimal-Path Connections}\label{subsubsec:internalSlow}
\noindent{\bf Overview}: In some planning problems the cost function may not be a metric, i.e., it may not obey the triangle
inequality. Specifically,  consider the setup where $f:\mathcal{X}
\to \reals$ is such that $0 < f_{\text{lower}} \le f(x) \le
f_{\text{upper}} < \infty$ for all $x \in \mathcal{X}$, and the cost of a path $\sigma$ is given by
\[
\int_{\sigma} \, f(s)\, ds.
\]
Note that in this setup a straight line is not generally the
lowest-cost connection between two samples $u,v\in \mathcal X$. \FMT\!, however,
relies on straight lines in two ways: adjacent nodes in the \FMT tree are
connected with a straight line, and two samples are considered to be
within $r$ of one another if the straight line connecting them has
cost less than $r$. In this section we consider a version of \FMT whereby two adjacent nodes in the \FMT tree are
connected with the \emph{optimal} path between them, and two nodes are considered to
be within $r$ of one another if the \emph{optimal} path connecting
them has cost less than $r$. 

\vspace{1 mm}
\noindent{\bf Changes to \FMT\!'s implementation}: Given two nodes $u, v \in V$, 
\[
\texttt{Cost}(u, v) = \min_{\sigma^{\prime}} \, \int_{\sigma^{\prime}} \, f(s)\, ds,\]
where $\sigma^{\prime}$ denotes a path connecting $u$ and $v$. Given a
set of nodes $V$, a node $v \in V$, and a positive number $r$,
$\texttt{Near}(V, v, r)$ returns the set of nodes $\{u \in V :
\texttt{Cost}(u,v) < r\}$. Every time a node is added to a tree, its
cost-optimal connection to its parent is also stored. Lastly, the
definition of $r_n$ needs to be multiplied by a factor of $f_{\text{upper}}$.

\subsubsection{Line Integral Costs with Straight-Line Connections}\label{subsubsec:internalFast}

\noindent{\bf Overview}: Computing an optimal path for a line integral cost for every
connection, as considered in Section \ref{subsubsec:internalSlow}, may
represent an excessive computational bottleneck. Two strategies to
address this issue are (1) precompute such optimal paths {\color{black}since} their
computation does not require knowledge of the obstacle set{\color{black},} or (2)
approximate such paths with cost-weighted, straight line paths and
study the impact on AO. In this section we study the latter approach, and we argue that AO does indeed still hold, by appealing
to asymptotics to show that the triangle inequality approximately
holds, with this approximation going away as $n\rightarrow
\infty$. 

%In this section we add the assumption that the cost function $f$ is globally Lipschitz, i.e., there exists a positive constant $L>0$ such that for any $x,y\in \mathcal X$
%\[
%\|f(x) - f(y) \|\leq L\, \|x-y\|.
%\]

\vspace{1 mm}
\noindent{\bf Changes to \FMT\!'s implementation}: Given two samples $u, v \in V$, 
\[
\texttt{Cost}(u, v) = \int_{\overline{uv}} \, f(s)\, ds.
\]
Given a set of samples $V$, a sample $v \in V$, and a positive number
$r$, $\texttt{Near}(V, v, r)$ returns the set of samples $\{u \in V :
\texttt{Cost}(u,v) < r\}$.   Lastly, the definition of $r_n$ needs to again
be increased by a factor of $f_{\text{upper}}$.

\subsection{\FMT Using $k$-Nearest-Neighbors} \label{subsec:FMTkNN}
\subsubsection{Overview}

A last variant of interest is to have a version of \FMT which makes
connections based on $k$-nearest-neighbors instead of a fixed cost radius. \edit{This} variant{\color{black},} referred to as \kFMT\!{\color{black},} has the
advantage of being more adaptive to different obstacle spaces than its
cost-radius counterpart. This is because \FMT will consider about half as
many connections for a sample very near an obstacle
surface as for a sample far from obstacles, since about half the
measure of the obstacle-adjacent-sample's cost ball is inside the
obstacle. \kFMT\!, on the other hand, will consider $k$
connections for \emph{every} sample. To prove AO for \kFMT\! (in
Appendix \ref{ao:knn}), we will stray slightly
from our main proof exposition in this paper and use the
similarities between \FMT and \PRMstar to leverage a similar proof for
$k$-nearest \PRMstar from \citep{Karaman.Frazzoli:IJRR2011}. 

\subsubsection{Changes to \FMT\!'s Implementation}
Two parts need to change in Algorithm
\ref{prtalg}, both about how $\texttt{Near}$ works. The first is in lines \ref{line:Nz} and \ref{line:intersectW},
where $N_z$ should be all samples $v \in
V\setminus\{z\}$ such that \emph{both} $v$
is a $k_n$-nearest-neighbor of $z$ and $z$ is a $k_n$-nearest-neighbor
of $v$. We refer to this set as the \emph{mutual} $k_n$-nearest-neighbor set of
$z$. The second change is that in lines \ref{save_2_0} and \ref{line:intersect}, $N_x$ should be
the usual $k_n$-nearest-neighbor set of $x$, namely all samples $v \in
V\setminus\{x\}$ such that $v$ is a $k_n$-nearest-neighbor of $x$. Finally, $k_n$ should be chosen so that 
\begin{equation}\label{eqn:FMTkNNrbound}
k_n = k_0 \log(n), \qquad \text{where } k_0 > 3^d \, e\,(1+1/d).
\end{equation}
{\color{black} With these changes, \kFMT works by repeatedly applying
  Bellman's equation \eqref{eq:Bellman} over a $k$-nearest-neighbor
  graph, analogously to what  is done in  the disk-connected graph
  case (see Theorem \ref{thrm:sp}).} {\color{black}When we want to refer
  to the generic algorithm \kFMT using the specific sequence $k_n$,
  and we want to make this use explicit, we will say \knFMT\!.}

\section{Numerical Experiments and Discussion}\label{sec:sims}
In this section we numerically investigate the advantages of \FMT over previous AO sampling-based 
motion planning algorithms.  Specifically, we compare \FMT against \RRTstar and \PRMstar\!, as these two algorithms are state-of-the-art within the class of AO planners, span the main ideas (e.g., roadmaps versus trees) in the field of sampling-based planning, and have open-source, high quality implementations. We first
present  in Section \ref{subsec:simSet} a brief overview of the simulation setup. We then compare \FMT\!, \RRTstar\!, and \PRMstar in Section \ref{advantages}. Numerical experiments confirm our theoretical and heuristic arguments by showing that \FMT\!, for a given
execution time, returns substantially better solutions than  \RRTstar and
\PRMstar in a variety of problem settings. \FMT\!'s
main computational speedups come from performing fewer collision
checks---the more {\color{black}expensive} collision-checking is, the more \FMT will \edit{excel}.
Finally, in Section  \ref{subsec:FMTStudy}, we study in-depth \FMT and
its extensions (e.g., general costs). In particular, we provide
{\color{black}practical} guidelines about how to implement and tune \FMT\!. 

\begin{figure}[!t]
  \centering
    \subfigure[$\text{SE}(2)$ bug trap.]{\label{fig:bg2}
      \includegraphics[width=0.45\textwidth]{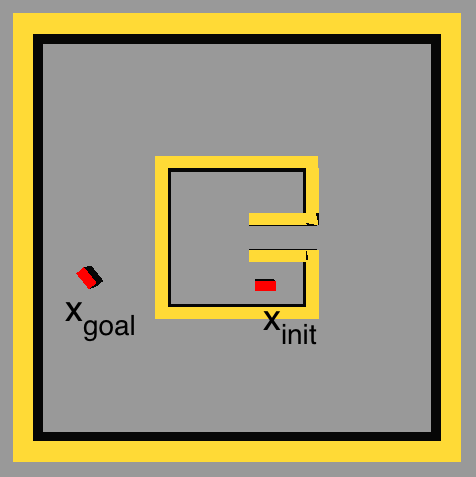}
    }
    \qquad
    \subfigure[ $\text{SE}(2)$ maze.]{\label{fig:s2m}
      \includegraphics[width=0.45\textwidth]{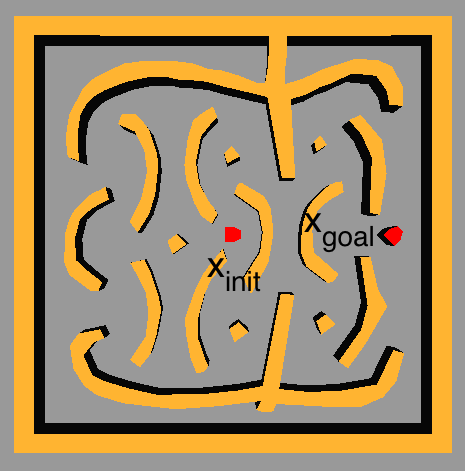}
    }
    \\
    \subfigure[ $\text{SE}(3)$ maze.]{\label{fig:s3m}
      \includegraphics[width=0.45\textwidth]{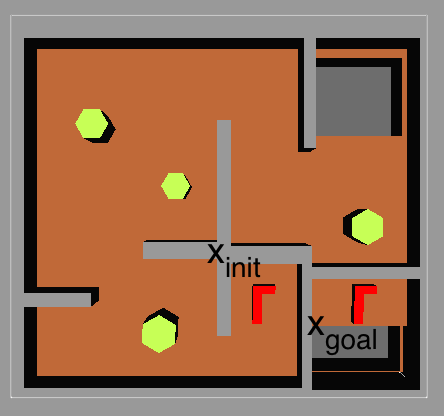}
    }
    \qquad
    \subfigure[$\text{SE}(3)$ Alpha puzzle.]{\label{fig:alphaP}
      \includegraphics[width=0.45\textwidth]{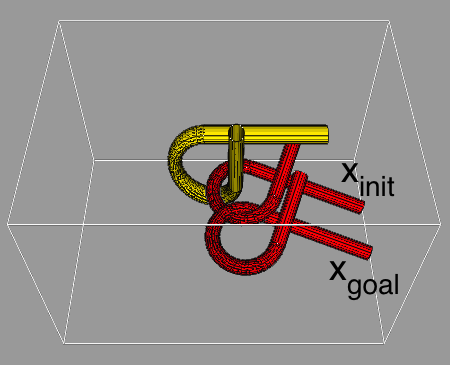}
    }
    \caption{Depictions of the OMPL.app $\text{SE}(2)$ and
            $\text{SE}(3)$ rigid body planning test problems. 
           }
\label{problempictures}
\end{figure}

\subsection{Simulation Setup}\label{subsec:simSet}
Simulations were written in a mix of C++ and Julia, and run using a Unix operating system with a
2.0 GHz processor and 8 GB of RAM. The C++ simulations were run through the Open Motion Planning Library (OMPL) \citep{Sucan.ea:RAM12}, from which the reference implementation of \RRTstar
was taken. We took the default values of \RRTstar parameters from OMPL (unless otherwise noted below), in particular a
steering parameter of 20\% of the maximum extent of the configuration
space, and a goal-bias probability of 5\%. Also, {\color{black} since
  the only OMPL implementation of \RRTstar is a}
$k$-nearest implementation, we adapted a $k$-nearest
version of \PRMstar and implemented a $k$-nearest version of \FMT\!,
both in OMPL{\color{black}; these are the versions used in
  Sections~\ref{subsec:simSet}--\ref{advantages}. In these
    two subsections, for notational simplicity, we will refer to the $k$-nearest versions of \FMT\!, \RRTstar\!, and \PRMstar simply as \FMT\!, \RRTstar\!, and \PRMstar\!, respectively.} The three algorithms were run on test problems drawn from the bank of standard
rigid body motion planning problems given in the OMPL.app graphical user interface. These problems, detailed below and depicted in Figure~\ref{problempictures}, are posed within the configuration spaces $\text{SE}(2)$ and $\text{SE}(3)$ which correspond to the kinematics (available translations and rotations) of a rigid body in 2D and 3D respectively. The dimension of the state space sampled by these planners is thus three in the case of $\text{SE}(2)$ problems, and six in the case of $\text{SE}(3)$ problems.

We chose the Julia programming language \citep{JB-SK-VS-AE:12} for the implementation of additional simulations because of its ease in accommodating the \FMT extensions studied in Section \ref{subsec:FMTStudy}.
We constructed experiments with a robot modeled as a union of hyperrectangles in high-dimensional Euclidean space moving amongst hyperrectangular obstacles. We note that for both simulation setups, \FMT\!, \RRTstar\!,
and \PRMstar used the \emph{exact same primitive routines}
(e.g., nearest-neighbor search, collision-checking, data handling,
etc.) to ensure a fair comparison. The choice of $k$ for the nearest-neighbor search phase of each of the planning algorithms is an important tuning parameter (discussed in detail for \FMT in Section \ref{subsubsec:rmtuning}). For the following simulations, unless otherwise noted, we used these coefficients for the nearest-neighbor count $k_n = k_0 \log(n)$: given a state space dimension $d$, for \RRTstar we used the OMPL default value $k_{0,\text{\RRTstar}} = e + e/d$, and for \FMT and \PRMstar we used the value $k_{0,\text{\FMT}} = k_{0,\text{\PRMstar}} = 2^d (e/d)$.
This latter coefficient differs from, and is indeed less than, the
lower bound in our mathematical guarantee of asymptotic optimality for
\kFMT\!, equation \eqref{eqn:FMTkNNrbound} (note that
$k_{0,\text{\RRTstar}}$ is also below the theoretical lower-bound presented
in \cite{Karaman.Frazzoli:IJRR2011}). We note, however, that for a fixed state space dimension $d$, the formula for $k_n$ differs only by a constant factor independent from the sample size $n$. 
Our choice of $k_{0,\text{\FMT}}$ in the experiments may be understood as a constant factor $e$ greater than the expected number of possible connections that would lie in an obstacle-free ball with radius specified by the lower bound in Theorem \ref{thrm:AO}, i.e., $\eta = e^{1/d} - 1 > 0$ in equation \eqref{radiusprt}. 
In practice we found that these coefficients for \RRTstar\!, \PRMstar\!, and \FMT worked well on the problem instances and sample size regimes of our experiments. Indeed, we note that the choice of $k_{0,\text{\RRTstar}}$\!, although taken directly from the OMPL reference implementation, stood up well against other values we tried when aiming to ensure a fair comparison. The implementation of \FMT and the
code used for algorithm comparison are available at: \url{http://www.stanford.edu/~pavone/code/fmt/}.

For each problem setup, we show a panel of six graphs. The first (top left) shows cost
versus time, with a point on the graph for each simulation run. These
simulations come in groups of 50, and within each group are run on the
same number of samples.
Note that this sample size is not necessarily the number of nodes in
the graph constructed by each algorithm; it indicates iteration count
in the case of \RRTstar\!, and free space sample count in the cases of
\FMT and \PRMstar\!. To be precise, \RRTstar
only keeps samples for which initial steering is collision-free. \PRMstar does use all of the sampled
points in constructing its roadmap, and while \FMT nominally constructs a tree as a subgraph of this roadmap, it may
terminate early if it finds a solution before all samples are considered.
There is also a line on the first plot tracing the
mean solution cost of \emph{successful} algorithm runs on a particular
sample count (1-standard-error of the mean error-bars are given
in both time and cost). Note that for a given algorithm, a group of simulations for a
given sample count is only plotted if it is at least 50\% successful
at finding a feasible solution. The plot below this one (middle left)
shows success rate as a function of time, with each point representing
a set of simulation{\color{black}s} grouped again by algorithm and node count. In this
plot, all sample counts are plotted for all algorithms, which is why the
curves may start farther to the left than those in the first plot. The
top right and middle right plots are the analogous plots to the first
two, but with sample count on the $x$-axis. Finally, the bottom left plot shows
execution time as a function of  sample count, and the bottom right plot shows
the number of collision-checks as a function of sample count. Note
that every plot shows vertical error bars, and horizontal error bars
where appropriate, {\color{black}of} length one standard-error of the mean,
although they are often too small to be distinguished from points.

\begin{figure}[!t]
  \centering
  \subfigure[]{
    \includegraphics[width=\figWidth\textwidth]{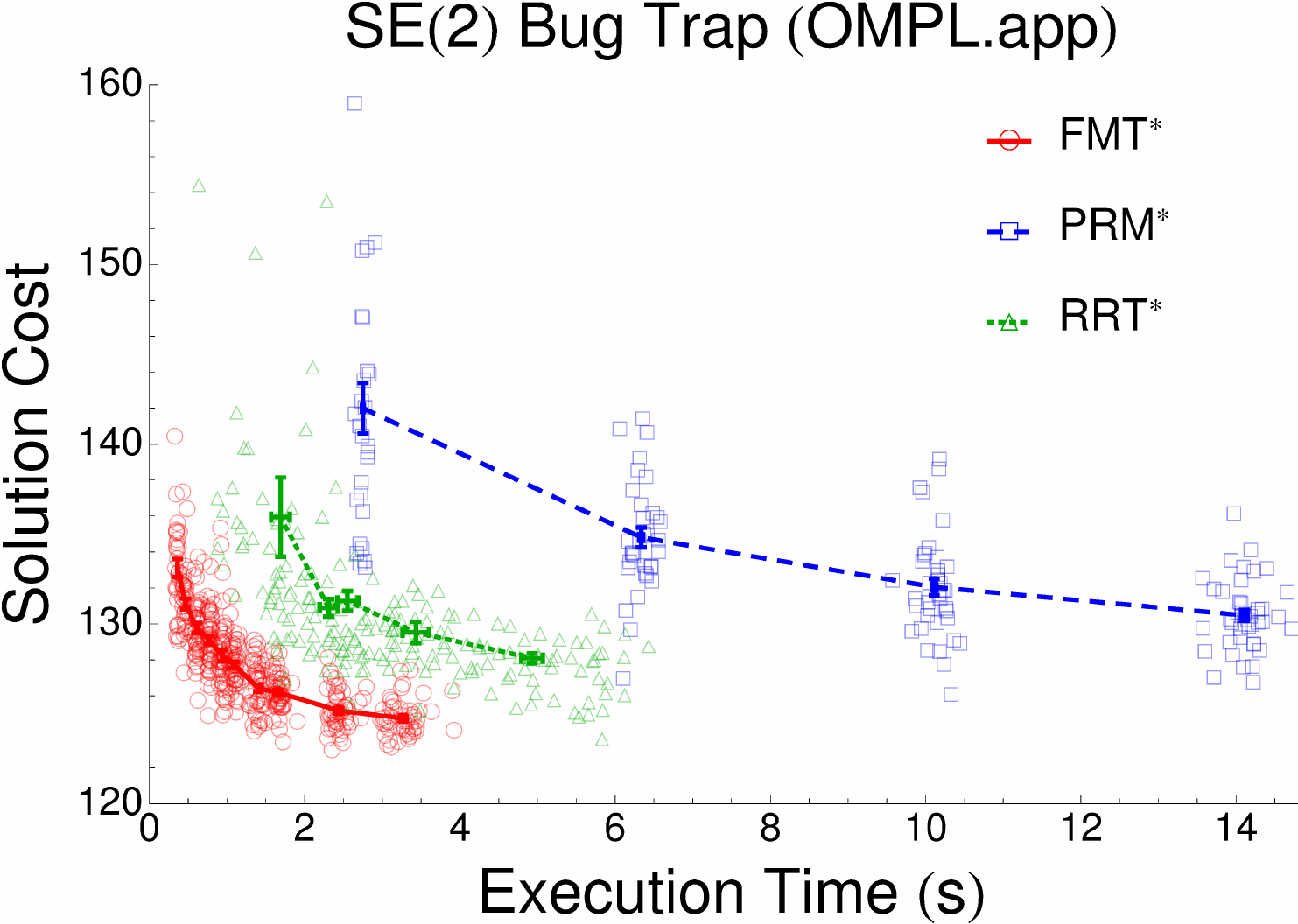}
  }
  \qquad \qquad
  \subfigure[]{
    \includegraphics[width=\figWidth\textwidth]{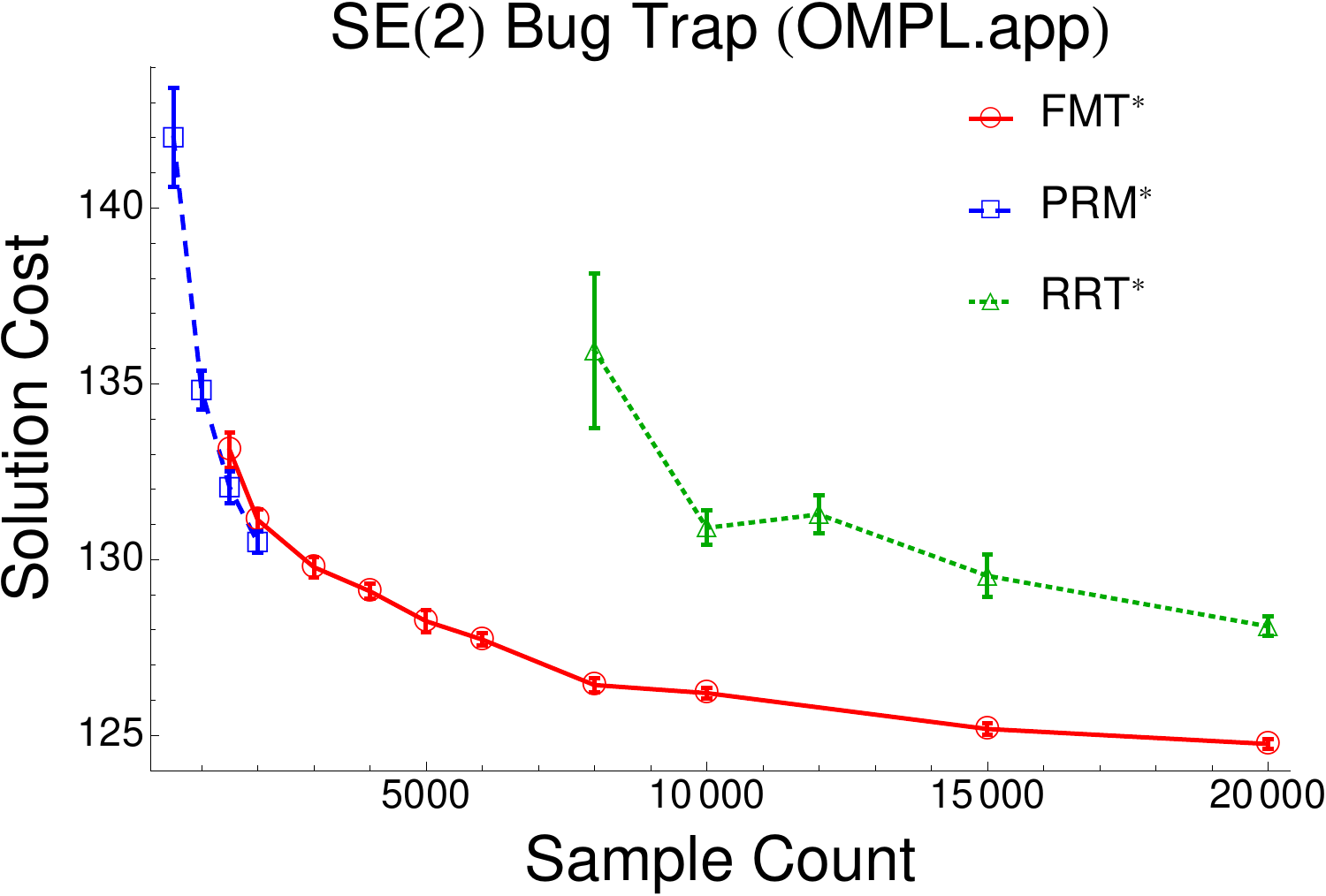}
  }
  \\
  \subfigure[]{
    \includegraphics[width=\figWidth\textwidth]{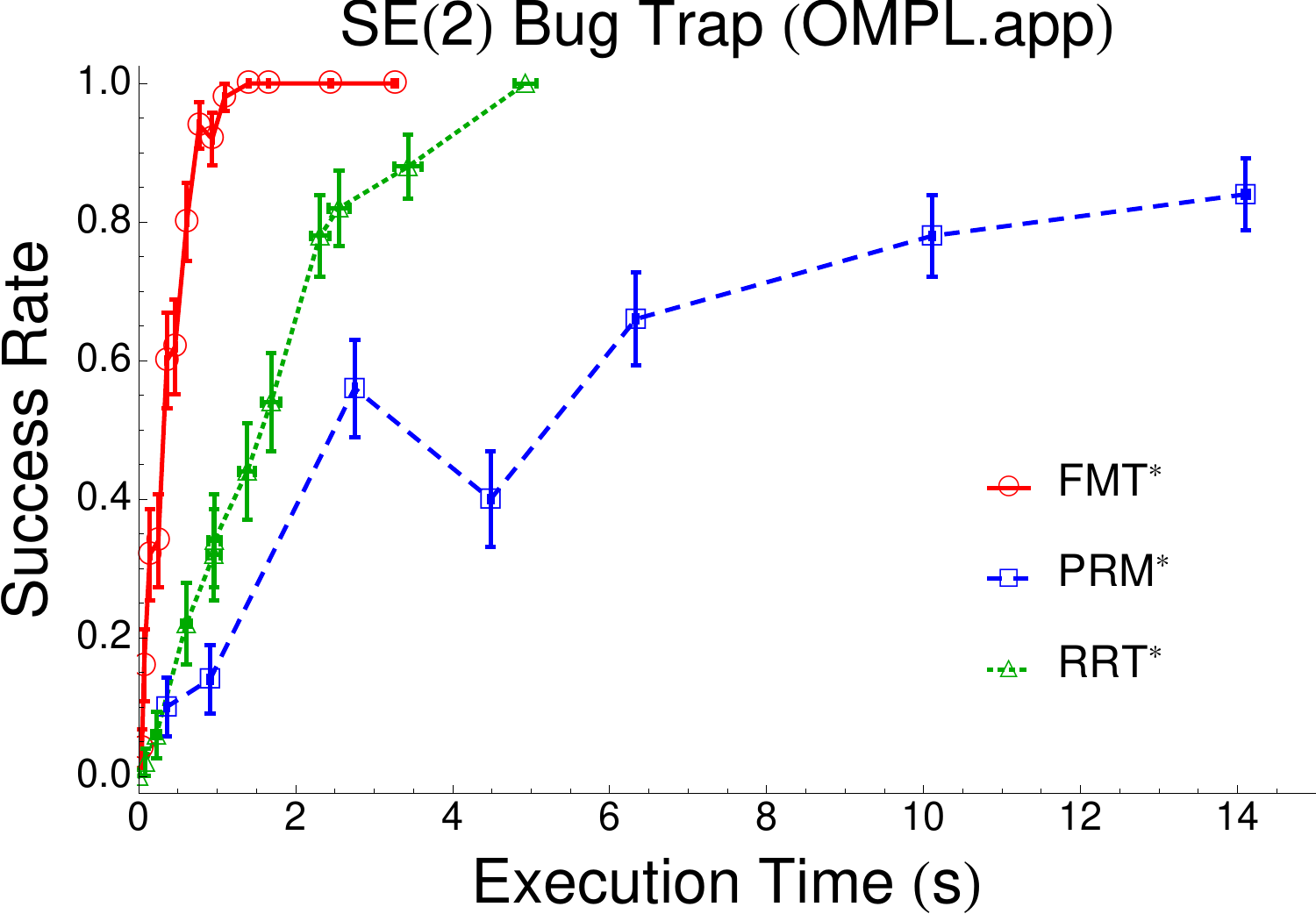}
  } 
  \qquad \qquad
  \subfigure[]{
    \includegraphics[width=\figWidth\textwidth]{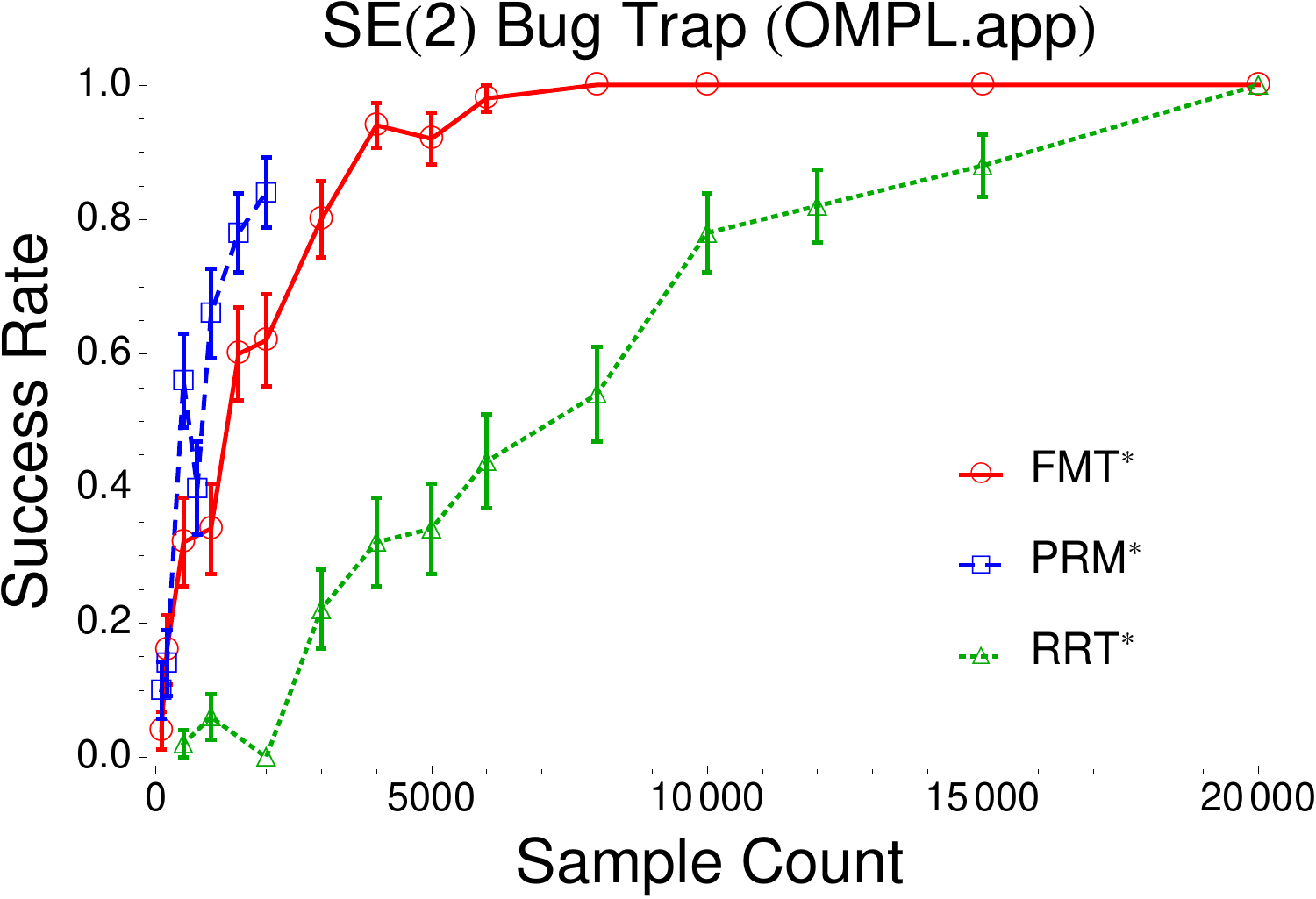}
  }
  \\
  \subfigure[]{
    \includegraphics[width=\figWidth\textwidth]{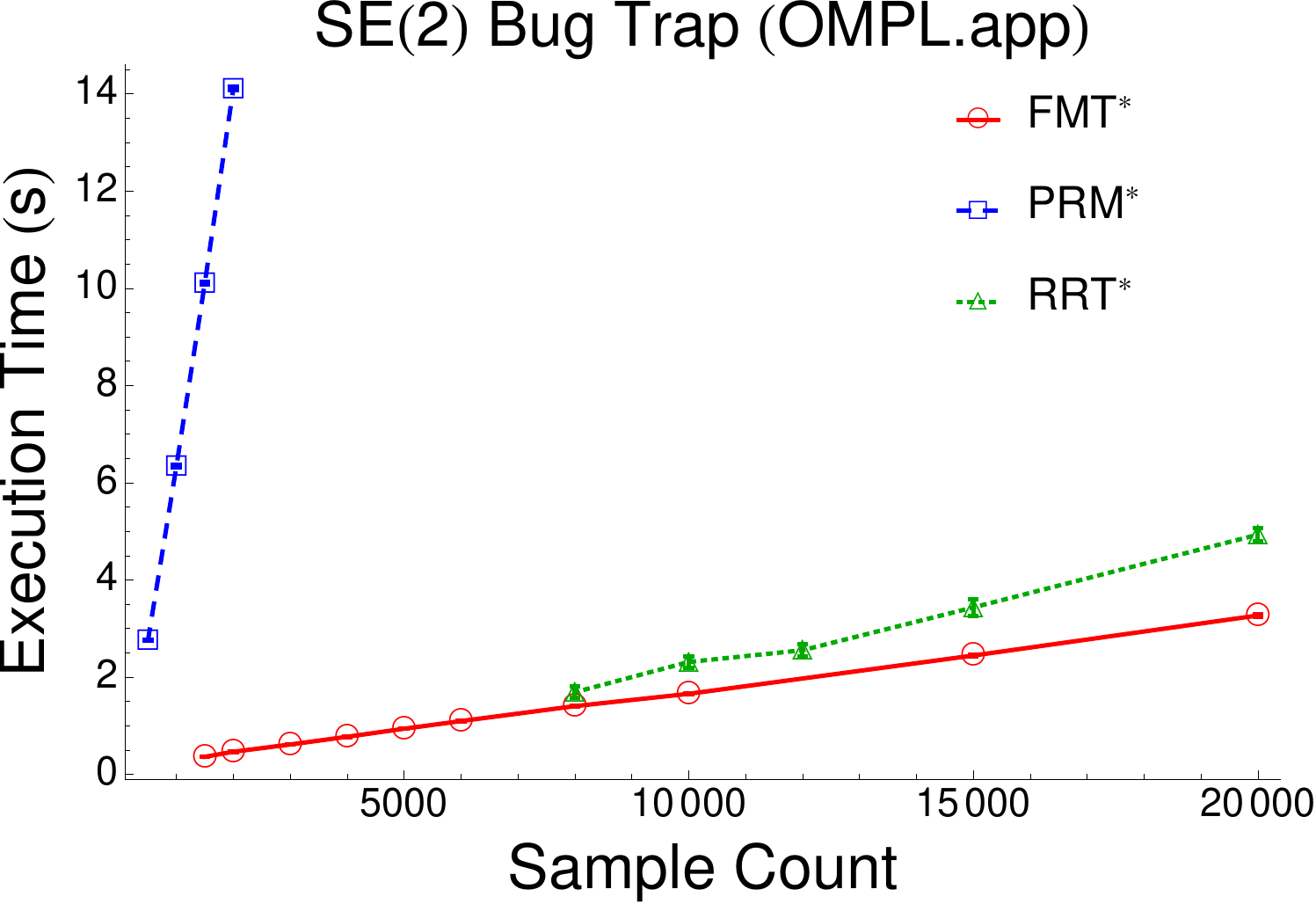}
  }
  \qquad \qquad
  \subfigure[]{
    \includegraphics[width=\figWidth\textwidth]{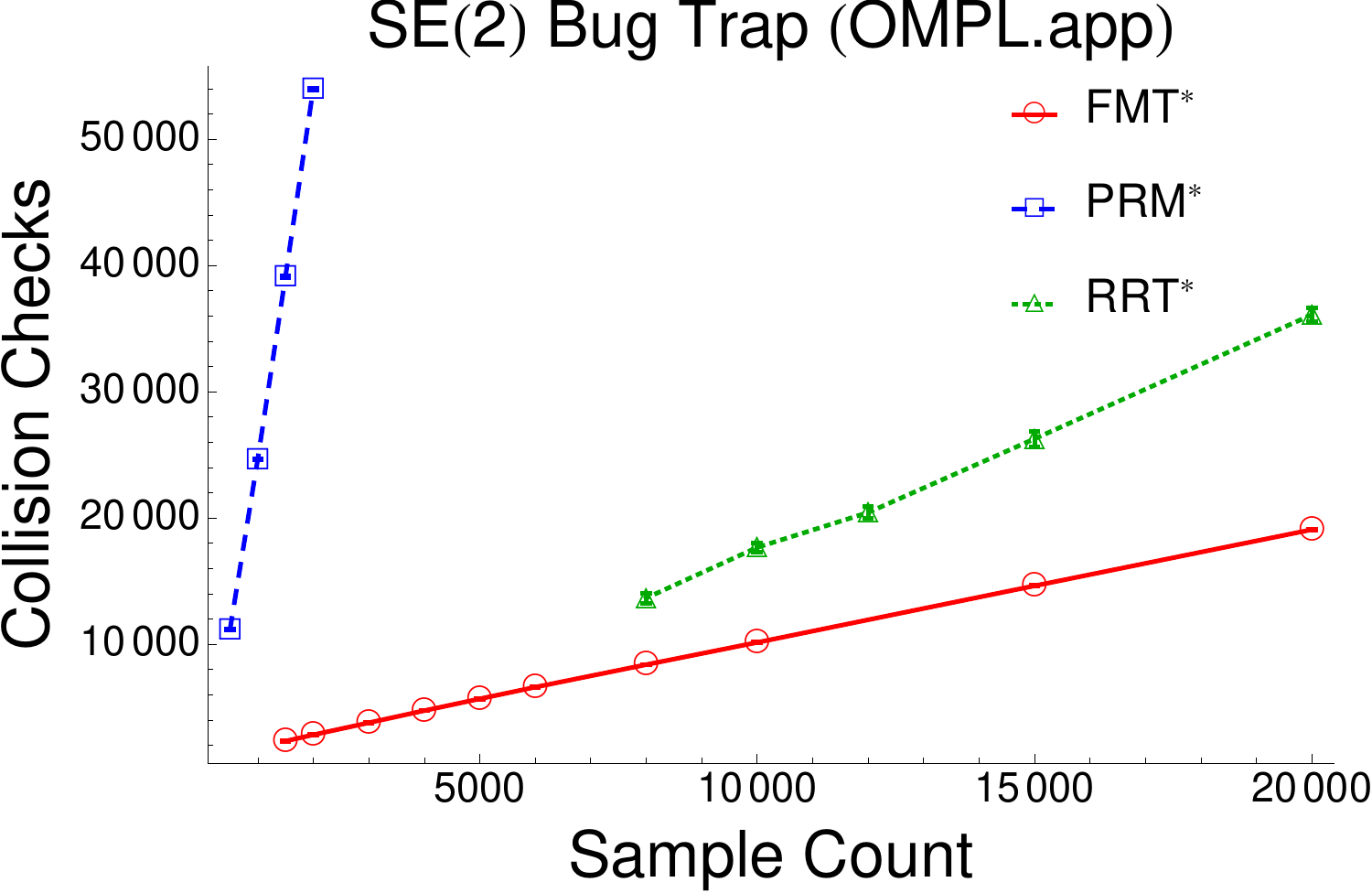}
  }
  \caption{Simulation results for a bug trap environment in 2D space.}
  \label{fig:se2bug}
\end{figure}

\begin{figure}[!t]
  \centering
  \subfigure[]{
    \includegraphics[width=\figWidth\textwidth]{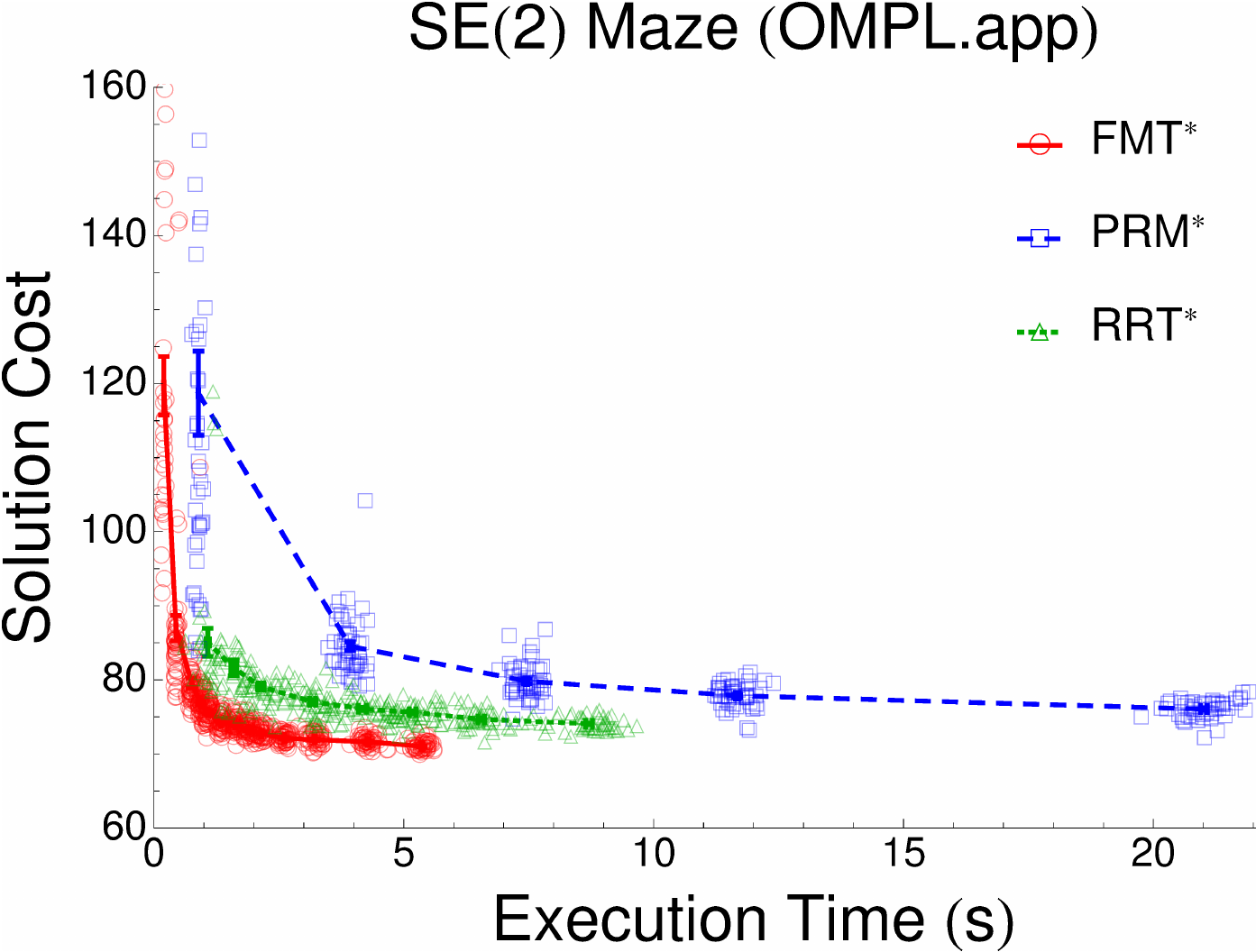}
  }
  \qquad \qquad
  \subfigure[]{
    \includegraphics[width=\figWidth\textwidth]{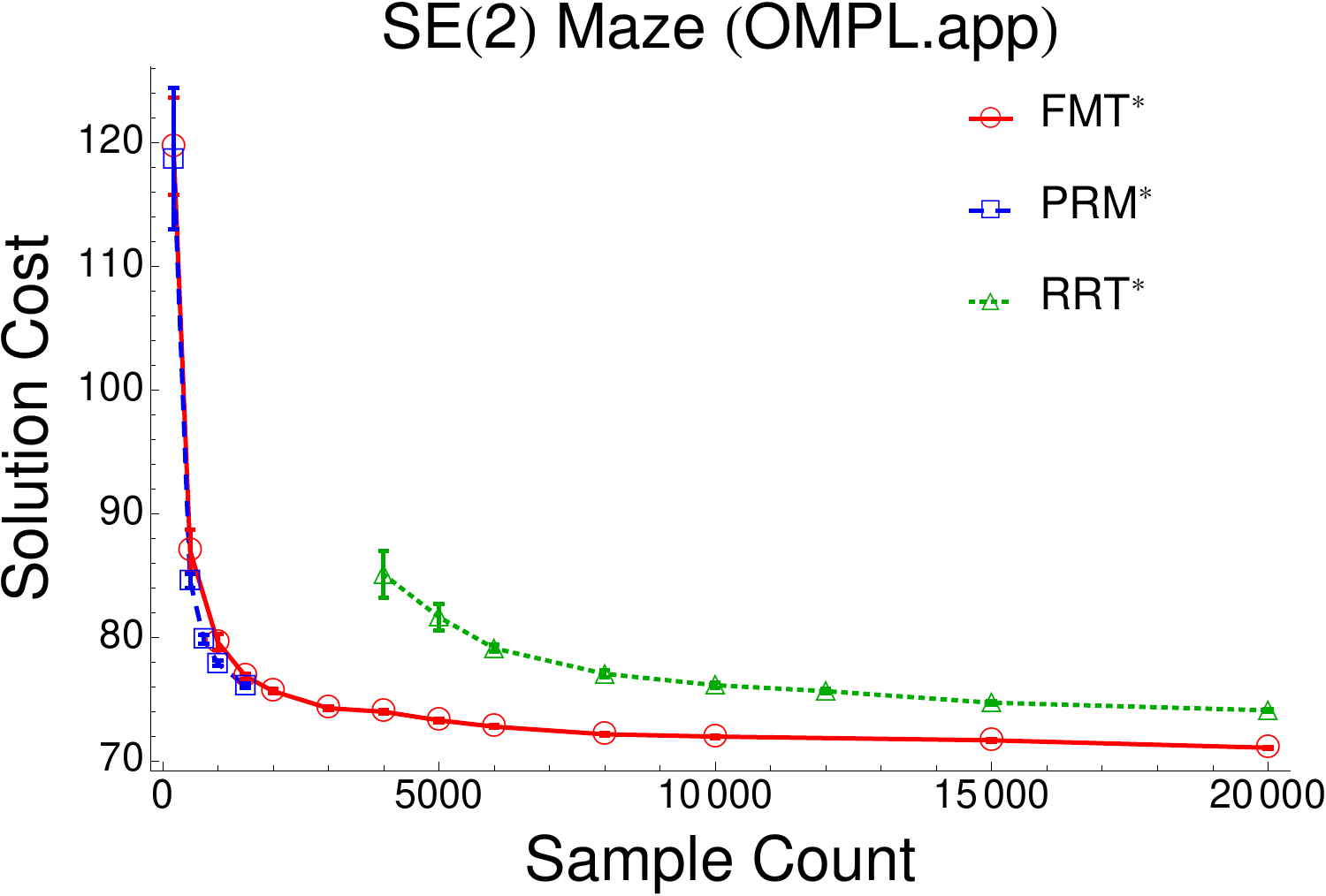}
  }
  \\
  \subfigure[]{
    \includegraphics[width=\figWidth\textwidth]{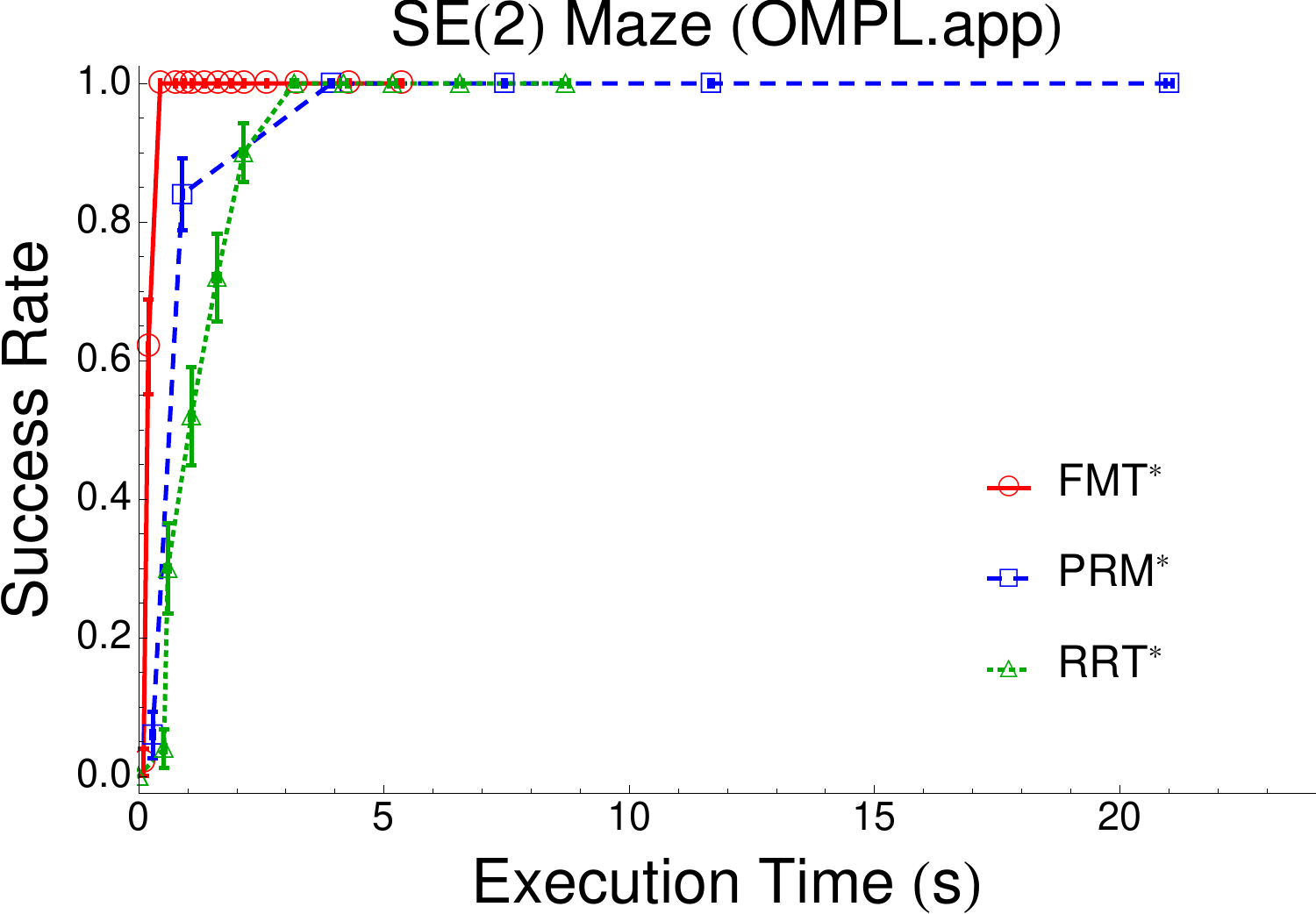}
  }
  \qquad \qquad
  \subfigure[]{
    \includegraphics[width=\figWidth\textwidth]{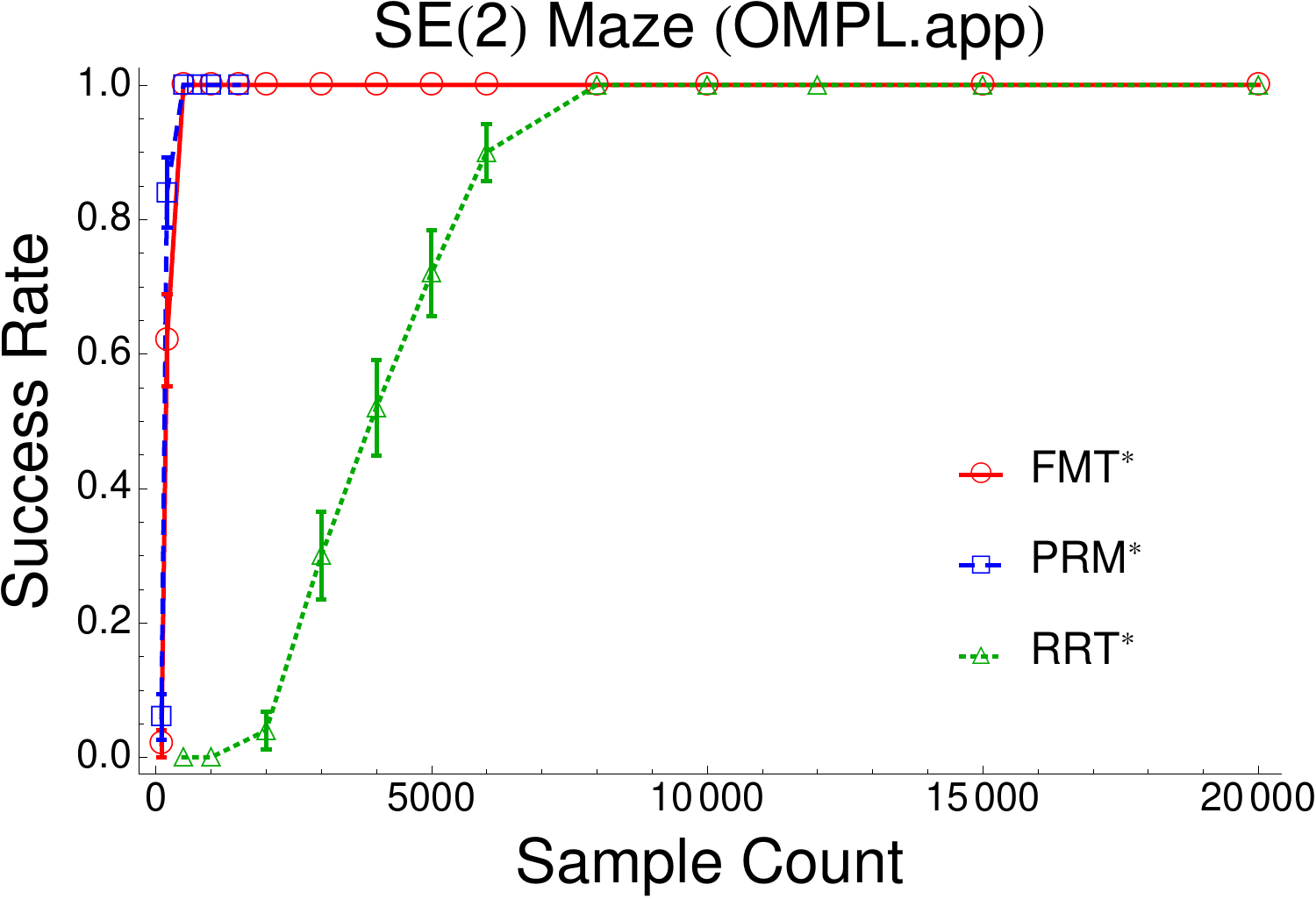}
  }
  \\
  \subfigure[]{
    \includegraphics[width=\figWidth\textwidth]{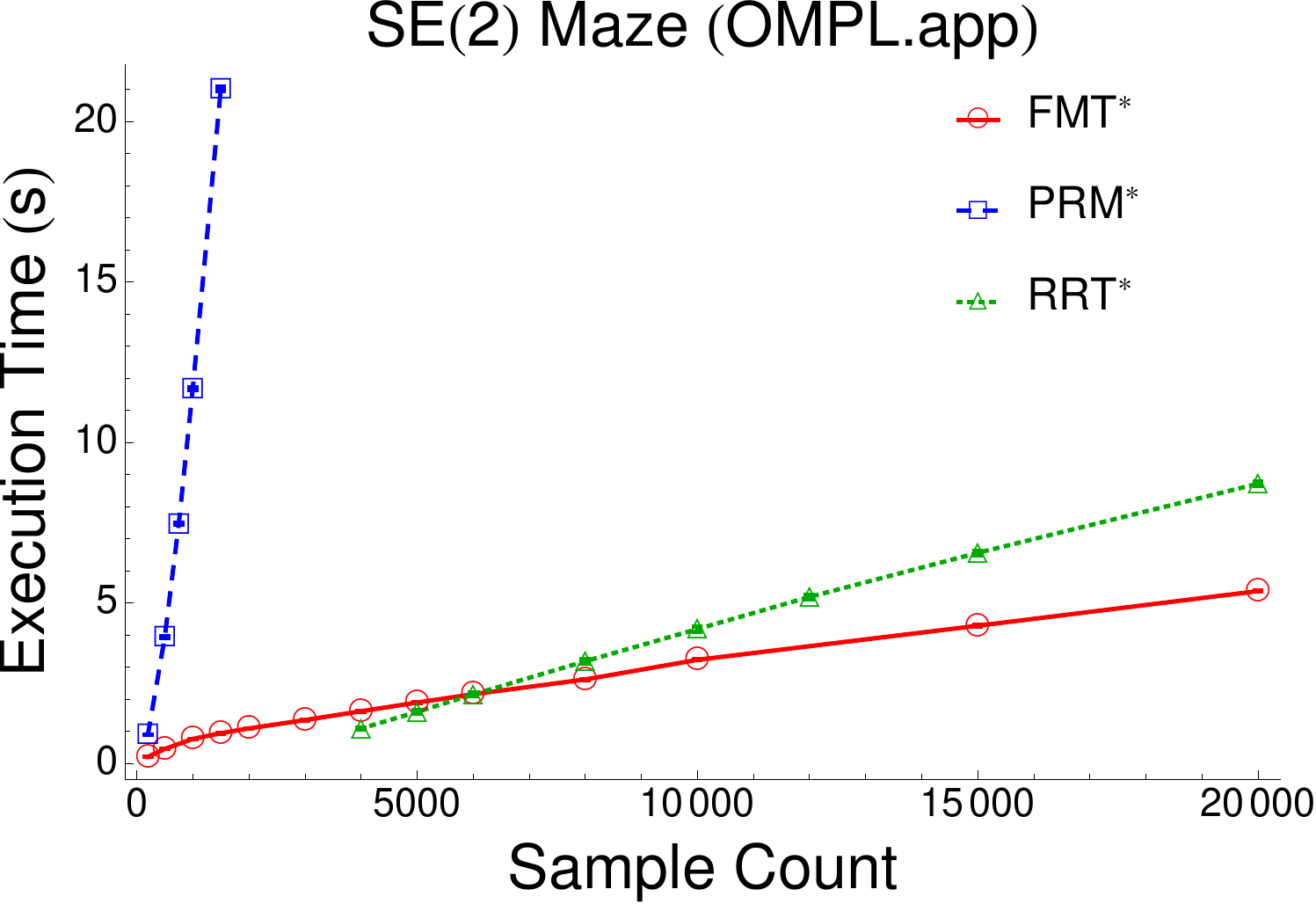}
  }
  \qquad \qquad
  \subfigure[]{
    \includegraphics[width=\figWidth\textwidth]{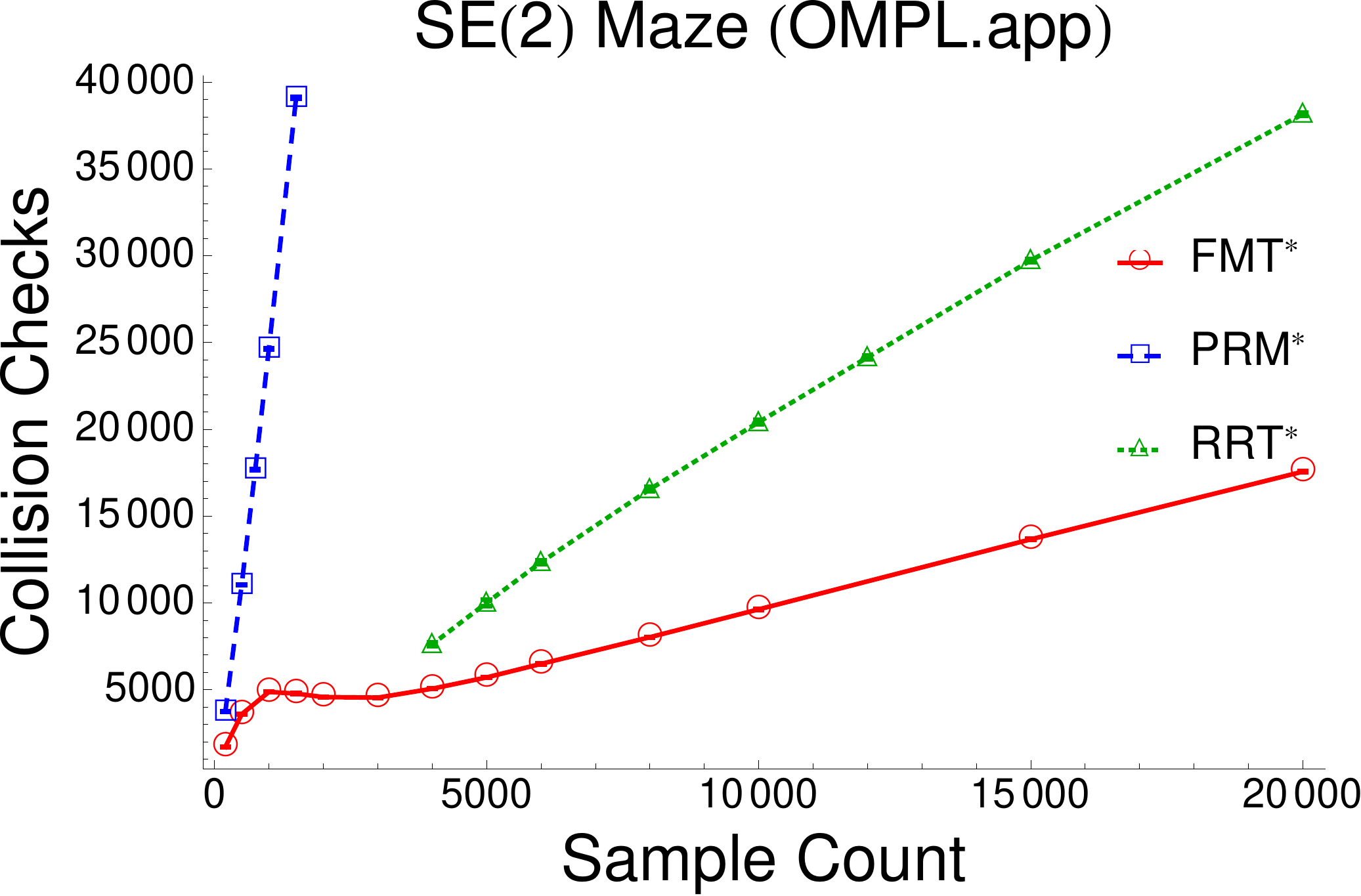}
  }
  \caption{Simulation results for a maze environment in 2D space.}
\label{fig:se2maze}
\end{figure}

\begin{figure}[!t]
  \centering
  \subfigure[]{
    \includegraphics[width=\figWidth\textwidth]{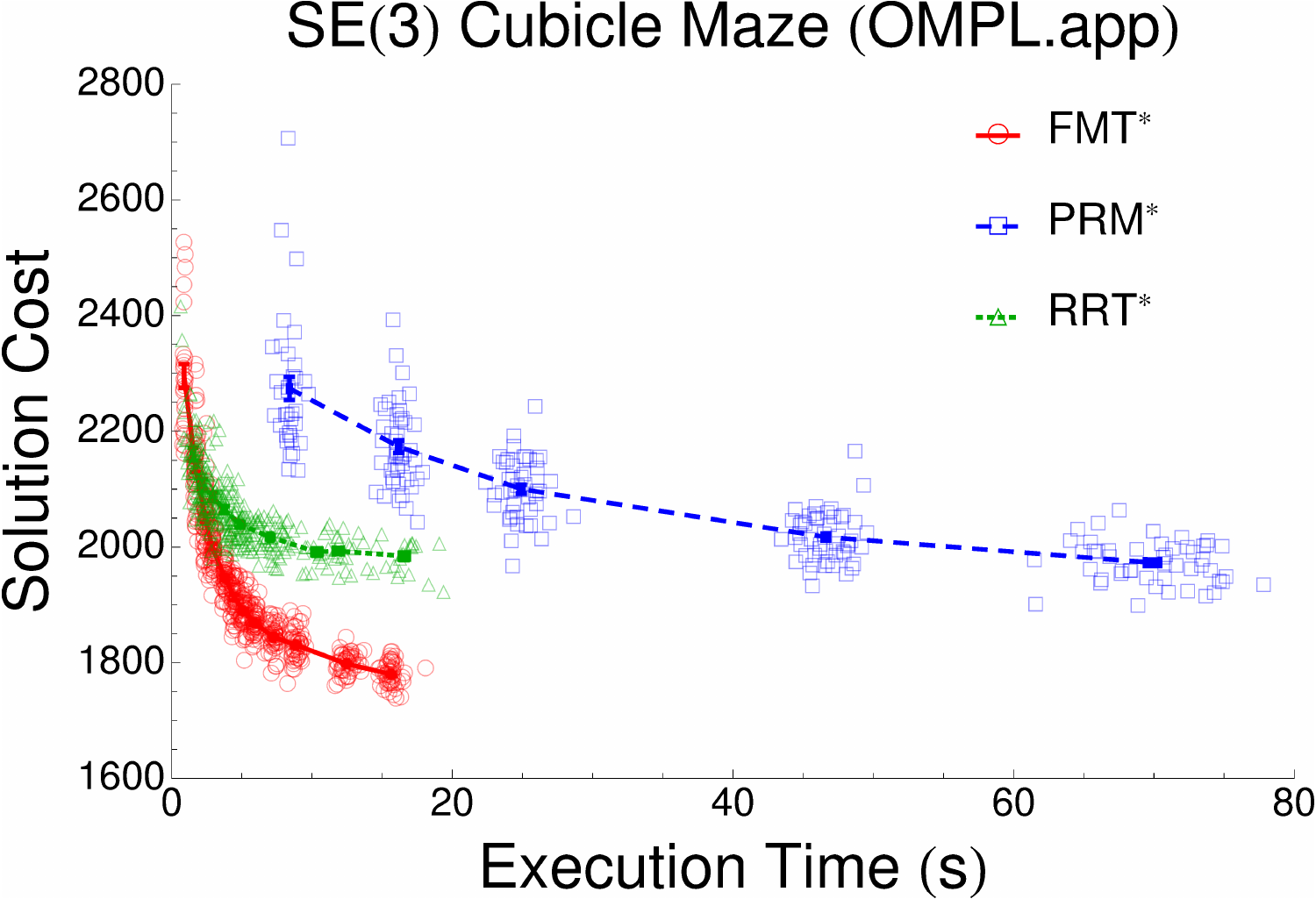}
  }
  \qquad\qquad
  \subfigure[]{
    \includegraphics[width=\figWidth\textwidth]{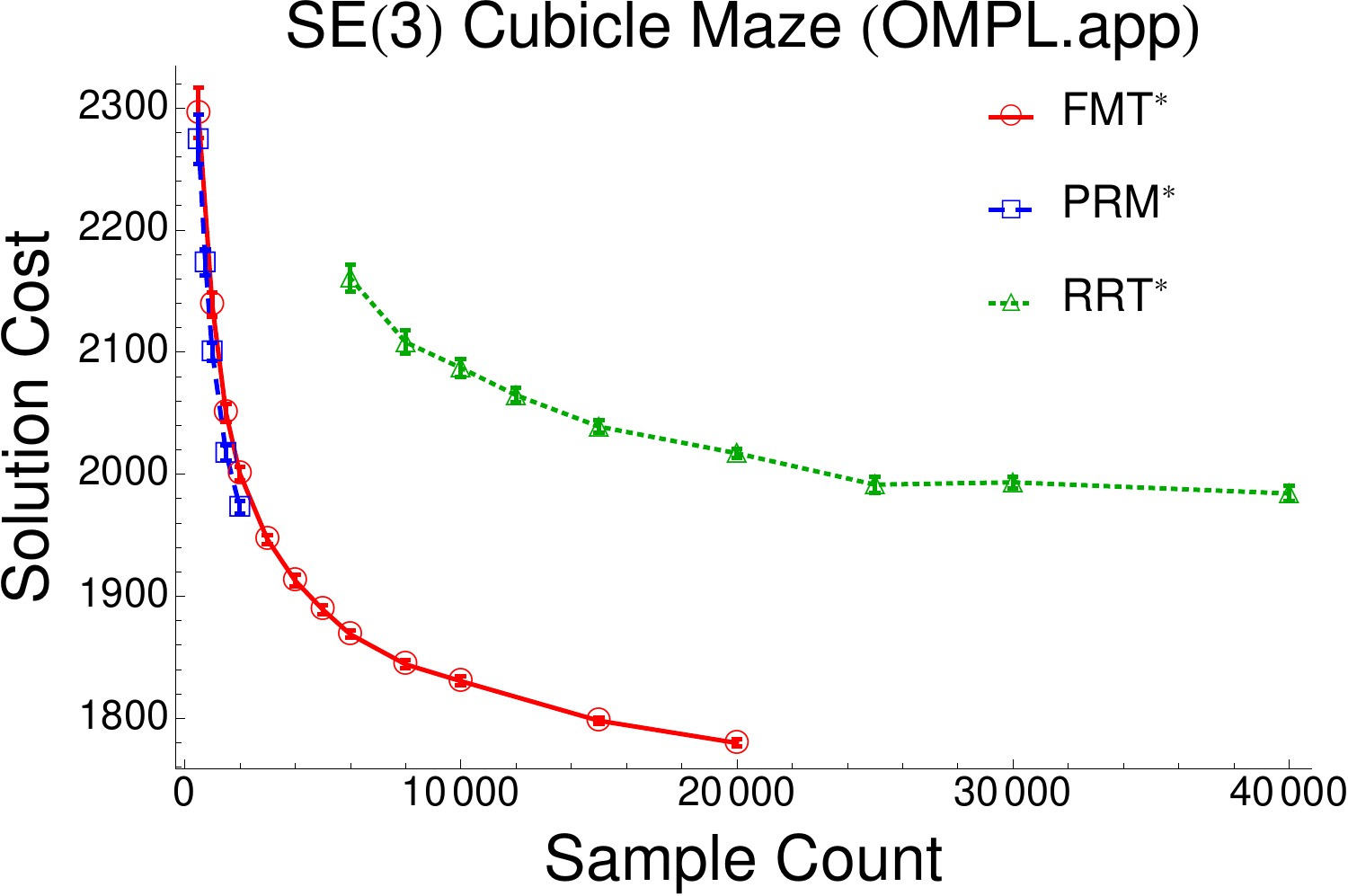}
  }
  \\
  \subfigure[]{
    \includegraphics[width=\figWidth\textwidth]{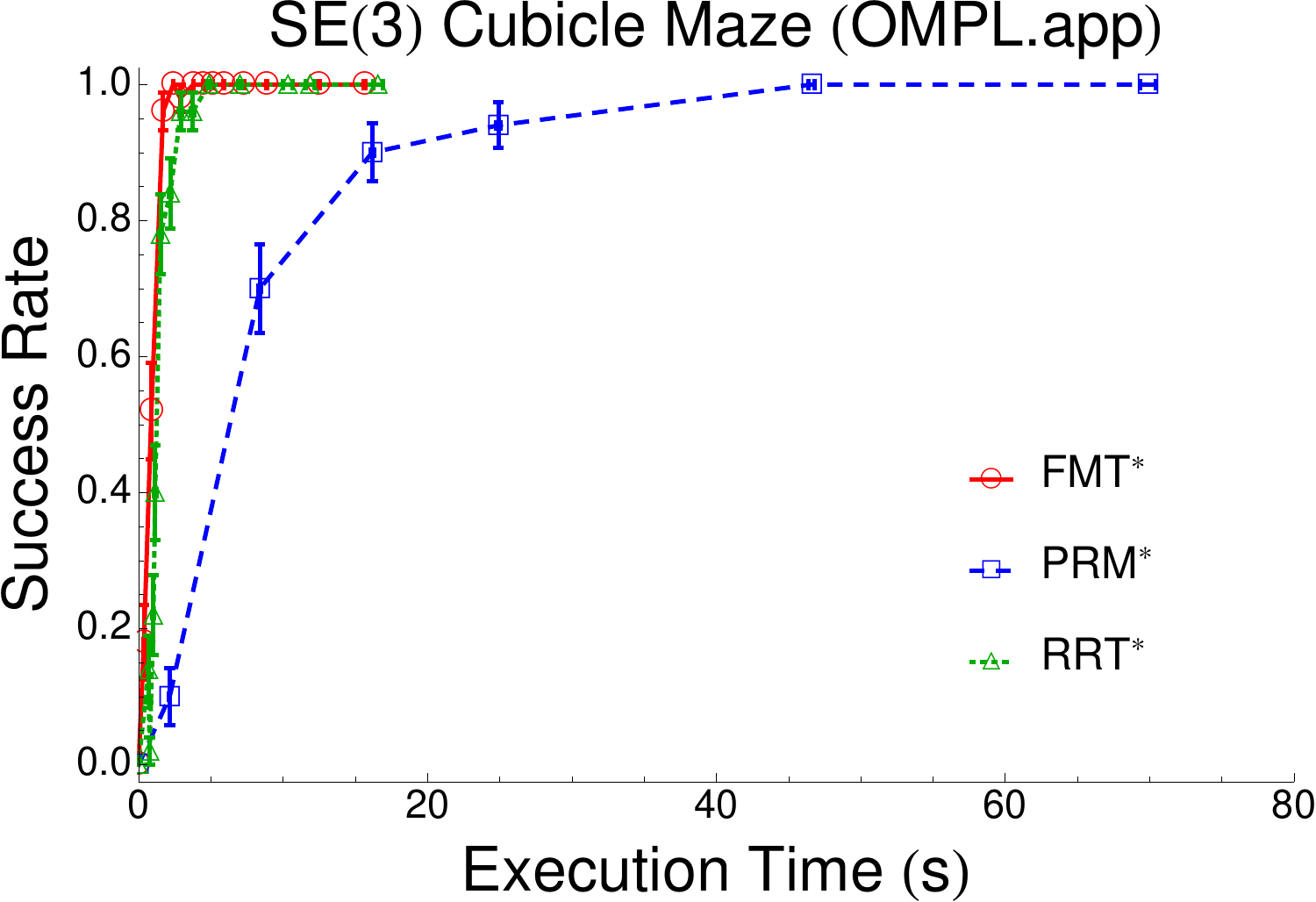}
  }
  \qquad\qquad
  \subfigure[]{
    \includegraphics[width=\figWidth\textwidth]{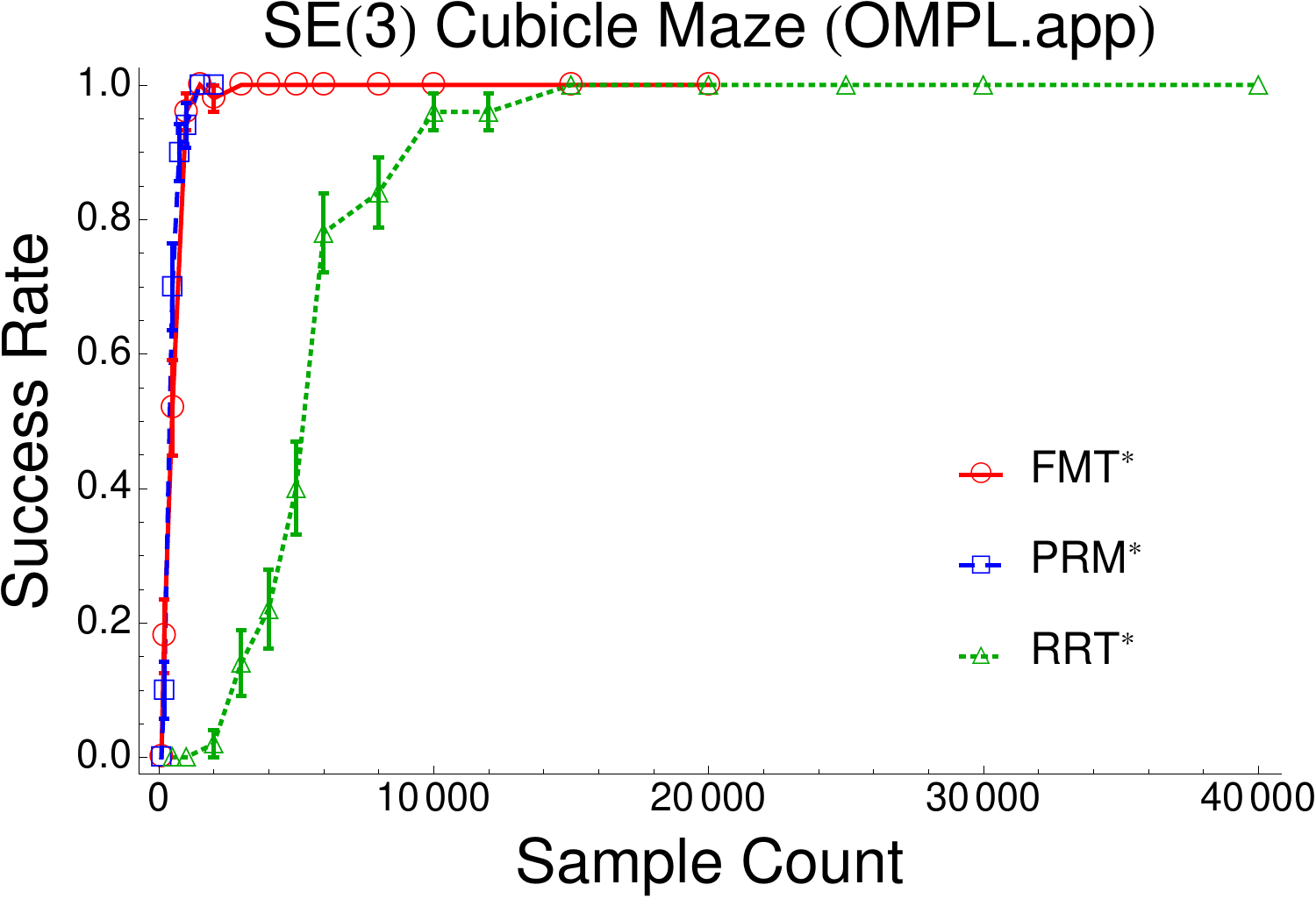}
  }
  \\
  \subfigure[]{
    \includegraphics[width=\figWidth\textwidth]{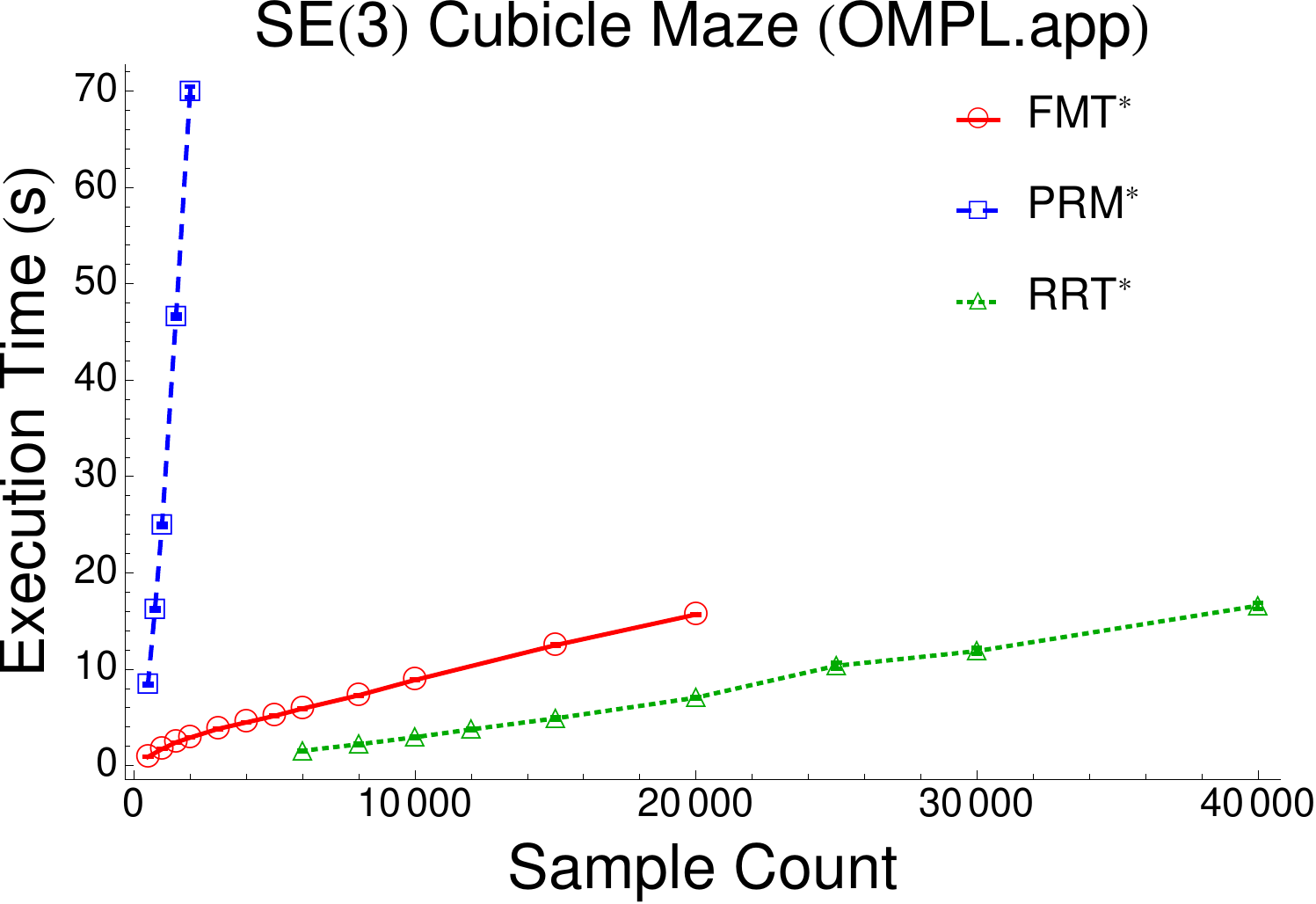}
  }
  \qquad\qquad
  \subfigure[]{
    \includegraphics[width=\figWidth\textwidth]{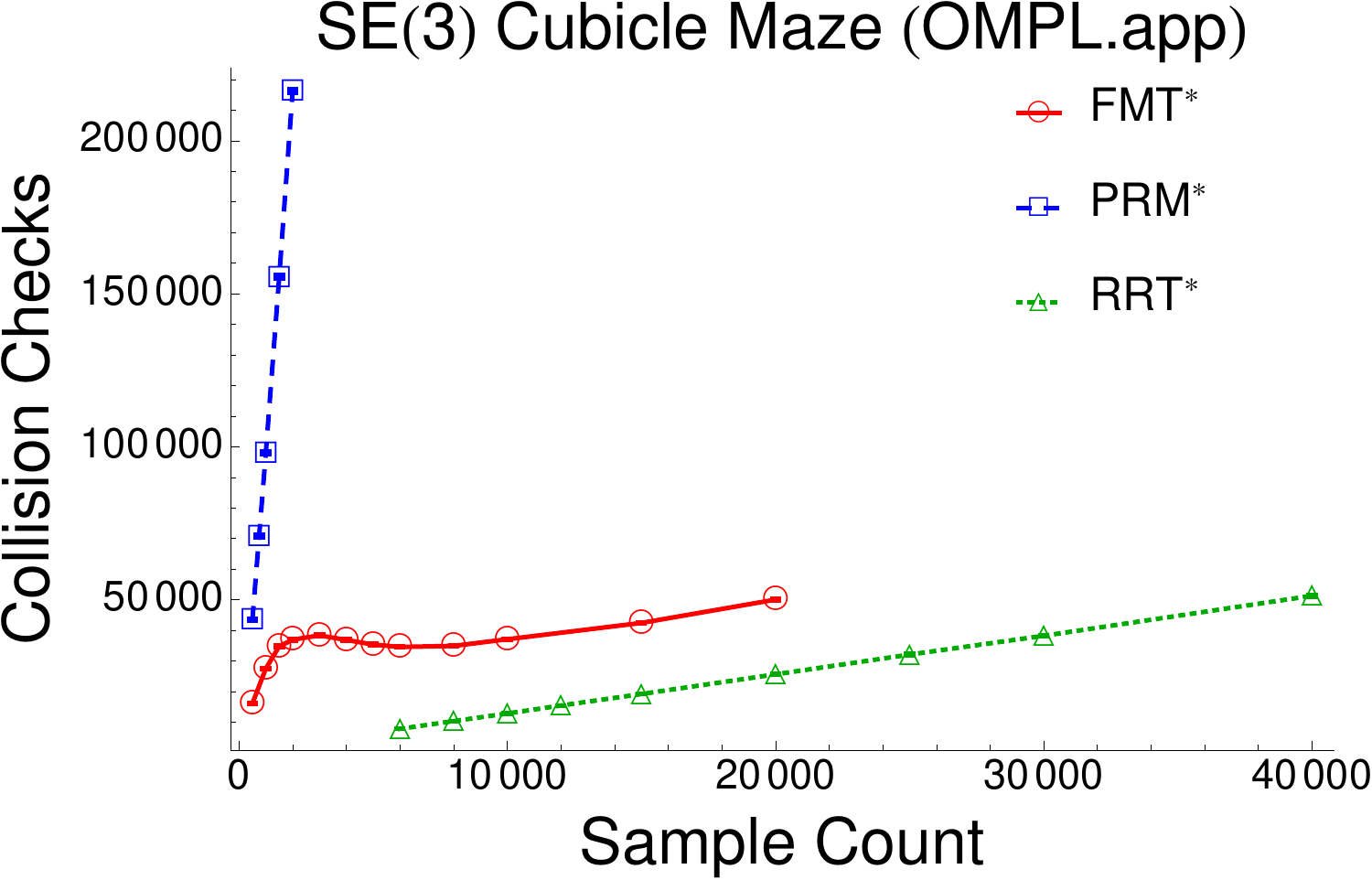}
  }
  \caption{Simulation results for a maze environment in 3D space.}
\label{fig:se3maze}
\end{figure}

\subsection{Comparison with Other AO Planning Algorithms}\label{advantages}
\subsubsection{Numerical Experiments in an $\text{SE}(2)$ Bug Trap}
\label{SE2bug}
The first test case is the classic bug trap problem in $\text{SE}(2)$ (Figure \ref{fig:bg2}), a prototypically challenging problem for sampling-based motion planners \citep{Lavalle:06}. The simulation results for this problem are depicted
graphically in Figure~\ref{fig:se2bug}. \FMT {\color{black}takes}
about half and one tenth the time to reach similar quality solutions
as \RRTstar and \PRMstar\!, respectively, on average. Note that \FMT
also is by far the quickest to reach high success rates, achieving
nearly 100\% in about one second, while \RRTstar takes about five
seconds and \PRMstar is still at 80\% success rate after 14
seconds. The plot of solution cost as a function of sample count shows
what we would expect: \FMT and \PRMstar return nearly identical-quality
solutions for the same number of samples, with \PRMstar very slightly
better, while \RRTstar\!, due to its greediness, suffers in
comparison. Similarly, \FMT and \PRMstar have similar success rates as
a function of sample count, both substantially higher than
\RRTstar. The reason that \RRTstar still beats \PRMstar in terms of cost 
versus time is explained by the plot of execution time versus sample
count: \RRTstar is much faster per sample than
\PRMstar. However, \RRTstar is still slightly slower per sample than
\FMT\!, as explained by the plot of collision-checks versus sample
count, which shows \FMT performing fewer collision-checks per sample
($O(1)$) than \RRTstar ($O(\log(n))$).

The lower success rate for \RRTstar may be explained as a consequence of its graph expansion process.
When iterating to escape the bug trap, the closest tree node to a new sample outside the mouth of the trap will nearly always lie in one of the ``dead end'' lips, and thus present an invalid steering connection.
Only when the new sample lies adjacent to the progress of the tree down the corridor will \RRTstar be able to advance. For \RRTstar to escape the bug trap, an \emph{ordered sequence} of samples must be obtained that lead the tree through the corridor.
\FMT and \PRMstar are not affected by this problem; their success rate is determined only by whether or not such a set of samples exists, not the order in which they are sampled by the algorithm.

\subsubsection{Numerical Experiments in an $\text{SE}(2)$ Maze}
Navigating a ``maze'' environment  is another prototypical benchmark for path planners \citep{Sucan.ea:RAM12}. This section, in particular, considers an $\text{SE}(2)$ maze (portrayed in Figure \ref{fig:s2m}). The plots for this environment, given in
Figure~\ref{fig:se2maze}, tell a very similar 
story to those of the $\text{SE}(2)$ bug trap. Again, \FMT reaches
given solution qualities faster
than \RRTstar and \PRMstar by factors of about 2 and 10, respectively{\color{black}.
A}lthough the success rates of all the algorithms go to 100\% quite
quickly{\color{black}, \FMT is} still the fastest. All other heuristic
relationships between algorithms in the other graphs remain the same
as in the case of the $\text{SE}(2)$ bug trap.

\subsubsection{Numerical Experiments in an $\text{SE}(3)$ Maze}
Figure~\ref{fig:se3maze} presents simulation results for a three-dimensional maze, specifically 
for the  maze in $\text{SE}(3)$ depicted in Figure \ref{fig:s3m}. These results show a few differences
from those in the previous two subsections. First of all, \FMT is an
even clearer winner in the cost versus time graph, with relative
speeds compared to \RRTstar and \PRMstar hard to compare due to the
fact that \FMT reaches an average solution quality in less than five
seconds that is below that achieved by \RRTstar and \PRMstar in about
20 seconds and 70 seconds, respectively. Furthermore, at 20 seconds,
the \FMT solution appears to still be improving faster than \RRTstar
after the same amount of time. The success rate as a function of time
for \RRTstar is much closer to{\color{black},} though still slightly below{\color{black},} \FMT than it was in
the previous two problem setups, with both algorithms reaching 100\%
completion rate in about three seconds.

A new feature of the $\text{SE}(3)$ maze is that
\RRTstar now runs faster per sample than \FMT\!, due to the fact that it
performs fewer collision-checks per sample than \FMT\!. The reason for
this has to do with the relative search radii of the two
algorithms. Since they work very differently, it is not unreasonable
to use different search radii, and although \FMT will perform fewer
collision-checks asymptotically, for finite sample sizes, the number
of collision-checks is mainly influenced by connection radius and
obstacle clutter. While \RRTstar\!'s radius has been smaller than
\FMT\!'s in all simulations up to this point, the previous two setups
had more clutter, forcing \RRTstar to frequently draw a sample,
collision-check its nearest-neighbor connection, and then remove it
when this check fails. As can be seen in Figure \ref{fig:s3m},
the $\text{SE}(3)$ maze is relatively open and contains fewer
traps as compared to the previous two problems, thereby utilizing more of 
the samples that it runs collision-checks for.

\subsubsection{Numerical Experiments for 3D, 5D, and 7D Recursive Maze}
\begin{figure}[!t]
  \centering
  \subfigure[2D recursive maze.]{
    \includegraphics[height=2in]{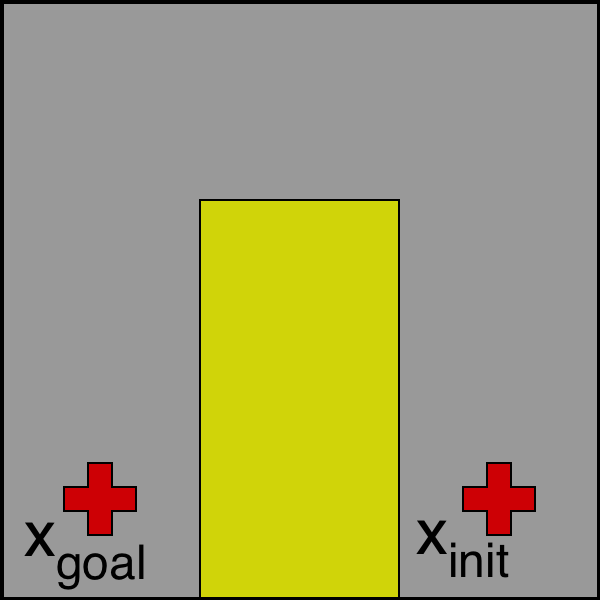}
  }
  \qquad\qquad
  \subfigure[3D recursive maze.]{
    \includegraphics[height=2in]{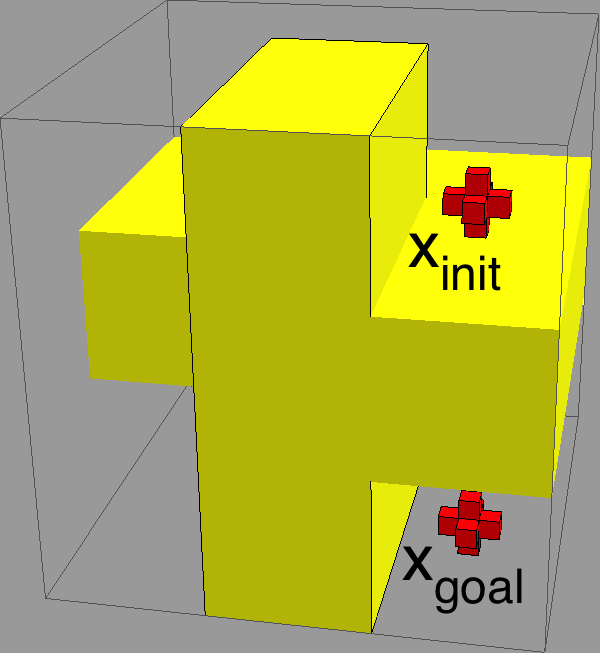}
  }
  \caption{Recursive maze environment.}
  \label{fig:recmazepics}
\end{figure}

\begin{figure}[!t]
  \centering
  \subfigure[]{
    \includegraphics[width=\figWidth\textwidth]{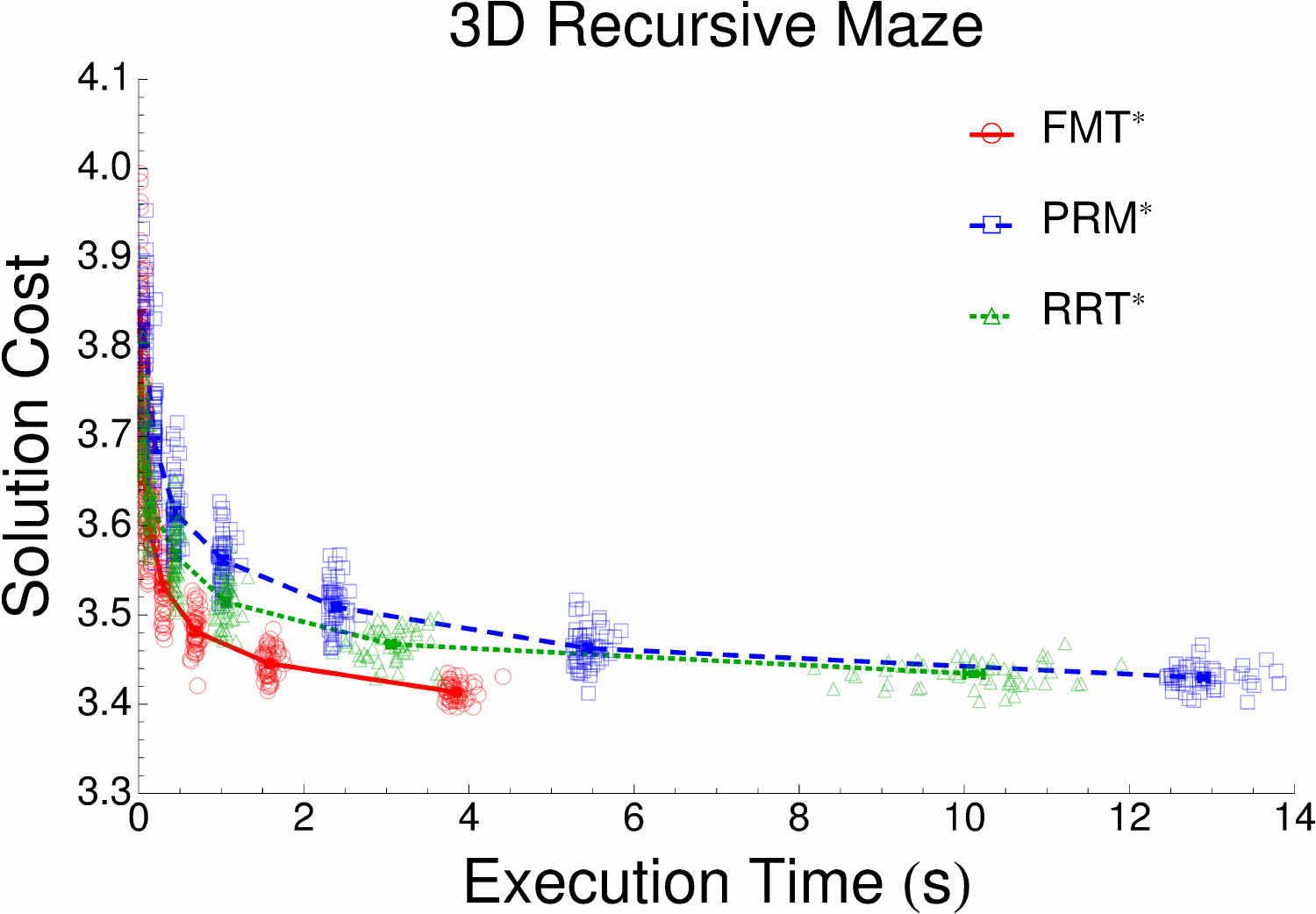}
  }
  \qquad  \qquad
  \subfigure[]{
    \includegraphics[width=\figWidth\textwidth]{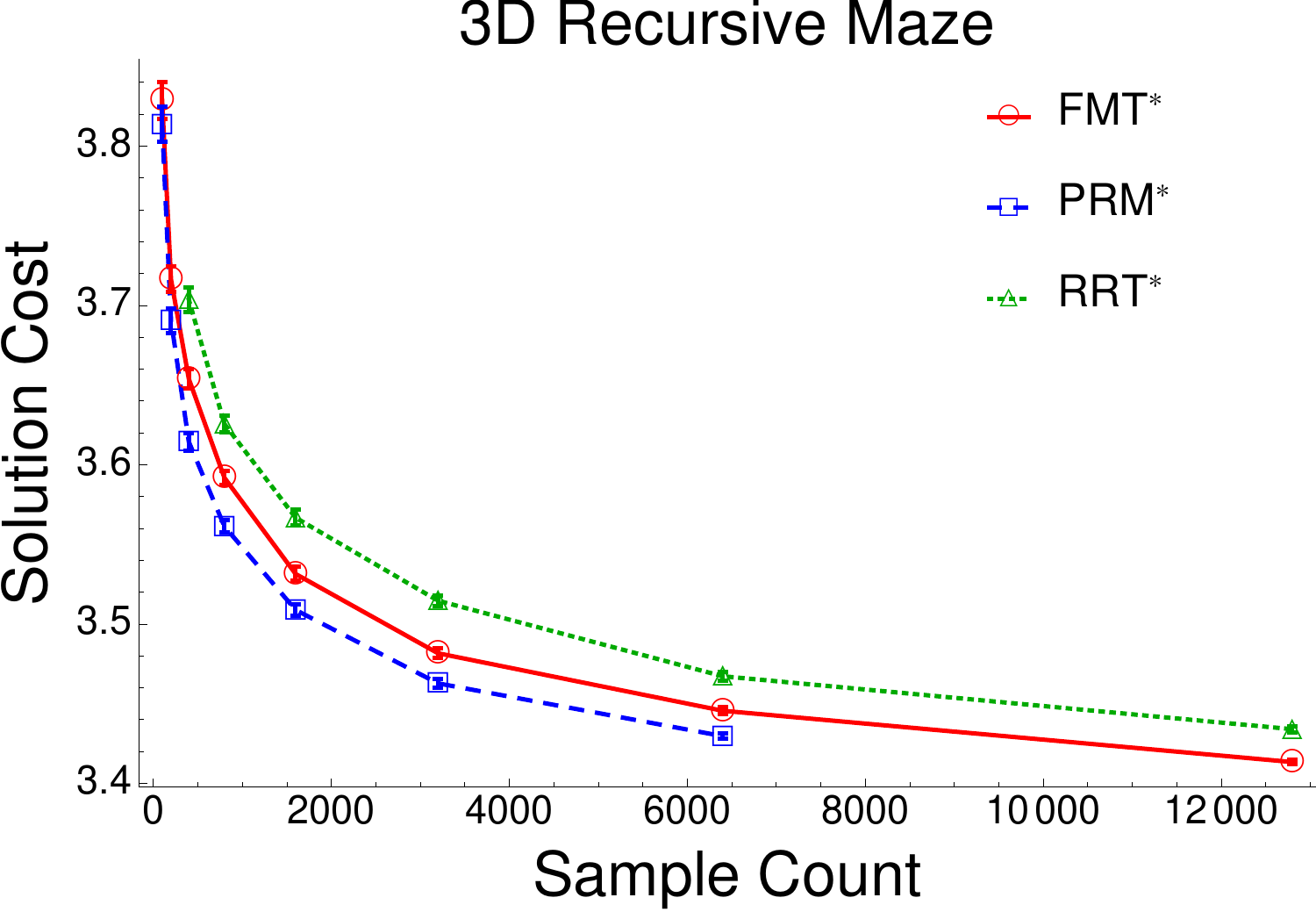}
  }
  \\
  \subfigure[]{
    \includegraphics[width=\figWidth\textwidth]{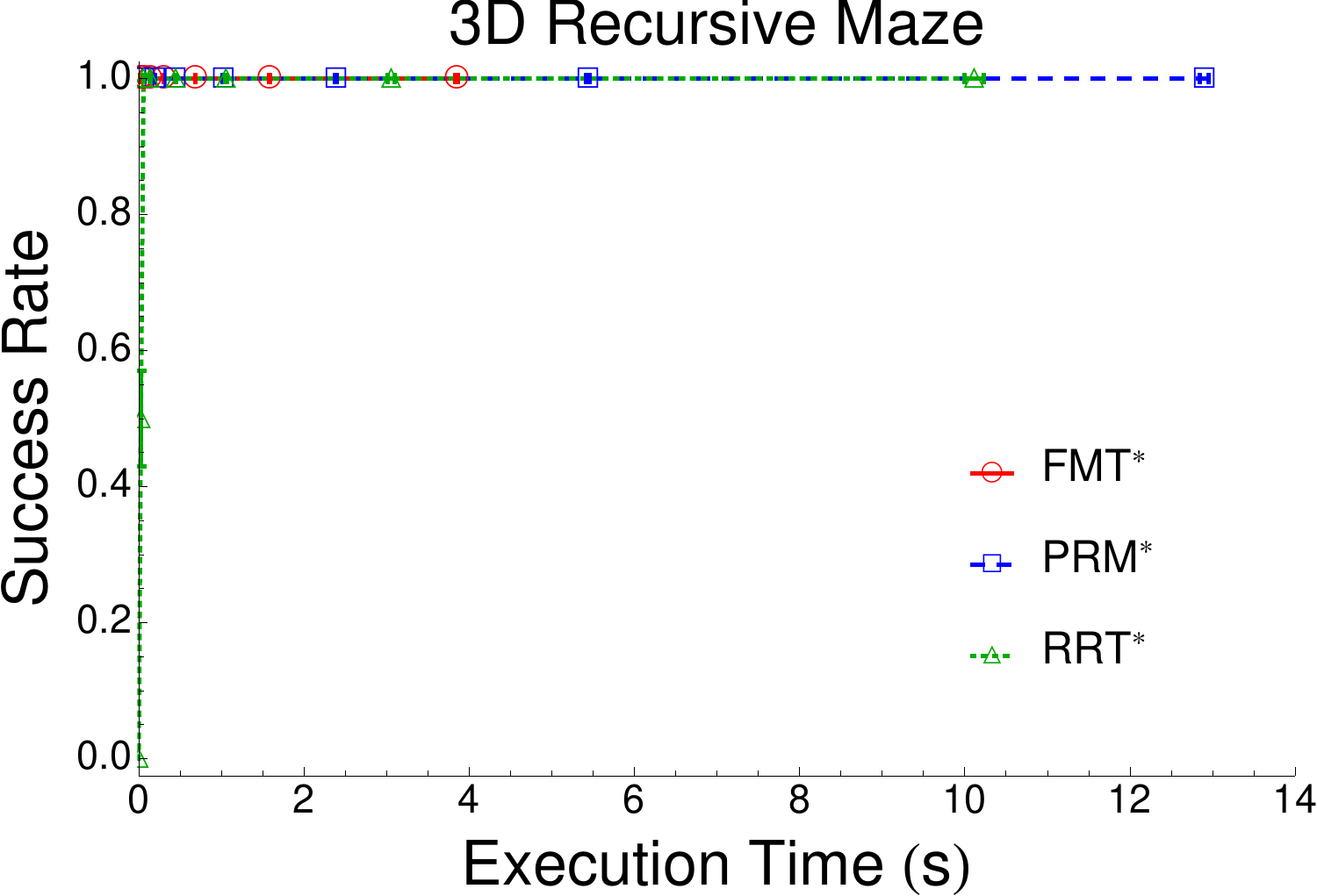}
  }
  \qquad  \qquad
  \subfigure[]{
    \includegraphics[width=\figWidth\textwidth]{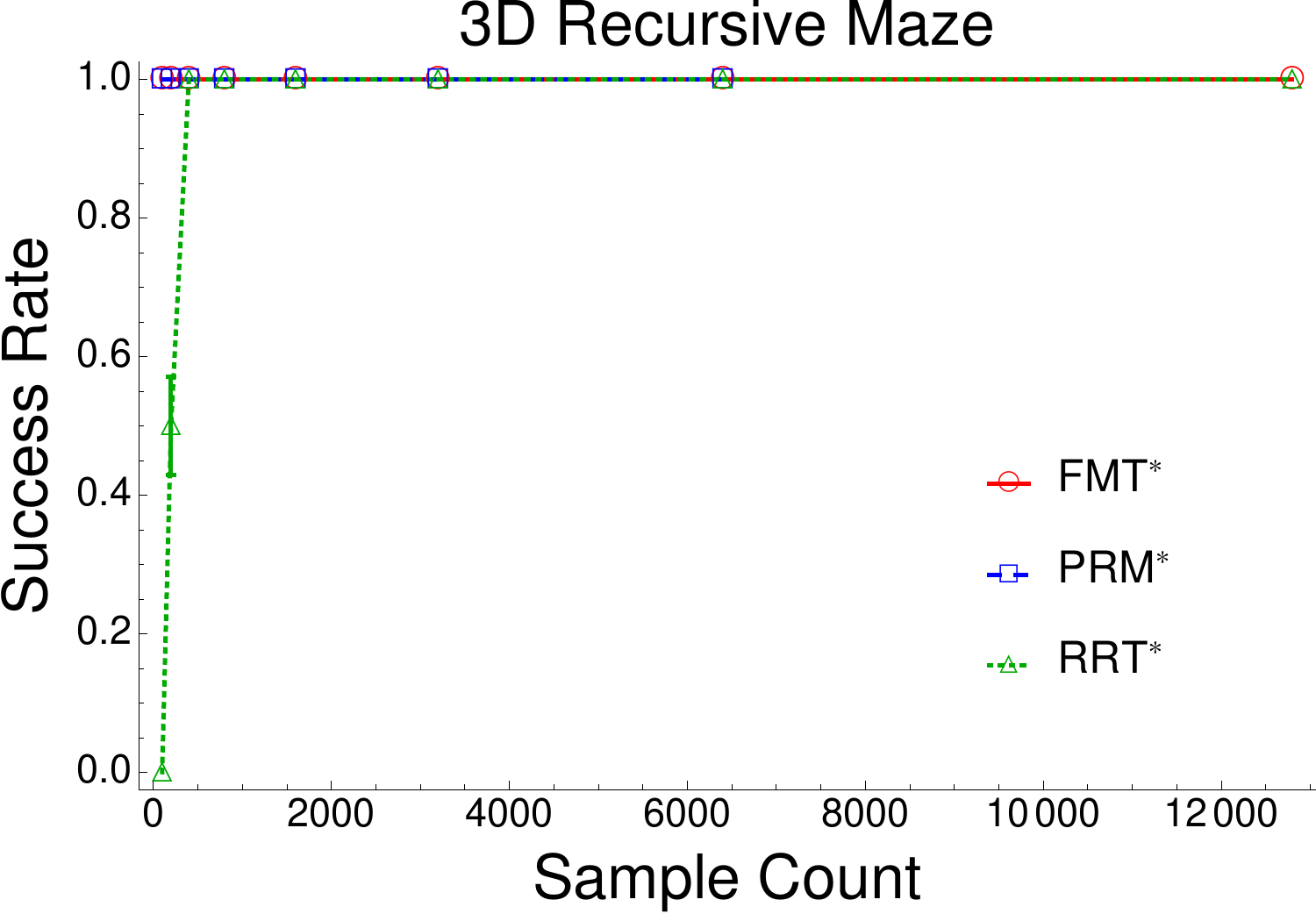}
  }
  \\
  \subfigure[]{
    \includegraphics[width=\figWidth\textwidth]{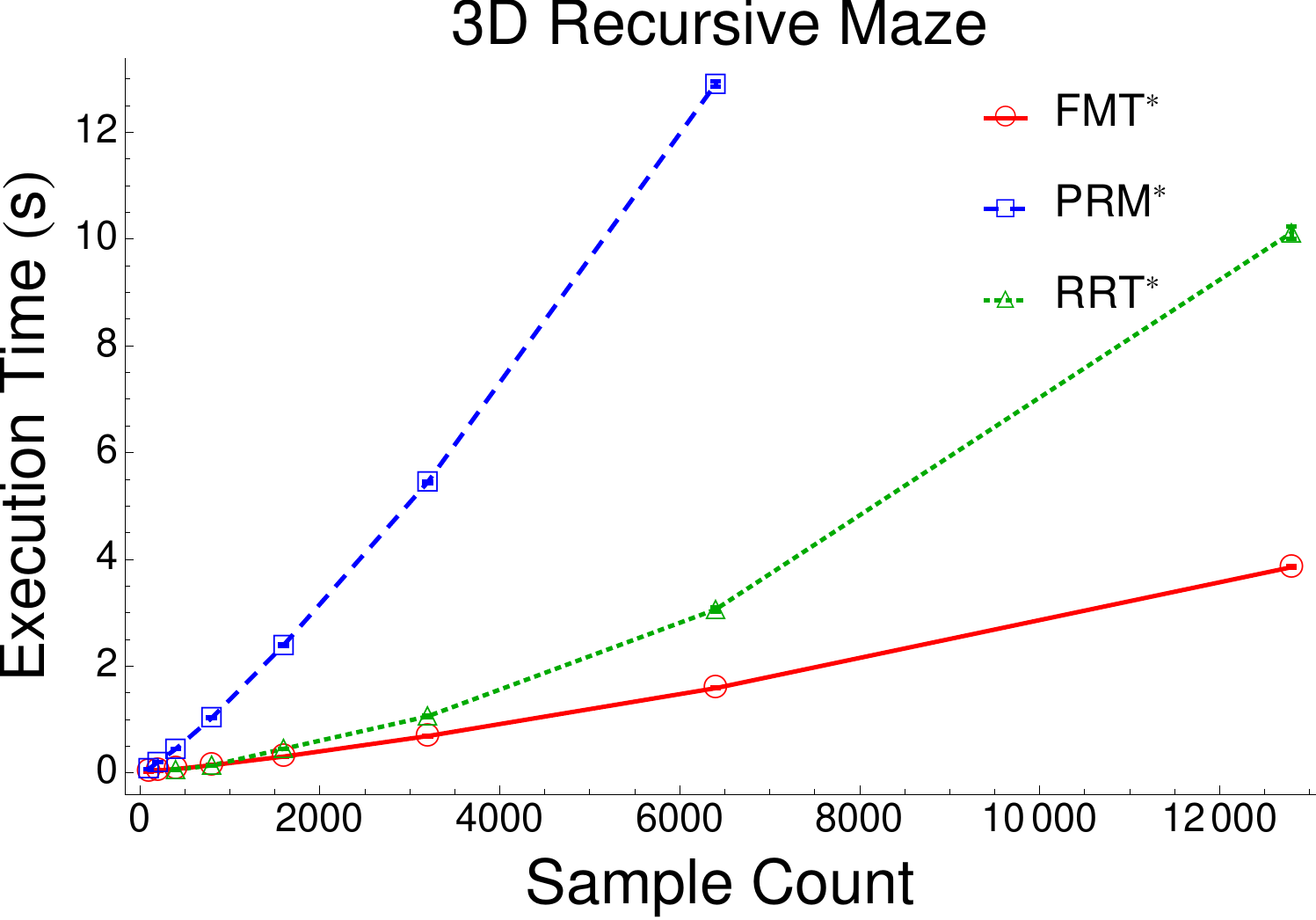}
  }
  \qquad  \qquad
  \subfigure[]{
    \includegraphics[width=\figWidth\textwidth]{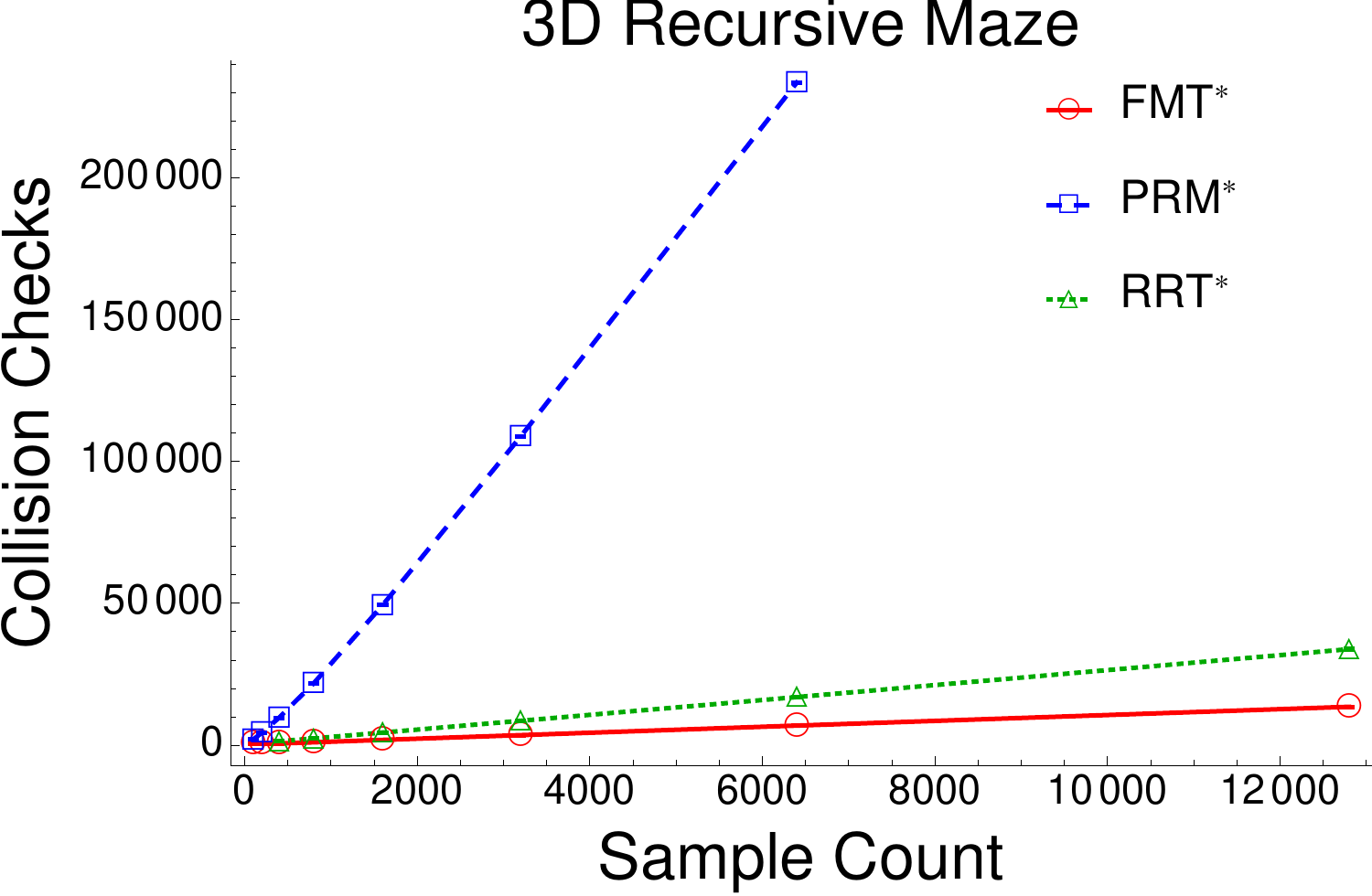}
  }
  \caption{Simulation results for a recursive maze environment in 3D.}
\label{fig:recmaze3}
\end{figure}

\begin{figure}[!t]
  \centering
  \subfigure{
    \includegraphics[width=\figWidth\textwidth]{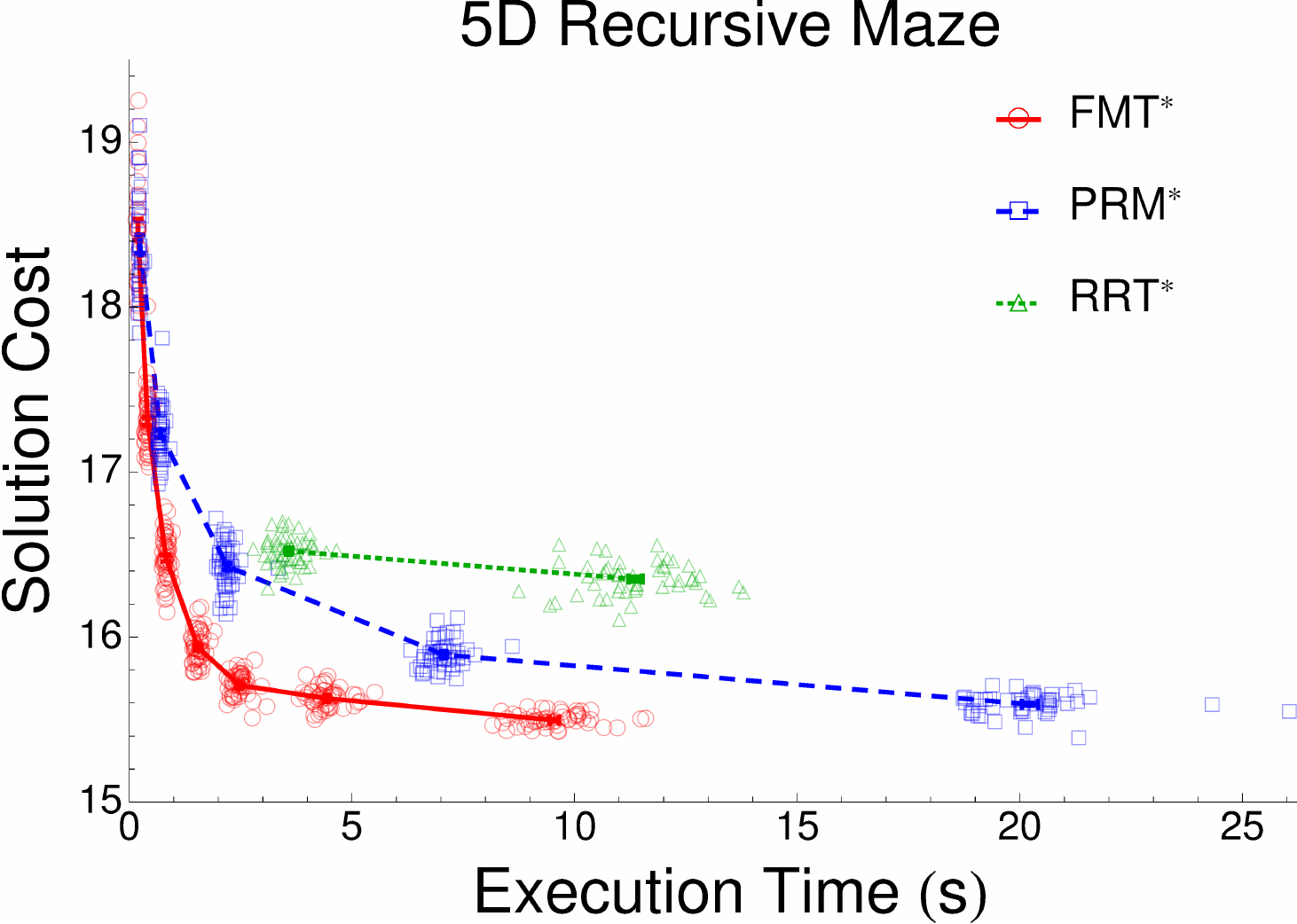}
  }
  \qquad \qquad
  \subfigure[]{
    \includegraphics[width=\figWidth\textwidth]{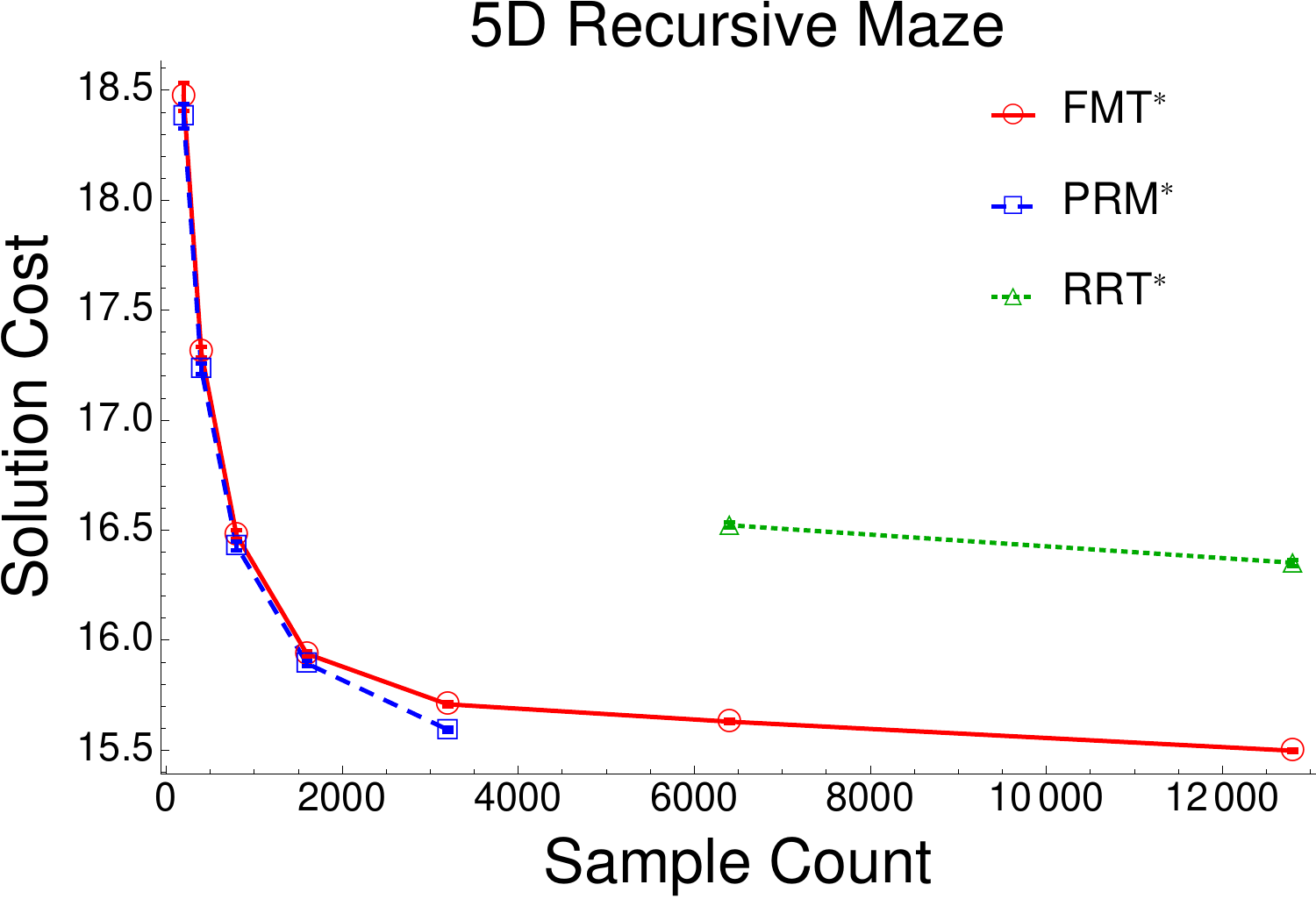}
  }
  \\
  \subfigure[]{
    \includegraphics[width=\figWidth\textwidth]{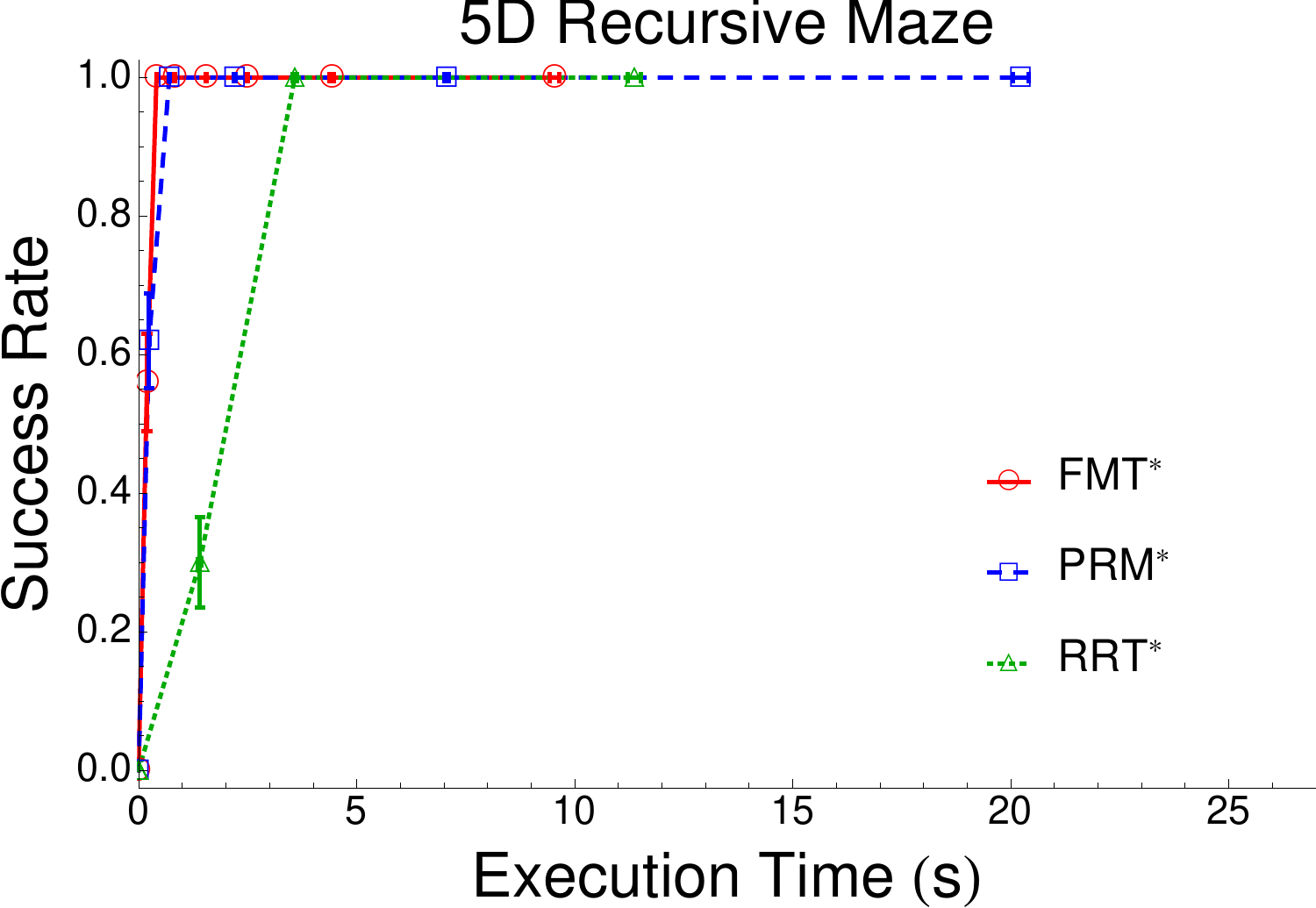}
  }
  \qquad \qquad
  \subfigure[]{
    \includegraphics[width=\figWidth\textwidth]{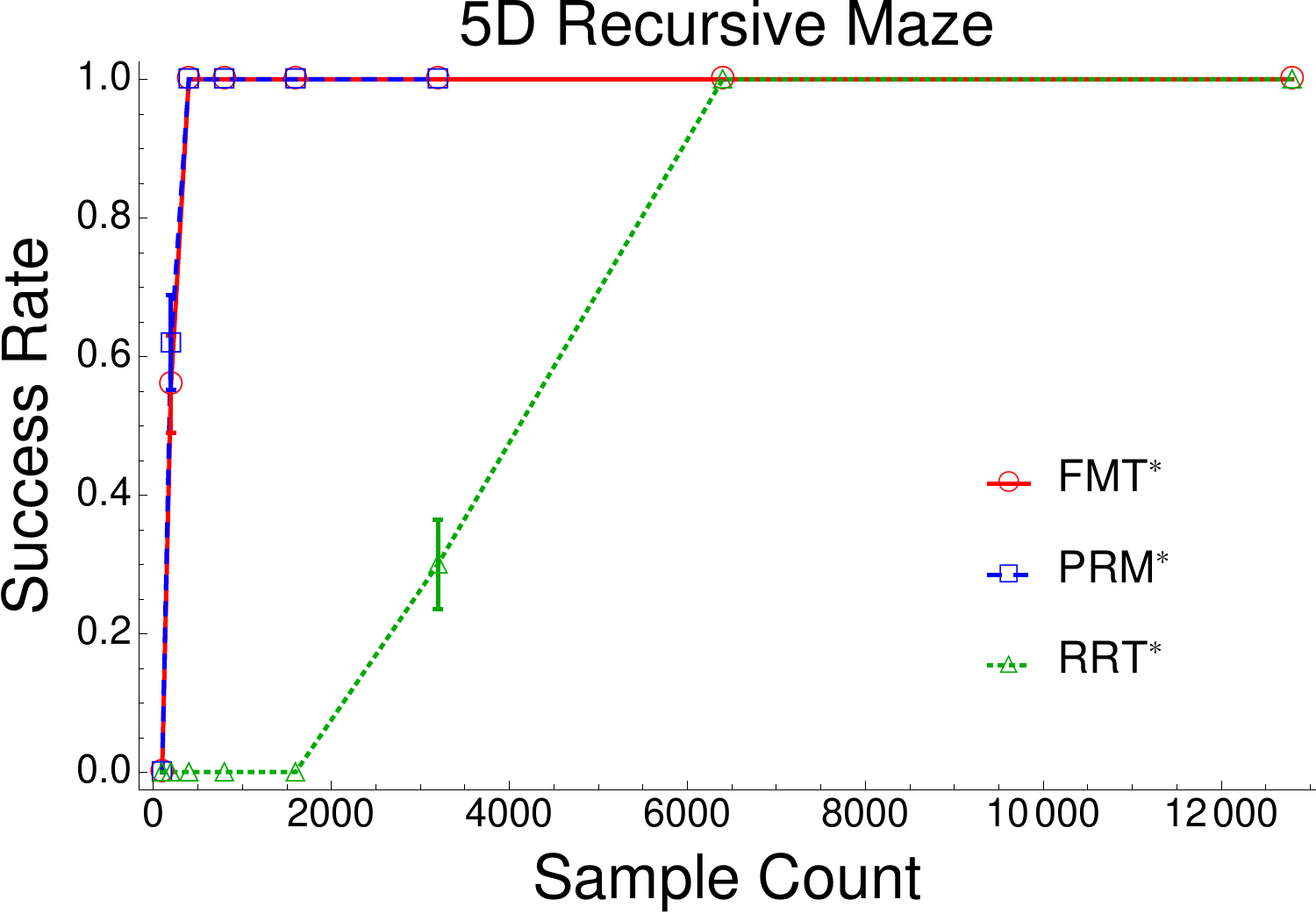}
  }
  \\
  \subfigure[]{
    \includegraphics[width=\figWidth\textwidth]{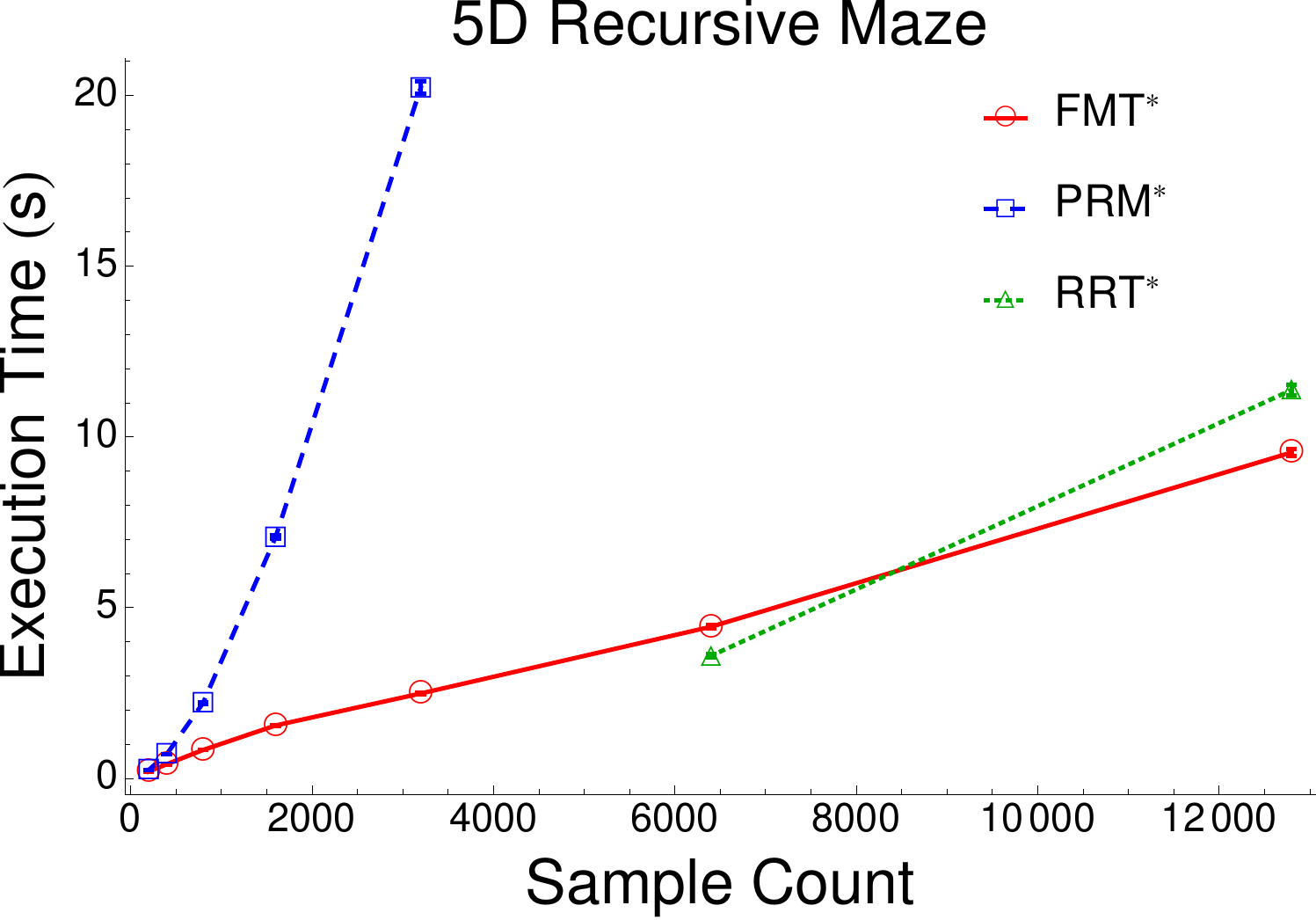}
  }
  \qquad \qquad
  \subfigure[]{
    \includegraphics[width=\figWidth\textwidth]{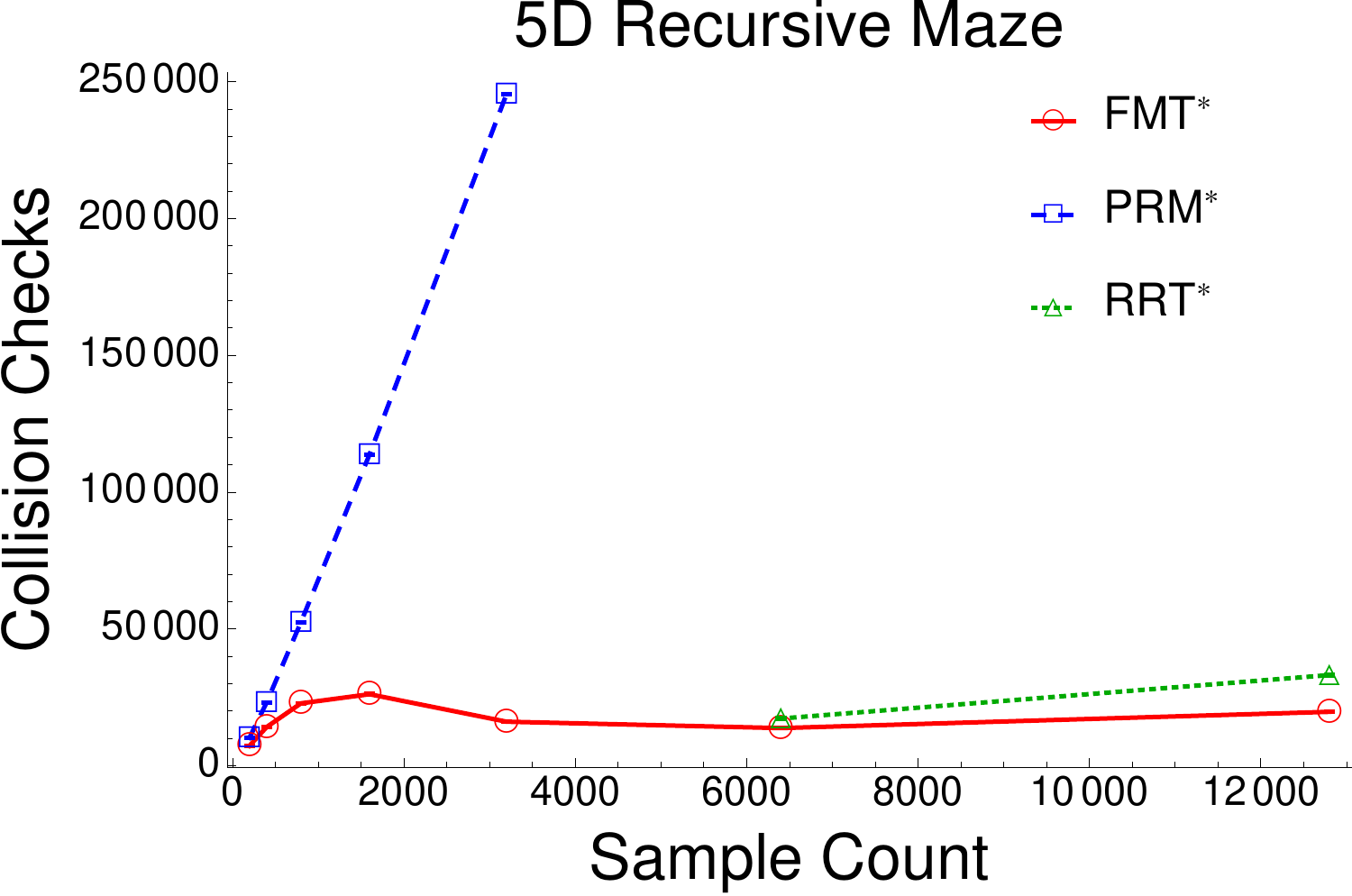}
  }
  \caption{Simulation results for a recursive maze environment in 5D.}
\label{fig:recmaze5}
\end{figure}

\begin{figure}[!t]
  \centering
  \subfigure{
    \includegraphics[width=\figWidth\textwidth]{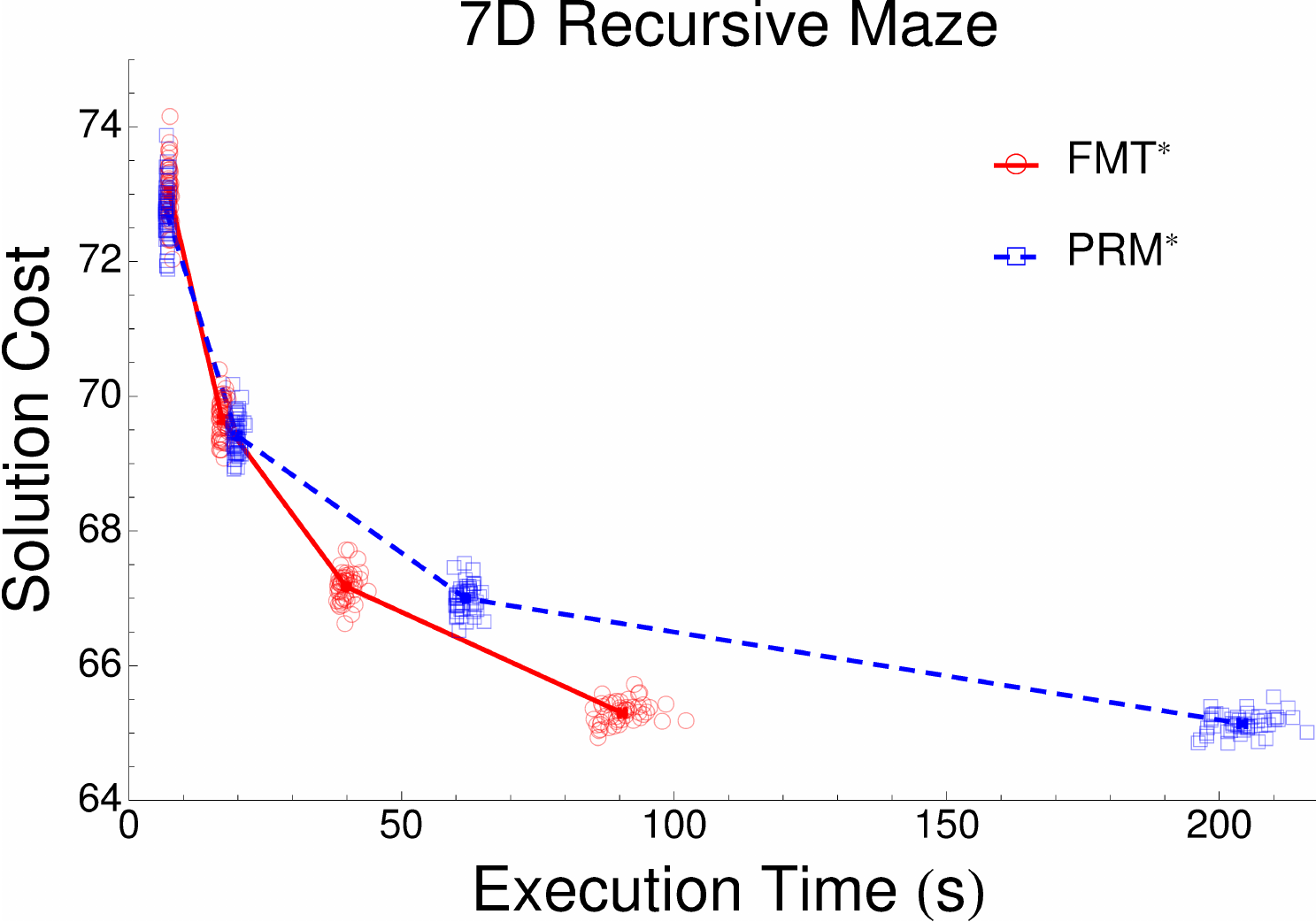}
  }
  \qquad \qquad
  \subfigure{
    \includegraphics[width=\figWidth\textwidth]{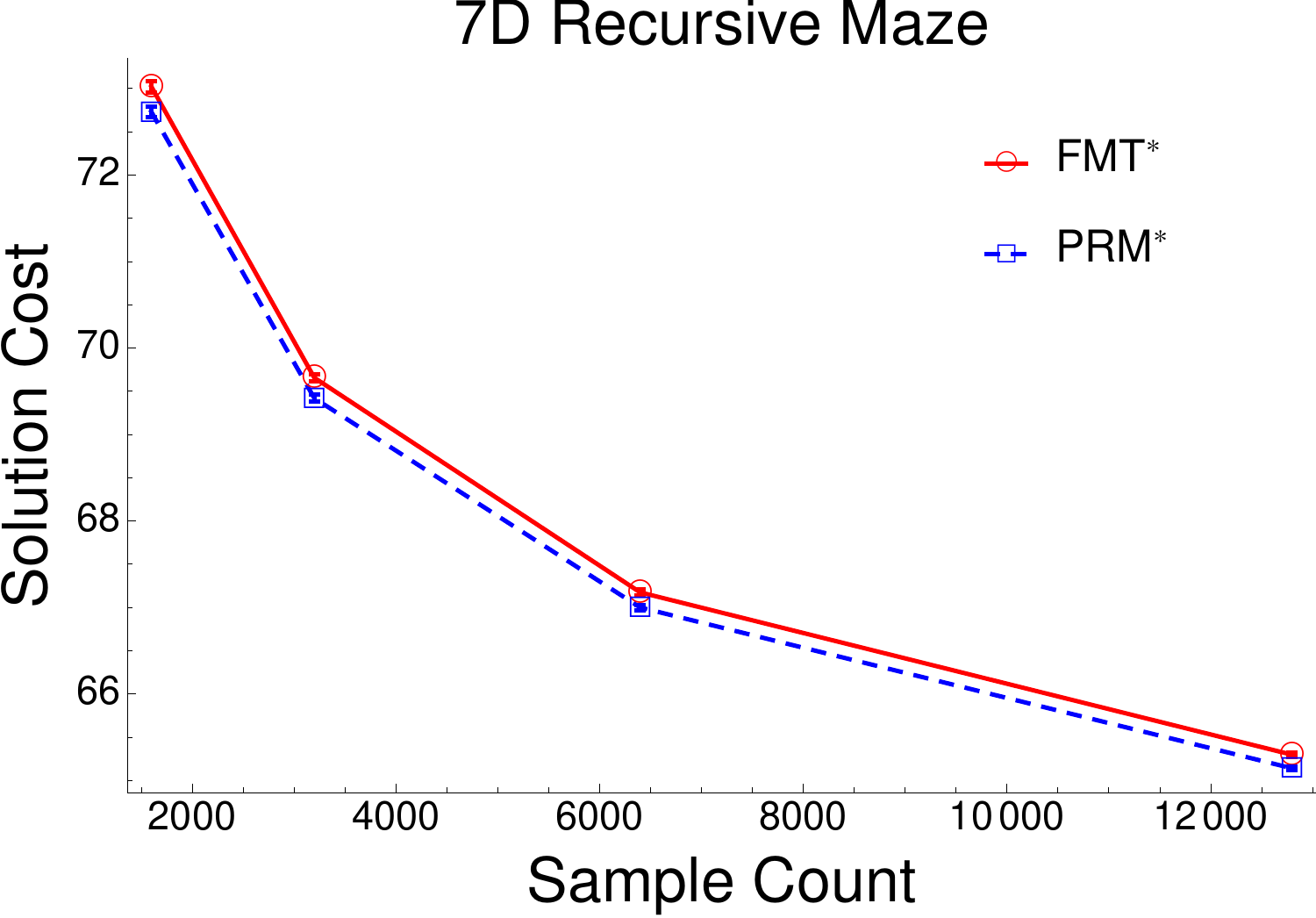}
  }
  \\
  \subfigure{
    \includegraphics[width=\figWidth\textwidth]{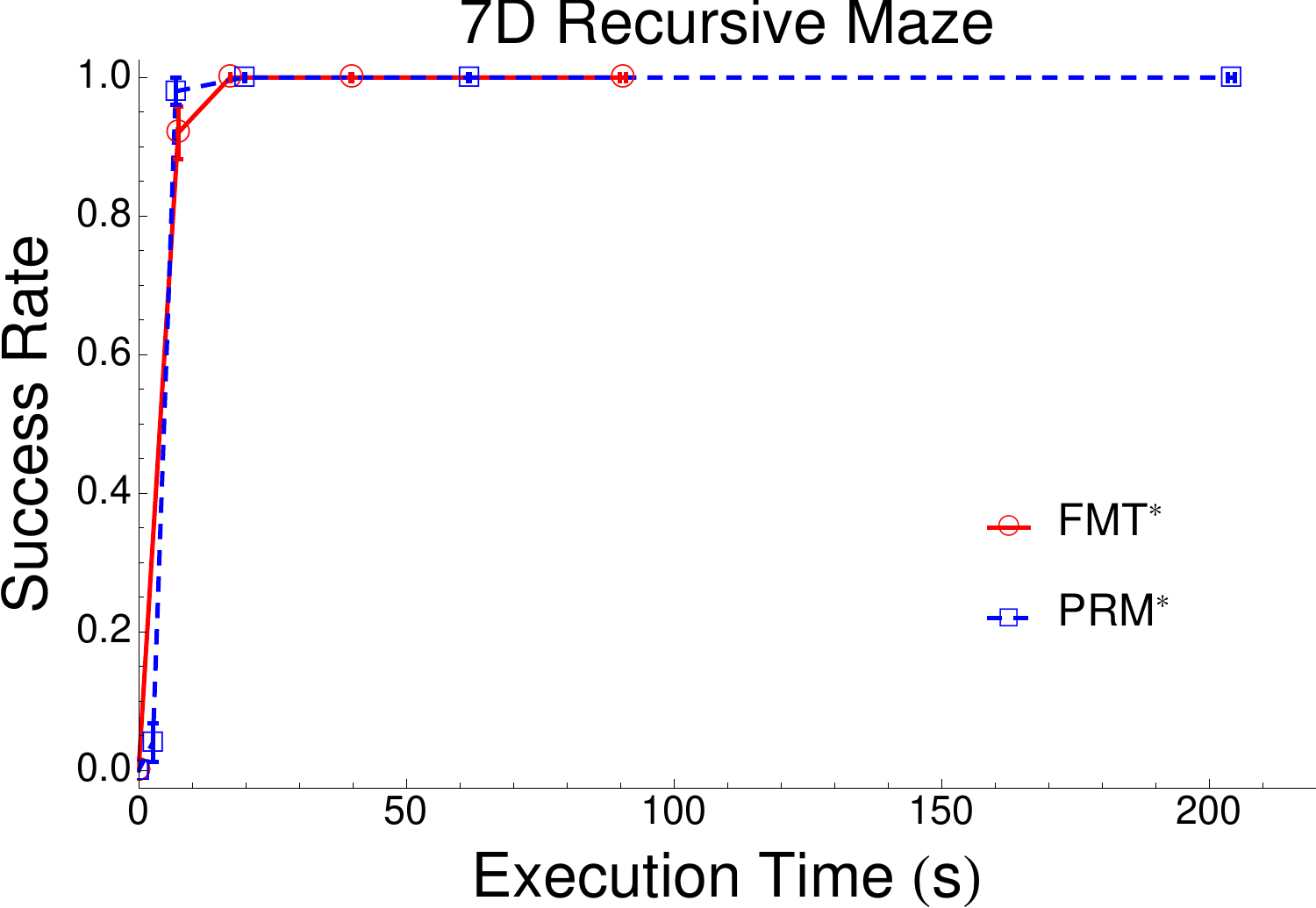}
  }
  \qquad \qquad
  \subfigure{
    \includegraphics[width=\figWidth\textwidth]{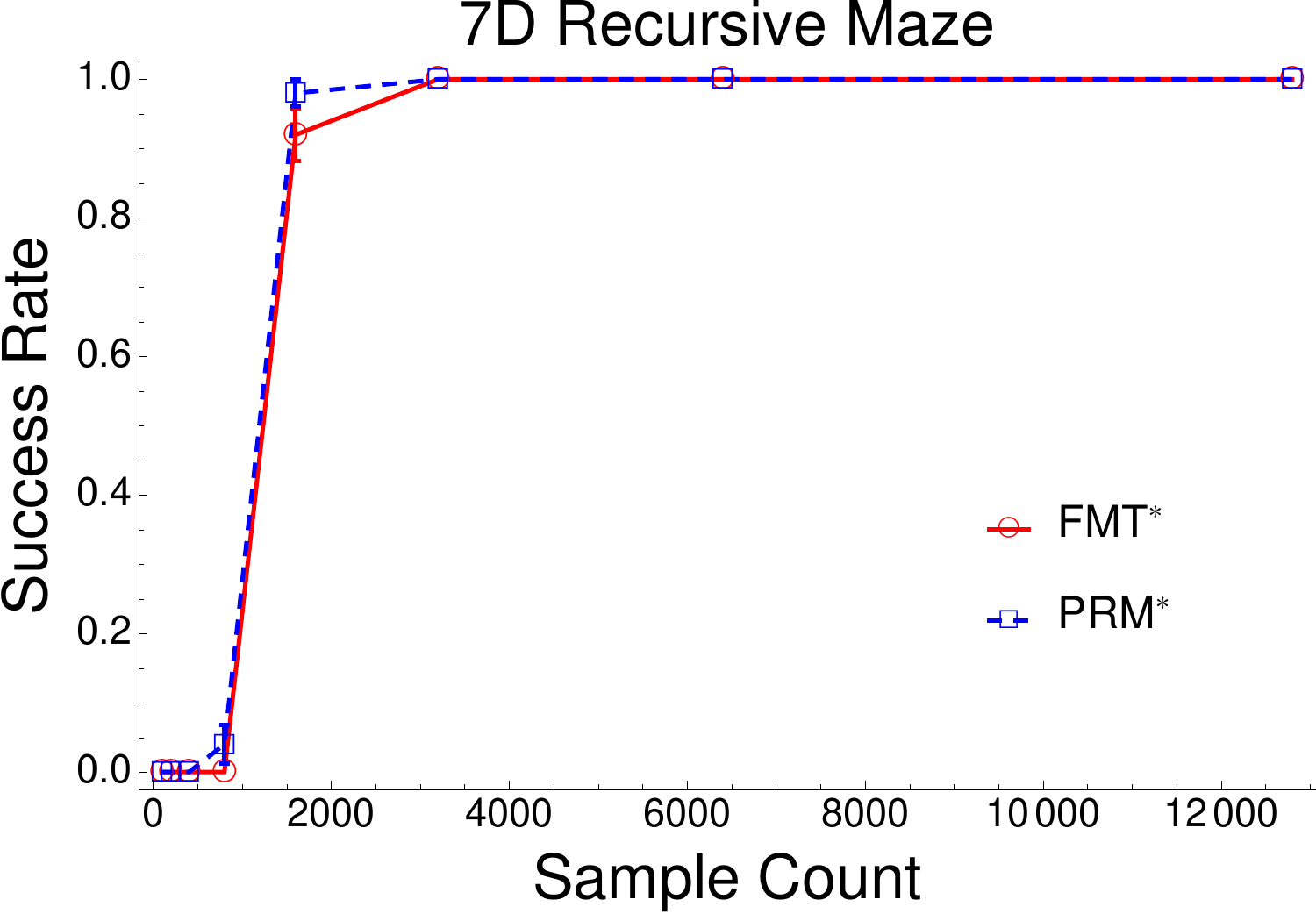}
  }
  \\
  \subfigure{
    \includegraphics[width=\figWidth\textwidth]{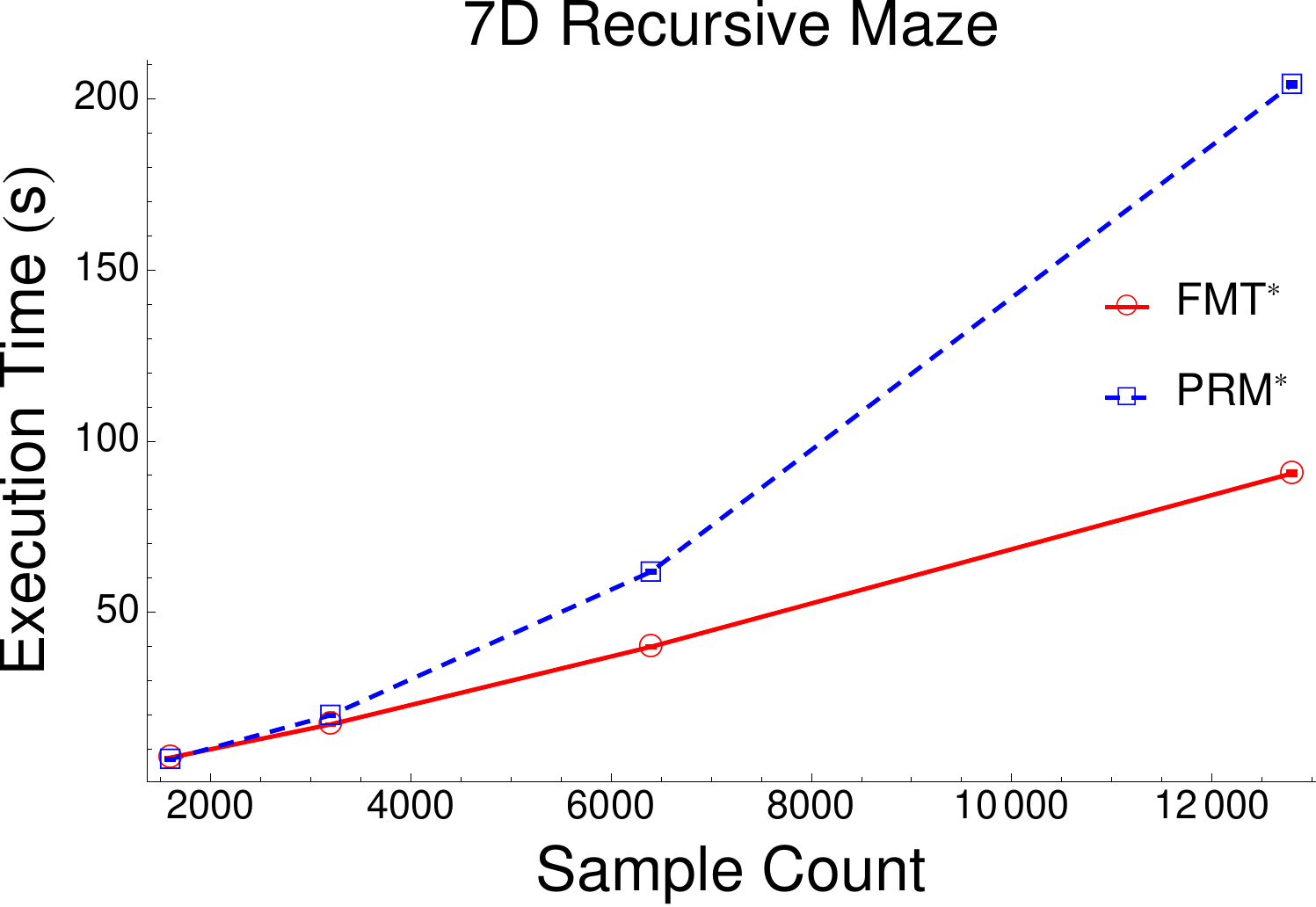}
  }
  \qquad \qquad
  \subfigure{
    \includegraphics[width=\figWidth\textwidth]{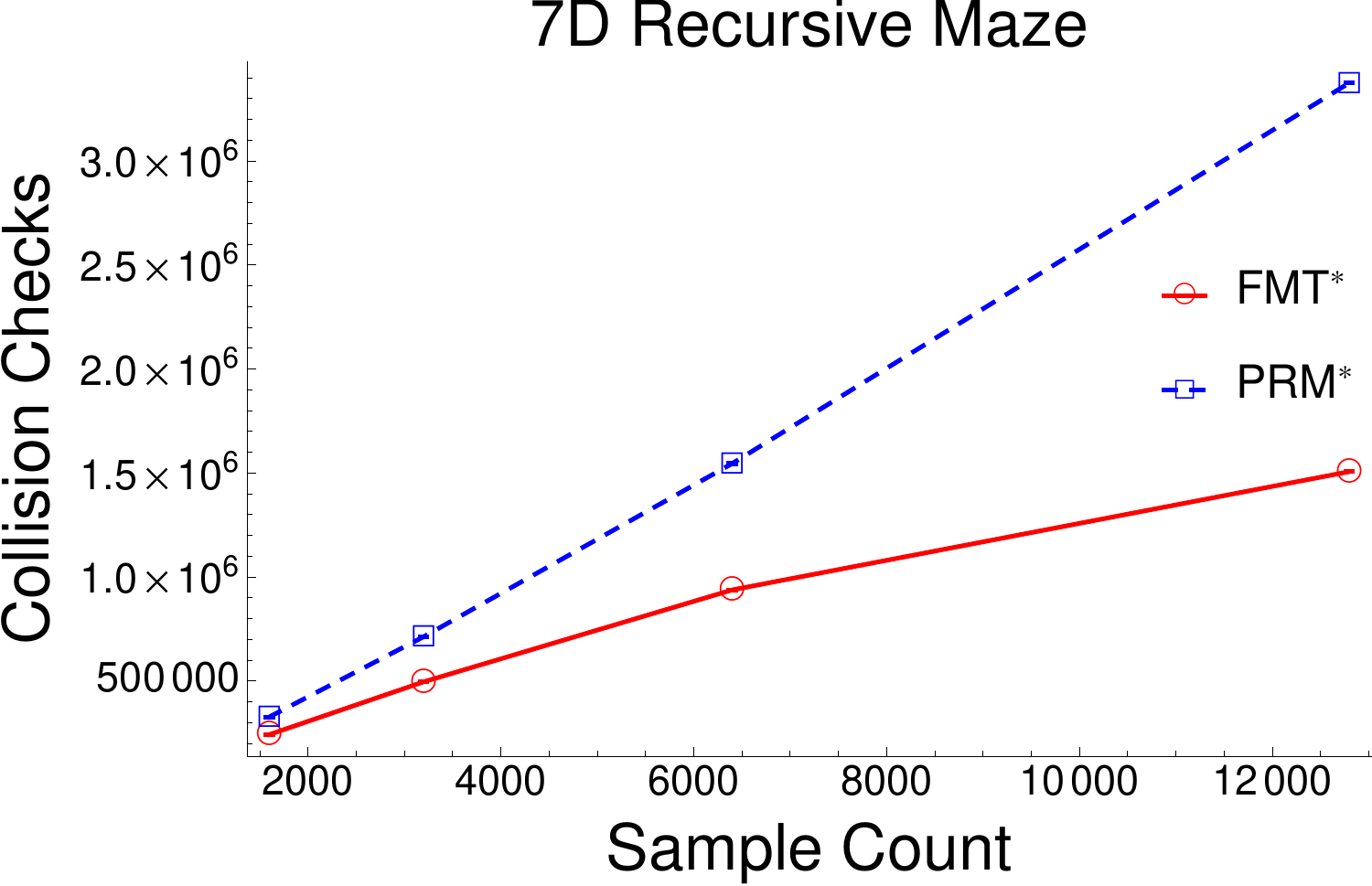}
  }
  \caption{Simulation results for a recursive maze environment in 7D.}
\label{fig:recmaze7}
\end{figure}

In order to illustrate a ``worst-case'' planning scenario in high dimensional space,
we constructed a recursive maze obstacle environment within the Euclidean unit hypercube. 
Essentially, each instance of the maze
consists of two copies of the maze in the previous dimension separated by a divider and connected through
the last dimension.
See Figure \ref{fig:recmazepics} for
the first two instances of the maze in two dimensions and three dimensions, respectively. 
This recursive nature has the effect of producing a problem environment with only one homotopy class of solutions,
any element of which is necessarily long and features sections that are spatially close, but far away from each other in terms
of their distance along the solution path. Our experiments investigated translating
a rigid body from one end of the maze to the other. 
The results of simulations in 3, 5, and 7 dimensional
recursive mazes are given in Figures \ref{fig:recmaze3},
\ref{fig:recmaze5}, and \ref{fig:recmaze7}. \FMT once again
reaches lower-cost solutions in less time than \RRTstar\!, with the
improvement increasing with dimension. The most notable trend between
\FMT and \RRTstar\!, however, is in success rate. While both algorithms
reach 100\% success rate almost instantly in 3D, \FMT reaches 100\% in
under a second, while \RRTstar takes closer to five seconds in 5D, and
most significantly \RRTstar was never able to find any solution in the
time alotted in 7D. This can be understood through the geometry of the
maze---the maze's complexity is exponentially increasing in dimension,
and in 7D, so much of free space is blocked off from every other part of
free space that \RRTstar is stuck between two bad options: it can use
a large steering radius, in which case nearly every sample fails to
connect to its nearest-neighbor and is thrown out, or it can use a
small steering radius, in which case connections are so short that the
algorithm has to figuratively crawl through the maze. Even if the steering parameter
were not an issue, the mere fact that \RRTstar operates on a steering graph-expansion
principle means that in order to traverse the maze, an ordered subsequence of $2^7$ nodes (corresponding to
each turn of the maze) must be in the sample sequence before a solution may be found. While this is an
extreme example, as the recursive maze is very complex in 7D (feasible solutions are at least 43
units long, and entirely contained in the unit cube), it accentuates
\FMT\!'s advantages in highly cluttered environments.

As compared to \PRMstar\!, \FMT still presents a substantial
improvement, but that improvement decreases with dimension. This can
be understood by noting that the two algorithms achieve nearly
identical costs for a given sample count, but \FMT is much faster due
to savings on collision-checks. However, as the plots show, the
relative decrease in collision-checks from \PRMstar to \FMT decreases
to only a factor of two once we reach 7D, and indeed we see
that, when both algorithms achieve low cost, \FMT does so in \edit{approximately}
half the time. This relative decrease in collision-checks comes from
the aforementioned extreme obstacle clutter in the configuration
space. \FMT makes big savings over \PRMstar when it connects many
samples on their first consideration, but when most samples are close
to obstacles, most samples will take multiple considerations to
finally be connected. Both algorithms achieve 100\% success rates in
\edit{approximately} the same amount of time.

\begin{figure}[!t]
  \centering
  \subfigure[]{\label{fig:alphaCvT}
    \includegraphics[width=\figWidth\textwidth]{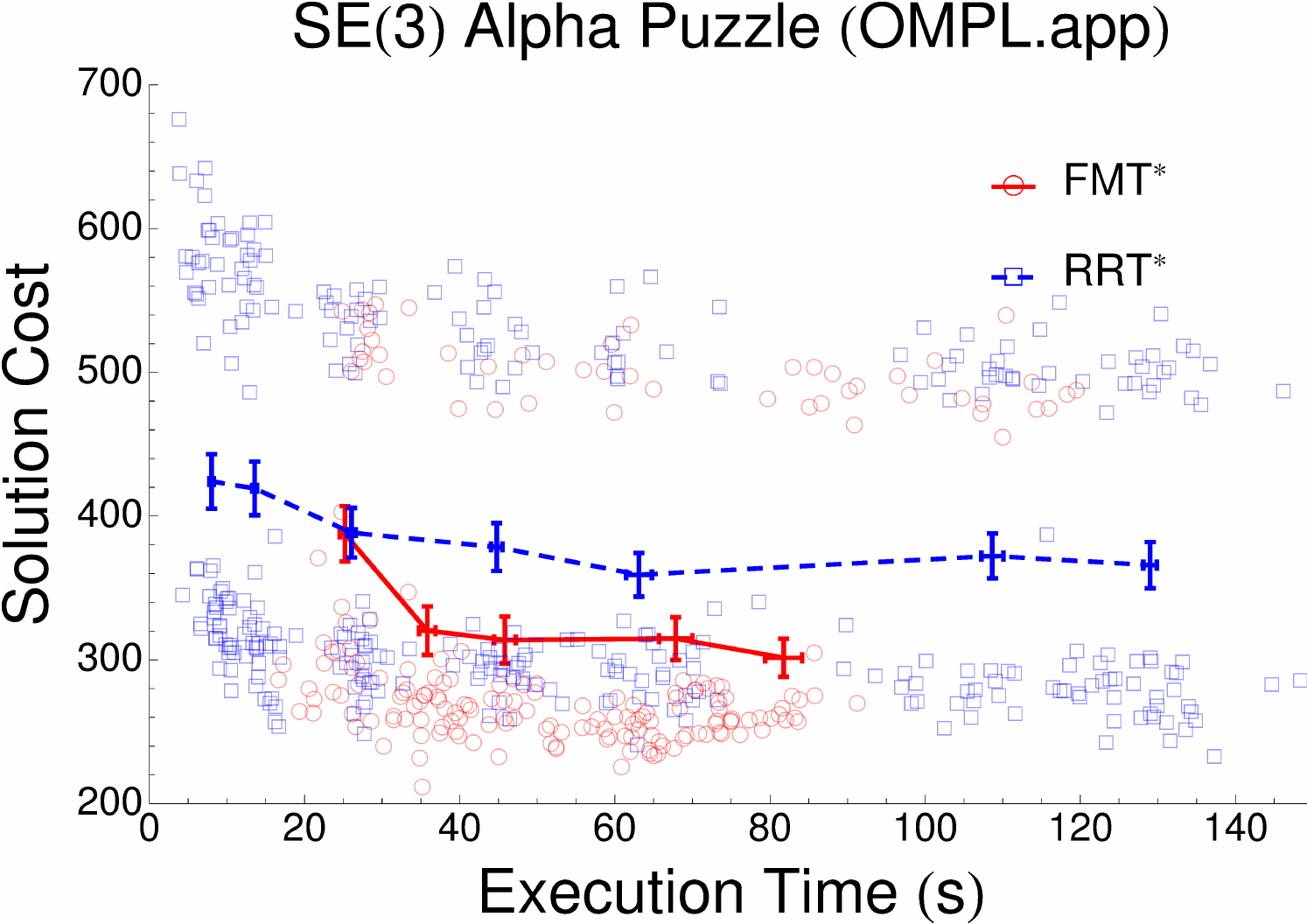}
  }
  \qquad
  \subfigure[]{\label{fig:alphaCvN}
    \includegraphics[width=\figWidth\textwidth]{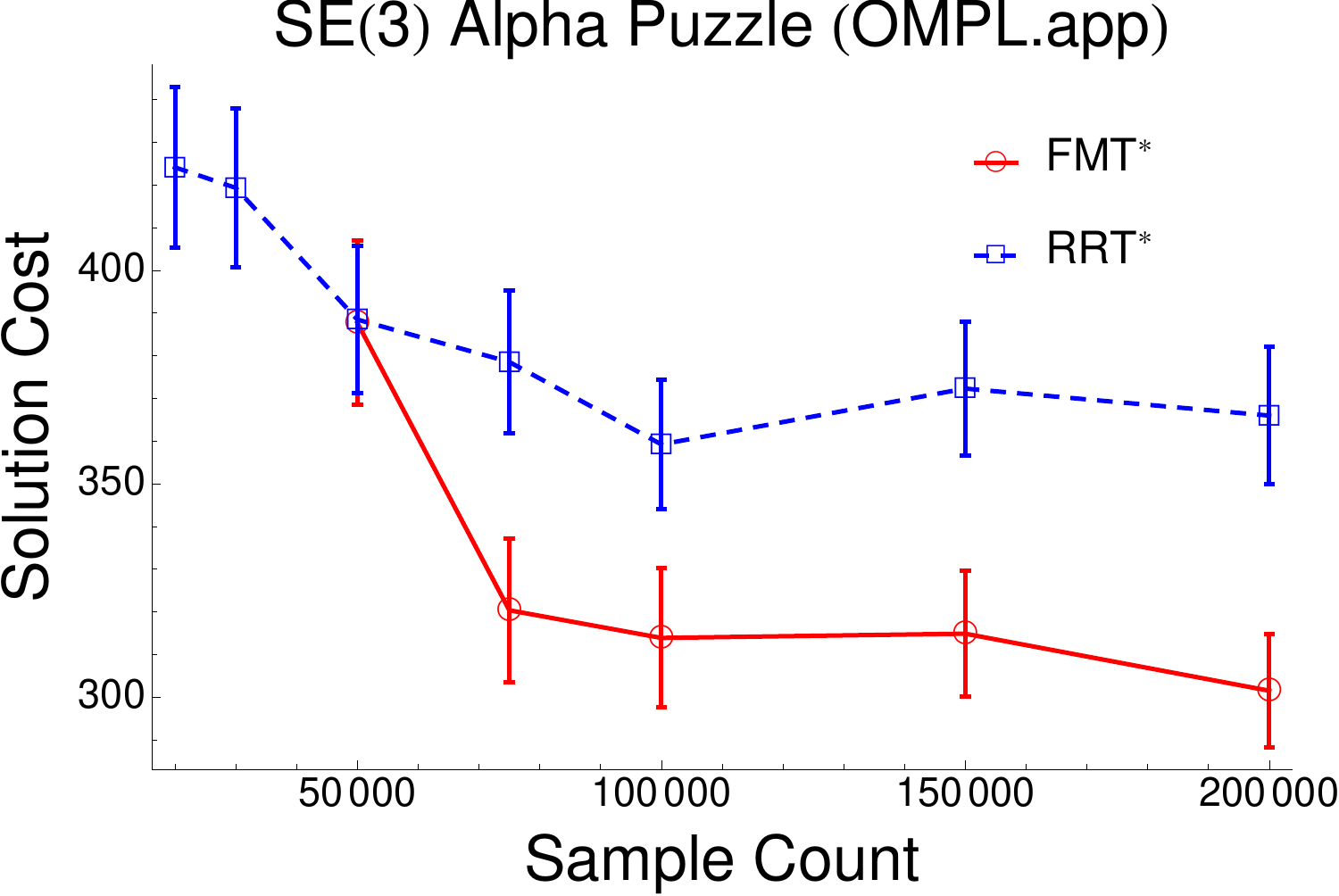}
  }
  \\
  \subfigure[]{\label{fig:alphaSRvT}
    \includegraphics[width=\figWidth\textwidth]{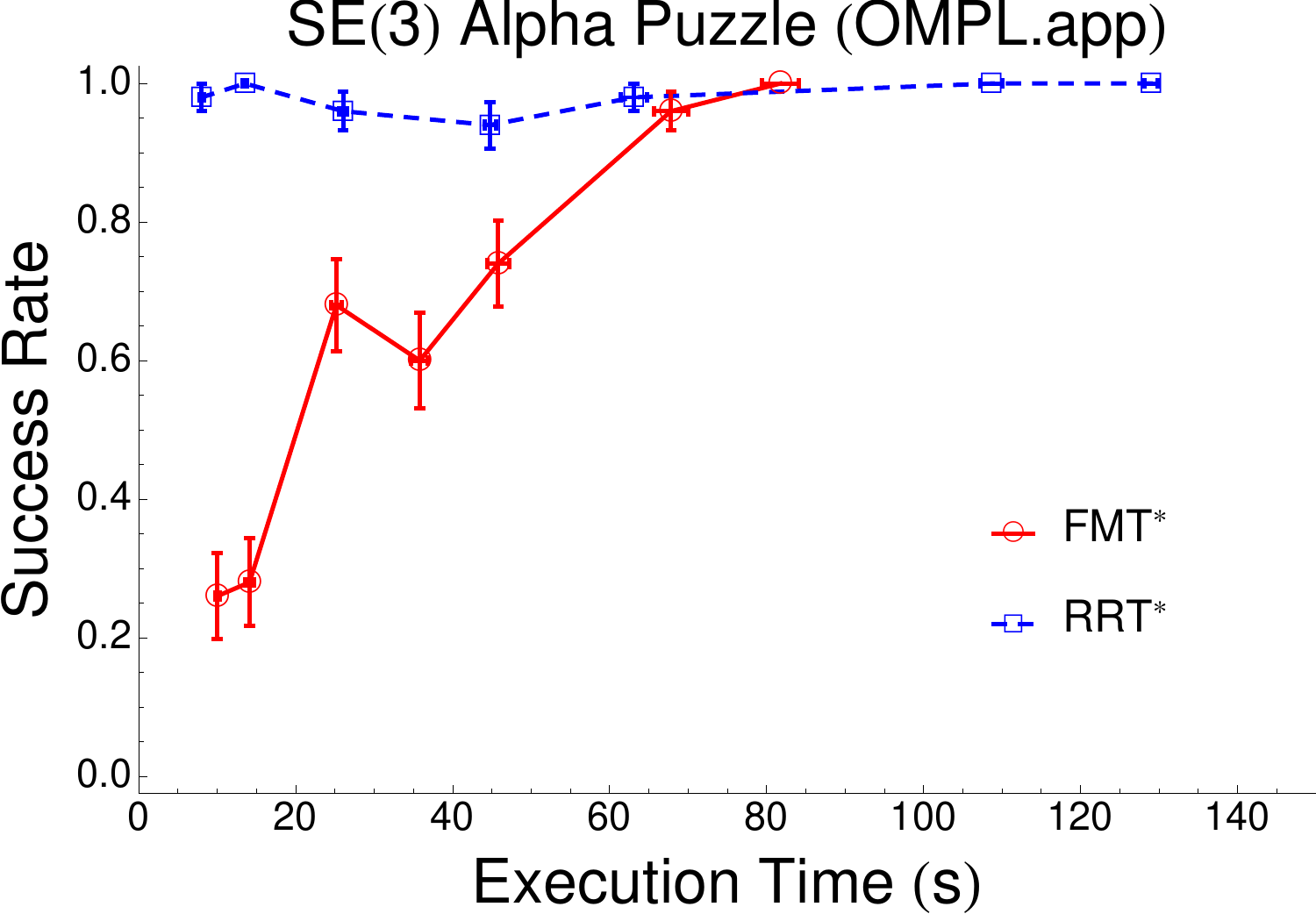}
  }
  \qquad
  \subfigure[]{
    \includegraphics[width=\figWidth\textwidth]{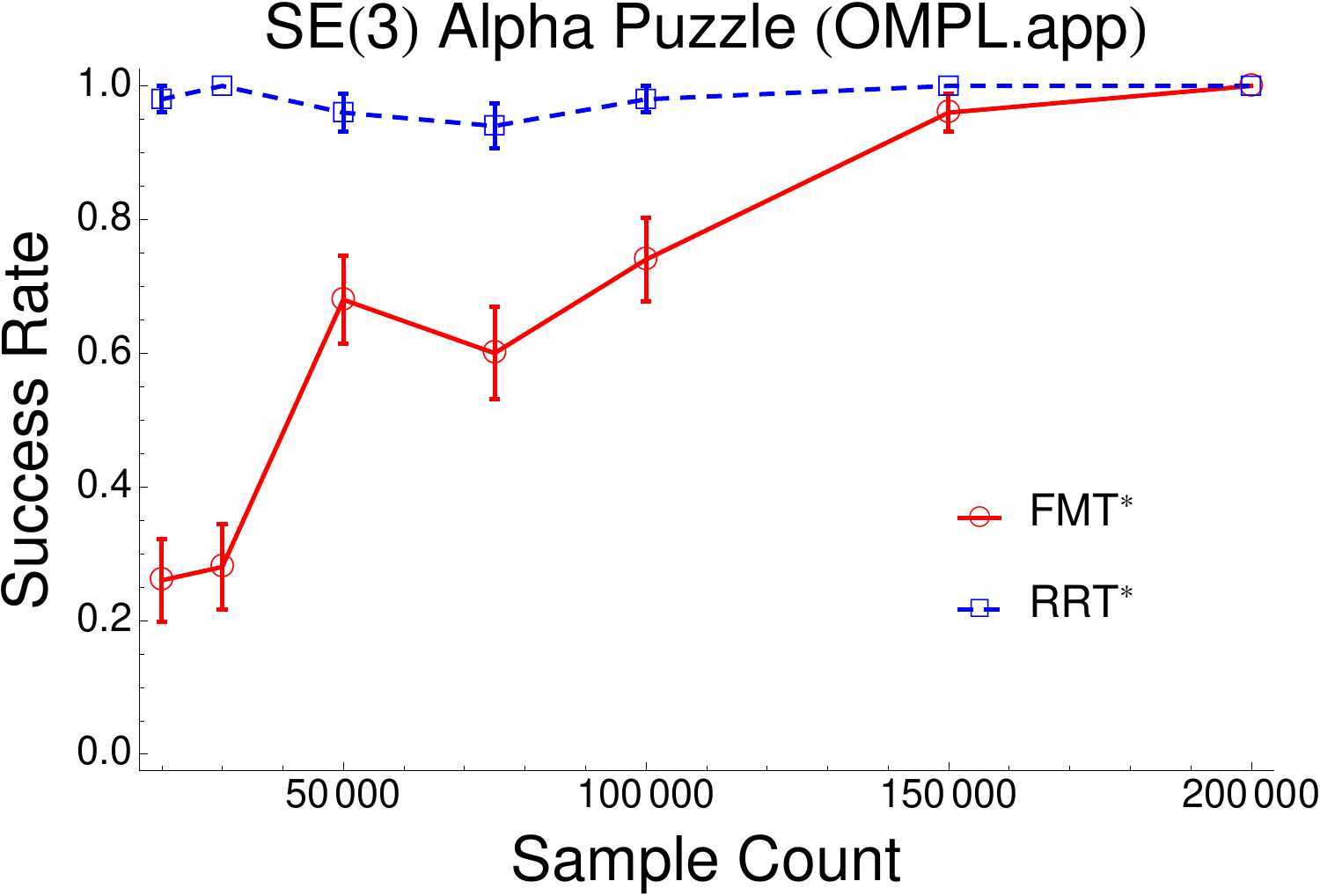}
  }
  \\
  \subfigure[]{
    \includegraphics[width=\figWidth\textwidth]{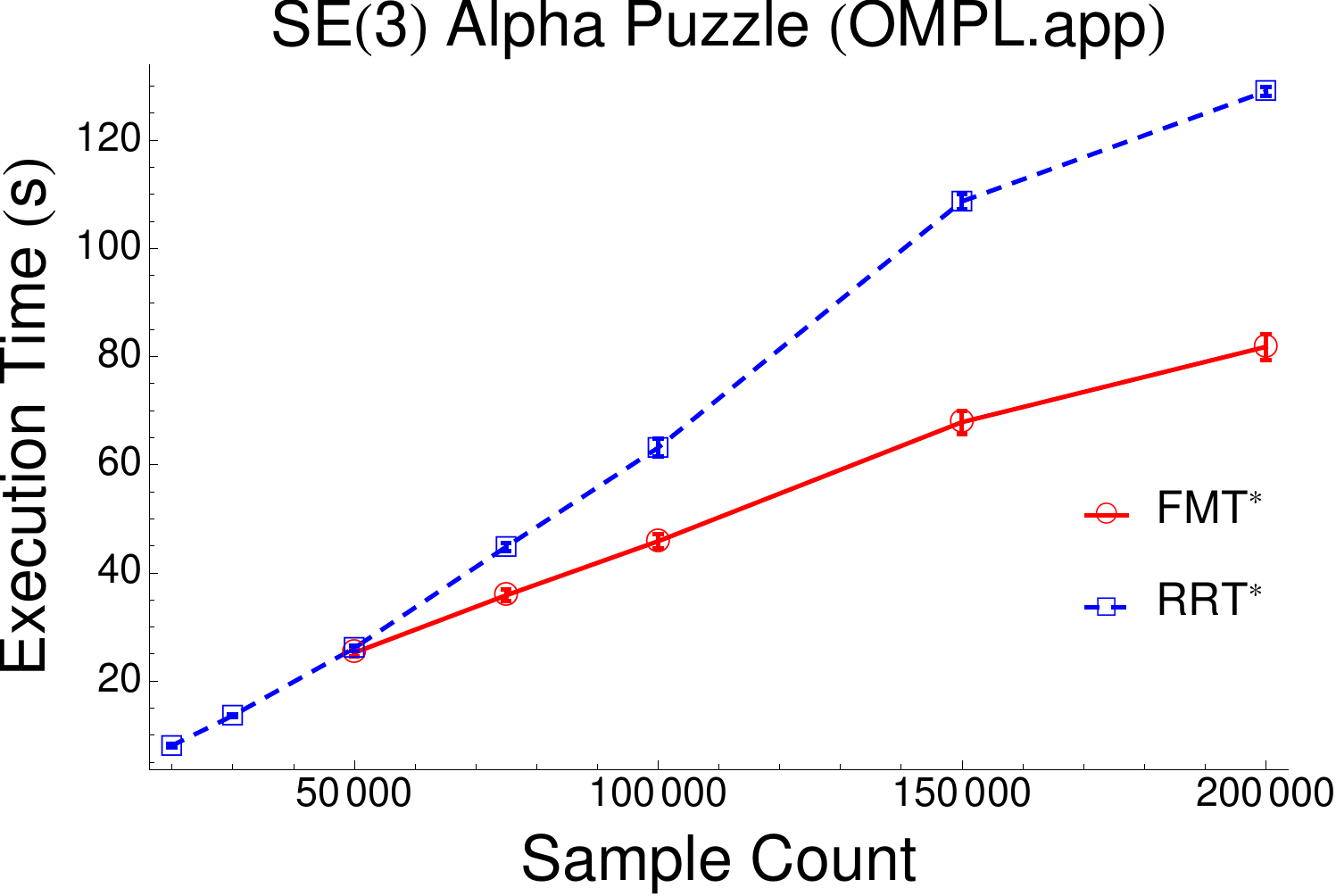}
  }
  \qquad
  \subfigure[]{
    \includegraphics[width=\figWidth\textwidth]{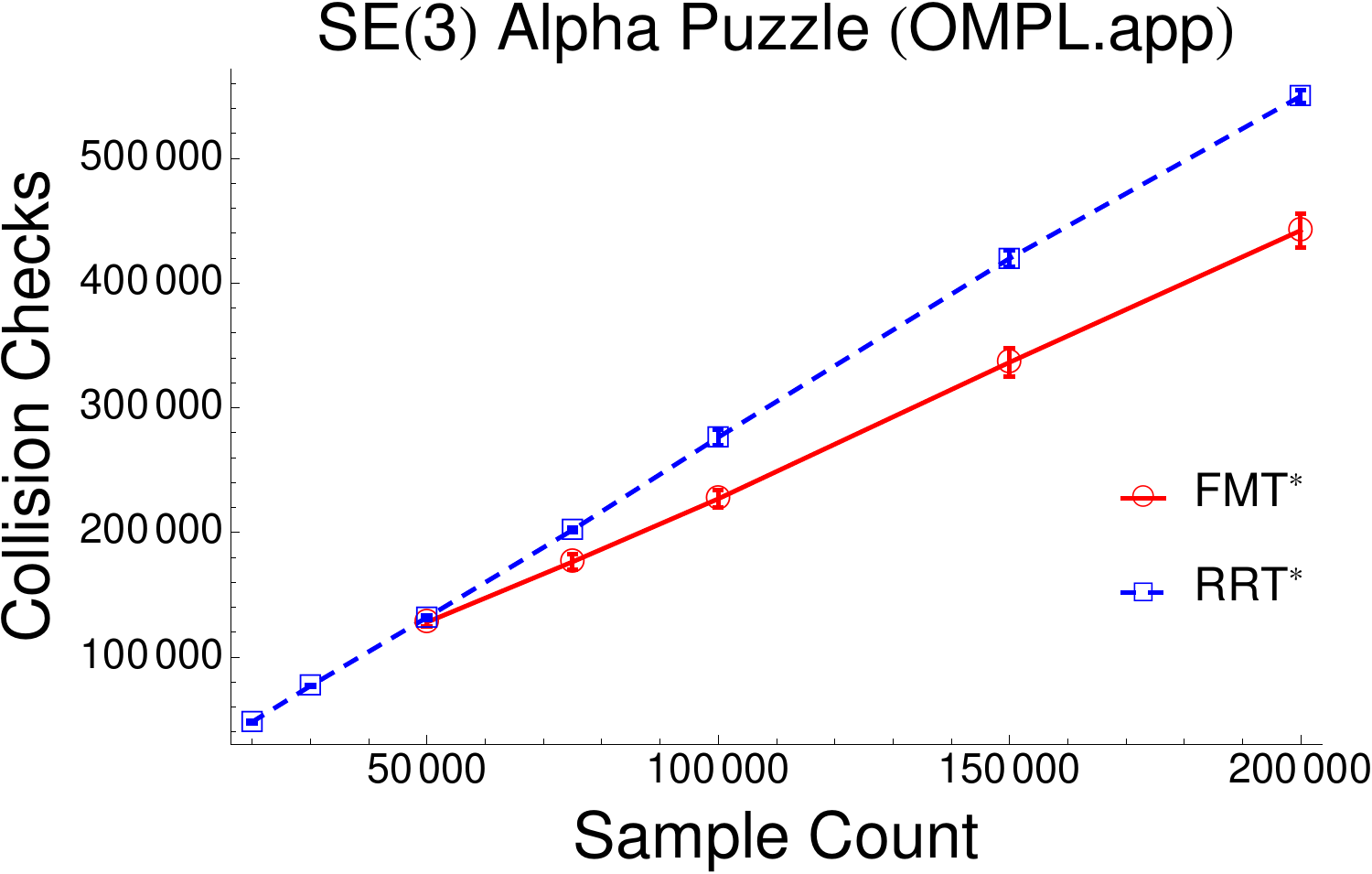}
  }
  \caption{Simulation results for a Alpha puzzle.}
  \label{fig:alpha}
\end{figure}

\subsubsection{Numerical Experiments for the $\text{SE}(3)$ Alpha Puzzle}\label{subsubsec:alpha}
Throughout our numerical evaluation of \FMT\!, we found only one planning problem where \FMT does not consistently outperform \RRTstar (\FMT outperformed \PRMstar in all of our numerical tests). The problem is the famous 
1.5 Alpha puzzle \citep{Amato.ea:98}, which consists of two tubes, each twisted in an $\alpha$ shape. The objective is to separate the intertwined tubes by a sequence of translations and rotations, which leads to extremely narrow
corridors in $\xfree$ through which the solution path must
pass (see Figure \ref{fig:alphaP}). Simulation results show that the problem presents two homotopy classes of paths (Figure \ref{fig:alphaCvT}). \FMT converges to a 100\% success rate more slowly  than \RRTstar  (Figure \ref{fig:alphaSRvT}), but when \FMT finds a solution{\color{black},} that solution tends to be in the ``right" homotopy class and of higher quality, see Figures  \ref{fig:alphaCvT} and \ref{fig:alphaCvN}. We note that in order to achieve this high success rate for \RRTstar\!, we adjusted the steering parameter to 1.5\% of the maximum extent of the configuration space{\color{black},} down from 20\%. Without this adjustment, \RRTstar was unable to find feasible solutions at the upper range of the sample counts considered.

This behavior can be intuitively explained as follows. The Alpha
puzzle presents ``narrow corridors" in $\xfree$
\citep{Amato.ea:98,Hsu.et.al:IJRR06}. When \FMT reaches their
entrance, if no sample is present in the corridors, \FMT\! stops its
expansion, while \RRTstar keeps trying to extend its branches through
the corridors{\color{black},} which explains its higher success rates at low sample counts. On the other hand, at high sample counts, samples are placed in the corridors with high probability, and when this happens the optimal (as opposed to greedy) way by which \FMT grows the tree usually leads to the discovery of a better homotopy class  and of a higher quality solution within it (Figure \ref{fig:alphaCvT}, execution times larger than $\sim 25$ seconds). As a result, \RRTstar outperforms \FMT for short execution times, while \FMT outperforms \RRTstar  in the complementary case. Finally, we note that the extremely narrow but short corridors in the Alpha puzzle present a different challenge to these algorithms than the directional corridor of the $\text{SE}(2)$ bug trap. As discussed in Section \ref{SE2bug}, the ordering of sampled points along the exit matters for \RRTstar in the bug trap configuration, while for the Alpha puzzle the fact that there are no bug-trap-like ``dead ends'' to present false steering connections means that a less intricate sequence of nodes is required for success.

On the one hand, allowing \FMT to sample new points around the leaves
of its tree whenever it fails to find a solution (i.e., when $\Hset$
becomes empty) might substantially improve its performance in the
presence of extremely narrow corridors. In a sense, such a
modification would introduce a notion of ``anytimeness"  and adaptive
sampling into \FMT\!, which would effectively leverage the steadily
outward direction by which the tree is constructed (see
\citep{JG-ES-SS-TB:14} for a conceptually related idea). This is a
promising area of future research (note that the theoretical
foundations for non-uniform sampling strategies are provided in
Section \ref{subsec:nonunif}). On the other hand, planning problems
with extremely narrow passages, such as the Alpha puzzle, do not
usually arise in robotics applications as, fortunately, they tend to
be \emph{expansive}, i.e., they enjoy ``good" visibility properties
\citep{Hsu.et.al:IJRR06}. Collectively, these considerations suggest
the superior performance of \FMT in {\color{black}most} practical settings.

\subsection{In-Depth Study of \FMT}\label{subsec:FMTStudy}

\subsubsection{Comparison Between \FMT and k{\color{black}-Nearest} \FMT}\label{subsubsec:r_v_k}
{\color{black} Since we are now comparing both versions of \FMT\!, we
  will explicitly} use radial-\FMT to denote {\color{black}the version
  of \FMT that uses a fixed Euclidean distance to determine neighbors,} {\color{black} and return to referring to \kFMT by its full name
throughout this section}. For this set of simulations, given in Figure
\ref{fig:kFMTvrFMT}, the formula for $k_n$ is still the same as in the
rest of the simulations, and for comparison, the radius $r_n$ of the
radial-\FMT implementation is chosen so that the expected number of
samples in a collision-free $r_n$-ball is exactly equal to
$k_n$. Finally, as a caveat, we point out that since
$k$-nearest-neighborhoods are fundamentally different from
$r$-radius-neighborhoods, the two algorithms depicted now use
\emph{different} primitive procedures. Since computing neighbors in
both algorithms takes a substantial fraction of the runtime, the cost
versus time plots should be interpreted with caution, since the
algorithms' relative runtimes could potentially change significantly
with a better implementation of one or both neighbor-finding primitive
procedures. With that said, we focus our attention more on the number
of collision-checks as a proxy for algorithm speed. Since this problem
has a relatively simple collision-checking module, we may expect that
for more complex problems in which collision-checking dominates
runtime, the number of collision-checks should approximate runtime well.

While the number of collision-checks in free space is the same between
the two algorithms, since all samples connect when they
are first considered, some interesting behavior is exhibited in the
same plot for the 5D maze. In particular, the number of collision
checks for {\color{black}\kFMT} increases quickly with sample count, then decreases
again and starts to grow more like the linear curve for
radial-\FMT\!. This hump in the curve corresponds to when the usual
connection distance for {\color{black}\kFMT} is greater than the width of the maze
wall, meaning that for many of the points, some of their
$k_n$-nearest-neighbors will be much farther along in the maze. Thus
{\color{black}\kFMT} tries to connect them to the tree, and fails because there is a
wall in between. The same problem doesn't occur for radial-\FMT
because its radius stays smaller than the width of the
maze wall. This is symptomatic of an advantage and disadvantage of
{\color{black}\kFMT\!}, namely that for samples near obstacles, connections may be
attempted to farther-away samples. This is an advantage because for
a point near an obstacle, there is locally less density around the point
and thus fewer nearby options for connection, making it harder for
radial-\FMT to find a connection, let alone a good one. For small sample sizes
relative to dimension however, this can cause a lot of extra
collision-checks by, as just described, having {\color{black}\kFMT} attempt
connections across walls. As this disadvantage goes away with enough
points, we still find that{\color{black}, although the difference in free space is very small,}
$k$-nearest \FMT outperforms radial-\FMT in both {\color{black}of the} settings shown, as the relative
advantage of {\color{black}\kFMT} in solution cost per sample is greater than the
relative disadvantage in number of collision-checks per sample.

\begin{figure}[!t]
  \centering
  \subfigure{
    \includegraphics[width=0.45\textwidth]{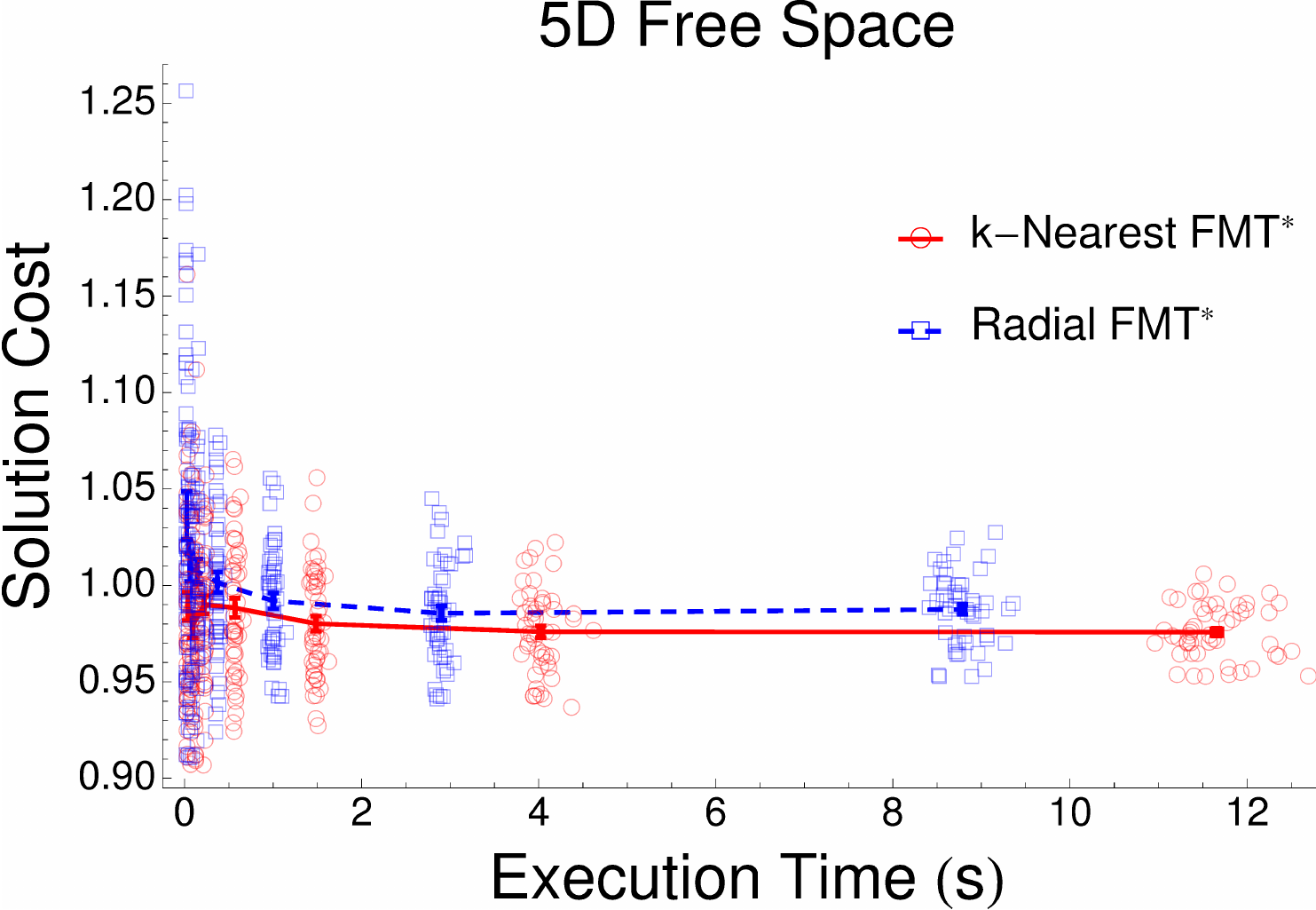}
  }
  \qquad
  \subfigure{
    \includegraphics[width=0.45\textwidth]{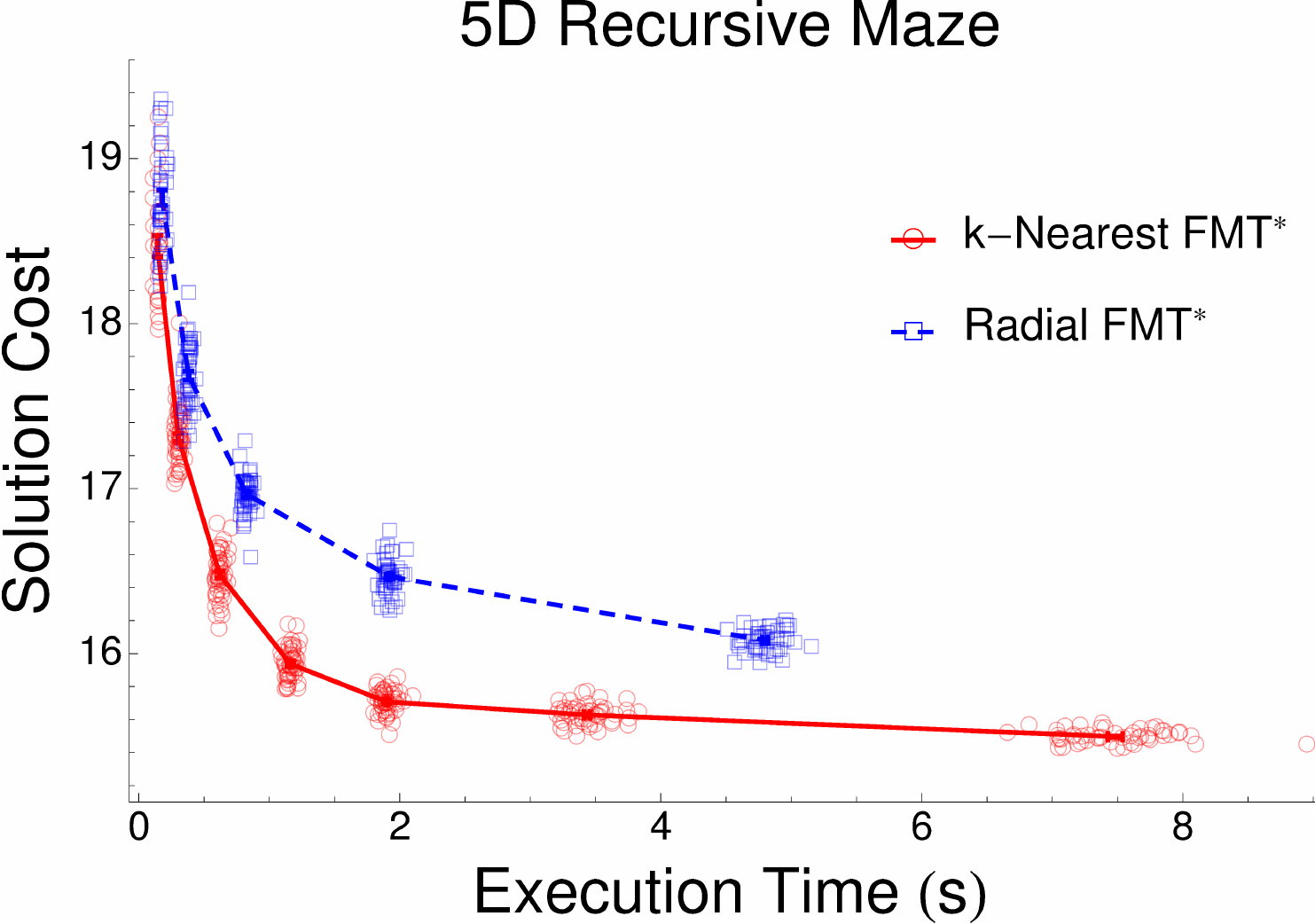}
  }
  \\
  \subfigure{
    \includegraphics[width=0.45\textwidth]{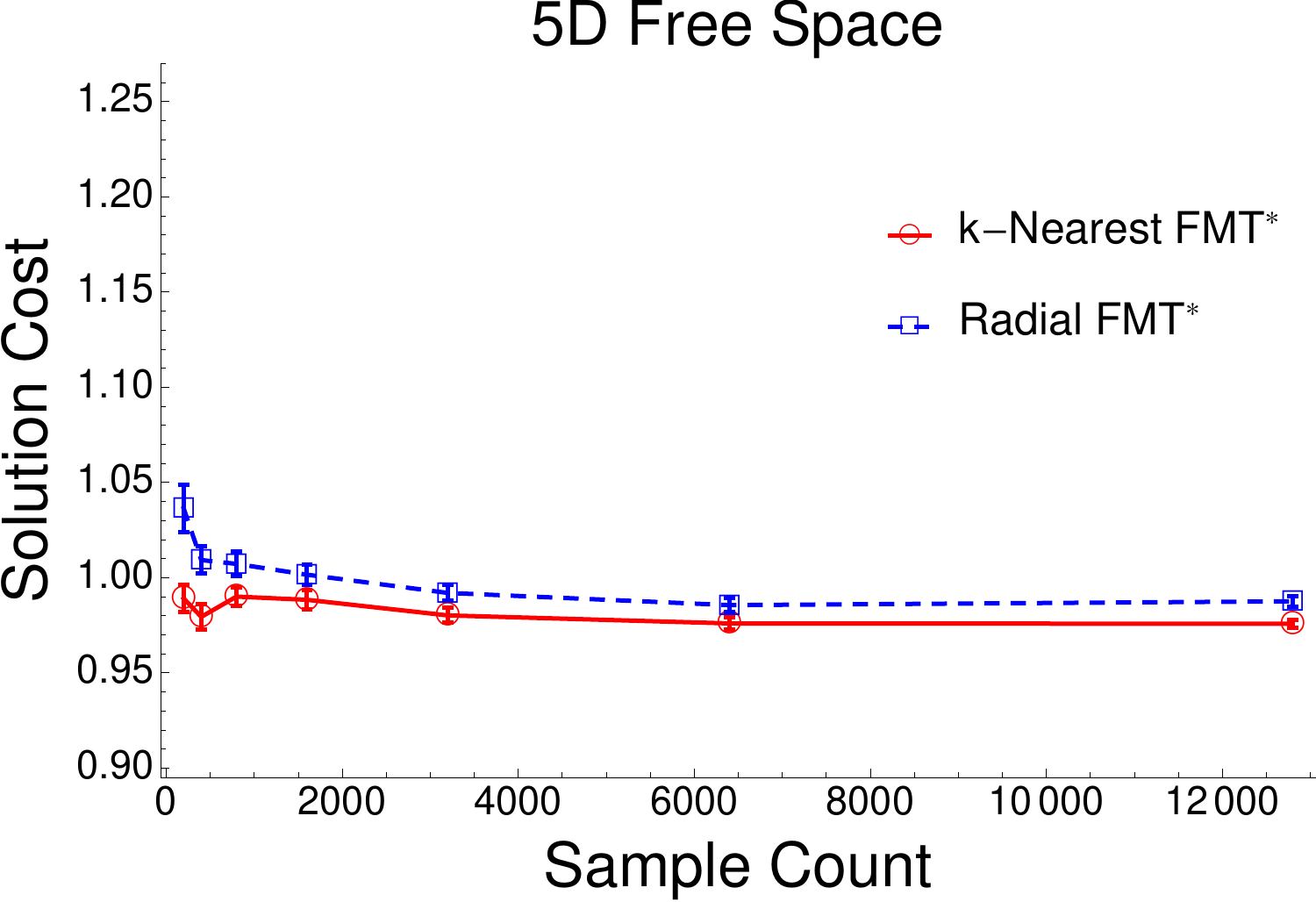}
  }
  \qquad
  \subfigure{
    \includegraphics[width=0.45\textwidth]{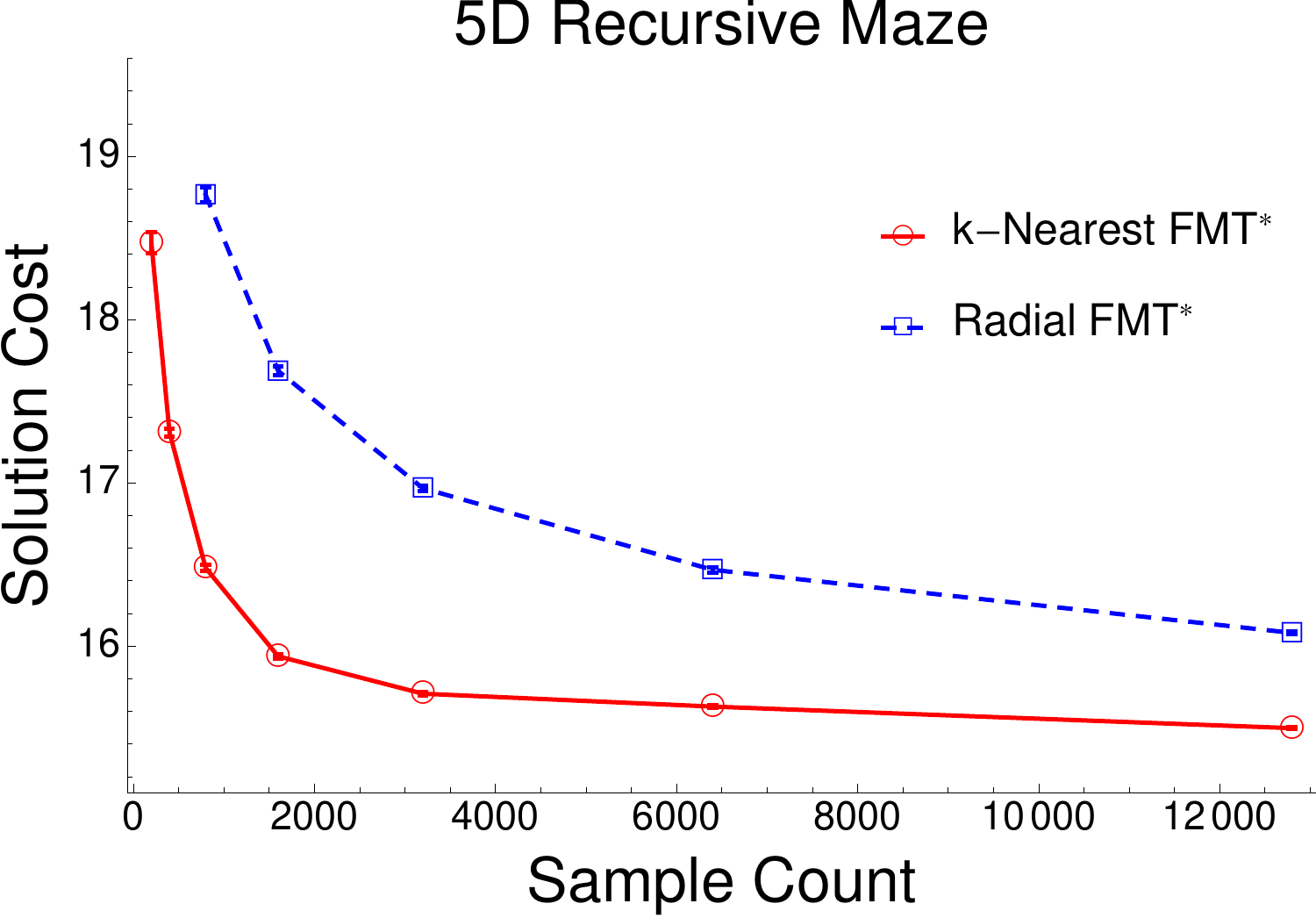}
  }
  \\
  \subfigure{
    \includegraphics[width=0.45\textwidth]{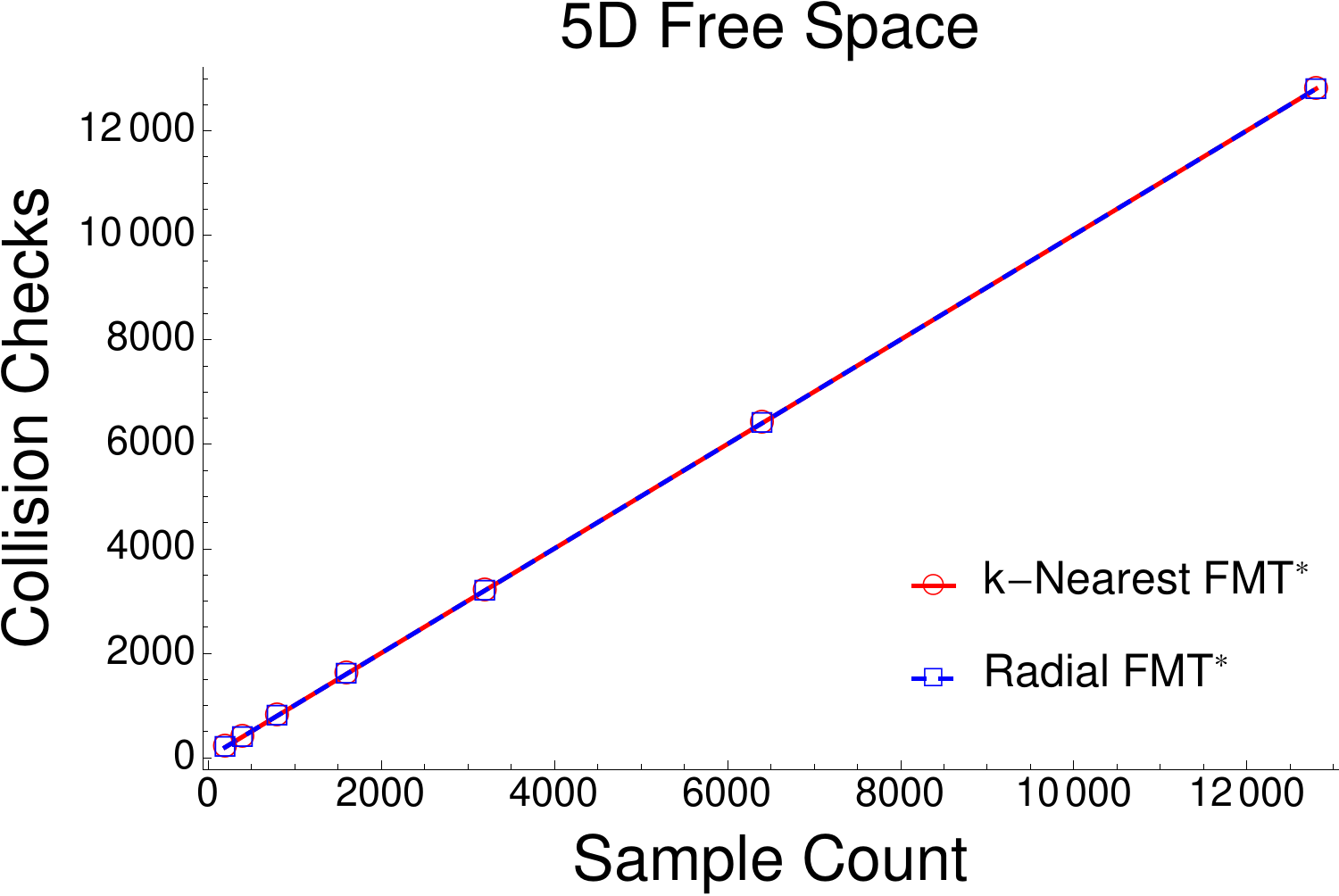}
  }
  \qquad
  \subfigure{
    \includegraphics[width=0.45\textwidth]{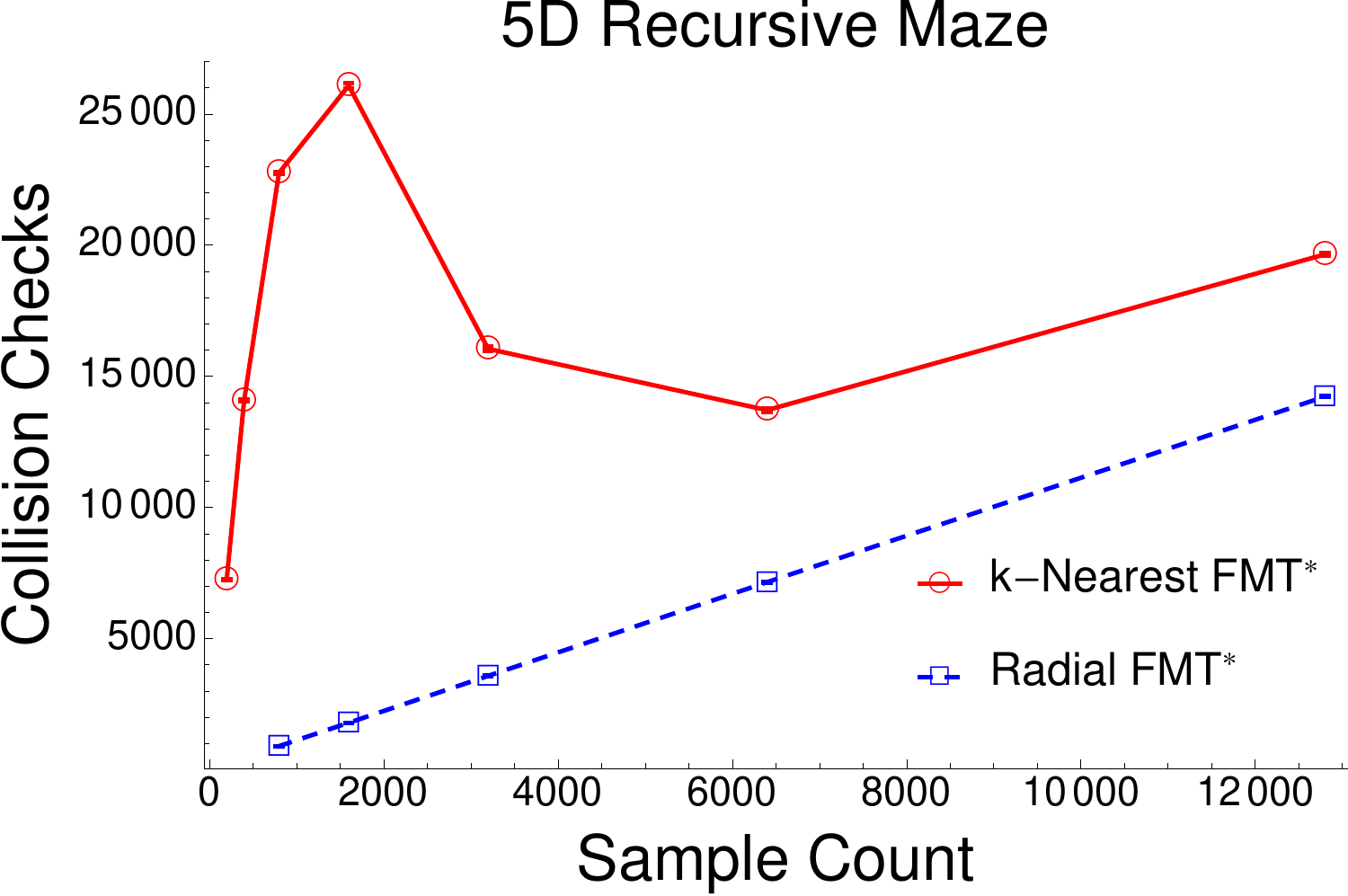}
  }
  \caption{Comparison between raidal-\FMT and {\color{black}\kFMT}.}
\label{fig:kFMTvrFMT}
\end{figure}

\subsubsection{Tuning the Radius Scale Factor}\label{subsubsec:rmtuning}
The choice of
tuning parameters is a challenging and pervasive problem in the sampling-based
motion planning literature. Throughout these numerical experiments, we
have used the same neighbor scaling factor, which we found empirically to work
well across a range of scenarios. In this section, we try to
understand the relationship of {\color{black}\kFMT} with this neighbor scaling
parameter, in the example of the  $\text{SE}(3)$ maze. The results of
running {\color{black}\kFMT} with a range 
of tuning parameters on this problem are shown in Figure
\ref{fig:tuning}. The values in the legend correspond to
{\color{black} a connection radius multiplier (RM)} of $k_{0,\text{\FMT}}$ as defined at the beginning of Section
\ref{sec:sims}, i.e., a value of {\color{black}$\text{RM} = 1$} corresponds to using exactly
$k_{0,\text{\FMT}}${\color{black}, and a value of $\text{RM} = 2$ corresponds to using $k_0 = 2^d\cdot k_{0,\text{\FMT}}$.} 
We point out that to reduce clutter, we have
omitted error bars from the plot, but note that they are small
compared to the differences between the curves.

This graph 
clearly shows the tradeoff in the scaling factor, namely that for small
values, {\color{black}\kFMT} rapidly reaches a fixed solution quality and then plateaus, while for larger values, the solution takes a while to reach 
lower costs, but continues to show improvement for longer, eventually
beating the solutions for small values. The fact that most of these
curves cross one another tells us that the choice of this tuning
parameter depends on available time and corresponding sample count. For this experimental setup, and for the other problems we tried, there appears
to be a sweet spot
around the value $\text{RM} = 1$. Indeed, this motivated our choice of
$k_{0,\text{\FMT}}$ in our simulations. We note that the curves for 0.7
through 0.9 start out at lower costs for very small execution times,
and it appears that the curve for 1.1 is going to start to
return better solutions than 1.0 before 35 seconds. That is, depending on the time/sample allowance,
there are at least four regimes in which different scaling factors outperform the others. For a different problem instance,
the optimal scaling profile may change, and for best performance some amount of manual tuning will be required. We note, however, that $\text{RM} = 1$ is never too far from the best in Figure~\ref{fig:tuning}, and should represent a safe default choice.

\begin{figure}[!t]
  \centering
  \subfigure{
    \includegraphics[width=0.75\textwidth]{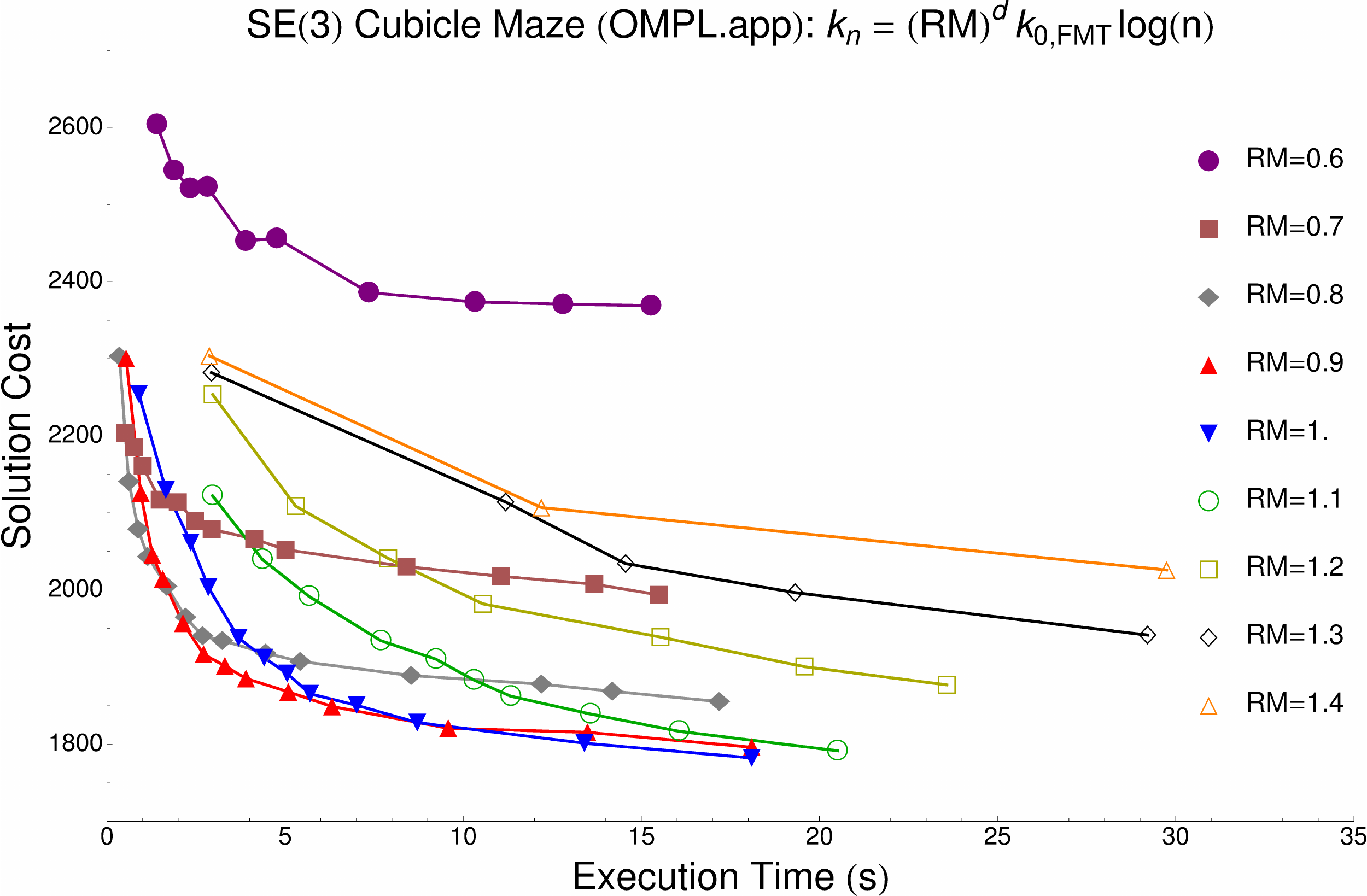}
  }\caption{Performance of {\color{black}\kFMT} for different values of the
    {\color{black} neighbor scaling parameter}.}
\label{fig:tuning}
\end{figure}

\subsubsection{Improvement on Convergence Rate with Simple Heuristics}
In any path planning problem, the optimal path tends to be quite
smooth, with only a few non-differentiable points. However, sampling-based
algorithms all locally-connect points with straight lines, resulting
in some level of {\color{black}``}jaggedness{\color{black}''} in the returned paths. A popular post-processing heuristic
for mitigating this problem is the \texttt{ADAPTIVE-SHORTCUT} smoothing heuristic described in \citep{DH:00}. In Figure \ref{fig:smoothing}, we show the effect of
applying the \texttt{ADAPTIVE-SHORTCUT} heuristic to {\color{black}\kFMT} solutions for the 5D recursive maze. We use a point
robot for this simulation as it allowed us to easily compute the true
optimal cost, and thus better place the improvement from the heuristic
in context. The improvement is substantial, and we see that we can obtain a solution
within 10\% of the optimal with fewer than 1000 samples in this
complicated 5D environment. Figure~\ref{fig:smoothing} also displays the fact that
adding the \texttt{ADAPTIVE-SHORTCUT} heuristic  only barely increases the number of
collision-checks. We place sample count on the $x$-axis because it is
more absolute than time{\color{black},} which is more system-dependent, and because the  \texttt{ADAPTIVE-SHORTCUT} heuristic
runs so \edit{quickly} compared to the overall algorithm that
sample count is able to act as an accurate proxy for time across the
two implementations of {\color{black}\kFMT}.

\begin{figure}[!t]
  \centering
  \subfigure{
    \includegraphics[width=0.45\textwidth]{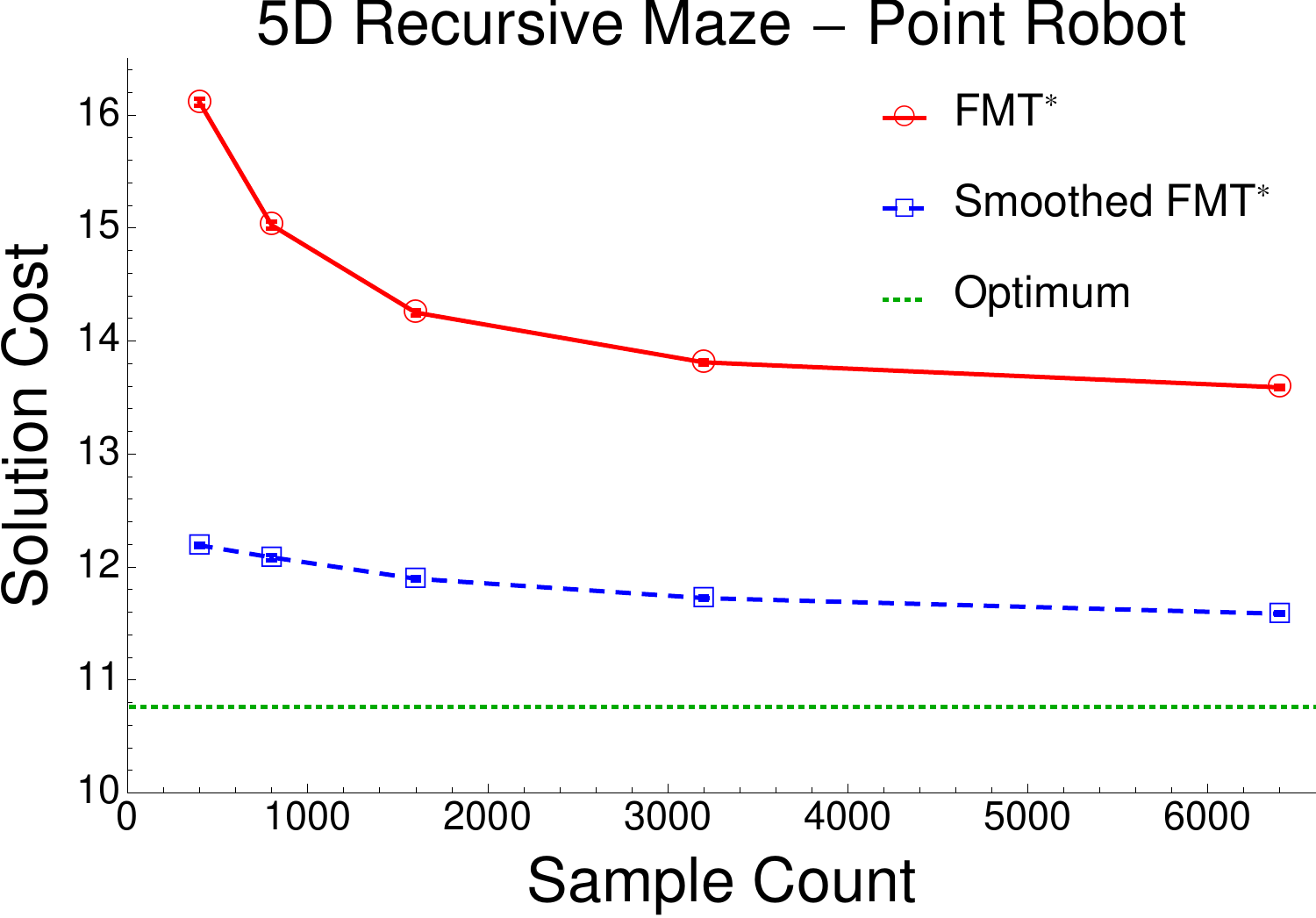}
  }
  \qquad
  \subfigure{
    \includegraphics[width=0.45\textwidth]{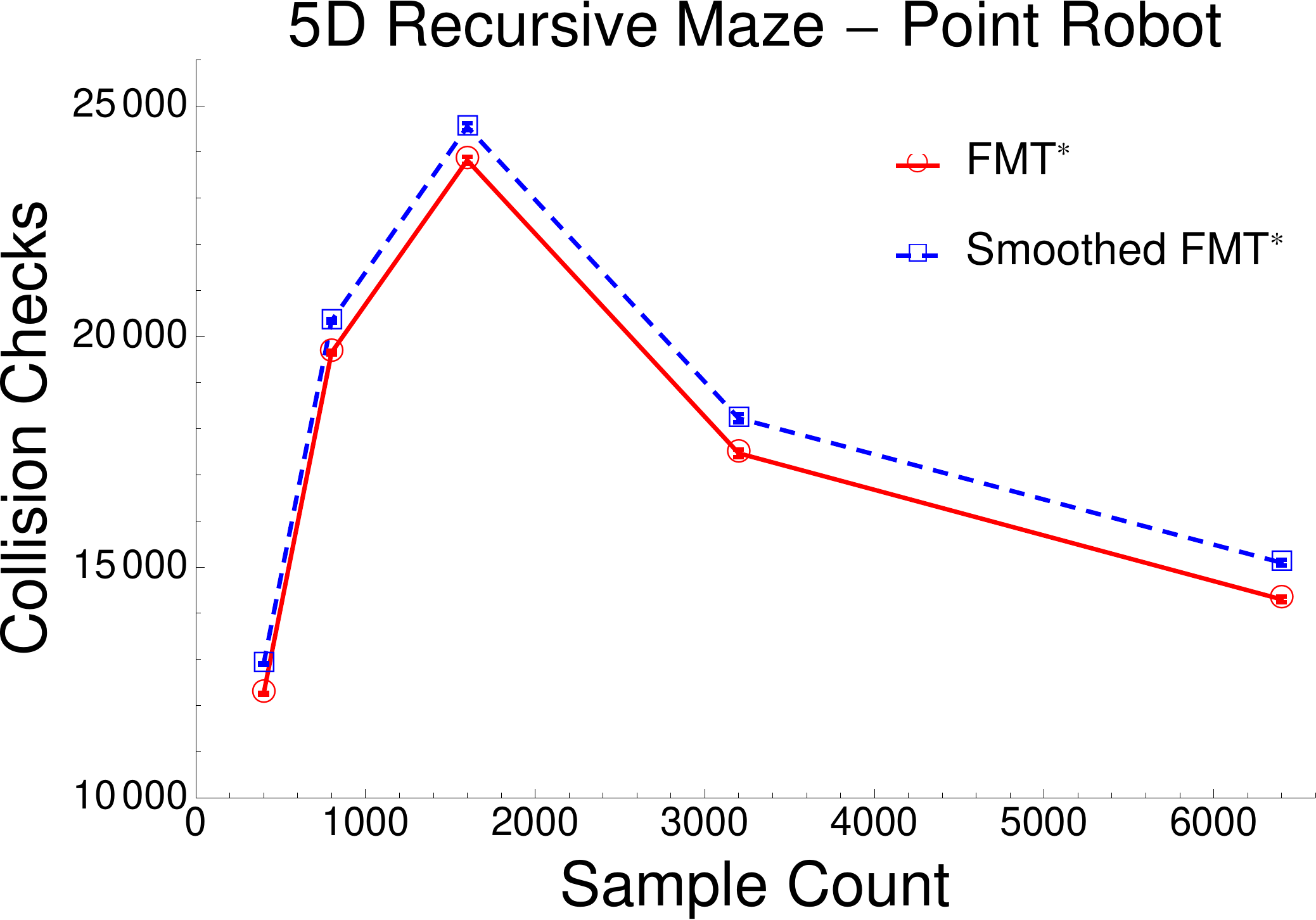}
  }
  \caption{Simulation results for a maze configuration in 2D space.}
\label{fig:smoothing}
\end{figure}

\subsubsection{Experiment With General Cost}
As a demonstration of the computationally-efficient {\color{black}\kFMT} implementation described in Section
\ref{subsubsec:internalFast}, we set up three environments with
non-constant cost-density over the configuration space. We have
kept them in two dimensions so that they can be considered visually,
see Figure \ref{fig:gencost}. In Figure \ref{fig:2x}, there is a
high-cost region near the root node and a low-cost region between it
and the goal region. {\color{black}\kFMT} correctly chooses the shorter path through
the high-cost region instead of going around it, as the extra distance
incurred by the latter option is greater than the extra cost incurred
in the former. In Figure \ref{fig:4x}, we have increased the cost
density of the high-cost region, and {\color{black}\kFMT} now correctly chooses to go
around it as much as possible. In Figure \ref{fig:2rad}, the cost density
function is inversely proportional to distance from the center, and {\color{black}\kFMT} smoothly makes its way
around the higher-cost center to reach the goal region. Note that in
all three plots, since cost-balls are used for considering
connections, the edges are shorter in higher-cost areas and longer
in lower-cost areas.

\begin{figure}[!t]
  \centering
  \subfigure[]{\label{fig:2x}
    \includegraphics[width=0.26\textwidth]{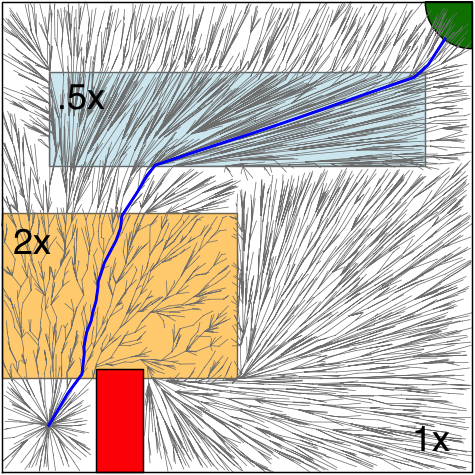}
  }
  \subfigure[]{\label{fig:4x}
    \includegraphics[width=0.26\textwidth]{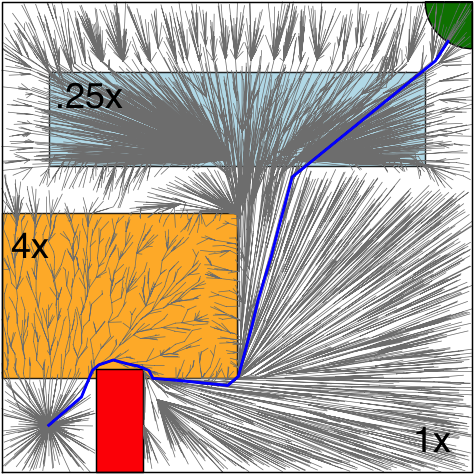}
  }
  \subfigure[]{\label{fig:2rad}
    \includegraphics[width=0.33\textwidth]{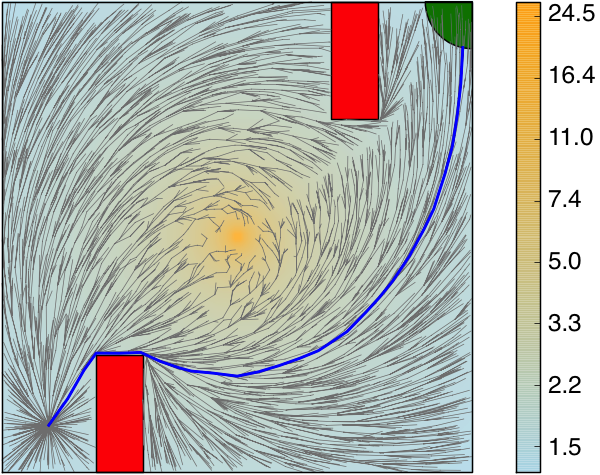}
  }
  \caption{Planning with general costs.}
\label{fig:gencost}
\end{figure}

\subsubsection{How to Best Use \FMT\!?}\label{subsubsec:howto}
\FMT relies on two parameters, namely the connection radius or number of neighbors, and the number of samples. As for the first parameter, numerical experiments showed that $k_{0, \text{\FMT}} = 2^{d}\,(e/d)$  represents an effective and fairly robust choice for the $k$-nearest version of \FMT---this is arguably the value that should be used in most planning problems. Correspondingly, for the radial version of \FMT\!, one should choose a connection radius as specified in the lower bound in Theorem \ref{thrm:AO} with $\eta = e^{1/d}-1$ (see Section \ref{subsec:simSet}). Selecting the number of samples is a more contentious issue, as it is very problem-dependent. A system designer should experiment with a variety of sample sizes for a variety of ``expected" obstacle configurations, and then choose the value that statistically performs the best within the available computational resources. Such a baseline choice could be adaptively adjusted via the resampling techniques discussed in \citep{OS-DH:14} or via the adaptive strategies discussed in \citep{JG-ES-SS-TB:14} and mentioned in Section \ref{subsubsec:alpha}.

For the problem environments and sample sizes considered in our
experiments, the extents of the neighbor sets ($k$-nearest or
radial) are macroscopic with respect to the obstacles. The decrease in
available connections for many samples when their radial neighborhoods significantly intersect the obstacle set seems to adversely affect algorithm performance (see Section \ref{subsubsec:r_v_k}).  The $k$-nearest version of \FMT avoids this issue by attempting connection to a fixed number of samples regardless of obstacle proximity. Thus $k$-nearest \FMT should be considered the default, especially for obstacle cluttered environments. If the application has mostly open space to plan through, however, radial \FMT may be worth testing and tuning.

In problems with a general cost function, \FMT provides a good standalone solution that provably converges to the optimum. In problems with a metric cost function, \FMT (as also \RRTstar and \PRMstar\!) should be considered as a backbone algorithm on top of which one should add a smoothing procedure such as \texttt{ADAPTIVE-SHORTCUT} \citep{DH:00}. In this regard, \FMT should be regarded as a fast ``homotopy finder{\color{black},}'' reflecting its quick initial convergence rate to a good homotopy class, which then needs to be assisted by a smoothing procedure to offset its typical plateauing behavior{\color{black},} i.e., slow convergence to an optimum solution \emph{within} a homotopy class. When combining \FMT with a smoothing procedure one should consider values for the connection radius or number of neighbors most likely equal to about 80\% or 90\% of the previously suggested values, so as to ensure very fast initial rates of convergence (see Section~\ref{subsubsec:rmtuning}). Additionally, non-uniform sampling strategies reflecting \emph{prior} knowledge about the problem may also improve the speed of finding the optimal homotopy class. Finally, a bidirectional implementation is usually preferable \citep{JS-ES-LJ-MP:14}.

\section{Conclusions}\label{sec:conc}
In this paper we have introduced and analyzed a novel probabilistic
sampling-based motion planning algorithm called the Fast Marching Tree
algorithm (\FMT$\!$). This algorithm is asymptotically optimal and
appears to converge \emph{significantly} faster then its
state-of-the-art counterparts for a wide range of challenging problem
instances. We used the weaker notion of convergence in probability, as
opposed to convergence almost surely, and showed that the extra
mathematical flexibility allowed us to compute convergence rate
bounds. Extensions (all retaining AO) to non-uniform sampling
strategies, general costs, and a $k$-nearest-neighbor implementation
were also presented.

This paper leaves numerous important extensions open for further
research. First, it is of interest to extend the \FMT algorithm to
address problems with differential motion constraints; the work
in \citep{ES-LJ-MP:14} and \citep{ES-LJ-MP:14b} presents preliminary
results in this direction (specifically, for systems with
driftless differential constraints, and with drift constraints and
linear affine dynamics, respectively). Second, we plan to explore
further the convergence rate bounds provided by the proof of AO given
here. Third, we plan to use this algorithm as the backbone for
scalable stochastic planning algorithms. Fourth, we plan to extend the
\FMT algorithm for solving the Eikonal equation{\color{black}, and more generally for} addressing problems characterized by partial differential equations. Fifth,
as discussed, \FMT requires the tuning of a scaling factor for either the search radius or the number of nearest-neighbors{\color{black},} and the selection of the number of samples. It is of interest to devise strategies whereby
these parameters are ``self regulating'' (see Section \ref{subsubsec:howto} for some possible strategies), thus effectively making the
algorithm parameter-free and anytime. Finally, we plan to test the performance of
\FMT on mobile ground robots operating in dynamic environments.

\section*{Acknowledgement}
The authors gratefully acknowledge the contributions of Wolfgang
Pointner and Brian Ichter to the implementation of \FMT and for help on the numerical experiments. This research was supported  by NASA under the Space Technology Research Grants Program, Grant NNX12AQ43G. 

\appendix

\titleformat{\section}
{\normalfont\Large\bfseries}
{Appendix \thesection:}{0.3em}{}

\section{Proofs for Lemmas \ref{lemma:claim_1}--\ref{lemma:claim_3}}\label{appA}
\begin{proof}[Proof of Lemma \ref{lemma:claim_1}]
To start, note that
$ \mathbb{P}(K^{\beta}_n \ge \alpha (M_n-1)) + \mathbb{P}(A_n^c)  \ge \mathbb{P}(\{K^{\beta}_n \ge \alpha (M_n-1)\} \cup A_n^c) 
 = 1 - \mathbb{P}(\{K^{\beta}_n < \alpha (M_n-1)\} \cap A_n)$, 
where the first inequality follows from the union bound and the second equality follows from De Morgan's laws. Note that the event $\{K^{\beta}_n < \alpha (M_n-1)\} \cap A_n$ is the event that each $B_{n,m}$ contains at least one node, and more than a $1-\alpha$ fraction of the $B^{\beta}_{n,m}$ balls also contains at least one node.  

When two nodes $x_i$ and $x_{i+1}$, $i \in \{1,\ldots,M_n-2\}$, are contained in adjacent balls $B_{n,i}$ and $B_{n,i+1}$, respectively, their distance apart $\|x_{i+1} - x_i\|$ can be upper bounded by,
\[ \left\{ \begin{array}{ll} \frac{\theta r_n}{2+\theta} + \frac{\beta r_n}{2+\theta} + \frac{\beta r_n}{2+\theta} & : \text{if}\, x_i \in B^{\beta}_{n,i} \text{ and } x_{i+1} \in B^{\beta}_{n,i+1} \\
\frac{\theta r_n}{2+\theta}  + \frac{\beta r_n}{2+\theta} + \frac{r_n}{2+\theta} & : \text{if}\, x_i \in B^{\beta}_{n,i} \text{ or } x_{i+1} \in B^{\beta}_{n,i+1} \\
\frac{\theta r_n}{2+\theta} + \frac{r_n}{2+\theta} + \frac{r_n}{2+\theta} & : \, \text{otherwise,} \end{array}\right. \]
where the three bounds have been suggestively divided into a term for the distance between ball centers and a term each for the radii of the two balls containing the nodes.  This bound also holds for $\|x_{M_n}-x_{M_n-1}\|$, although necessarily in one of the latter two bounds, since $B^{\beta}_{n,M_n}$ being undefined precludes the possibility of the first bound. Thus we can rewrite the above bound, for $i \in \{1,\ldots, M_n-1\}$, as $\|x_{i+1} - x_i\| \le \bar{c}(x_i) + \bar{c}(x_{i+1})$, 
where
\begin{equation}
\label{cbar}
\bar{c}(x_k) := \left\{ 
\begin{array}{ll}
\frac{\theta r_n}{2(2+\theta)} + \frac{\beta r_n}{2+\theta} & : \, x_k \in B^{\beta}_{n,k}, \\
\frac{\theta r_n}{2(2+\theta)} + \frac{r_n}{2+\theta} & : \, x_k \notin B^{\beta}_{n,k}. \\
\end{array} \right.
\end{equation}
Again, $\bar{c}(x_{M_n})$ is still well-defined, but always takes the second value in equation \eqref{cbar} above. Let $L_{n,\alpha,\beta}$ be the length of a path that sequentially connects a set of nodes $\{x_1 = x_{\text{init}}, x_2, \dots, x_{M_n}\}$, such that $x_m \in B_{n,m} \; \forall m \in \{1,\dots,M_n\}$, and more than a $(1-\alpha)$ fraction of the nodes $x_1,\dots,x_{M_n-1}$ are also contained in their respective $B^{\beta}_{n,m}$ balls. The length $L_{n,\alpha,\beta}$ can then be upper bounded as follows

\allowdisplaybreaks
\begin{align}
\label{pathL1}
L_{n,\alpha,\beta} & = \sum^{M_n-1}_{k = 1} \|x_{k+1} - x_k\|  \le \sum^{M_n-1}_{k = 1} 2\bar{c}(x_k) - \bar{c}(x_1) + \bar{c}(x_{M_n})\nonumber \\
& \le (M_n - 1)\frac{\theta r_n}{2+\theta} + \lceil (1-\alpha)(M_n-1) \rceil \frac{2\beta r_n}{2+\theta} + \lfloor \alpha (M_n-1) \rfloor \frac{2 r_n}{2+\theta} + \frac{(1-\beta) r_n}{2+\theta}\nonumber \\
& \le (M_n - 1) \, r_n \, \frac{\theta + 2\alpha + 2(1-\alpha)\beta}{2+\theta} + \frac{(1-\beta)r_n}{2+\theta}\displaybreak[3]\nonumber \\
& \le M_n \,r_n \frac{\theta + 2\alpha + 2\beta}{2+\theta} + \frac{r_n}{2+\theta}.
\end{align}

In equation \ref{pathL1}, $\lceil x \rceil$ denotes the smallest integer not less than $x$, while $\lfloor x \rfloor$ denotes the  largest integer not greater than $x$.
Furthermore, we can upper bound $M_n$ as follows,
\begin{align}
\label{Mbound}
c(\sigma_n') & \ge\! \!\sum_{k = 1}^{M_n-2} \|\sigma_n(\tau_{k+1}) \!- \!\sigma_n(\tau_k)\| \!\!+\!\! \|\sigma_n'(1) \!-\! \sigma_n(\tau_{M_n-1})\|  \ge (M_n\! -\! 2) \frac{\theta r_n}{2+\theta} \!+\! \frac{r_n}{2(2+\theta)}\nonumber \\
& = M_n \frac{\theta r_n}{2+\theta} + \Bigl (\frac{1}{2} - 2\theta \Bigr) \frac{r_n}{2+\theta}  \ge M_n \frac{\theta r_n}{2+\theta},
\end{align}
where the last inequality follows from the assumption that $\theta < 1/4$. Combining equations (\ref{pathL1}) and (\ref{Mbound}) gives
\begin{equation}
\label{pathL2}
\begin{split}
L_{n,\alpha,\beta} & \le c(\sigma_n')\, \biggl (1+\frac{2\alpha + 2\beta}{\theta} \biggr) + \frac{r_n}{2+\theta}  = \kappa(\alpha,\beta,\theta)\, c(\sigma_n') + \frac{r_n}{2+\theta}.
\end{split}
\end{equation}

We will now show that when $A_n$ occurs, $c_n$ is no \edit{greater} than the
length of the path connecting any sequence of $M_n$ nodes tracing through the balls $B_{n,1},\dots,B_{n,M_n}$ (this
inequality of course also implies $c_n < \infty$). Coupling this fact with equation \eqref{pathL2}, {\color{black}we} can then conclude that the event $\{K^{\beta}_n < \alpha (M_n-1)\} \cap A_n$ implies that $c_n \leq \kappa(\alpha,\beta,\theta)\, c(\sigma_n') + \frac{r_n}{2+\theta}$, which, in turn, would prove the lemma.

Let $x_1 = x_{\text{init}}$, $x_2 \in B_{n,2}$, $\dots$, $x_{M_n} \in B_{n,M_n} \subseteq \mathcal{X}_{\text{goal}}$. Note that the $x_i$'s need not all be distinct.  The following property holds for all $m \in \{2, \dots, M_n-1\}$:
\begin{equation*}
\begin{split}
\|x_m - x_{m-1}\| & \le \|x_m - \sigma_n(\tau_m)\| + \|\sigma_n(\tau_m) - \sigma_n(\tau_{m-1})\|  + \|\sigma_n(\tau_{m-1}) - x_{m-1}\| \\& \le \frac{r_n}{2+\theta} + \frac{\theta r_n}{2+\theta} + \frac{r_n}{2+\theta} = r_n.
\end{split}
\end{equation*}
Similarly, {\color{black}we} can write $\|x_{M_n} - x_{M_n-1}\| \le \frac{r_n}{2+\theta} + \frac{(\theta + 1/2)r_n}{2+\theta} + \frac{r_n}{2(2+\theta)} = r_n$.  Furthermore, we can lower bound the distance to the nearest obstacle for $m \in \{2, \dots, M_n-1\}$ by:
\begin{equation*}
\begin{split}
\inf_{w \in X_{\text{obs}}}\|x_m - w\| & \ge \inf_{w \in X_{\text{obs}}}\|\sigma_n(\tau_m) - w\| - \|x_m - \sigma_n(\tau_m)\|  \ge \frac{3+\theta}{2+\theta}r_n - \frac{r_n}{2+\theta} = r_n,
\end{split}
\end{equation*}
where the second inequality follows from the assumed
$\delta_n$-clearance of the path $\sigma_n$.  Again, similarly, {\color{black}we}
can write $\inf_{w \in X_{\text{obs}}}\|x_{M_n} - w\| \ge \inf_{w \in
  X_{\text{obs}}}||x_m - \sigma_n(1)\| - \|\sigma_n(1) - w\| \ge
\frac{3+\theta}{2+\theta}r_n - \frac{r_n}{2+\theta} = r_n$.  Together,
these two properties imply that, for $m \in \{2,\dots, M_n\}$, when a
connection is attempted for $x_m$, $x_{m-1}$ will be in the search
radius and there will be no obstacles in that search radius.  In
particular, this fact implies that either the algorithm will return a feasible path before considering $x_{M_n}$, or it will consider $x_{M_n}$ and connect it.  Therefore, \FMT is guaranteed to return a feasible solution when the event $A_n$ occurs.  Since the remainder of this proof assumes that $A_n$ occurs, we will also assume $c_n < \infty$.

Finally, assuming $x_m$ is contained in an edge, let $c(x_m)$ denote the unique cost-to-{\color{black}arrive} of $x_m$ in the graph generated by \FMT at the end of the algorithm, just before the path is returned. If $x_m$ is not contained in an edge, we set $c(x_m) = \infty$.  Note that $c(\cdot)$ is well-defined, since if $x_m$ is contained in any edge, it must be connected through a unique path to $x_{\text{init}}$. We claim that for all $m \in \{2,\dots, M_n\}$, either
$ c_n \le \sum_{k = 1}^{m-1} \|x_{k+1} - x_k\|$, 
or
$c(x_m) \le \sum_{k = 1}^{m-1} \|x_{k+1} - x_k\|$. 
In particular, taking $m = M_n$, this inequality would imply that $c_n \le \min\{c(x_{M_n}),\sum_{k = 1}^{M_n-1} \|x_{k+1} - x_k\|\} \le \sum_{k = 1}^{M_n-1} \|x_{k+1} - x_k\|$, which, as argued before, would imply the claim.

The claim is proved by induction on $m$. The case of $m = 1$ is trivial, since the first step in the \FMT algorithm is to make every collision-free connection between $x_{\text{init}} = x_1$ and the nodes contained in $B(x_{\text{init}}; r_n)$, which will include $x_2$ and, thus, $c(x_2) = \|x_2 - x_1\|$.  Now suppose the claim is true for $m-1$.  There are four exhaustive cases to consider:
\begin{itemize}
\itemsep0.5em
\item[1.] $c_n \le \sum_{k = 1}^{m-2}\|x_{k+1} - x_k\|$,
\item[2.] $c(x_{m-1}) \le \sum_{k = 1}^{m-2}\|x_{k+1} - x_k\|$ and \FMT {\color{black}terminates} before considering $x_m$,
\item[3.] $c(x_{m-1}) \le \sum_{k = 1}^{m-2}\|x_{k+1} - x_k\|$ and $x_{m-1} \in \Hset$ when $x_m$ is first considered,
\item[4.] $c(x_{m-1}) \le \sum_{k = 1}^{m-2}\|x_{k+1} - x_k\|$ and $x_{m-1} \notin \Hset$ when $x_m$ is first considered.
\end{itemize}

Case 1: $c_n \le \sum_{k = 1}^{m-2}\|x_{k+1} - x_k\| \le \sum_{k =
  1}^{m-1}\|x_{k+1} - x_k\|$, thus the claim is true for $m$. Without
loss of generality, for cases 2--4 we assume that case 1 does not occur.

Case 2: $c(x_{m-1}) < \infty$ implies that $x_{m-1}$ enters $\Hset$ at some point during \FMT\!. However, if $x_{m-1}$ were ever the minimum-cost element of $\Hset$, $x_m$ would have been considered, and thus \FMT must have returned a feasible solution before $x_{m-1}$ was ever the minimum-cost element of $\Hset$.  Since the end-node of the solution returned must have been the minimum-cost element of $\Hset$, $c_n \le c(x_{m-1}) \le \sum_{k = 1}^{m-2}\|x_{k+1} - x_k\| \le \sum_{k = 1}^{m-1}\|x_{k+1} - x_k\|$, thus the claim is true for $m$.

Case 3: $x_{m-1} \in \Hset$ when $x_m$ is first considered, $\|x_m - x_{m-1}\| \le r_n$, and there are no obstacles in $B(x_m; r_n)$. Therefore, $x_m$ must be connected to some parent when it is first considered, and $c(x_m) \le c(x_{m-1}) + \|x_m - x_{m-1}\| \le \sum_{k = 1}^{m-1} \|x_{k+1} - x_k\|$, thus the claim is true for $m$.

Case 4: When $x_m$ is first considered, there must exist $z \in B(x_m; r_n)$ such that $z$ is the minimum-cost element of $\Hset$, while $x_{m-1}$ has not even entered $\Hset$ yet.  Note that again, since $B(x_m; r_n)$ intersects no obstacles and contains at least one node in $\Hset$, $x_m$ must be connected to some parent when it is first considered. Since $c(x_{m-1}) < \infty$, there is a well-defined path $\mathcal{P} = \{v_1, \dots, v_q\}$ from $x_{\text{init}} = v_1$ to $x_{m-1} = v_q$ for some $q \in \mathbb{N}$.  Let $w = v_j$, where $j = \max_{i \in \{1,\dots,q\}} \{i : v_i \in \Hset \text{ when } x_m \text{ is first considered}\}$. Then there are two subcases, either $w \in B(x_m; r_n)$ or $w \notin B(x_m; r_n)$.  If $w \in B(x_m; r_n)$, then,
\[ \begin{split}
c(x_m) & \le c(w) + \|x_m - w\| \le c(w) + \|x_{m-1} - w\| + \|x_m - x_{m-1}\| \\
& \le c(x_{m-1}) + \|x_m - x_{m-1}\|  \le \sum_{k = 1}^{m-1} \|x_{k+1} - x_k\|,
\end{split} \]
thus the claim is true for $m$ (the second and third inequalities follow from the triangle inequality).  If $w \notin B(x_m; r_n)$, then,
\[ \begin{split}
c(x_m) & \le c(z) + \|x_m - z\|  \le c(w) + r_n  \le c(x_{m-1}) + \|x_m - x_{m-1}\|  \le \sum_{k = 1}^{m-1} \|x_{k+1} - x_k\|,
\end{split} \]
where the third inequality follows from the fact that $w \notin B(x_m, r_n)$, which means that any path through $w$ to $x_m$, in particular the path $\mathcal{P} \cup {x_m}$, must traverse a distance of at least $r_n$ between $w$ and $x_m$.  Thus, in the final subcase of the final case, the claim is true for $m$.

Hence, we can conclude that $c_n \le \sum_{k = 1}^{M_n-1} \|x_{k+1} - x_k\|$. As argued before, coupling this fact with equation \eqref{pathL2}, {\color{black}we} can conclude that the event $\{K^{\beta}_n < \alpha (M_n-1)\} \cap A_n$ implies that $c_n \leq \kappa(\alpha,\beta,\theta)\, c(\sigma_n') + \frac{r_n}{2+\theta}$, and the claim follows.
\end{proof}

\begin{proof}[Proof of Lemma \ref{lemma:claim_2}]
The proof relies on a Poissonization argument.  For $\nu \in (0,1)$,
let $\tilde{n}$ {\color{black}be} a random variable drawn from a Poisson distribution
with parameter $\nu\, n$ (denoted as Poisson$(\nu \,n)$). Consider the
set of nodes $\widetilde{V} := \texttt{SampleFree}(\tilde{n})$, and
for the remainder of the proof, ignore $x_{\text{init}}$ (adding back
$x_{\text{init}}$ only decreases the probability in question, which we
are showing goes to zero anyway).  Then the locations of the nodes in
$\widetilde{V}$ are distributed as a spatial Poisson process with
intensity $\nu n / \mu(\mathcal{X}_{\text{free}})$. Therefore, for a Lebesgue-measurable region $R \subseteq \mathcal{X}_{\text{free}}$, the number of nodes in $R$ is distributed as a Poisson random variable with distribution Poisson$\Bigl (\nu \, n \, \mu(R) / \mu(\mathcal{X}_{\text{free}})\Bigr)$, independent of the number of nodes in any region disjoint with $R$ \citep[Lemma 11]{Karaman.Frazzoli:IJRR2011}. 

Let $\widetilde{K}_n^{\beta}$ be the Poissonized analogue of $K_n^{\beta}$, namely $\widetilde{K}_n^{\beta} := \card\Bigl \{m \in \{1, \dots, M_n-1\}: B^{\beta}_{n,m} \cap \widetilde{V} = \emptyset \Bigr \}$.  Note that only the distribution of node locations has changed through Poissonization, while the balls $B^{\beta}_{n,m}$ remain the same. From the definition of $\widetilde{V}$, we can see that $\p{K_n^{\beta} \ge \alpha (M_n-1))} = \probcond{\widetilde{K}_n^{\beta} \ge \alpha (M_n-1)}{\tilde{n} = n}$.  Thus, we have
\begin{equation}
\label{poissonization}
\begin{split}
  \p{\widetilde{K}_n^{\beta} \ge \alpha (M_n-1))}  &= \sum_{j = 0}^{\infty} \probcond{\widetilde{K}_n^{\beta} \ge \alpha (M_n-1)}{\tilde{n} = j} \cdot \p{\tilde{n} = j} \\
&  \ge \sum_{j = 0}^{n} \probcond{\widetilde{K}_n^{\beta} \ge \alpha (M_n-1)}{  \tilde{n} = j}\,\p{\tilde{n} = j} \\
&  \ge \sum_{j = 0}^{n} \probcond{\widetilde{K}_n^{\beta} \ge \alpha (M_n-1)}{ \tilde{n} = n} \, \p{\tilde{n} = j} \\
&  = \probcond{\widetilde{K}_n^{\beta} \ge \alpha (M_n-1)}{ \tilde{n} = n} \, \p{\tilde{n} \le n} \\
& = \p{K_n^{\beta} \ge \alpha (M_n-1)} \, \p{\tilde{n} \le n} \\
& \ge (1-e^{-a_{\nu}n}) \p{K_n^{\beta} \ge \alpha (M_n-1)},
\end{split}
\end{equation}
where $a_{\nu}$ is a positive constant that depends only on $\nu$.  The third line follows from the fact that $\mathbb{P}(\widetilde{K}_n^{\beta} \ge \alpha (M_n-1) | \tilde{n} = j)$ is nonincreasing in $j$, and the last line follows from a tail approximation of the Poisson distribution \citep[p. 17]{Penrose:03} and the fact that $\mathbb{E}[\tilde{n}] < n$.  Thus, since $\lim_{n \rightarrow \infty}(1-e^{-a_{\nu}n}) = 1$ for any fixed $\nu \in (0,1)$, it suffices to show that $\lim_{n \rightarrow \infty}\mathbb{P}(\widetilde{K}_n^{\beta} \ge \alpha (M_n-1)) = 0$ to prove the statement of the lemma.

Since by assumption $\beta < \theta/2$, $B^{\beta}_{n,1},\dots,
B^{\beta}_{n,M_n-1}$ are all disjoint.  This disjointness means that
{\color{black}for fixed $n$,} the number of the Poissonized nodes that fall in each {\color{black}$B^{\beta}_{n,m}$} is independent of the others and identically distributed as a Poisson random variable with mean equal to
\[ \frac{\mu(B^{\beta}_{n,1})}{\mu(\mathcal{X}_{\text{free}})}\nu n = \frac{\zeta_d \bigl (\frac{\beta r_n}{2+\theta} \bigr )^d}{\mu(\mathcal{X}_{\text{free}})}\nu n = \frac{\nu \zeta_d \beta^d\gamma^d\log(n)}{(2+\theta)\mu(\mathcal{X}_{\text{free}})} := \lambda_{\beta,\nu}\log(n),\]
where $\lambda_{\beta,\nu}$ is positive and does not depend on $n$.
From this equation we get that for $m \in \{1,\dots,M_n-1\}$,
\[ \mathbb{P}(B^{\beta}_{n,m} \cap \widetilde{V} = \emptyset) = e^{-\lambda_{\beta,\nu} \log(n)} = n^{-\lambda_{\beta,\nu}}. \]
Therefore, $\widetilde{K}_n^{\beta}$ is distributed according to a
binomial distribution, in particular according to {\color{black}the}
Binomial($M_n-1$, $n^{-\lambda_{\beta, \nu}}$) {\color{black}distribution}. Then for $n > (e^{-2}\alpha)^{-\frac{1}{\lambda_{\beta,\nu}}}$, $e^2\mathbb{E}[\widetilde{K}_n^{\beta}] < \alpha (M_n-1)$, so from a tail approximation to the Binomial distribution \citep[p. 16]{Penrose:03}, 
\begin{equation}
\label{binomialbound}
\mathbb{P}(\widetilde{K}_n^{\beta} \ge \alpha(M_n-1)) \le e^{-\alpha
  (M_n-1)}.
\end{equation}
Finally, since \edit{by assumption} $\xinit \notin \xgoal$, the optimal
cost is positive, i.e., $c^* > 0$; this positivity implies that there is a lower-bound on feasible path length.  Since the ball radii decrease to 0, it must be that $\lim_{n \rightarrow \infty} M_n = \infty$ in order to cover the paths, and the lemma is proved.
\end{proof}

\begin{proof}[Proof of Lemma \ref{lemma:claim_3}]
Let $c_{\text{max}} := \max_{n \in \mathbb{N}} c(\sigma_n')$;  the convergence of $c(\sigma_n')$ to a limiting value that is also a lower bound implies that $c_{\text{max}}$ exists and is finite. Then we have,
\begin{equation}
\label{ballintersections}
\begin{split}
\p{A_{n,\theta}^c} & \le \sum_{m = 1}^{M_n} \p{B_{n,m} \cap V = \emptyset}  = \sum_{m = 1}^{M_n} \, \biggl (1-\frac{\mu(B_{n,m})}{\mu(\mathcal{X}_{\text{free}})} \biggr)^n  = \sum_{m = 1}^{M_n-1}\biggl  (1- \frac{\zeta_d (\frac{r_n}{2+\theta})^d}{\mu(\mathcal{X}_{\text{free}})} \biggr)^n \\
&\qquad \qquad + \biggl (1- \frac{\zeta_d (\frac{r_n}{2(2+\theta)})^d}{\mu(\mathcal{X}_{\text{free}})}\biggr)^n \\
& \le M_n \biggl (1- \frac{\zeta_d \gamma^d \log(n)}{n(2+\theta)^d \mu(\mathcal{X}_{\text{free}})} \biggr)^n   +  \biggl (1- \frac{\zeta_d \gamma^d \log(n)}{n(4+2\theta)^d \mu(\mathcal{X}_{\text{free}})} \biggr)^n \\
& \le M_n e^{-\frac{\zeta_d \gamma^d \log(n)}{(2+\theta)^d \mu(\mathcal{X}_{\text{free}})}} + e^{-\frac{\zeta_d \gamma^d \log(n)}{(4+2\theta)^d \mu(\mathcal{X}_{\text{free}})}} \\
& \le \frac{(2+\theta)c(\sigma_n')}{\theta r_n} n^{-\frac{\zeta_d \gamma^d}{(2+\theta)^d \mu(\mathcal{X}_{\text{free}})}} + n^{-\frac{\zeta_d \gamma^d}{(4+2\theta)^d \mu(\mathcal{X}_{\text{free}})}} \\
& \le \frac{(2+\theta)c_{\text{max}}}{\theta \gamma}\log(n)^{-\frac{1}{d}} n^{\frac{1}{d}-\frac{\zeta_d \gamma^d}{(2+\theta)^d \mu(\mathcal{X}_{\text{free}})}} + n^{-\frac{\zeta_d \gamma^d}{(4+2\theta)^d \mu(\mathcal{X}_{\text{free}})}}, \\
\end{split}
\end{equation}
where the third inequality follows from the inequality $(1-\frac{1}{x})^n \leq e^{-\frac{n}{x}}$, and the fourth inequality follows from the bound on $M_n$ obtained in the proof of Lemma \ref{lemma:claim_1}. As $n \rightarrow \infty$, the second term goes to zero for any $\gamma > 0$, while the first term goes to zero for any $\gamma > (2+\theta)\Bigl  (\mu(\mathcal{X}_{\text{free}})/(d \zeta_d) \Bigr)^{1/d}$, which is satisfied by $\theta < 2\eta$. Thus $\p{A_{n,\theta}^c} \rightarrow 0$ and the lemma is proved.
\end{proof}

\section{Proof of Convergence Rate Bound}
\label{appD}
\begin{proof}[Proof of Theorem \ref{thrm:convrate}]
We proceed by first proving the tightest bound possible,
  carrying through all terms and constants, and then we make
  approximations to get to the final simplified result. Let $\varepsilon > 0$, $\theta \in (0, \min(2\eta,\, 1/4) )$, $\alpha, \beta \in (0, ,
\min(1, \, \varepsilon)\, \theta/8)$, and $\nu \in (0,1)$. Let $H(a) =
1+ a(\log(a) - 1)$, and $\gamma = 2(1+\eta)\Bigl (\frac{1}{d
  \zeta_d}\Bigr)^{1/d}$ so that $r_n = \gamma
\Bigl(\frac{\log(n)}{n}\Bigr)^{1/d}$. Letting $n_0 >
\Bigl(\alpha/e^2\Bigr)^{-\frac{(2+\theta)}{\nu \zeta_d \beta^d
    \gamma^d}}$ and such that 
    \[
    r_{n_0} < \min \biggl\{ 2\, \xi
(2+\theta), \frac{2+\theta}{3+\theta} \delta,
\frac{\varepsilon(2+\theta)}{8} c^* \biggr\},\] 
then for $n \ge n_0$, we
claim that\footnote{Note that the
    convergence rate bound is slightly different from that presented in the
    conference version of this paper, reflecting a corrected
    typographical error.},
\begin{equation}
\label{longrate}
\begin{split}
\mathbb{P}(c_n > (1+\varepsilon)c^*) & < \frac{1}{1 - e^{-\nu n
    H(\frac{n+1}{\nu n})}} e^{-\frac{\alpha}{2}\Bigl \lfloor
  \frac{2+\theta}{\theta r_n} c^* \Bigr\rfloor \Bigl(\log\Bigl(\alpha \Bigl \lfloor \frac{2+\theta}{\theta r_n} c^* \Bigr\rfloor\Bigr) + \zeta_d \Bigl(\frac{\beta r_n}{2+\theta}\Bigr)^d \nu n \Bigr)}\\
& \quad \,+\biggl \lfloor \frac{2+\theta}{\theta r_n} c^* \biggr\rfloor \Bigl(1 - \zeta_d \Bigl (\frac{r_n}{2+\theta}\Bigr)^d\Bigr)^n  + \Bigl(1 - \zeta_d \Bigl(\frac{r_n}{2(2+\theta)}\Bigr)^d\Bigr)^n.
\end{split}
\end{equation}
To prove equation~\eqref{longrate}, note that from the proof of Theorem~\ref{thrm:AO}, equation~\eqref{opt1} and
Lemma~\ref{lemma:claim_1} combine to give (using the same notation),
\[ \mathbb{P}(c_n > (1+\varepsilon)c^*) \le  \mathbb{P}(K^{\beta}_n
\ge \alpha (M_n-1)) + \mathbb{P}(A_{n,\theta}^c). \]
From Equation~\eqref{poissonization} in the proof of
Lemma~\ref{lemma:claim_2}, and a more precise tail bound
\citep[page 17]{Penrose:03} relying on the
assumptions of $n_0$,
\[ \mathbb{P}(c_n > (1+\varepsilon)c^*) \le  \left(\frac{1}{1-e^{\nu n
    H(\frac{n+1}{\nu n})}}\right)\mathbb{P}(\tilde{K}^{\beta}_n
\ge \alpha (M_n-1)) + \mathbb{P}(A_{n,\theta}^c). \]
By the same arguments that led to equation~\eqref{binomialbound}, but
again applied with slightly more precise tail bounds
\citep[page 16]{Penrose:03}, we get,
\[ \mathbb{P}(c_n > (1+\varepsilon)c^*) \le  \left(\frac{1}{1-e^{\nu n
    H(\frac{n+1}{\nu
      n})}}\right)e^{-\frac{\alpha(M_n-1)}{2}\left(\log\big(\alpha(M_n-1)\big) +
  \frac{\nu \zeta_d \beta^d
    \gamma^d}{(2+\theta)\mu(\mathcal{X}_{\text{free}})}
  \log(n)\right)} + \mathbb{P}(A_{n,\theta}^c). \]
By the first three lines of equation~\eqref{ballintersections} from the
proof of Lemma~\ref{lemma:claim_3} (there we upper-bounded $M_n-1$ by
$M_n$ for simplicity, here we carry through the whole term),
\begin{equation*}
\begin{split}
\mathbb{P}(c_n > (1+\varepsilon)c^*) \le & \left(\frac{1}{1-e^{\nu n
    H(\frac{n+1}{\nu
      n})}}\right)e^{-\frac{\alpha(M_n-1)}{2}\left(\log\big(\alpha(M_n-1)\big) +
  \frac{\nu \zeta_d \beta^d
    \gamma^d}{(2+\theta)\mu(\mathcal{X}_{\text{free}})}
  \log(n)\right)} \\
& + (M_n-1) \biggl (1- \frac{\zeta_d \gamma^d \log(n)}{n(2+\theta)^d
  \mu(\mathcal{X}_{\text{free}})} \biggr)^n   +  \biggl (1-
\frac{\zeta_d \gamma^d \log(n)}{n(4+2\theta)^d
  \mu(\mathcal{X}_{\text{free}})} \biggr)^n. \\
\end{split}
\end{equation*}
Finally, in this simplified configuration space, we can set all the
approximating paths $\sigma_n$ from the proof of Theorem~\ref{thrm:AO}
to just be the optimal path, allowing us to compute $M_n-1 = \Bigl \lfloor
\frac{2+\theta}{\theta r_n} c^*\Bigr \rfloor$. Plugging this formula in, noting
that $\mu(\mathcal{X}_{\text{free}}) \le 1$, and simplifying by
collecting terms into factors of $r_n$ gives equation~\eqref{longrate}.

Grouping together constants in equation~\eqref{longrate} into positive
superconstants $A$, $B$, $C$, $D$, and $E$ for simplicity and dropping
the factor of $\frac{1}{1 - e^{-\nu n H(\frac{n+1}{\nu n})}}$ (which goes to 1 as $n \rightarrow \infty$) from the first term, the
bound becomes,
\begin{equation}
\begin{split}
\mathbb{P}(c_n > (1+\varepsilon)c^*) & <
e^{-A\left(\frac{n}{\log(n)}\right)^{1/d}\left(\log(B)+\frac{1}{d}\log\left(\frac{n}{\log(n)}\right)+C\log(n)\right)} \\
& \qquad  + \left(\frac{n}{\log(n)}\right)^{1/d}\cdot\left(1-D\frac{\log(n)}{n}\right)^n
  + \left(1-E\frac{\log(n)}{n}\right)^n, \\
& \le
B^{-A\left(\frac{n}{\log(n)}\right)^{1/d}}\cdot\left(\frac{n}{\log(n)}\right)^{-\frac{A}{d}\left(\frac{n}{\log(n)}\right)^{1/d}}\cdot
n^{-A\,C\left(\frac{n}{\log(n)}\right)^{1/d}}
\\
& \qquad + \left(\frac{n}{\log(n)}\right)^{1/d}\cdot n^{-D} + n^{-E},
\end{split}
\end{equation}
where the second inequality is just a rearrangement of the first term,
and uses the inequality $(1-x/n)^n \le e^{-x}$ for the last two
terms. As both $n$ and $\frac{n}{\log(n)}$ approach $\infty$, the
first term must become negligible compared to the last two terms, no
matter the values of the superconstants. Now noting that $E =
\frac{D}{2^d}$, we can write the asymptotic upper bound for
$\mathbb{P}(c_n > (1+\varepsilon)c^*)$ as,
\begin{equation*}
\left(\log(n)\right)^{-\frac{1}{d}} n^{\frac{1}{d} - D} + n^{-\frac{D}{2^d}}.
\end{equation*}
Therefore the deciding factor in which term asymptotically dominates
is whether or not $\frac{1}{d} - D \le -\frac{D}{2^d}$. Plugging in for
the actual constants composing $D$, we get,
\begin{equation*}
\frac{1}{d} \le \left(1-\frac{1}{2^d}\right)\frac{1}{d}\left(\frac{2(1+\eta)}{2+\theta}\right)^d,
\end{equation*}
or, equivalently,
\begin{equation}
\label{etacond}
\eta \ge \frac{2+\theta}{(2^d-1)^{1/d}} - 1.
\end{equation}
However, since $\theta$ is a proof parameter that can be taken
arbitrarily small (and doing so improves the asymptotic rate), if $\eta > \frac{2}{(2^d-1)^{1/d}} - 1$, then
$\theta$ can always be chosen small enough so that
equation~\eqref{etacond} holds. Finally, we are left with,
\begin{equation}
\mathbb{P}(c_n > (1+\varepsilon)c^*) \in \left\{\begin{array}{lcl}
     O\left(\left(\log(n)\right)^{-\frac{1}{d}}n^{\frac{1}{d}\left(1-\left((1+\eta)\frac{2}{2+\theta}\right)^d\right)}\right)
    & \text{ if } & \eta \le \frac{2}{(2^d-1)^{1/d}} - 1,\\
    O\left(n^{-\frac{1}{d}\left(\frac{1+\eta}{2+\theta}\right)^d}\right)
    & \text{ if } & \eta > \frac{2}{(2^d-1)^{1/d}} - 1, \\
    \end{array} \right.
\end{equation}
for arbitrarily small $\theta$. By replacing $\theta$ by an
arbitrarily small parameter $\rho$ that is additive in the exponent,
the final result is proved.
\end{proof}

\section{Proof of Computational Complexity}
\label{appB}
\begin{proof}[Proof of Theorem \ref{thrm:CC}]
We first prove two results that are not immediately obvious from the
description of the algorithm, about the number of computations of edge
cost and how many times a node is considered for connection.
\begin{lemma}[Edge-Cost Computations]
\label{costcalls}
Consider the setup of Theorem \ref{thrm:CC}. Let $M^{(1)}_{\text{\FMT}}$ be the
number of computations of edge cost
when \FMT is run on $V$ using $r_n$. Similarly, let $M^{(1)}_{\text{PRM}^*}$
be the number of computations of edge cost when \PRMstar is run on $V$
using $r_n$. Then in expectation,
\[ M^{(1)}_{\text{\FMT}} \le M^{(1)}_{\text{PRM}^*} \in O(n\log(n)). \]

\begin{proof}
\PRMstar computes the cost of every edge in its graph.  For a given node, edges are only created between that node and nodes in the $r_n$-ball around it.  The expected number of nodes in an $r_n$-ball is less than or equal to $(n/\mu(\mathcal{X}_{\text{free}})) \zeta_d r_n^d = (\zeta_d/\mu(\mathcal{X}_{\text{free}})) \gamma \log(n)$, and since there are $n$ nodes, the number of edges in the \PRMstar graph is $O(n\log(n))$.  Therefore, $M^{(1)}_{\text{PRM}^*}$ is $O(n\log(n))$.

For each node $x \in V$, \FMT saves the associated set $N_x$ of
$r_n$-neighbors.  Instead of just saving a reference for
each node $y \in N_x$, $N_x$ can also allocate memory for the real
value $\texttt{Cost}(y,x)$. Saving this value whenever it is
first computed guarantees that \FMT will never compute it more than
once for a given pair of nodes.  Since the only pairs of nodes
considered are exactly those considered in \PRMstar\!, it is guaranteed that $M^{(1)}_{\text{\FMT}} \le M^{(1)}_{\text{PRM}^*}$.  Note that computation of the cost-to-{\color{black}arrive} of a node already connected in the \FMT graph was not factored in here, because it is just a sum of edge costs which have already been computed.
\end{proof}
\end{lemma}

The following Lemma shows that lines 10--18 in Algorithm
  2 are only run $O(n)$ times, despite being contained in the loop at line
  6, which runs $O(n)$ times, and the loop at line 9, which runs
  $O(\log(n))$ times, which would seem to suggest that lines 10--18 are
  run $O(n\log(n))$ times.

\begin{lemma}[Node Considerations]
\label{considerations}
Consider the setup of Theorem \ref{thrm:CC}. We say that a node is
`considered' when it has played the role of $x \in
X_{\text{near}}$ in line \ref{line:forXnear} of Algorithm
\ref{prtalg}. Let $M^{(2)}_{\text{\FMT}}$ be
the number of node considerations when \FMT is run on $V$ using
$r_n${\color{black},} including multiple considerations of the same node. Then in expectation,
\[ M^{(1)}_{\text{\FMT}} \in O(n). \]

\begin{proof}
Note that $X_{\text{near}} \subset \Wset$ and nodes are permanently removed from
$\Wset$ as soon as they are connected to a parent. Furthermore, if there
are no obstacles within $r_n$ of a given node, then it must be
connected to a parent when it is \emph{first} considered. Clearly
then, considerations involving these nodes account for at most $n$
considerations, so it suffices to show that the number of considerations
involving nodes within $r_n$ of an obstacle (denote this value by
$M_{obs}$) is $O(n)$ in expectation.

Any node can only be considered as many times as it has neighbors,
which is $O(\log(n))$ in expectation. Furthermore, as $n \rightarrow
\infty$, the expected number of nodes within $r_n$ of an obstacle can be
approximated arbitrarily well by $n \cdot S_{obs}\cdot r_n$, where
$S_{obs}$ is the constant surface area of the obstacles. This
equation is just the
density of points{\color{black},} $n${\color{black},} times the volume formula for a thin shell
around the obstacles, which will hold in the large $n$ limit, since
$r_n \rightarrow 0$. Since $r_n \in O((\log(n)/n)^{1/d})$, these
combine to give,
\[ M_{obs} \in O(n (\log(n)/n)^{1/d} \log(n)) =
O((\log(n))^{1+1/d}n^{1-1/d}) \in O(n) \]
in expectation, proving the lemma.
\end{proof}
\end{lemma}

We are now ready to show that the computational complexity of \FMT is
$O(n \log(n))$ in expectation. As already pointed out in Lemma
\ref{costcalls}, the number of calls to \texttt{Cost} is
$O(n\log(n))$. By Lemma \ref{considerations} and the fact that
\texttt{CollisionFree} is called if and only if a node is under
consideration, the number of calls to \texttt{CollisionFree} is
$O(n)$. The number of calls to \texttt{Near} {\color{black}in which any
  computation is done, as opposed to just loading a set from memory,}
is bounded by $n$, since {\color{black}neighbor sets are} saved and thus
{\color{black}are} never {\color{black}computed} more than once for
each node.  Since \texttt{Near} {\color{black}computation} can be implemented to arbitrarily
close approximation in $O(\log(n))$ time \citep{Arya.Mount:ACM95}, the
calls to \texttt{Near} also account for $O(n\log(n))$ time complexity.
Since each node can have at most one parent in the graph $T$, $E$ can
only have at most $n$ elements and since edges are only added, never
subtracted, from $E$, the time complexity of building $E$ is $O(n)$.
Similarly, $\Wset$ only ever has nodes subtracted and starts with $n$
nodes, so subtracting from $\Wset$ takes a total of $O(n)$ time. 

Operations on $\Hset$ can be done in $O(n\log(n))$ time if $\Hset$ is
implemented as {\color{black}a} binary min heap.  As pointed out in Theorem
\ref{termination}, there are at most $n$ additions to $\Hset$, each taking
$O(\log(\text{card } \Hset))$, and since card $\Hset \le n$, these additions
take $O(n\log(n))$ time.  Finding and deleting the minimum element of
$\Hset$ also happens at most $n$ times and also takes $O(\log(\text{card}
\Hset))$ time, again multiplying to $O(n\log(n))$ time.  There are also
the {\color{black} intersections}.  Using hash maps, intersection can be
implemented in time linear in the size of the smaller of the two
sets \citep{BD-ACK:11}.
Both intersections, in lines \ref{line:intersectW} and
\ref{line:intersect}, have $N_x$ as one of the sets, which will have
size $O(\log(n))$.  Since the {\color{black}intersection} in line
\ref{line:intersectW} happens once per while loop iteration, it
happens at most $n$ times, taking a total of $O(n \log(n))$ run
time. Also, the {\color{black}intersection} at line \ref{line:intersect} is
called exactly once per consideration, so again by Lemma
\ref{considerations}, this operation takes a total of $O(n \log(n))$
time.  Finally, each computation of $y_{\text{min}}$ in line
\ref{line:ymin} happens once per consideration and takes time linear
in card $Y_{\text{near}} = O(\log(n))$ (note that computing $y_{\text{min}}$ does not require
  \emph{sorting} $Y_{\text{near}}$, just finding its minimum, and that
  computations of cost have 
already been accounted for), leading to $O(n \log(n))$ in total for
this operation. Note that the solution is returned upon algorithm
completion, so there is no ``query" phase.  In addition, $V$, $\Hset$,
$E$, and $\Wset$ all have maximum size of $n$, while saving $N_x$ for up
to $n$ nodes requires $O(n\log(n))$ space, so \FMT has space
complexity $O(n \log(n))$.
\end{proof}

\section{AO of \FMT with Non-Uniform
  Sampling}
\label{ao:nonunif}

Imagine sampling from $\varphi$ by decomposing it into a 
mixture density as follows. With probability $\ell$, draw a sample
from the uniform density, and with probability $1-\ell$, draw a sample
from a second distribution {\color{black}with} probability density function
$(\varphi-\ell)/(1-\ell\mu(\mathcal{X}_{\text{free}}))$. If \FMT is
run on only the (approximately $n\ell$) 
nodes that were drawn from the uniform distribution, the entire proof
of asymptotic optimality in Theorem \ref{thrm:AO} goes through after adjusting up the
connection radius $r_n$ by a factor of $(1/\ell)^{1/d}$. This fact can
be seen by observing that the
proof only relies on the expected value of the number of nodes in a $r_n$-ball,
and the lower density and larger ball radius 
cancel out in this expectation, leaving the expected value the same as in the original
proof. This cancellation formalizes the intuition that
  sparsely-sampled regions require searching \edit{wider} to
  make good connections.
Finally, note that adding samples before running \FMT
(while holding all parameters of \FMT fixed, in particular acting as
if $n$ were the number of original samples for the purposes of
computing $r_n$) can only improve the paths in the tree which do not
come within a radius of the obstacles. Since the proof of \FMT\!'s AO only employs
approximating paths that are bounded away from the obstacles by at
least $r_n$, the cost of these paths can only decrease if more points
are added, and thus their costs must still converge to the optimal
cost in probability. Thus when the (approximately $n(1-\ell)$) nodes that were drawn
from the second distribution are added back to the sample space, thus
returning to the original non-uniform sampling distribution,
asymptotic optimality still holds. 

In {\color{black}our discussion of non-uniform sampling,} we have repeatedly characterized a sampling
distribution by its probability density function $\varphi$. We note for
mathematical completeness that probability density functions are only
defined up to an arbitrary set of Lebesgue measure 0. Thus all conditions
stated in this {\color{black}discussion} can be slightly relaxed in that they only
have to hold on a set of Lebesgue measure $\mu(\mathcal{X}_{\text{free}})$.

\section{AO of \FMT for General Costs}
\label{ao:cost}
\noindent {\bf Asymptotic optimality of \FMT for metric costs}:
% Note that the $\texttt{Near}$ function already refers to
%the specific cost at hand, and thus the algorithm now considers
%\emph{cost-balls} instead of Euclidean balls. 
To make the proof of AO go through, we do have the additional
requirement that the cost be such that $\zeta_d$, the measure of the
unit cost-ball, be contained in $(0, \infty)$. Such
a cost-ball must automatically be contained in a Euclidean
ball of the same center; denote the radius of this ball by
$r_{\text{outer}}$. Then just three more things need to be adjusted in
the proof of Theorem \ref{thrm:AO}: First, condition (1) in the third paragraph of the proof needs to
change to $\frac{r_{\text{outer}}r_n}{2(2+\theta)} < \xi$. Second, condition
(2) right after it needs to be changed to
$\frac{3+\theta}{2+\theta}r_{\text{outer}}\, r_n < \delta$. Finally, every
time that length is mentioned, excepting cases when distance to the
edge of obstacles or the goal region is being considered, length
should be considered to mean cost instead. These replacements include
the radii of the covering balls (so they {\color{black}are} covering
cost-balls), the $\|\cdot\|$ function in the definition of $\Gamma_m$,
and the definition of $L_{n,\alpha,\beta}$, for instance. The first
two changes ensure that a sample is still drawn in the goal region so
that the returned path is feasible, and ensure that the covering
cost-balls remain collision-free. The third change is \edit{only}
notational, and the rest of the proof follows, since the triangle
inequality still holds.

\vspace{1 mm}
\noindent{\bf Asymptotic optimality of \FMT for line integral costs
  with optimal-path connections}:
Because of the bounds on the cost
density, the resulting new cost-balls with cost-radius $r$ contain, and
are contained in, Euclidean balls of radius
$r/f_{\text{upper}}$ and $r/f_{\text{lower}}$, respectively. Thus by
adjusting constants for obstacle clearance and
covering-cost-ball-radius, {\color{black}we} can still ensure that the covering
cost-balls have sufficient points sampled within them, and that they
are sufficiently far from the obstacles. Furthermore, by only
considering optimal connections, \emph{{\color{black}we are} back to having a triangle
inequality}, since the cost of the optimal path connecting $u$ to
$v$ is no greater than the sum of the costs of the optimal paths
connecting $u$ to $w$ and $w$ to $v$, for any $w$. Therefore
we are again in a situation where the AO proof in Theorem \ref{thrm:AO} holds
nearly unchanged.

\vspace{1 mm}
\noindent{\bf Asymptotic optimality of \FMT for line integral costs
  with straight-line connections}:
Assume that we can partition all of $\mathcal{X}$ except some set of
Lebesgue measure zero into finitely many connected, open regions, such
that on each such region $f$ is Lipschitz. Assume further that each of
the optimum-approximating paths (from the definition of
$\delta$-robust feasibility) can be chosen such that it contains only
finitely many points on the 
boundary of all these open regions. Note that this property does \emph{not}
have to hold for the optimal path itself, indeed the optimal path may
run along a region's boundary and still be arbitrarily approximated by
paths that do not. Since the Lipschitz regions are 
open, each approximating path $\sigma_n$ can be chosen such that
there exist two sequences of strictly positive constants
$\{\phi_{n,i}\}_{i=1}^{\infty}$ and $\{\psi_{n,i}\}_{i=1}^{\infty}$
such that: (a) $\phi_{n,i} \stackrel{i \rightarrow \infty}{\longrightarrow}
0$, and (b) for each $i$, for any point $x$ on $\sigma_n$ that is more than a
distance $\phi_{n,i}$ from any of the finitely many points on
$\sigma_n$ that lie on the boundary of a Lipschitz region, the
$\psi_{n,i}$-ball around $x$ is entirely contained in a single
Lipschitz region. This condition essentially requires that nearly all
of each approximating path is bounded away from the edge of any of the
Lipschitz regions. 
%Assume that the optimum-approximating paths (from the 
%definition of $\delta$-robust feasibility) each have only finitely
%many nondifferentiable points, and between each of those points in
%each path, there exists a small tube around the path on which $f$ is
%Lipschitz. 
Taken together, these conditions allow for very general cost functions,
including the common setting of $f$ piecewise constant on
finitely-many regions in $\mathcal{X}$. To see how these conditions
help prove AO, we examine the two reasons that the lack of a triangle
inequality hinders the proof of AO for \FMT\!.

The first is
that, even if \FMT returned a path that is optimal with respect to the
straight-line \PRMstar graph (this is the graph with nodes
  $V$ and edges connecting every pair of samples {\color{black} that
    have a} straight-line
  connection {\color{black}that} is collision-free and has cost less than $r_n$), there would be an extra cost associated with
each edge (in the straight-line \PRMstar graph too) for being the suboptimal path between its endpoints, and
this is not accounted for in the proof. The second reason is that to
ensure that each sample that is sufficiently far from the obstacles is
optimally connected to the existing \FMT tree, the triangle inequality
is used only in the first subcase of case 4 at the end of the proof of
Lemma \ref{lemma:claim_1}, where it is shown that the path returned by \FMT is at
least as good as any path $\mathcal{P}$ that traces through samples in
the covering balls in a particular way. This subcase is for when, at
the time when a given sample $x \in \mathcal{P}$ is added to \FMT\!,
$x$'s parent in $\mathcal{P}$ (denoted $u$) is not in $\Hset$, but one of $x$'s
ancestors in $\mathcal{P}$ is in $\Hset$. If the triangle inequality
fails, then it is possible that connecting $x$ to the \FMT tree
through a path that is entirely contained in $x$'s search radius and
runs through $u$ (which \FMT cannot do, since $u \notin \Hset$) would have given $x$ a
better cost-to-{\color{black}arrive} than what it ends up with in the \FMT
solution. Therefore, for the proof to go through, either the triangle
inequality needs to hold within all of the search balls of points
contained in the covering balls (or on a sequence of balls
$B^{\text{search}}_{n,m}$ centered at the covering balls but with an
$r_n$-larger radius), or the
triangle inequality needs to hold approximately such that this
approximation, summed over all the $B^{\text{search}}_{n,m}$, goes to zero as $n
\rightarrow \infty$. We venture to show that the latter case holds, using
the fact that, on each of the portions of $\mathcal{X}$ on which $f$ is
Lipschitz, we have an approximate triangle inequality, and the
approximation goes to zero quickly as $n \rightarrow \infty$. 

In particular, for a given optimum-approximating path, there are
$O(1/r_n)$ of the $B^{\text{search}}_{n,m}$, {\color{black}with} radii $O(r_n)$, and we can forget
about the $B^{\text{search}}_{n,m}$ containing points on the boundary
of a Lipschitz region. Let $i_n = \min\{i: \phi_{i,n} > \text{ the
  radius of } B^{\text{search}}_{n,m}\}$. Note that $\sigma_n$ can be
taken to converge to the optimal path slowly enough that $\phi_{i_n,n}
\stackrel{n \rightarrow \infty}{\longrightarrow} 0$ and
$r_n/\psi_{i_n,n} \stackrel{n \rightarrow \infty}{\longrightarrow}
0$. The boundary-containing
$B^{\text{search}}_{n,m}$ can be ignored because $\phi_{i_n,n}
\stackrel{n \rightarrow \infty}{\longrightarrow} 0$ ensures that the
boundary-containing balls cover an asymptotically negligible length
of the $\sigma_n$'s, and thus connections within them contribute negligibly to the 
cost of the \FMT solution as $r_n \rightarrow 0$. Furthermore, since
$r_n/\psi_{i_n,n} \stackrel{n \rightarrow \infty}{\longrightarrow} 0$, we are left with
$O(1/r_n)$ balls which, for $r_n$ small enough, are each entirely
inside a Lipschitz region, of which there are only
finitely many, and thus there exists a global Lipschitz constant $L$
that applies to all those balls, and does not change as $r_n
\rightarrow 0$. The suboptimality of a straight line contained in a
ball of radius $r$ on a $L$-Lipschitz region is upper-bounded by its
length ($2r$) times the maximal cost-differential on the ball
($2Lr$). Thus the total cost penalty on \FMT over all the
$B^{\text{search}}_{n,m}$ of interest is $O(r_n^2/r_n) = O(r_n)$, and
$r_n \rightarrow 0$, so we expect straight-line \FMT to return a
solution that is asymptotically no worse than that produced by the
``optimal-path" \FMT in Section \ref{subsubsec:internalSlow}, and is
therefore AO.

\section{AO of \kFMT}
\label{ao:knn}

Henceforth, we will call
mutual-$k_n$-nearest \PRMstar the \PRMstar\!-like algorithm in which the
graph is constructed by placing edges only between \emph{mutual}
$k_n$-nearest-neighbors. 
%The basic idea is that the
%mutual-$k_n$-nearest \PRMstar graph contains good paths, \FMT finds
%similarly good paths when they aren't too close to obstacles, and for
%large enough $n$ a good path exists with high probability that is not too close to an
%obstacle.
%
%\subsubsection{$k$-Nearest \FMT}
Three key
facts ensure AO of \kFMT\!, namely: (1) the mutual-$k_n$-nearest
\PRMstar graph arbitrarily approximates (in bounded variation norm) any path in
$\mathcal{X}_{\text{free}}$ for $k_n = k_0
\log(n)$, $k_0 > 3^de(1+1/d)$, (2) \knFMT 
returns at least as good a solution as any feasible path in the mutual-$k_n$-nearest
\PRMstar graph for which no node in the path has an obstacle between it and one of its
$k_n$-nearest-neighbors, and (3) for any fixed positive clearance
$\Upsilon$ and $k_n = k_0 \log(n)$, $k_0 > 3^de(1+1/d)$, the length of
the longest edge containing a $\Upsilon$-clear node in the
$k_n$-nearest-neighbor graph (not mutual,  this time) goes to zero in
probability. Paralleling the terminology adopted in Section \ref{prtintro}, we refer to samples in the mutual-$k_n$-nearest
\PRMstar graph as nodes.
Leveraging these facts, {\color{black}we} can readily show that \knFMT with $k _n =
k_0 \log(n)$, $k_0 > 3^de(1+1/d)$ arbitrarily approximate an optimal solution with
arbitrarily high probability as $n \rightarrow \infty$. Specifically, because the
problem is $\delta$-robustly feasible, {\color{black}we} can take an
arbitrarily-well-approximating path $\sigma$ that still has positive obstacle
clearance, and arbitrarily approximate that path in the mutual-$k_n$-nearest
\PRMstar graph by (1). By taking $n$ larger and larger, since $\sigma$'s clearance is
positive and fixed, the best approximating path in the
mutual-$k_n$-nearest \PRMstar graph will eventually have some positive
clearance with arbitrarily high probability. Then by (3), the length of the longest
edge containing a point in the approximating path goes to zero in
probability, and thus the probability that any node in the best
approximating path in the mutual-$k_n$-nearest \PRMstar graph will have
one of its $k_n$-nearest-neighbors be farther away than the nearest
obstacle goes to zero. Therefore by (2), $k_n$-nearest \FMT on the same samples
will find at least as good a solution as that approximating path with
arbitrarily high probability as $n \rightarrow \infty$, and the result
follows.

{\bf Proof of fact (1)}: To see why fact (1) holds, we need to adapt
the proof of Theorem 35 from \cite{Karaman.Frazzoli:IJRR2011}, which
establishes AO of $k$-nearest \PRMstar\!. Since nearly all of the arguments are
the same, we will not recreate it in its entirety here, but only point
out the relevant differences, of which there are three. (a) We
consider a slightly different geometric construction (with explanation
why), which adds a factor of
$3^d$ to their $k_n$ lower bound, (b) we adjust the proof for
mutual-$k_n$-nearest \PRMstar\!, as opposed to regular $k_n$-nearest
\PRMstar\!, and (c) we generalize to show that there exist paths in the
mutual-$k_n$-nearest \PRMstar graph that arbitrarily approximate
\emph{any} path in $\mathcal{X}_{\text{free}}$, as opposed to just the
optimal path. 

To explain difference (a), where the radius of the $B'_{n,m}$
was equal to $\delta_n$ (defined at the beginning of Appendix D.2 in
\cite{Karaman.Frazzoli:IJRR2011}), it should instead be given by,
\begin{equation}
\min \left\{
  \delta,3(1+\theta_1)\left(\frac{(1+1/d+\theta_2)\mu(\mathcal{X}_{\text{free}})}{\zeta_d}\right)^{1/d}\left(\frac{\log(n)}{n}\right)^{1/d}\right\},
\end{equation}
with the salient difference being an extra factor of 3 in the second
element of the $\min$ as compared to $\delta_n$. Note we are \emph{not}
redefining $\delta_n$, which is used to construct the smaller balls
$B_{n,m}$ as well as to determine the separation between ball centers
for both sets of balls. Thus this change leaves the $B_{n,m}$ ball
unchanged, and the centers of the $B'_{n,m}$ balls unchanged, while
asymptotically tripling the radius of the $B'_{n,m}$ balls. Note that this
changes the picture given in \cite[Figure
26]{Karaman.Frazzoli:IJRR2011}, in that the outer circle should have
triple the
radius. This change is needed because in the second sentence in the
paragraph after the proof of their Lemma 59, which says ``Hence, whenever
the balls $B_{n,m}$ and $B_{n,m+1}$ contain at least one node each,
and $B'_{n,m}$ contains at most $k(n)$ vertices, the $k$-nearest \PRMstar
algorithm attempts to connect all vertices in $B_{n,m}$ and
$B_{n,m+1}$ with one another'' might not hold in some cases. With the
definition of $B'_{n,m}$ given there, for $\theta_1$ arbitrarily small
(which it may need to be), $B'_{n,m}$ is just barely wider than
$B_{n,m}$ (although it does still contain it and $B_{n,m+1}$, since
their centers get arbitrarily close as well). Then the point on the
edge of $B_{n,m}$ farthest from the center of $B_{n,m+1}$ is exactly
$\delta_n - \frac{\delta_n}{1+\theta_1} =
\frac{\theta_1\delta_n}{1+\theta_1}$ (the difference in radii of
$B'_{n,m}$ and $B_{n,m}$) from the nearest point on the
edge of $B'_{n,m+1}$, while it is $\frac{2+\theta_1}{1+\theta}\delta_n$ (the
sum of the radii of $B_{n,m}$ and $B_{n,m+1}$ and the distance between
their centers) from the farthest point in 
$B_{n,m+1}$. Therefore, there may be a sample $x_m \in B_{n,m}$ and a sample
in $x_{m+1} \in B_{n,m+1}$ that are much farther apart from one
another than $x_m$ is from some points which are just outside
$B'_{n,m}$, and therefore $x_{m+1}$ may not be one of $x_m$'s
$k$-nearest-neighbors, no matter how few samples fall in
$B'_{n,m}$. However, for $n$ large enough, our proposed radius for
$B'_{n,m}$ is exactly $3\delta_n$, which results in the point on the
edge of $B_{n,m}$ farthest from the center of $B_{n,m+1}$ being
$3\delta_n - \frac{\delta_n}{1+\theta_1} =
\frac{2+3\theta_1}{1+\theta_1}\delta_n$ (the difference in radii of
$B'_{n,m}$ and $B_{n,m}$) from the nearest point on the
edge of $B'_{n,m+1}$, while it is
$\frac{2+\theta_1}{1+\theta}\delta_n$ (the sum of the radii of $B_{n,m}$ and
$B_{n,m+1}$ and the distance between their centers) from the farthest point in 
$B_{n,m+1}$ (See Figure~\ref{fig:kfmt}). 
\begin{figure}
  \centering
  \includegraphics[width=100mm]{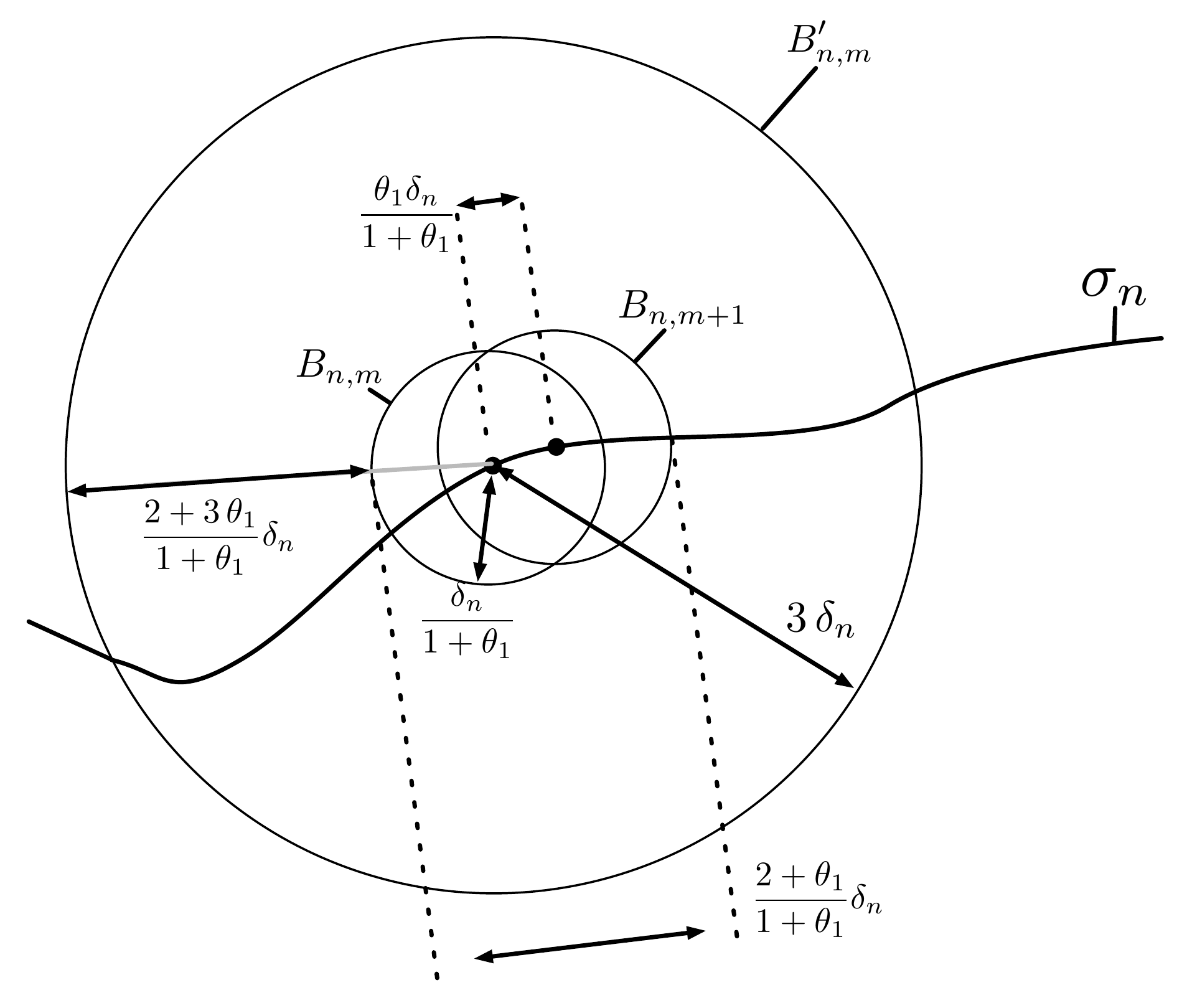}
  \caption{An illustration of $B_{n,m}$ and $B'_{n,m}$ in
    the proof of AO of $k$-nearest \edit{\FMT}.}
  \label{fig:kfmt}
\end{figure}
Therefore, $\frac{2+3\theta_1}{1+\theta_1}\delta_n >
\frac{2+\theta_1}{1+\theta}\delta_n$ implies that any point in
$B_{n,m}$ is closer to \emph{every} point in $B_{n,m+1}$ than it is to
\emph{any} point outside $B'_{n,m}$. This fact implies that if there are at
most $k$ samples in $B'_{n,m}$, at least one of which $x_{m+1}$ is in
$B_{n,m+1}$ and one of which $x_m$ is in $B_{n,m}$ (assume $x_{m+1}
\neq x_{m}$ or they are trivially connected), then any point that is
closer to $x_m$ than $x_{m+1}$ must be inside $B'_{n,m}$, of which
there are only $k$ in total, and thus $x_{m+1}$ must be one of $x_m$'s
$k$-nearest-neighbors. We have increased the volume of the $B'_{n,m}$
by a factor of $3^d$, making it necessary to increase the
$k_{\text{PRM}}$ lower-bound (used in their Lemmas 58 and 59) by the
same factor of $3^d$. This factor allows for the crucial part of the proof
whereby it is shown that no more than $k_n$ samples fall in each of
the $B'_{n,m}$. On the subject of changing the $k_{\text{PRM}}$
lower-bound, we note that it may be possible to reduce $k_{\text{FMT}}
:= 3^de(1+1/d)$ to $3^de/d$ by the same ideas used in our Theorem \ref{thrm:AO},
since for this proof we only need convergence in probability, while
\cite{Karaman.Frazzoli:IJRR2011} prove the stronger convergence almost surely.

For difference (b), note that the proof of
\cite{Karaman.Frazzoli:IJRR2011} states that when the event $A'_n$
holds, all samples in $B_{n,m+1}$ must be in the
$k_n$-nearest-neighbor sets of any samples in $B_{n,m}$. However a
symmetrical argument shows that all
samples in $B_{n,m}$ must also be in the $k$-nearest-neighbor sets of
any samples in $B_{n,m+1}$, and thus all samples in both balls must be
mutual-$k$-nearest-neighbors. Since this argument is the only place in their
proof that uses connectedness between samples, the entire proof holds just as well for
mutual-$k_n$-nearest \PRMstar as it does for $k_n$-nearest
\PRMstar\!. For difference (c), there is nothing to prove, as the
exposition in  \cite{Karaman.Frazzoli:IJRR2011} does not use anything about the cost
of the path being approximated until the last paragraph of their
Appendix D. Up until then, a path (call it $\sigma$) is chosen
and it is shown that the path in the $k_n$-nearest \PRMstar graph that is
closest to $\sigma$ in bounded variation norm converges to
$\sigma$ in the same norm.

{\bf Proof of fact (2)}: To see why fact (2) holds, consider the nodes
along a feasible path
$\mathcal{P}$ in the mutual-$k_n$-nearest
\PRMstar graph, such that all of the nodes are farther from any obstacle
than they are from any of their $k_n$-nearest-neighbors. We will show
that for any point $x$ along $\mathcal{P}$, with parent in
$\mathcal{P}$ denoted by $u$, if \knFMT is
run through all the samples (i.e., it ignores the stopping condition
of $z \in \mathcal{X}_{\text{goal}}$ in line \ref{stoppingcond}), then
the cost-to-{\color{black}arrive} of $x$ in the 
solution path is no worse than the cost-to-{\color{black}arrive} of $x$ in
$\mathcal{P}$, assuming the same is true for all of $x$'s ancestors in
$\mathcal{P}$. By feasibility, the endpoint of $\mathcal{P}$ is in
$\mathcal{X}_{\text{free}}${\color{black},} and then induction on the nodes in
$\mathcal{P}$ implies that this endpoint either is the end of a \knFMT
solution path with cost
no greater than that of $\mathcal{P}$, or \knFMT stopped
before the endpoint of $\mathcal{P}$ was considered, in which case \knFMT
returned an even lower-cost solution than the path that would have
eventually ended at the endpoint of $\mathcal{P}$. Note that we are
not restricting the edges in $k$-nearest \FMT to be drawn from those
in the mutual-$k$-nearest \PRMstar graph, indeed \knFMT can now potentially
return a solution strictly better than \emph{any} feasible path through the
mutual-$k$-nearest \PRMstar graph.

We now show that $x$'s cost-to-{\color{black}arrive} in the $k_n$-nearest \FMT solution is at
least as good as $x$'s cost-to-{\color{black}arrive} in $\mathcal{P}$, given that the
same is true of all of $x$'s ancestors in $\mathcal{P}$. Recall that
by assumption, all connections in $\mathcal{P}$ are to
\emph{mutual}-$k_n$-nearest-neighbors, and that for all nodes $x$ in
$\mathcal{P}$, \emph{all} of $x$'s (not-necessarily-mutual)
$k_n$-nearest-neighbors are closer to $x$ tha{\color{black}n} the nearest obstacle is
to $x${\color{black},} and thus the line connecting $x$ to any of its
$k_n$-nearest-neighbors must be collision-free. Note also that $x$'s
parent in $\mathcal{P}$, denoted by $u$, has finite cost in the
\knFMT tree by assumption, which means it must enter
$\Hset$ at some point in the algorithm. Since we are not stopping early, $u$ must also be the
minimum-cost node in $\Hset$ at some point, at which point $x$ will be
considered for addition to $\Hset$ if it had not been already. Now
consider the following four exhaustive cases for when $x$ is
\emph{first} considered (i.e., $x$'s first iteration in the for loop
at line \ref{line:forXnear} of Algorithm \ref{prtalg}). (a) $u \in \Hset$: then $u \in
Y_{\text{near}}$ and $\overline{ux}$ is collision-free, so when $x$ is
connected, its cost-to-{\color{black}arrive} is less than that of $u$ added 
to $\texttt{Cost}(u,x)$, which in turn is less than the
cost-to-{\color{black}arrive} of $x$ in $\mathcal{P}$ (by the triangle inequality). (b)
$u$ had already entered and was 
removed from $\Hset$: this case is impossible, since $u$ and $x$ are
both among one anothers' $k_n$-nearest-neighbors, and thus $x$ must
have been considered at the latest when $u$ was the lowest-cost-to-{\color{black}arrive} node in
$\Hset$, just before it was removed. (c) $u \in V_{\text{unvisited}}$
and $x$'s closest ancestor in $\Hset$, denoted $w$, is a
$k_n$-nearest-neighbor of $x$: by assumption, $\overline{wx}$
is collision-free, so when $x$ is connected, its cost-to-{\color{black}arrive} is no
more than that of $w$ added to $\texttt{Cost}(w,x)$, which in
turn is less than that of the cost-to-{\color{black}arrive} of $x$ in $\mathcal{P}$
(again, by the triangle inequality). (d) $u \in \Wset$
and $w$ (defined as in the previous case) is not a
$k_n$-nearest-neighbor of $x$: denoting the current
lowest-cost-to-{\color{black}arrive} node in $\Hset$ by $z$, we know that the
cost-to-{\color{black}arrive} of $z$ is no more than that of $w$, and since $x$ is a
mutual-$k_n$-nearest-neighbor of $z$, $z$ is also a
$k_n$-nearest-neighbor of $x$. Furthermore, we know
that since $w$ is \emph{not} a $k_n$-nearest-neighbor of $x$,
$\texttt{Cost}(w,x) \ge
\texttt{Cost}(z,x)$. Together, these facts give us that when $x$
is connected, its cost-to-{\color{black}arrive} is no more than that of $z$ added to
$\texttt{Cost}(z,x)$, which is no more than that of $w$
added to $\texttt{Cost}(w,x)$, which in turn is no more than
the cost-to-{\color{black}arrive} of $x$ in $\mathcal{P}$ (again, by the triangle
inequality). 

{\bf Proof of fact (3)}: To see why fact (3) holds, denote the longest edge in the
$k_n$-nearest-neighbor graph by $\hat{e}_n^{\text{max}}$, let 
\[
e_n :=
\biggl(\frac{e \, k_0 \, \mu(\mathcal{X}_{\text{free}})\log(n)}{\zeta_d\,  (n-1)}\biggr)^{1/d},
\]
and note that
\begin{equation}
\begin{split}
\mathbb{P}(\hat{e}_n^{\text{max}} > e_n) \le & \, \mathbb{P}(\text{any }
e_n\text{-ball around a sample  contains fewer than } k_n \text{ neighbors}) \\
\le & \, n \, \mathbb{P}(\text{the }
e_n\text{-ball around }v\text{ contains fewer than } k_n \text{
  neighbors}), \\
\end{split}
\end{equation}
where $v$ is some arbitrary sample. Finally, observe that
$e_n \stackrel{n \rightarrow \infty}{\longrightarrow} 0$ and for $e_n
< \Upsilon$, the number of neighbors in the $e_n$-ball around any sample is a
binomial random variable with parameters $n-1$ and
$\frac{e k_n}{n-1}$, so we can use the bounds in
\cite[page 16]{Penrose:03} to obtain, 
\begin{equation}
\begin{split}
\mathbb{P}(\hat{e}_n^{\text{max}} > e_n) \le & \, n
e^{-e k_n H(\frac{k_n-1}{k_n}e)} \\
\le & \, n^{1-e k_0 H(\frac{k_n-1}{k_n}e)} \\
\le & \, n^{-16} \qquad \text{for } n \ge 2, \\
\end{split}
\end{equation}
where $H(a) = 1+a-a\log(a)$. Thus since $n^{-16} \stackrel{n\rightarrow
  \infty}{\longrightarrow} 0$ and $e_n \stackrel{n\rightarrow
  \infty}{\longrightarrow} 0$, we have the result.

\bibliographystyle{plainnat}

\bibliography{references}

\end{document}